\documentclass[twoside,11pt]{article}

\PassOptionsToPackage{unicode}{hyperref}
\usepackage{jmlr2e}

\usepackage[utf8]{inputenc} %
\usepackage[T1]{fontenc}    %
\usepackage{enumitem}

\usepackage{url}            %
\usepackage{booktabs}       %
\usepackage{amsfonts}       %
\usepackage{nicefrac}       %
\usepackage{amsmath} %
\usepackage{microtype}      %
\usepackage{lipsum}         %
\usepackage{todonotes}

\usepackage{tcolorbox}

\usepackage{titletoc}

\newcommand{\Ex}{\mathbb{E}}
\newcommand{\Exk}{\mathbb{E}_k}
\newcommand{\taut}{{\tau_k}}

\hypersetup{
    colorlinks=true,
    citecolor=blue,
    linkcolor=blue,
}

\usepackage{listings} %

\definecolor{codegreen}{rgb}{0,0.6,0}
\definecolor{codegray}{rgb}{0.5,0.5,0.5}
\definecolor{codepurple}{rgb}{0.58,0,0.82}
\definecolor{codeblue}{rgb}{0,0,1}
\definecolor{backcolour}{rgb}{0.95,0.95,0.92}
\definecolor{key-color}{rgb}{0.8, 0.47, 0.196}

\lstdefinestyle{mystyle}{
    backgroundcolor=\color{backcolour},   
    commentstyle=\color{codegreen},
    numberstyle=\tiny\color{codegray},
    stringstyle=\color{codepurple},
    basicstyle=\ttfamily\footnotesize,
    breakatwhitespace=false,         
    breaklines=true,                 
    captionpos=b,                    
    keepspaces=true,                 
    numbers=left,                    
    numbersep=5pt,                  
    showspaces=false,                
    showstringspaces=false,
    showtabs=false,                  
    tabsize=2,
    language=Python,
    emph={lm},
    emphstyle={\color{blue}},
    classoffset=1, %
    otherkeywords={sum},
    morekeywords={rm, mean},
    keywordstyle=\color{codegreen},
    classoffset=0,
}
\newcommand\nnfootnote[1]{%
  \begin{NoHyper}
  \renewcommand\thefootnote{}\footnote{#1}%
  \addtocounter{footnote}{-1}%
  \end{NoHyper}
}

\usepackage{xcolor}         %
\usepackage{subfigure}
\usepackage{algorithm}
\usepackage{algorithmic}
\usepackage{wrapfig}
\usepackage{color}
\usepackage{framed}
\definecolor{shadecolor}{rgb}{0.92,0.92,0.92}
\colorlet{shadecolor}{orange!15}
\usepackage{mathtools}
\allowdisplaybreaks[1]
\usepackage{mathrsfs} 
\usepackage{ulem}
\usepackage{makecell}
\usepackage[capitalize,noabbrev]{cleveref}

\newtheorem{assumption}[theorem]{Assumption}
\newcommand{\qed}{\unskip\nobreak\hfill\BlackBox}

\newcommand{\BLUE}[1]{{\color{blue} #1}}

\newcommand{\mlki}{m_{l,k,i}}

\newcommand{\betabeta}{(\beta_1,\beta_2)}
\newcommand{\functionclass}{\mathcal{F}_{L, D_0, D_1}^{n}  (\mathbb{R}^d)}

\newcommand{\RED}[1]{{\color{black} #1}}

\newcommand{\yushunrevise}[1]{\textcolor{black}{ #1}}
\newcommand{\brrevise}[1]{\textcolor{black}{ #1}}
\newcommand{\yushunrevisediverge}[1]{\textcolor{black}{ #1}}

\ShortHeadings{Adam Converges Without Any Modification On Update Rules}{Zhang, Li, Chen, Shi, Luo and Sun}
\firstpageno{1}

\title{Adam Converges Without Any Modification On Update Rules}

\author{%
  \name Yushun Zhang$^{1,2}$ \email yushunzhang@link.cuhk.edu.cn \\
  \name Bingran Li$^{1}$ \email bingranli@link.cuhk.edu.cn \\
  \name Congliang Chen$^{1}$ \email congliangchen@link.cuhk.edu.cn \\
  \name Zhi-Quan Luo$^{1,2}$ \email luozq@cuhk.edu.cn \\
  \name Ruoyu Sun$^{1,2\dagger}$ \email sunruoyu@cuhk.edu.cn \\
  \addr $^1$The Chinese University of Hong Kong, Shenzhen, China \\
  $^2$Shenzhen Research Institute of Big Data \\
}

\editor{Not yet assigned}

\begin{document}

\maketitle
\numberwithin{equation}{section} %
\nnfootnote{$\dagger$: Correspondence author.}
\nnfootnote{This manuscript is an extended journal version of the conference paper ``Adam Can Converge Without Any Modification On Update Rules'' appeared at NeurIPS 2022. This manuscript contains a simplified proof for Adam under random shuffling (simplified based on the conference paper), and a  new proof for Adam under with-replacement sampling (inspired by the conference paper).  }

\vspace{-2mm}
\begin{abstract}%
Adam is the default algorithm for training neural networks, including large language models (LLMs). However, \citet{reddi2019convergence} provided an example that Adam diverges, raising concerns for its deployment in AI model training. We identify a key mismatch between the divergence example and practice: \citet{reddi2019convergence} pick the problem after picking the hyperparameters of Adam, i.e., $(\beta_1,\beta_2)$; while practical applications often fix the problem first and then tune $(\beta_1,\beta_2)$. In this work, we prove that Adam converges with proper problem-dependent hyperparameters. First, we prove that Adam converges when $\beta_2$ is large and $\beta_1 < \sqrt{\beta_2}$. Second, when $\beta_2$ is small, we point out a region of $(\beta_1,\beta_2)$ combinations where Adam can diverge to infinity. Our results indicate a phase transition for Adam from divergence to convergence when changing the $(\beta_1, \beta_2)$ combination. To our knowledge, this is the first phase transition in $(\beta_1,\beta_2)$ 2D-plane reported in the literature, providing rigorous theoretical guarantees for Adam optimizer. We further point out that the critical boundary $(\beta_1^*, \beta_2^*)$ is problem-dependent, and particularly, dependent on batch size. This provides suggestions on how to tune $\beta_1$ and $\beta_2$: when Adam does not work well, we suggest tuning up $\beta_2$ inversely with batch size to surpass the threshold $\beta_2^*$, and then trying $\beta_1< \sqrt{\beta_2}$. Our suggestions are supported by reports from several empirical studies, which observe improved LLM training performance when applying them.

\end{abstract}


\section{Introduction}
\label{section:intro}

Machine learning tasks often aim to solve the following empirical risk minimization (ERM) problem. 
{\small \begin{equation} \label{finite_sum}
  \operatorname*{minimize}_{x \in \mathbb{R}^{d}} \quad f(x):=\frac{1}{n}\sum_{i=0}^{n-1} f_{i}(x),
\end{equation}}
where $x\in \mathbb{R}^{d}$ denotes the trainable parameters, $n \in \mathbb{N}$ is the number of mini-batches that partition the dataset, and $f_i(x)$ denotes the loss on the $i$-th mini-batch data. For a fixed dataset of size $\mathcal{D}$, the batch size in each mini-batch is $\mathcal{D} /n $. 
In deep learning, Adam \citep{kingma2014adam} is one of the most popular algorithms for solving \eqref{finite_sum}. 
It has been applied to various domains such as natural language processing (NLP) and computer vision (CV) (e.g., \citep{vaswani2017attention,dosovitskiy2021an}).
Its impact is also evidenced by over 230,000 citations as of December 2025, a number that continues to grow rapidly \citep{adamcitation}.

 In the era of large language models (e.g., \citep{openai_chatgpt_2022}), Adam plays a central role in large-scale training. 
 Adam is reported to be used to train many mainstream LLMs, including Llama series \citep{touvron2023llama},  Qwen series \citep{bai2023qwen}, and DeepSeek series \citep{liu2024deepseek}, etc. Adam is clearly serving as a major horsepower behind the advancement of AI.  Its influence was recently recognized when it received the {\it ICLR 2025 Test-of-Time Award} \citep{ICLR2025TestOfTime}.

Despite its prevalence, an influential paper
\citep{reddi2019convergence} (the winner of {\it ICLR 2018 Best Paper Award} \citep{ICLR2018BestPaper}) provides an example that 
Adam diverges with a wide range of hyperparameters. A main result in  \citep{reddi2019convergence} states that: 
\begin{snugshade}
\begin{center}
   {\it \citep{reddi2019convergence}: For any $\beta_1, \beta_2$ s.t. $0 \leq \beta_1 < \sqrt{\beta_2} <1$, there exists a problem such that Adam diverges. } 
\end{center}
\end{snugshade}
Here, $\beta_1$ and $\beta_2$ are the hyperparameters to control Adam's 1st-order and 2nd-order momentum $m_k$ and $v_k$.
The divergence region is visualized in Figure \ref{fig:intro_paper} (a). This finding raises serious concerns for Adam's deployment in AI model training,  where the divergence can raise alerts of unpredictable training failures. Since then, many new variants have been designed. For instance, 
AMSGrad \citep{reddi2019convergence} enforces $v_k$ (defined later in Algorithm \ref{algorithm_wr}) to be non-decreasing; AdaBound \citep{luo2018adaptive} imposes constraint $v_k \in [C_l, C_u]$ to ensure the boundedness on the effective stepsize.

\begin{figure*}[t]
  \vspace{-0.5cm}
    \centering
    \subfigure[Divergent region claimed by \citep{reddi2019convergence}]{
    \begin{minipage}[t]{0.24\linewidth}
    \centering
    \includegraphics[width=\linewidth]{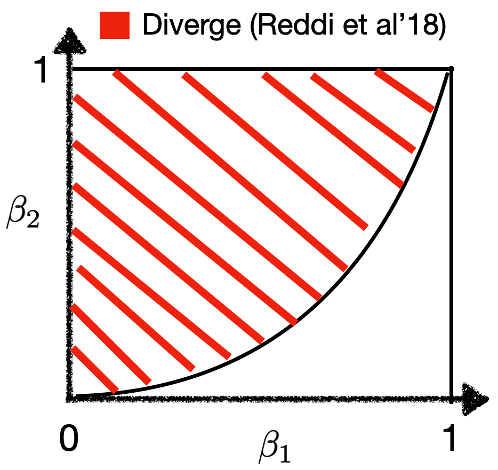}
    \end{minipage}%
    }%
    \subfigure[Our contribution]{
      \begin{minipage}[t]{0.24\linewidth}
      \centering
   \includegraphics[width=\linewidth]{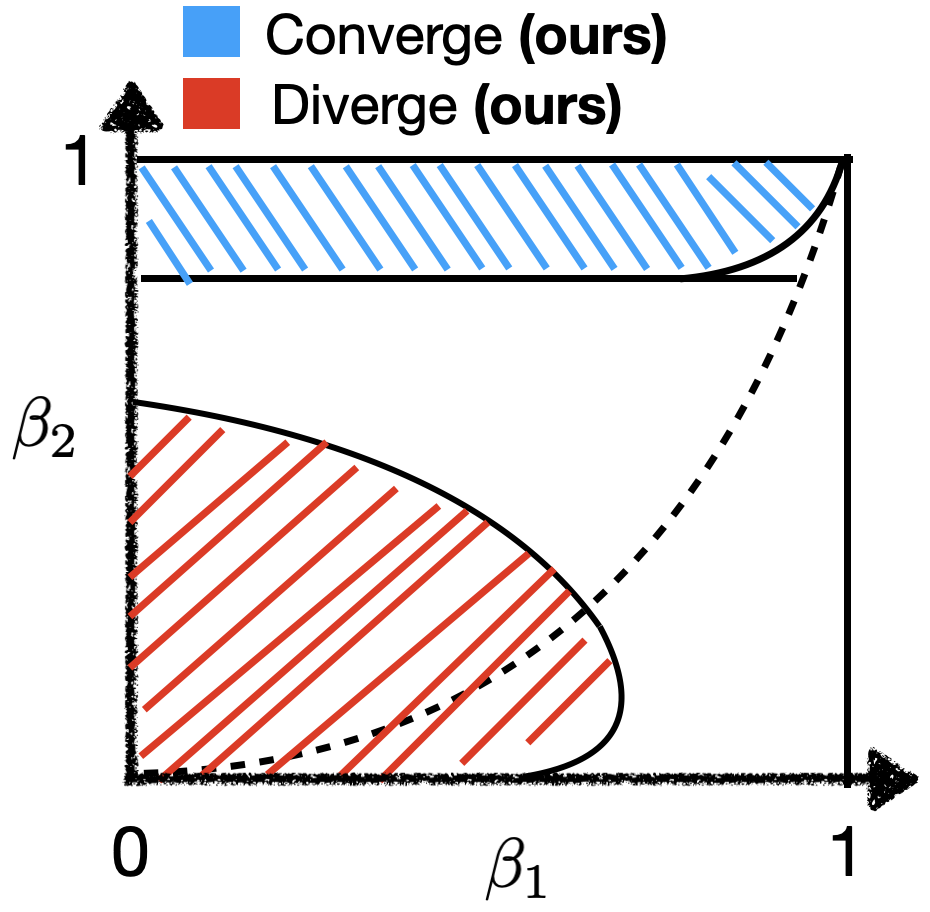}
      \end{minipage}%
      }%
    \subfigure[MNIST]{
      \begin{minipage}[t]{0.24\linewidth}
      \centering
    \includegraphics[width=\linewidth]{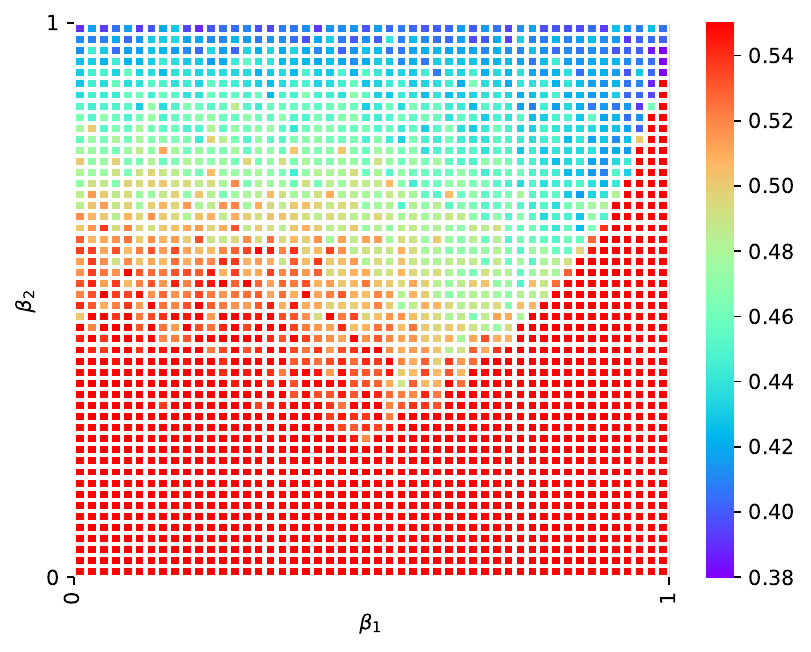}
      \end{minipage}%
      }%
    \subfigure[CIFAR-10]{
      \begin{minipage}[t]{0.24\linewidth}
      \centering
    \includegraphics[width=\linewidth]{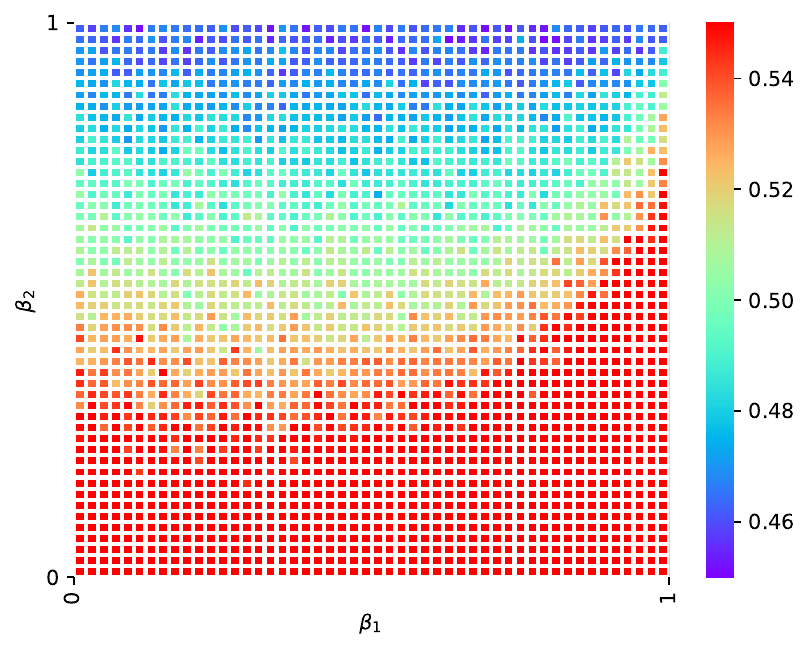}
      \end{minipage}%
      }%
    \centering
    \vspace{-4mm}
    \caption{ {\small {\bf (a)}: The divergent region of Adam claimed by \citep{reddi2019convergence}. They fix $\betabeta$ first and then pick a problem to construct the divergence example.   {\bf (b)}: An illustration of our contribution in $\betabeta$ phase diagram. We fix the problem before picking $\betabeta$.  Note that this is a different setting from {\bf (a)}, so there is no contradiction. {\bf Both boundaries of the red and blue regions depend on batch size (shown later).} The shape of the region follows the solution to our analytic conditions.  The dotted curve satisfies $\beta_1 =\sqrt{\beta_2}$. {\bf (c), (d)}:  The training loss on MNIST and CIFAR-10.  We sweep $\beta_1$ and  $\beta_2$ in grids $\{(k_1/50,k_2/50)| k_1 = 0,\cdots,49,k_2=0, \cdots, 49\}$, resulting in 2,500 trials. The performance of Adam reconciles with our theoretical characterization in {\bf (b)}.   }}%
    \label{fig:intro_paper}
\vspace{-5mm}
\end{figure*}

On the other hand, counter-intuitively, vanilla Adam remains exceptionally popular.  Without any modification on its update rules, Adam works well in practice.  This is rather surprising due to the existence of divergence theory.
 Even more mysteriously, we find that the commonly reported hyperparameters actually satisfy the divergence condition stated earlier. 
 For instance, 
\citet{kingma2014adam} claim that  $\betabeta=(0.9,0.999)$ is a ``good choice for the tested machine learning problems'' and it is indeed the default setting in deep learning libraries. 
GAN researchers (e.g. \citet{radford2015unsupervised,isola2017image}) use $\betabeta=(0.5,0.999)$. For LLMs such as GPT-3,  Llama series, and DeepSeek series \citep{brown2020language, touvron2023llama, liu2024deepseek}, $\betabeta$ is chosen to be $(0.9,0.95)$.
All these hyperparameters live in the divergence region $\beta_1 < \sqrt{\beta_2}$.
Surprisingly, instead of observing the divergence issue, these hyperparameters achieve good performances and they actually show the sign of convergence.

 Why does Adam work well despite its theoretical divergence issue? Is there any mismatch between deep learning problems and the divergence example? 
We take a closer look into the divergence example and find out the mismatch {\it does} exist. 
In particular, we notice an important (but often ignored) characteristic of the divergence  example:
 \citep{reddi2019convergence} picks   $\betabeta$ {\it before} picking the problem (or precisely, the \# mini-batches $n$).
 That is, to construct the divergence example, they change $n$ for different $\betabeta$.
   For instance,  for $\betabeta=(0,0.99)$, they use one $n$ to construct the  divergent example; for $\betabeta = (0, 0.9999)$, they use  another $n$ to construct another divergent example. 
On the other hand, in practical applications of Adam listed above, 
practitioners tune hyperparameters $\betabeta$ \textit{after} the problem (or $n$) is fixed. 
So there is a gap between the setting of theory and practice: the order of picking the problem and $\betabeta$ is different.

Given the good practical performance of Adam, we conjecture that Adam can converge when the problem is fixed. 
Unfortunately, the behavior of vanilla Adam is far less studied than its variants (perhaps due to the criticism of divergence).
 To verify this conjecture, we run Adam for the classification problem on data MNIST and CIFAR-10 as shown in Figure \ref{fig:intro_paper} (c) and (d). For these two problems, we find that:
 First, when $\beta_2$ is large, the optimization error is small for almost all values of $\beta_1$;  Second, when $\beta_1,\ \beta_2$ are both small, there is a 
red region with 1.4x larger error.

While Adam's performances seem unstable in the red region, we find that {\it it always performs well in the top blue region in Figure \ref{fig:intro_paper}.} This seems to suggest that: when the problem is fixed, Adam can converge without modification after proper tuning of $\beta_1$ and $\beta_2$.  
We ask  the following question:

\begin{snugshade}
    \begin{center}
        {\it Can  Adam provably converge without any modification on its update rules?} 
\end{center}
\end{snugshade}

In this work, we theoretically characterize Adam's behavior across different choices of $(\beta_1,\beta_2)$. 
We show that vanilla Adam (without algorithmic modification) exhibits two qualitatively different regimes: 
in a \textit{safe} region of $(\beta_1,\beta_2)$ it provably converges to the set of critical points (in the realizable case) or to a neighborhood of that set (in the non-realizable case), 
whereas in a \textit{danger} region it can diverge to infinity. 
Together, these results reveal a divergence--convergence phase transition in the $(\beta_1,\beta_2)$ plane.

Our contributions are visualized in Figure~\ref{fig:intro_paper}(b) and summarized as follows.

\begin{itemize} 
\item \textbf{Convergence for large $\beta_2$.} 
When {\small $0\leq\beta_1< \sqrt{\beta_2}<1$} and {\small $\beta_2$} is larger than a problem-dependent threshold, Adam converges to critical points in the realizable case and to a neighborhood of critical points in the non-realizable case. 
The threshold for $\beta_2$ depends on the problem class, and it increases with $n$ (equivalently, it decreases with batch size).   

\item \textbf{Divergence for small $\beta_2$.} 
For any $n \geq 3$ and any problem-class parameters $(L,D_0,D_1)$ with large enough $D_1$---where $L$ is the smoothness constant (Assumption~\ref{assum1}) and $(D_0,D_1)$ quantify the gradient variance condition (Assumption~\ref{assum2})---there exists an instance in the problem class such that Adam diverges to infinity for all $(\beta_1,\beta_2)$ in a certain region. The boundary of the divergence region expands with $n$ (or decreases with batch size).

\item 
\textbf{Key characteristics.} 
We emphasize the following aspects of our results. 

\textbf{(1) Phase transition}. 
The divergence result considers the same setting as our convergence result  (fixed problem class), 
indicating the existence of a critical boundary $(\beta_1^*, \beta_2^*)$ that demarcates a phase transition from divergence to convergence.

\textbf{(2) Problem-dependent bounds}. 
Our convergence and divergence regions of $\betabeta$ depend on the problem-class parameters. This is drastically different from \citep{reddi2019convergence}, which considers problem-independent $\betabeta$. 

\textbf{(3) Non-asymptotic characterization}. 
Our ``divergence region'' of $\betabeta$ expands as $n$ increases and converges to the whole region $ [0,1)^2$ as $n$ goes to infinity, 
 which recovers (actually stronger than) the problem-independent divergence result of \citep{reddi2019convergence}. 
 In this sense, we can view the divergence result of \citep{reddi2019convergence} 
as  an asymptotic characterization  of the divergence region (as $n \rightarrow \infty$)
 and our divergence result as a non-asymptotic characterization (for any fixed and finite $n$). 

\textbf{(4) No assumption of bounded gradients.}
 Our convergence analysis does {\it not} need a 
 bounded gradient assumption, which is commonly used in the literature. 
Removing such assumption is crucial for revealing the phase transition: with gradients bounded a priori, the gradients of Adam cannot diverge, while we prove that it can happen under certain $\betabeta$.

\item {\bf Proof difficulties.} The primary challenge in the convergence proof is characterizing the limit behavior of a stochastic non-linear dynamics. 
We handle this by considering a concentration effect of Adam's {\small $1/\sqrt{v_k}$} around {\small $1/ \sqrt{\Ex(v_k)}$}. One specific difficulty is that $v_k$ can be arbitrarily close to 0, and $1/\sqrt{v_k}$ can behave badly in the worst case.   We find that a large $\beta_2$ helps the concentration and stabilizes the dynamic system. The intuition is that a large $\beta_2$ slows down the changes of $v_k$, and its behavior becomes predictable.

 \item  {\bf Tuning suggestions for $\betabeta$.} 
 Our positive and negative results provide guidance for tuning $\beta_1$ and $\beta_2$: when Adam does not work well, we suggest tuning up $\beta_2$ {\it inversely with batch size} to surpass the threshold $\beta_2^*$, and then trying $\beta_1 <\sqrt{\beta_2}$. These suggestions are supported by empirical findings in LLM pre-training \citep{porian2024resolving,zhang2024does}.

\end{itemize}

We believe our results advance the theoretical understanding of Adam. 
While \citet{reddi2019convergence} state that ``Adam can diverge'', our results show the other side of the story: when the problem is fixed, proper tuning can make Adam converge without any modification. 
In particular, the convergence region of $\betabeta$ is guaranteed to be nonempty. 
This is analogous to gradient descent on an $L$-smooth objective function, where a convergent stepsize always exists (e.g. $\eta < 2/L$).
Overall, our results provide theoretical support for vanilla Adam and offer concrete guidance for selecting $\betabeta$ in practice.

\paragraph{Relation to the conference version \citep{zhang2022adam}.}
This is an extended journal version of our NeurIPS 2022 conference paper \citep{zhang2022adam}. 
The current manuscript presents a streamlined convergence proof for Adam under random shuffling (greatly simplified relative to the conference version) and establishes a new convergence result and proof under with-replacement sampling, which was not covered in the conference version. 
Technically, this journal version develops new concentration results along Adam's trajectory and introduces a systematic procedure to deploy these concentration tools in the presence of momentum and unbounded gradients. 
These changes help reduce the length of the random-shuffling proof from approximately $44$ pages to about $18$ pages, and they enable the with-replacement convergence analysis. 
These proof techniques are new relative to the conference version and may be broadly applicable to analyzing Adam-type algorithms. 

Beyond the convergence analysis, we also strengthen the conference
paper’s divergence result. Using a revised counterexample, we additionally proves that the gradient
norm diverges to infinity, whereas the conference version established divergence of the iterates and
objective values.

\section{Preliminaries}
\subsection{Review of Adam}
\label{section:preliminaries}

We consider problem \eqref{finite_sum}.
We use $x$  to denote the optimization variable. 
We denote $\nabla f_i$ as the gradient of $f_i$. 
We consider two implementations of Adam based on data sampling strategies:  Adam under with-replacement sampling (Algorithm \ref{algorithm_wr}) and under random shuffling (Algorithm \ref{algorithm_rr}).  Algorithm  \ref{algorithm_wr} is more theoretically oriented, and Algorithm \ref{algorithm_rr} is more widely used in practical deployment of Adam \footnote{For instance, GPT-3 technical report states that ''Data are sampled without replacement during
training'' \citep{brown2020language}.}.

\begin{algorithm}
    \caption{Adam under With-Replacement Sampling}
    \label{algorithm_wr}
    \begin{algorithmic}
     \STATE Initialize $m_{0}$, $v_{0}$, and  $x_{1}$.
      \FOR {$k=1\to\infty$}
      \STATE \BLUE{Sample $\tau_k$ uniformly from $\{0, 1, \dots, n-1\}$.}
      \STATE 
      $m_{k} = \beta_1 m_{k-1} + (1 - \beta_1) \nabla f_{\tau_k}(x_k)$ %
      \STATE $v_{k} = \beta_2 v_{k-1} + (1 - \beta_2) \nabla f_{\tau_k}(x_k) \circ \nabla f_{\tau_k}(x_k)$ %
      \STATE $x_{k+1} = x_{k} - \frac{\eta_k}{\sqrt{v_k} + \epsilon} \circ m_k$
       \ENDFOR
    \end{algorithmic}
  \end{algorithm}

\begin{algorithm}
    \caption{Adam under Random Shuffling }
    \label{algorithm_rr}
    \begin{algorithmic}
     \STATE Initialize  $m_{1,-1}$, $v_{1,-1}$, and $x_{1,0}$.
      \FOR {$k=1\to\infty$}
      \STATE \BLUE{Sample $\{\tau_{k,0},\tau_{k,1},\cdots,\tau_{k,n-1}\}$ as a random permutation of $\{0,1,2,\cdots,n-1\}$.}
        \FOR  {$i=0 \to n-1$}
      \STATE 
      $m_{k,i}=\beta_1 m_{k,i-1}+\left(1-\beta_1\right)\nabla f_{\tau_{k,i}}(x_{k,i})$ %
      \STATE $v_{k,i}= \beta_2 v_{k,i-1}+\left(1-\beta_2\right)\nabla f_{\tau_{k,i}}(x_{k,i})\circ \nabla f_{\tau_{k,i}}(x_{k,i})$ %
      \STATE $x_{k,i+1}=x_{k,i}-\frac{\eta_{k}}{\sqrt{v_{k,i}}+\epsilon}\circ m_{k,i}$	
      \ENDFOR
      \STATE $x_{k+1,0}=x_{k,n}$; $v_{k+1,-1}=v_{k,n-1}$; $m_{k+1,-1}=m_{k,n-1}$
       \ENDFOR
    \end{algorithmic}
  \end{algorithm}
  
In Algorithm  \ref{algorithm_wr} and \ref{algorithm_rr}, $m$ and $v$ denote the 1st-order and 2nd-order momentum, respectively. The product $\circ$, division, and square-root operator are component-wise. 
Regarding Algorithm \ref{algorithm_wr}, 
we denote $x_{k}, m_{k}, v_{k} \in \mathbb{R}^{d}$ as the value of $x, m, v$ at the $k$-th iteration, respectively. 
Regarding Algorithm \ref{algorithm_rr}, 
we denote $x_{k, i}, m_{k, i}, v_{k, i} \in \mathbb{R}^{d}$ as the value of $x, m, v$ at the $k$-th outer loop (epoch) and the $i$-th inner loop (batch), respectively.   We choose {\small $\eta_{k}=\frac{\eta_{0}}{\sqrt{k}}$} as the stepsize. 

In Algorithm  \ref{algorithm_wr} and \ref{algorithm_rr}, $\epsilon$ is adopted to avoid the corner case where $v_k$'s are constantly 0 along the trajectory, in which case Adam is not well-defined. To ensure the well-definedness of Adam,  one can either use (i) a small  $\epsilon$ (e.g., default $\epsilon$ is $10^{-8}$); or (ii) use $\epsilon = 0$ with a non-zero initialization $v \succ 0$. Our theory supports both cases. In the main body of the proof, we adopt (ii), which makes the results cleaner. We also provide the proof of case (i), i.e., the convergence of Adam with ``non-zero $\epsilon$'' in Appendix \ref{appendix:bias_correction}.  In either case, our final convergence results do not have any dependencies on $\epsilon$ or initialization $v$, and thus our result remains non-vacuous regardless of $\epsilon$ or $v$. \yushunrevise{Following the suggestion by \citep{kingma2014adam}, we use $m_0 = m_{1,-1} = 0$ in our proofs. Extending to non-zero initialization only involves a change of several constant terms. }

In the original paper by \citep{kingma2014adam}, the authors introduce an additional ``bias correction'' step, which can be implemented by changing the stepsize $\eta_k$ to {\small $\hat{\eta}_k=\frac{\sqrt{1-\beta_2^k}}{1-\beta_1^k} \eta_k$}.  Note that  {\small $\hat{\eta}_k \in [\sqrt{1-\beta_2} \eta_k,\frac{1}{1-\beta_1}\eta_k]$} is well-bounded near $\eta_k$, so $\eta_k$ and $\hat{\eta}_k$ bring the same convergence rate. In the main body of our proof, we follow the forms of Algorithm  \ref{algorithm_wr} and \ref{algorithm_rr}, which make the results cleaner. For completeness, we add the proof on the convergence of Adam with ``bias correction'' steps in Appendix \ref{appendix:bias_correction}.

In our analysis, we make the following mild assumptions on $f_i(x) $ and $f(x)$  in the ERM problem \eqref{finite_sum}.

\begin{assumption} \label{assum1}
   There exists $ L \geq 0$, such that $\|\nabla f_i(x) - \nabla f_i(y)\|_2 \leq L \|x -y\|_2 $ for  any $i \in [n]$ and any $ x, y \in  \mathbb{R}^d $.
   In addition, $f(x)$ is lower bounded by a finite constant $f^*$. 
 \end{assumption}

\begin{assumption}\label{assum2}
  There exist $D_0, D_1 \geq 0$ such that $ \sum_{i=0}^{n-1}\left\|\nabla f_{i}(x)\right\|_{2}^{2} \leq D_{1}\|\nabla f(x)\|_{2}^{2}+D_{0}$ for any $ x\in \mathbb{R}^d$.
  \end{assumption}
  \vspace{-1mm}

When $n,  L, D_0,$ and $D_1$ are fixed a priori, we define the corresponding problem class
{\small
\begin{equation}
\functionclass := \bigl\{ f(x)\mid {}f(x)=\tfrac1n\sum_{i=0}^{n-1}f_i(x), x\in\mathbb{R}^d, \text{Assumptions~\ref{assum1}--\ref{assum2} hold with }(L,D_0,D_1)\bigr\}.
\end{equation}
}

     In Assumption \ref{assum1}, the Lipschitz condition for component functions is standard for ERM problems  (e.g., \citep{bertsekas1996neuro,bertsekas2000gradient, bubeck2015convex,schmidt2017minimizing,allen2018katyusha}).

     Assumption \ref{assum2} covers a class of variance or growth conditions in the literature.  
     We now discuss how Assumption \ref{assum2} is reduced to different conditions under different choices of $D_1$ and $D_0$.

      We firstly discuss $D_1$. Under a slightly more restricted condition of $D_1 > n$,  Assumption \ref{assum2} becomes the ``affine variance'' condition in \eqref{eq_a2_affine_variance}, which controls the deviation between mini-batch and full gradient. The ``affine variance'' condition is originally proposed by \citep{bertsekas2000gradient} and is later popularized by \citep{bottou2018optimization}.
    {\small 
        \begin{equation}
        \label{eq_a2_affine_variance}
       \frac{1}{n} \sum_{i=0}^{n-1}\left\|\nabla f_{i}(x) - \nabla f(x) \right\|_{2}^{2} = \frac{1}{n}\left( \sum_{i=0}^{n-1}\left\|\nabla f_{i}(x)\right\|_{2}^2 \right)  -  \|\nabla f(x)\|_{2}^2  \leq   \frac{(D_{1} -n)}{n}\|\nabla f(x)\|_{2}^{2} + \frac{D_0}{n}, \forall x\in \mathbb{R}^d.
    \end{equation}
    }

      When $D_1=n$, Assumption \ref{assum2} and \eqref{eq_a2_affine_variance} reduce to the ``bounded variance'' in \eqref{eq_a2_bdd_variance}:
    {\small
    \begin{equation}
        \label{eq_a2_bdd_variance}
  \frac{1}{n} \sum_{i=0}^{n-1}\left\|\nabla f_{i}(x) - \nabla f(x) \right\|_{2}^{2} = \frac{1}{n}\left( \sum_{i=0}^{n-1}\left\|\nabla f_{i}(x)\right\|_{2}^2 \right)  -  \|\nabla f(x)\|_{2}^2  \leq   \frac{D_0}{n}, \forall x\in \mathbb{R}^d.
    \end{equation}
    }

The ``bounded variance'' condition is widely used for analyzing both SGD and adaptive gradient methods (e.g. \citep{polyak1992acceleration,ghadimi2013stochastic, ghadimi2016mini, Manzil2018adaptive}).   %
Assumption \ref{assum2}  allows the variance to grow with the gradient norm, making it strictly weaker than the bounded variance condition \eqref{eq_a2_bdd_variance}. This relaxation is meaningful: bounded variance is often too restrictive and can easily fail. For  example, consider a convex quadratic minimization problem \citep{bottou2018optimization}:
{\small \begin{equation}
\label{eq_a2_quadratic_example}
   \operatorname*{minimize}_{x\in\mathbb{R}^d} \quad f(x) = \frac{1}{2n} \|Ax - b\|_2^2 = \frac{1}{2n} \sum_{i=0}^{n-1} (a_i^\top x - b_i)^2 := \frac{1}{n}\sum_{i=0}^{n-1}  f_i(x),
\end{equation}}
where $A$ has rank at least 2. In this case, the left-hand side of \eqref{eq_a2_bdd_variance} grows quadratically with $\|x\|_2$, and no finite $D_0$ satisfies \eqref{eq_a2_bdd_variance}. In contrast, Assumption \ref{assum2} holds with finite $D_1$ and  $D_0$. We provide a more detailed justification in Appendix \ref{appendix_bdd_variance}. 

Another drawback of the bounded variance condition \eqref{eq_a2_bdd_variance} is that it excludes some divergence counter-examples of SignSGD (which is equivalent to Adam with $\beta_1 = \beta_2 = 0$). As such, analysis under \eqref{eq_a2_bdd_variance} may not reveal the full picture of Adam.  We provide more detailed explanation later in Section \ref{sec_div}.

There is also a stronger condition that requires $D_1 = 0$, i.e.,  $ \sum_{i=0}^{n-1}\left\|\nabla f_{i}(x)\right\|_{2}^{2} \leq  D_0$. This condition is sometimes called ``bounded 2nd-order moment'' (e.g., \citep{nemirovski2009robust}). This is strictly stronger than \eqref{eq_a2_bdd_variance} and implies bounded gradient condition, so it is even more restricted.

As a result, Assumption \ref{assum2} with generic $D_1$ is recommended in \citep{bottou2018optimization} as it is {\it ``relatively minor''} and {\it ``variance is allowed to grow quadratically in any direction.''}

Now we discuss $D_0$.  
When $D_{0}=0$, the condition becomes ``strong growth condition'' (SGC)  \citep{solodov1998incremental,schmidt2013fast}.
An implication of SGC is that 
when $\|\nabla f(x)\|=0$, we have $\left\|\nabla f_i(x)\right\|=0$ for all $i$.
This condition is considered reasonable in the overparameterized regime where neural networks  can interpolate all data points
\citep{vaswani2019fast,shi2020rmsprop}. 
We will show that Adam converges to exact critical points when 
SGC holds. 

For the general case where $D_{0} > 0$, Adam is not guaranteed to reach the exact critical points. 
Instead, it can only converge to a neighborhood of critical points \citep{Manzil2018adaptive,shi2020rmsprop}. 
This phenomenon indeed occurs for Adam in experiments, even with diminishing stepsizes (see Figure~\ref{fig:cifar_nlp_mnist}(b)).
This behavior is in line with a classical phenomenon in stochastic gradient methods: constant-stepsize SGD is known to converge to a neighborhood of critical points whose size scales with the noise level \citep{luo1991convergence,bertsekas1997nonlinear,yan2018unified,yu2019linear,liu2020improved}.

Finally, we emphasize that we do {\it not} add the bounded gradient assumption {\small $\|\nabla f(x)\|\leq G$}, which is commonly used in the literature. 
This is crucial for revealing the phase transition: with gradients bounded a priori, the gradients of Adam cannot diverge, while we prove that this can happen under certain $\betabeta$.

\subsection{A Brief Review of the Counter-example in \citep{reddi2019convergence}}
\label{sec:reddi}

Before stating our theoretical results, we first restate the counter-example by \citep{reddi2019convergence}. For the consistency of notation, we will restate their results under our notation. 
They consider the following one-dimensional convex problem:  $\operatorname*{minimize} \frac{1}{n} \sum_{i=0}^{n-1} f_i(x)$ where $x \in [-1,1]$, $n \geq 3$:
	{\small \begin{equation}\label{counterexample_reddi}
		f_{i}(x)=\left\{\begin{array}{ll}n x, & \text { for } i =0 \\ -x, & \text { otherwise, }\end{array}\right.
	\end{equation}}

Note that \eqref{counterexample_reddi} satisfies both Assumption \ref{assum1} and \ref{assum2} (with $D_1= n^4 +n^3 -n^2$ and $D_0 =0$), so our assumptions do not rule out this counter-example a priori.
Nevertheless, problem \eqref{counterexample_reddi} is a constrained problem with feasible set $x \in [-1,1]$, and the optimal solution is $x^*=-1$. 
Since this is a constrained problem, the term ``divergence'' here actually means the iterates will stay in a huge region with the size of the whole feasible set. 
Since this is a constrained problem, the term ``divergence'' here actually means that the iterates will stay on the boundary of the feasible region and are far from the optimal solution $x* = -1$.
Note that both cases describe algorithmic behavior that is opposite to convergence, and we do not emphasize this distinction between them.

  In  \citep{reddi2019convergence}, 
the function \eqref{counterexample_reddi} is presented as an ``online optimization problem'' (not a finite-sum problem).
We rewrite \eqref{counterexample_reddi} in a finite-sum form so that it matches our ERM notation in \eqref{finite_sum}. 
We will use the same cyclic sampling order
$ (f_0, f_1, f_2 ; f_0, f_1, f_2; \dots)  $ as in \citet{reddi2019convergence}. 
For completeness, 
we restate their main claim below in our notation.

\begin{theorem}[Theorem 2 in \citep{reddi2019convergence}]
\label{thm2_reddie}
For any fixed $\betabeta$ satisfying $\beta_1 < \sqrt{\beta_2}$, there exists a sufficiently large $n$, s.t., applying Adam to the function \eqref{counterexample_reddi} (under cyclic sampling) converges to the sub-optimal point $x=1$.
\end{theorem}

We briefly discuss the divergent condition for Theorem \ref{thm2_reddie}.  As stated in Eq. (7), Appendix B in \citep{reddi2019convergence},  for every fixed $\betabeta$, they need an ``$n$ that depends on $\beta_1$ and $\beta_2$''. 
As such, they require different $n$ to cause divergence on different $\betabeta$. The considered problem class is constantly changing.

For completeness, we further restate Theorem 1 in \citep{reddi2019convergence}.

\begin{theorem}[Theorem 1 in \citep{reddi2019convergence}]
\label{thm1_reddi}
 For function \eqref{counterexample_reddi}, when $\beta_1=0$ and $\beta_2 = 1/(n^2+1)$, Adam will  converge to a sub-optimal point $x=1$.
\end{theorem}

Theorem \ref{thm1_reddi} considers choosing $\betabeta$ after $n$. However, this result only shows divergence on {\it a single hyper-parameter choice} $\betabeta = (0, 1/(n^2+1))$. This configuration lies somewhere on the left boundary of Figure \ref{fig:intro_paper} (a). 
It therefore provides no characterization beyond this single point: it remains unclear which directions of changing $\betabeta$ preserve divergence and which directions suppress it or induce convergence.

\begin{figure}[t]
\begin{center}
\vspace{-0.5cm}
\subfigure[Cyclic update order]{
    \begin{minipage}[t]{0.3\linewidth}
    \centering
    \includegraphics[width=\textwidth]{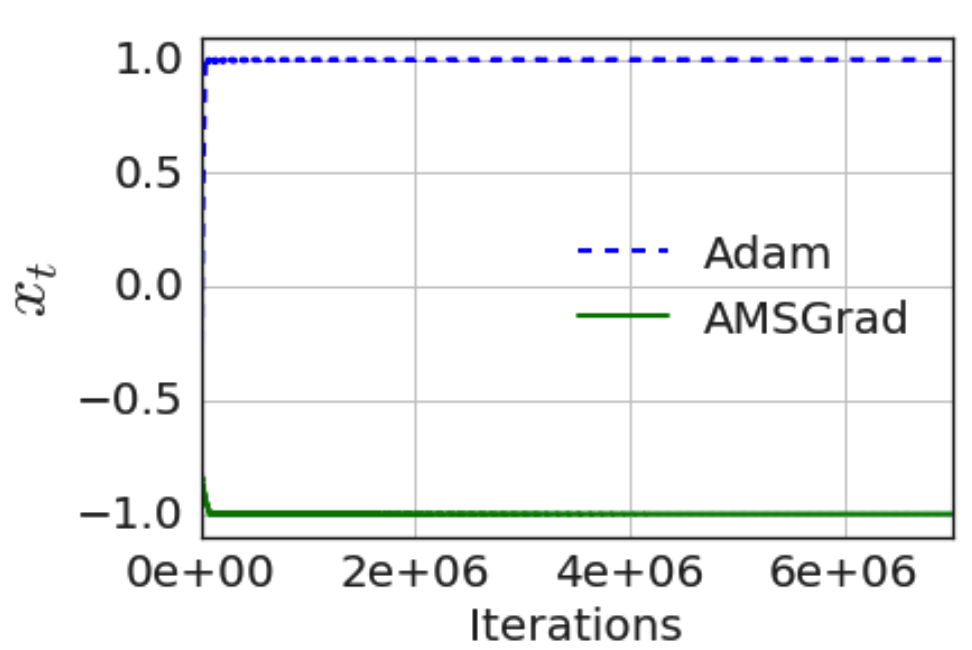}
    \end{minipage} }%
    \subfigure[Randomized update order]{
    \begin{minipage}[t]{0.3\linewidth}
    \centering
    \includegraphics[width=\linewidth]{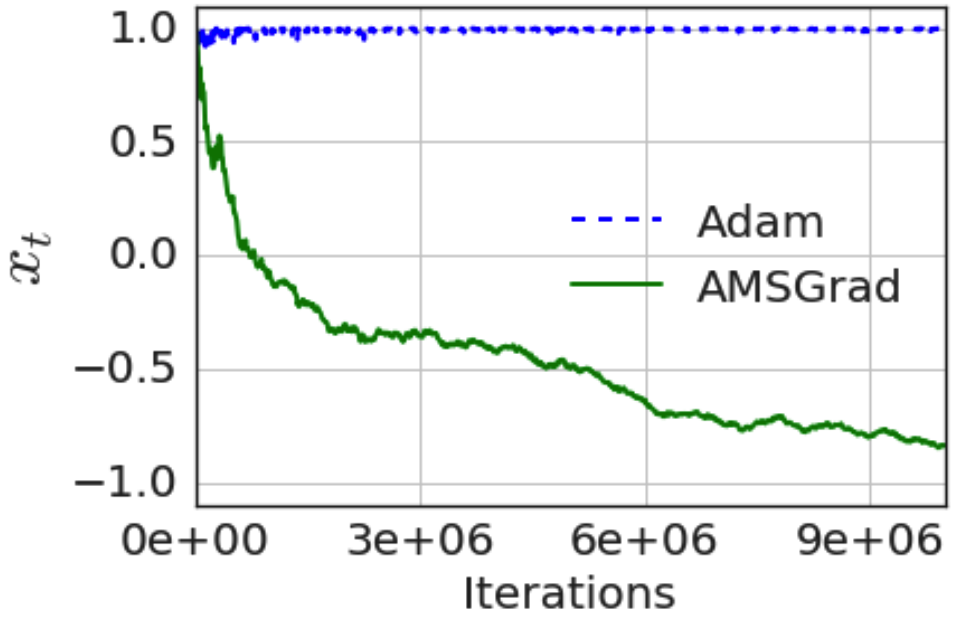}
    \end{minipage}%
    }%
    \subfigure[Diminishing stepsize]{
    \begin{minipage}[t]{0.28\linewidth}
    \centering
    \includegraphics[width=\linewidth]{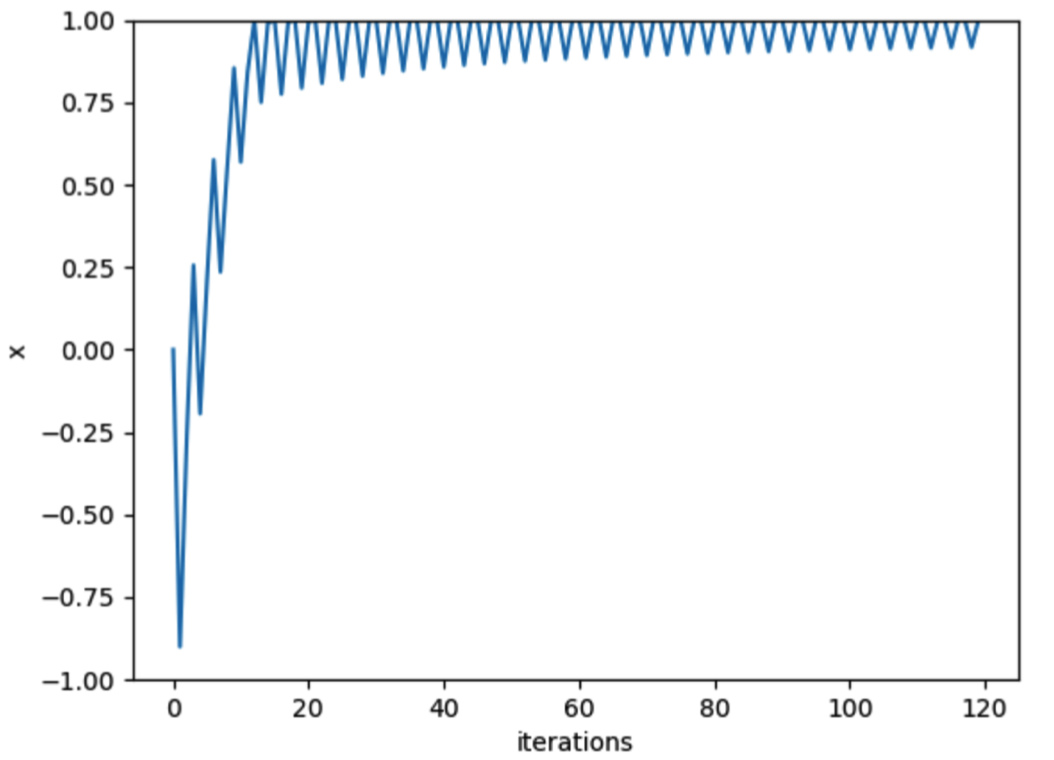}
    \end{minipage}%
    }%
\caption{{\small 
{\bf (a, b):} We restate Figure 1 from \citep{reddi2019convergence}. The divergence of Adam happens under both cyclic and random update orders, so randomization cannot prevent divergence.  Since they consider constrained problems, the term “divergence” here means getting stuck at the sub-optimal solution $x=1$.  In the figures, AMSGrad is a different method, and it is not our focus.  {\bf (c):} Diminishing stepsize $\eta_k = \frac{1}{\sqrt{k}}$ does not prevent divergence. }
}
\label{fig:reddi_figures}
\vspace{-5mm}
\end{center}
\end{figure}

Finally, these divergence results also hold when randomized update orders are used instead of cyclic orders, as proved in \citep[Theorem 3]{reddi2019convergence}. Consequently, randomization does not prevent divergence. 
Further, the proof of \citep{reddi2019convergence} uses $\eta_k \propto \frac{1}{\sqrt{k}}$, so diminishing stepsize does not prevent divergence, either. These claims are supported by numerical evidence in Figure \ref{fig:reddi_figures}.

\citet{reddi2019convergence} wrote: {\it ``These results have important consequences insofar as one has to use \textit{problem-dependent} $\epsilon$, $\beta_1$, and $\beta_2$ in order to avoid bad convergence behavior.  In high-dimensional problems, this typically amounts to using...a different $\epsilon$, $\beta_1$, and $\beta_2$ for each dimension. However, this defeats the purpose of adaptive methods since it requires tuning a large set of parameters.''}
In the following, we show that in our setting such divergence can be avoided without changing the algorithm or introducing new hyperparameters.
The key is to choose $\betabeta$ in a problem-dependent manner, with a particular dependence on the number of mini-batches $n$ (equivalently, the batch size).
Importantly, this does \textit{not} require per-coordinate tuning of $\betabeta$---a direction mentioned by \citet{reddi2019convergence}.

\subsection{Related Works}
\label{section:related}

Ever since \citet{reddi2019convergence} pointed out the divergence issue, there are many attempts on designing new variants of Adam. 
Since we focus on understanding vanilla Adam {\it without modification},  we do not discuss these variants here. We relegate the introduction to more Adam variants later in Appendix \ref{appendix:more_related}. 
Compared with proposing new variants, 
the convergence of vanilla Adam is far less studied than its variants (perhaps due to the criticism of divergence).  We discuss as follows.

\citet{shi2020rmsprop} study randomly-shuffled RMSprop proposed in \citep{hinton2012neural} (Algorithm \ref{algorithm_rr} with $\beta_1 =0$) \footnote{We notice that \citet{shi2020rmsprop} also analyze randomly-shuffled Adam  with $\beta_1$ close enough to $0$. However, \citet{yushun2022doesadamconverge} show that they require $\beta_1<10^{-7}$. Thus, their result provide limited information beyond randomly-shuffled RMSProp.}.   We believe it is still important to study Adam: numerically, Adam often substantially outperforms RMSProp (e.g. \citep{agarwal2020optimistic}). Furthermore, all mainstream LLMs are trained using Adam, and RMSProp is reported to perform $\geq 20\%$ worse than Adam \citep[Table 1]{orvieto2025search}. 
Theoretically, literature on RMSProp cannot reveal  how these hyperparameters jointly affect (or jeopardize) the convergence of Adam in the $\betabeta$-2D plane. 
We note that \citet{shi2020rmsprop} conjecture the existence of a critical threshold $\beta_2^*$ for RMSProp. We generalize this conjecture to $\betabeta$-2D plane for Adam and confirm the existence of a critical boundary $(\beta_1^*, \beta_2^*)$. 
One additional difference is that  \citet{shi2020rmsprop} only study randomly-shuffled RMSprop, whereas we consider more general cases of both with-replacement and randomly-shuffled Adam.

\citet{defossez2022a} analyze Adam with $0<\beta_2 <1$  and $\beta_1<\beta_2$. However, they require bounded gradient assumption $\|\nabla f(x)\|\leq G$,  which prevents the potential divergence a priori. Further, their bound is proportional to $1/\epsilon$, where $\epsilon$ is the stability hyperparameter and is usually small (e.g., default $\epsilon$ is $10^{-8}$). As a result, they require sufficiently large non-zero  \(\epsilon\) to keep the result non-vacuous. These assumptions incur  artificial boundedness property on Adam: $ 0 <\epsilon \leq \sqrt{v} +\epsilon \leq G$, and thus Adam will never diverge and is essentially changed to another algorithm called AdaBound \citep{luo2019adaptive}. In contrast, our result is independent of $\epsilon$ and allows $\epsilon=0$. We do not impose the bounded gradient condition and reveal the di-convergence phase transition under different $\betabeta$.

\paragraph{New proofs of Adam after the online appearance of the conference version in 2022.} After the conference version of this manuscript \citep{zhang2022adam} appeared online in August 2022, there has been a series of works providing new convergence proofs of Adam. Here, we summarize how their assumptions differ from ours.

Recent works \citep{xie2024adam,JMLR2025,li2025frac} provide new proofs for RMSProp and Adam. Compared to our work, their proof relies on a ``coordinate-wise bounded variance'' condition, which imposes the inequality of \eqref{eq_a2_bdd_variance} on each coordinate individually.
Note that this is stronger than \eqref{eq_a2_bdd_variance}, which only imposes the inequality on the overall gradient.  Moreover, since \eqref{eq_a2_bdd_variance} is already known to be restrictive— for instance, it fails even on simple quadratic problems \eqref{eq_a2_quadratic_example}—its coordinate-wise variant is necessarily more limiting. 
Besides the coordinate-wise bounded variance, they also required $T$-dependent $\betabeta$, where $T$ is the pre-determined total iteration budget. In practice,  $T$-dependent $\betabeta$ is rarely used or reported by practitioners. In theory, we argue that such dependence on $T$ is unnecessary and can be removed. In contrast, our proof does not require these two conditions.

\citet{wang2023closing} provide a novel convergence proof for Adam under the ``coordinate-wise affine variance'' condition and $T$-dependent $\betabeta$.   Similarly,  the coordinate-wise version is strictly stronger than the affine variance condition \eqref{eq_a2_affine_variance} or Assumption \ref{assum2}; and the $T$-dependence is unnecessary.  \citet{wang2024provable} further relax these two conditions, but their analysis focused on Adam under random shuffling and does not cover with-replacement sampling.

\citet{li2023convergence,peng2025simple,hong2023high,hong2024convergence,zhang2024convergence} provide refined convergence proofs of Adam. One limitation is that their complexity is proportional to $1/\epsilon$, where  $\epsilon$ is the stability hyperparameter. These bounds
become vacuous as $\epsilon$ approaches $0$, which brings concerns since $\epsilon$ is usually small in practice (e.g., default $\epsilon$ is $10^{-8}$). We argue that the dependency on $\epsilon$ is artificial, and in contrast, our result is independent of $\epsilon$ and allows arbitrarily small $\epsilon$ including 0. In addition to the dependency on $\epsilon$, \citet{li2023convergence,peng2025simple} require the bounded variance condition \eqref{eq_a2_bdd_variance} or its stronger version; and \citet{hong2023high,hong2024convergence,zhang2024convergence} require an almost-sure or coordinate-wise version of the affine variance condition \eqref{eq_a2_affine_variance}. All these assumptions are stronger than Assumption \ref{assum2}.

Another line of new works prove the convergence of Adam under bounded gradient or iterate, or bounded $v_k$ assumptions \citep{ding2023adam,jiang2023uadam,xiao2024adam,liang2025convergence}, while we do not impose these boundedness conditions.  

Besides the difference in assumptions, we further highlight an important difference between our work and the literature above. The above works focus on proving convergence upper bounds for Adam under certain choices of $\betabeta$. However, a good upper bound is just one side of the story.  In contrast, we establish the {\it phase transition from divergence to convergence} in different regions of   $\betabeta$, which presents a more complete picture.  Our work points out two missing facts in the literature.
\begin{itemize}[nosep]
    \item {\bf First}, we show that tuning $\betabeta$ is not only sufficient {\it but also necessary} for convergence. Specifically, we prove the existence of a phase transition: as $\beta_2$ increases from 0 to 1, the behavior of Adam shifts from divergence (a bad lower bound  in Theorem  \ref{thm_diverge}) to convergence (good upper bounds in Theorem \ref{thm_wr} and \ref{thm_rr}). This reveals that tuning $\beta_2$ is not merely about optimizing the convergence rate; it is crucial for preventing the fundamental divergence danger.  To our knowledge, this is  the first phase transition in $\betabeta$ 2D-plane for Adam reported in the literature.
    \item {\bf Second}, we show that this phase transition occurs at a {\it problem-dependent} boundary $(\beta_1^*, \beta_2^*)$, which grows {\it inversely with batch size}.
    This contrasts with prior theoretical works that suggest  $\beta_2$ should increase with the total iteration budget $T$ and does not reveal the dependency on the considered problem. 
\end{itemize}

\vspace{-4mm}
\section{Main Results}
\label{sec_main}

\subsection{Convergence Results}
\label{sec_conv_wr}
Now we prove that Adam converges with proper problem-dependent hyperparameters.
We first present the result for Algorithm \ref{algorithm_wr} and then present the result for Algorithm \ref{algorithm_rr}.

\begin{theorem}
\label{thm_wr}
{\bf (Convergence result for Algorithm \ref{algorithm_wr})}
Assume Algorithm \ref{algorithm_wr} satisfies:
$0\leq\beta_1 < \sqrt{\beta_2} <1$;  
$\beta_2$ satisfies {\small$ \beta_2 \geq  \gamma_1(n) = 1- \mathcal{O}\left(\frac{1-\beta_1^n}{n^5}\right)$} where
$\gamma_1(n)$ is defined in
\eqref{eq:beta2_sufficient_range};
and {\small $\eta_k = \frac{\eta_0}{\sqrt{k}}$}. 
For any {\small $f(x) \in \functionclass$}, when $T\in\mathbb{N}$ is sufficiently large, we have

{\small
\begin{equation}
\label{eq_thm_wr} 
  \min_{k \in [1, T]}
  \Ex\left[\mathcal{M}_{D_0,D_1}\bigl(\|\nabla f(x_k)\|_2\bigr)\right]
  =
  \mathcal{O}\left(\frac{\log T }{(1-\beta_2)\sqrt{T}} \right)
  +\mathcal{O}((1-\beta_2)\sqrt{D_0}),
\end{equation}
}

where 
{\small
\begin{equation}
\label{eq:stationarity_measure}
\mathcal{M}_{D_0,D_1}(u):=
\begin{cases}
\displaystyle
\min\left\{\frac{u^2}{\sqrt{D_0}},
           \frac{u}{2\sqrt{dD_1}}\right\},
    & D_0>0,\ D_1>0,\\[2mm]
\displaystyle \frac{u}{2\sqrt{dD_1}},
    & D_0=0<D_1,\\[2mm]
\displaystyle \frac{u^2}{\sqrt{D_0}},
    & D_1=0<D_0,\\[2mm]
u,  & D_0=D_1=0.
\end{cases}
\end{equation}
}

\end{theorem}

  \paragraph{Remark 1: the choice of $\beta_2$.} Our theory suggests that large $\beta_2$ should be used to ensure convergence. This message matches our experiments in Figure \ref{fig:intro_paper}. We emphasize that the requirement of ``large $\beta_2$'' is necessary, because small $\beta_2$ will indeed lead to provable divergence (shown later in Theorem \ref{thm_diverge}).
    We here comment a bit on the threshold $\gamma_1(n)$. By the proof of Theorem  \ref{thm_wr}, and particularly, by \eqref{eq_delta_1_beta2} in Section \ref{appendix:lemma_descent_WR}, $\beta_2$ needs to satisfy $\beta_2 \geq 1- \mathcal{O}\left(\frac{1-\beta_1^n}{n^5}\right) $.  %
    We also remark that $\gamma_1(n)$ slowly increases with $\beta_1$. This property is visualized in Figure \ref{fig:intro_paper} (b) where the lower boundary of the blue region slightly lifts up when $\beta_1$ increases.
    Note that our threshold of $\beta_2$ is a sufficient condition for convergence and the power is not claimed tight.  Tightening the threshold for $\beta_2$ will be an interesting future direction.

  \paragraph{Remark 2: $\beta_2$ and batch size.}  We find that the condition of $\beta_2$: {\small $\beta_2 \geq \gamma_1(n) = 1- \mathcal{O}\left(\frac{1-\beta_1^n}{n^5}\right)$} is problem-dependent and it increases with $n$. This property suggests that we need larger $\beta_2$ when $n$ is large, or equivalently, we need larger $\beta_2$ when the batch size is small \footnote{This is because batch size equals $\frac{\mathcal{D}}{n}$, where $\mathcal{D}$ denotes the total sample size and $n$ denotes the number of mini-batches into which the dataset is divided.}. This aligns with our experiments in Fig. \ref{fig:cifar_nlp_mnist} (a):  on MNIST, smaller batch size $\frac{\mathcal{D}}{n}$ (which is equivalent to larger $n$) requires a larger $\beta_2$ to reach small loss.

\begin{figure*}[t]
  \vspace{-0.5cm}
    \centering
    \subfigure[batch size v.s. $\beta_2$]{
    \begin{minipage}[t]{0.3\linewidth}
        \centering
        \includegraphics[width=\linewidth]{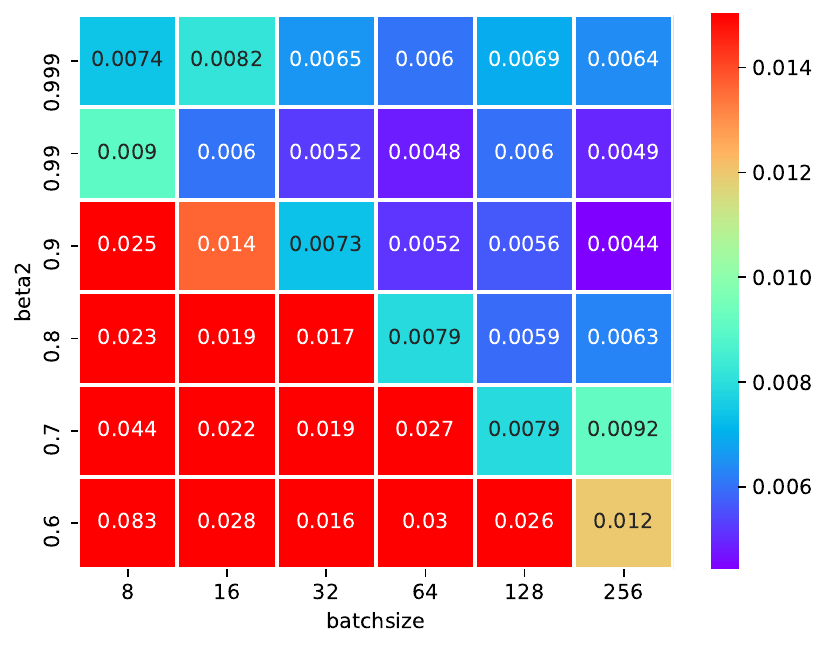}
    \end{minipage}    
    }%
    \subfigure[Large $\beta_2$ and $D_0>0$]{
      \begin{minipage}[t]{0.3\linewidth}
      \centering
    \includegraphics[width=\linewidth]{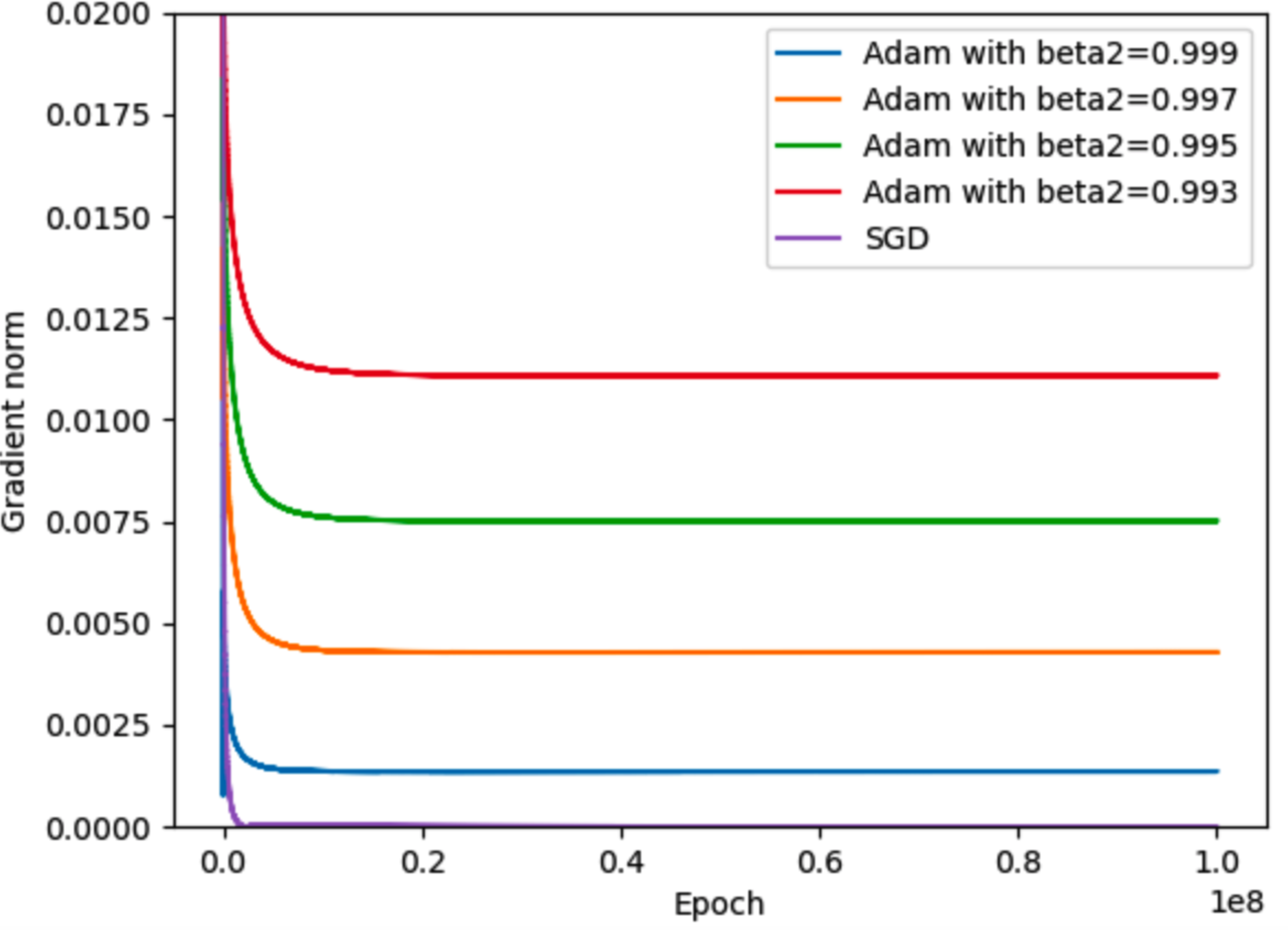}
      \end{minipage}%
      }%
    \subfigure[Large $\beta_2$ and $D_0= 0$]{
    \begin{minipage}[t]{0.3\linewidth}
    \centering
\includegraphics[width=\linewidth]{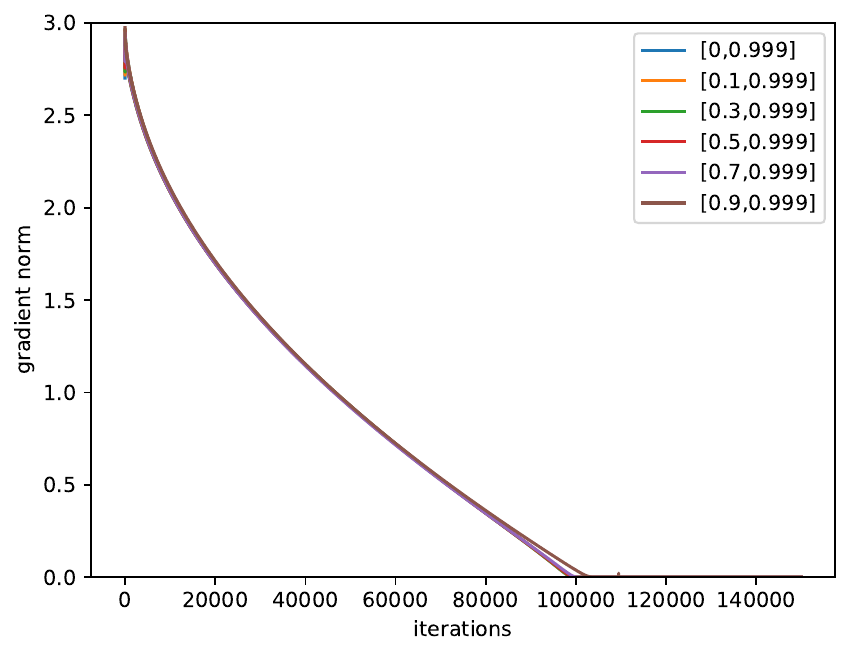}
    \end{minipage}%
    }%
    \centering
    \vspace{-3mm}
    \caption{{\small {\bf (a):} The training
loss of Adam on MNIST under different batch size and $\beta_2$. The trend aligns with our theory: we need a larger $\beta_2$ when batch size is small. Here, we used the default $\beta_1=0.9$. 
    {\bf (b) (c):} 
  large-$\beta_2$ Adam converges to a neighborhood of critical points when $D_0>0$ and converges to exact critical points when $D_0 = 0$. We use diminishing stepsize $\eta_k = 0.1 /\sqrt{k}$ as in our theory. Experimental details are shown in Appendix \ref{appendix:exp_setting}.
    }}
    \vspace{-2mm}
    \label{fig:cifar_nlp_mnist}
\end{figure*}

 \paragraph{Remark 3: the choice of $\beta_1$.} Theorem \ref{thm_wr} requires  $\beta_1 < \sqrt{\beta_2}$.
  Since $\beta_2$ is suggested to be large, our convergence result can cover flexible choice of $\beta_1 \in [0,1)$. For instance,  $\beta_2=0.999$ brings the threshold of $\beta_1 < 0.9995$, which covers  basically all practical choices of $\beta_1$ reported in the literature, including the default setting $\beta_1=0.9$.  Our theory aligns with our grid-search experiments on MNIST and CIFAR-10 (Figure \ref{fig:intro_paper} (c,d) in Section \ref{section:intro}), where Adam shows good performance for a wide range of $\beta_1$ when $\beta_2$ is large.

\paragraph{Remark 4: guidance to LLM pre-training.} Our theory indicates that larger $\beta_2$ is required when $n$ is large. This equivalently indicates that a larger $\beta_2$ is required when the batch size, which equals $\frac{\mathcal{D}}{n}$, is small. Our subsequent divergence theory (presented later) further implies that making $\beta_2$ dependent on the batch size is necessary to avoid divergence. Collectively, these messages can provide guidance for hyperparameter tuning in LLM pre-training, as confirmed by various literature.  We list some numerical evidence  as follows.
\begin{itemize}
    \item \citet{zhao2024deconstructing,orvieto2025search} report that training fails if both $\betabeta$  are close to 0, and the performance substantially improves when use $\beta_1 = \beta_2 =0.95$ or 0.975. This finding aligns with our results.  We restate some of their numerical results in Figure \ref{fig:otherworks} (a). 
    \item \citet{sreckovic2025your} state that {\it ``\citet{zhang2022adam} shows that higher $\beta_2$ values substantially improve small-batch training, and  \citet{marek2025small} highlights the importance of scaling $\beta_2$ in this regime.''} We restate some of their numerical results in Figure \ref{fig:otherworks} (b) and (c). 
    \item  \citet{zhang2024does} emphasize that  {\it ``larger $\beta_2$ (increase from 0.95 to 0.99 or  0.999) substantially improves small batch size training ... aligning with findings in \citep{zhang2022adam}''}.  We restate some of their numerical results in Figure \ref{fig:otherworks} (d, e). 
    \item \citet{porian2024resolving} report that {\it ``enlarging $\beta_2$ (from 0.95 to 0.99 or 0.999) is essential at lower batch sizes ... This matches the theoretical work \citep{zhang2022adam}.''}
\end{itemize}

\paragraph{Remark 4: Is larger $\beta_2$ always better?}
\yushunrevise{To ensure convergence, Theorem \ref{thm_wr} states that $\beta_2$ needs to be larger than a batch-size-dependent threshold $\beta_2^*$. But what is the optimal $\beta_2$  within the range of $[\beta_2^*, 1)$? We provide an initial discussion based on our current results. }

In Theorem \ref{thm_wr}, the rate includes the constant  $\frac{1}{1-\beta_2}$, which is magnified as $\beta_2 \rightarrow 1$.  We note that such dependency is somewhat unavoidable as it is embedded in the design of Adam: when $\beta_2 = 1$ and initial $v_0= 0$, $v_k$ always equals $0$ and Adam is not well-defined. Consequently, our theoretical bound must naturally capture this worst-case singularity. Nevertheless, the bound in \eqref{eq_thm_wr} remains strictly bounded and valid provided that $\beta_2 \in [0,1)$. In practical scenarios such as $\beta_2 =0.95$ or $ 0.999$, this corresponds to a finite multiplier of 20 or 1000. 

Theorem \ref{thm_wr} suggests that: within the convergence range $[\beta_2^*, 1)$, a larger $\beta_2$ is not necessarily better.  This aligns with the empirical observation in trend in LLM where: when Adam converges under $\beta_2 = 0.95$, a larger $\beta_2$ such as $0.999$ can yield worse performance \citep{wortsman2023small,bai2025adaptive}.

Theoretically, finding the optimal $\beta_2$ requires substantially more effort, and we leave it as a future direction.
\yushunrevise{Empirically, \citep{marek2025small} suggests a heuristic strategy, which can be helpful for practitioners: starting from the popular choice, e.g., $\beta_2 = 0.95$, and if we were to scale the batch size from $B_1$ to $B_2$,  change $\beta_2$ to $\beta_2^{B_2/B_1}$. This strategy aligns with our theory. This tuning strategy is recently adopted in Marin -- an open LLM pretraining project \citep{held2026delphi}.}

\begin{figure*}[t]
  \vspace{-0.5cm}
    \centering
    \subfigure[Fig. 12 from \citep{orvieto2025search}]{
      \begin{minipage}[t]{0.35\linewidth}
      \centering
    \includegraphics[width=\linewidth]{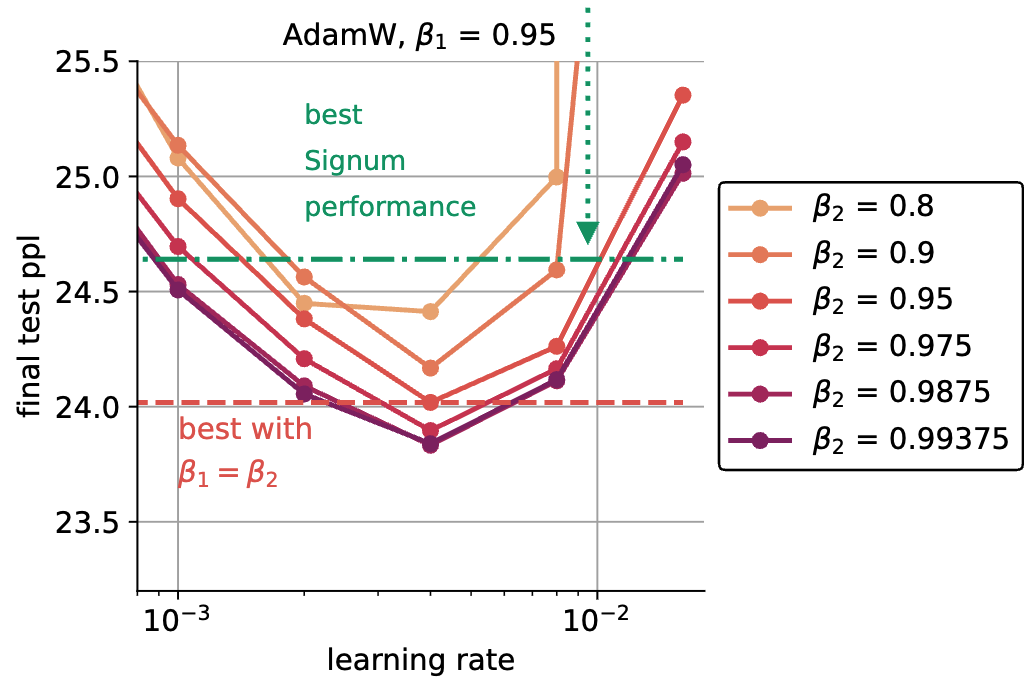}
      \end{minipage}%
      }%
    \subfigure[Fig. 5 from \citep{sreckovic2025your}]{
    \begin{minipage}[t]{0.30\linewidth}
        \centering
        \includegraphics[width=0.65\linewidth]{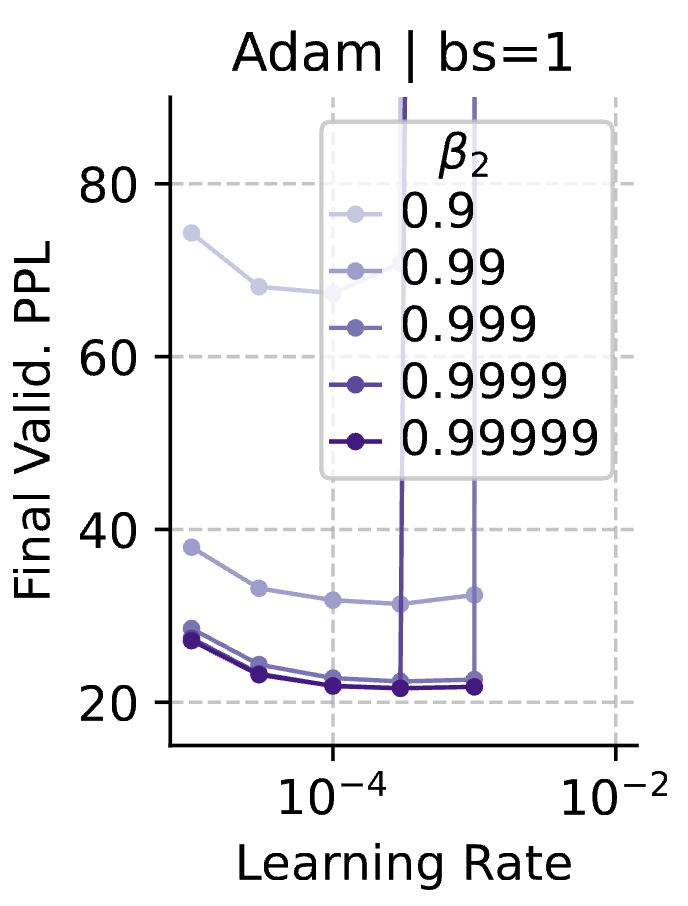}
    \end{minipage}    
    }%
    \subfigure[Fig.4 from \citep{marek2025small}]{
    \begin{minipage}[t]{0.30\linewidth}
    \centering
\includegraphics[width=0.8\linewidth]{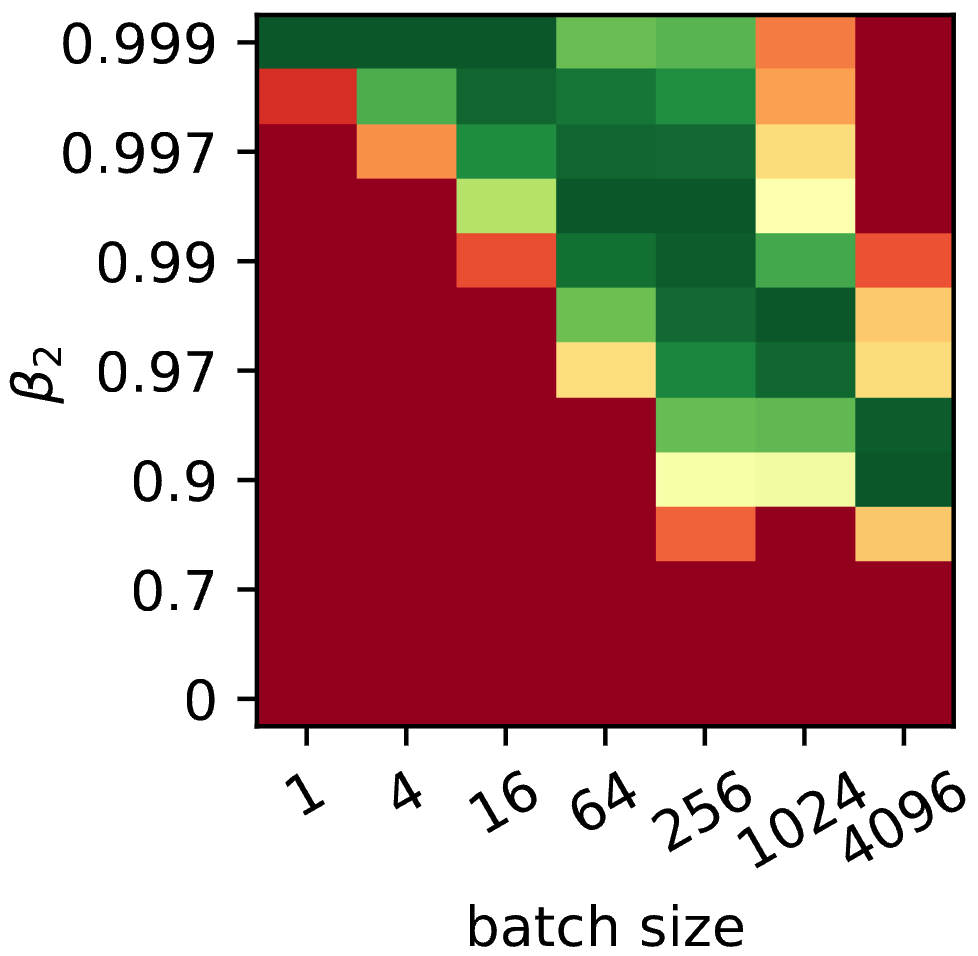}
    \end{minipage}%
    }%
    \\
    \subfigure[Figure 8 (b) from  \citep{zhang2024does}]{
    \begin{minipage}[t]{0.35\linewidth}
    \centering
    \includegraphics[width=\linewidth]{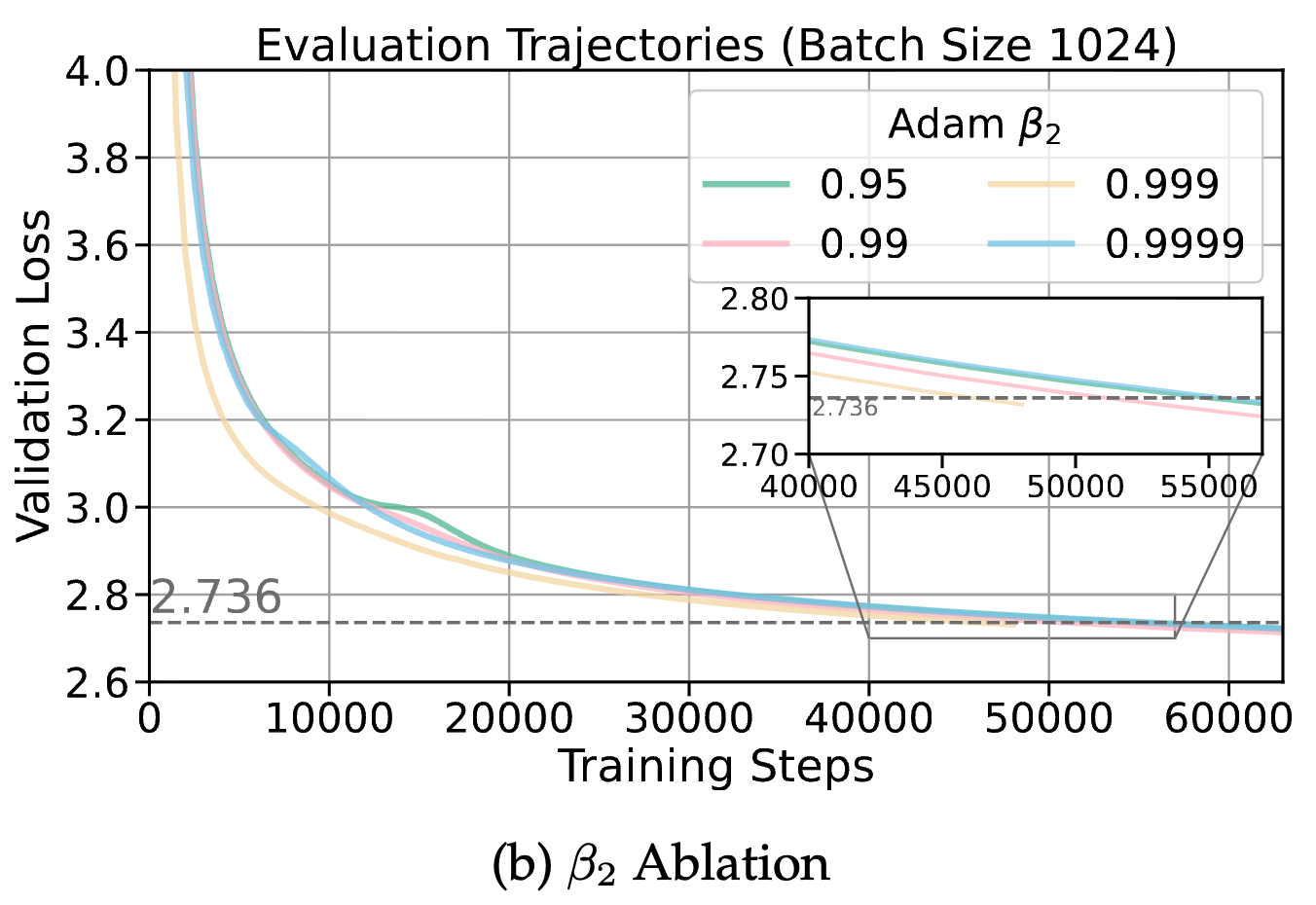}
    \end{minipage}%
    }%
    \subfigure[Table 4 from  \citep{zhang2024does}]{
    \begin{minipage}[t]{0.35\linewidth}
    \centering
    \includegraphics[width=0.65\linewidth]{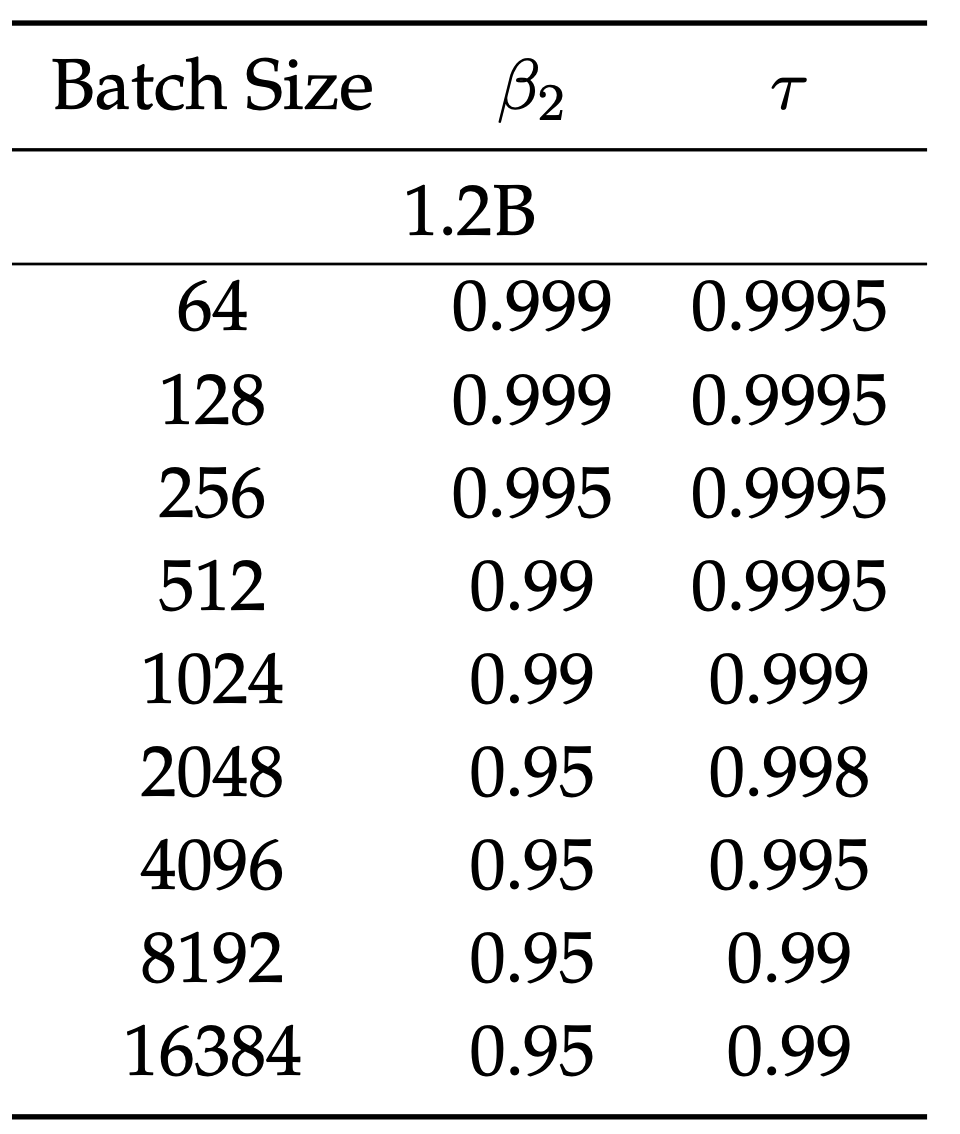}
    \end{minipage}%
    }%
    \centering
    \vspace{-3mm}
    \caption{{\small  On the effect of $\beta_2$ on LLM pre-training from recent literature. {\bf (a,b,c)}: Final validation loss of LLMs trained under different $\beta_2$ and batch size. {\bf (c):} greener color indicates lower validation loss. {\bf (d, e):} The optimal $\beta_2$ to train a LLM with 1.2B parameters under different batch size. Here, $\tau$ serves for other training tricks that are independent of our discussion.  These results reach a consistent conclusion: Larger $\beta_2$ helps boost performance and shall be tuned up under small batch-size regimes.  These results confirm that our theory provides valid guidance for hyperparameter tuning in LLM pre-training. }}
    \vspace{-2mm}
    \label{fig:otherworks}
\end{figure*}

  \paragraph{Remark 5: convergence to a neighborhood of critical points.}
  When {\small $D_0>0$}, Adam converges to a neighborhood of critical points, in lieu of the exact critical points. We emphasize that this is {\it not} due to the limitation of the analysis, and this phenomenon is also observed in practice: even for simple convex quadratic function with  {\small $D_0>0$},  Adam with diminishing stepsize {\it cannot} reach exactly zero gradient (see Figure \ref{fig:cifar_nlp_mnist} (b). We state the non-realizable function in Appendix \ref{appendix:exp_setting}).
  
 In the non-realizable case ({\small $D_0>0$}),  converging to the neighborhood of critical points is common for stochastic gradient methods, including constant-stepsize SGD  \citep{luo1991convergence,bertsekas1997nonlinear, yan2018unified,yu2019linear,liu2020improved} and diminishing-stepsize RMSProp \citep{Manzil2018adaptive,shi2020rmsprop}.  This is because: even though $\eta_k$ is decreasing, the effective stepsize $\eta_k/\sqrt{v_{k}}$ might not decay.    \yushunrevise{ The good news is that the size of the neighborhood {\small $\mathcal{O}(\delta(\beta_2)\sqrt{D_0})$} vanishes to 0 as $\beta_2$ goes to 1 (both in theory and experiments in Figure \ref{fig:cifar_nlp_mnist} (b)). This can be seen in the expression of $\delta(\beta_2)$ in \eqref{def_delta_C}  in Appendix \ref{appendix:lemma_descent_WR}. The size shrinks to 0 because the movement of  $\sqrt{v_{k}}$ shrinks as  $\beta_2$ increases.}

\vspace{1mm}
  As a corollary of Theorem \ref{thm_wr}, we have Corollary \ref{thm1_sgc_wr}  in the realizable case (i.e., $D_0=0$).
  
\begin{corollary}\label{thm1_sgc_wr}
  Under the setting in Theorem \ref{thm_wr}. When $D_0=0$ for Assumption \ref{assum2}, we have
{\small
\begin{align*}
\vspace{-2mm}
     \min_{k \in [1, T]} \Ex\|\nabla f(x_{k})\|_2
   =\mathcal{O}\left(\frac{\log T }{\sqrt{T}} \right).\nonumber
   \vspace{-2mm}
\end{align*}
}%
\end{corollary}

  In the realizable case (i.e. $D_0=0$), Corollary \ref{thm1_sgc_wr} states that Adam can converge to critical points.  This matches our numerical results in Figure \ref{fig:cifar_nlp_mnist} (c).   The  rate in Theorem \ref{thm_wr} and  Corollary \ref{thm1_sgc_wr}  is comparable to that of SGD under the same condition in \citep{vaswani2019fast}.

Similarly to Theorem \ref{thm_wr}, we now present the convergence result for Adam under random shuffling (Algorithm \ref{algorithm_rr}). All the remarks above also apply to the random-shuffling case.

\brrevise{
\begin{theorem}
\label{thm_rr}
{\bf (Convergence result for Algorithm \ref{algorithm_rr})}
\brrevise{
Assume that
$0\leq\beta_1<\sqrt{\beta_2}<1$, $\beta_2\geq\gamma_2(n)
:=
1-\widetilde c_\gamma
\frac{1-\beta_1^n}{n^{5.5}}$,
where $\widetilde c_\gamma\in(0,1]$ is independent of
$n,\beta_1,T$, and
$\eta_k=\eta_0/\sqrt{k}$.
For any $f\in\functionclass$, when $T\in\mathbb{N}$ is sufficiently large, we have
{\small \begin{equation}
\label{eq_thm_rr}
\min_{k\in[1,T]}
\Ex\left[
\mathcal{M}_{D_0,D_1}
\bigl(\|\nabla f(x_{k,0})\|_2\bigr)
\right]
=
\mathcal{O}\left(\frac{\log T}{\sqrt{T}}\right)
+
\mathcal{O}\left(
\widetilde\delta_{D_0,D_1}(\beta_2)\sqrt{D_0}
\right),
\end{equation}}
where $\mathcal{M}_{D_0,D_1}$ is defined in
\eqref{eq:stationarity_measure}, and
$\widetilde\delta_{D_0,D_1}(\beta_2)$ is nonnegative, independent of
$k,T$, and converges to zero as $\beta_2\to1$.}
\end{theorem}
}

  As a corollary of Theorem \ref{thm_rr}, we have the following result in the realizable case (i.e., $D_0=0$).

\begin{corollary}\label{thm1_sgc_rr}
  Under the setting in Theorem \ref{thm_rr}. When $D_0=0$ for Assumption \ref{assum2}, we have
{\small
\begin{align*}
     \min_{k \in [1, T]} \Ex\|\nabla f(x_{k,0})\|_2
   =\mathcal{O}\left(\frac{\log T }{\sqrt{T}} \right).\nonumber  
\end{align*}
}%
 \end{corollary}

The proof of Theorem \ref{thm_wr} and Theorem \ref{thm_rr} can be seen in Section \ref{section:lemma_concentrate_v} -- \ref{appendix_main_body_wr} and  Appendix \ref{appendix:thm_rr}, respectively.

\subsection{Divergence Results}
\label{sec_div}
Now we prove that small-$\beta_2$ Adam can diverge, and the divergence region is problem-dependent. 
The divergence of small-$\beta_2$ Adam suggests that ``large $\beta_2$'' is necessary for Theorem \ref{thm_wr}.

\yushunrevisediverge{
We construct a counter-example in {\small $\functionclass$}.  Consider
{\small $f(x)= \frac{1}{n}\sum_{i=0}^{n-1} f_i(x)$} for $n \geq 3$ and
$x \in \mathbb{R}$, and define $f_i(x)$ as:
{\small \begin{equation}
f_{i}(x)=
\left\{
\begin{array}{ll}
\frac{1+(n-1)a}{2}x^2,
& i=0, \\[2mm]
-\frac{a}{2}x^2,
& i=1,\ldots,n-1,
\end{array}
\right.
\label{counterexample1}
\end{equation}}
where $a>0$ is arbitrary finite positive constant. Summing up all the $f_i(x)$, one can see that
{\small \begin{equation}
  f(x)=\frac{x^2}{2n}.
\end{equation}}
Thus, for any finite positive $a>0$, $f(x)$ belongs to $\functionclass$.
For instance, when $a=1/(n-1)^2$, we have $D_1 =
\frac{19}{8}n^2$, $D_0=0$, and $f(x)$ is lower bounded with optimal
$x^*=0$. Further, it satisfies Assumption \ref{assum2}, but it does not
satisfy the bounded variance \eqref{eq_a2_bdd_variance} condition which
restricts $D_1$ to be $n$.
}

\yushunrevisediverge{
Function \eqref{counterexample1} is inspired by
\citep{reddi2019convergence}, but we allow iterates, 
gradients, and function values to diverge to infinity.
We present the divergence result in Theorem \ref{thm_diverge}. We note
that Sign-SGD is a special case of Adam with $\beta_1=\beta_2=0$, so the
following divergence result applies to Sign-SGD as well.
}

\begin{figure}[t]
\begin{center}
\vspace{-0.5cm}
      \subfigure[Divergence region when $n=10$]{
        \begin{minipage}[t]{0.32\linewidth}
        \centering
        \includegraphics[width=\linewidth]{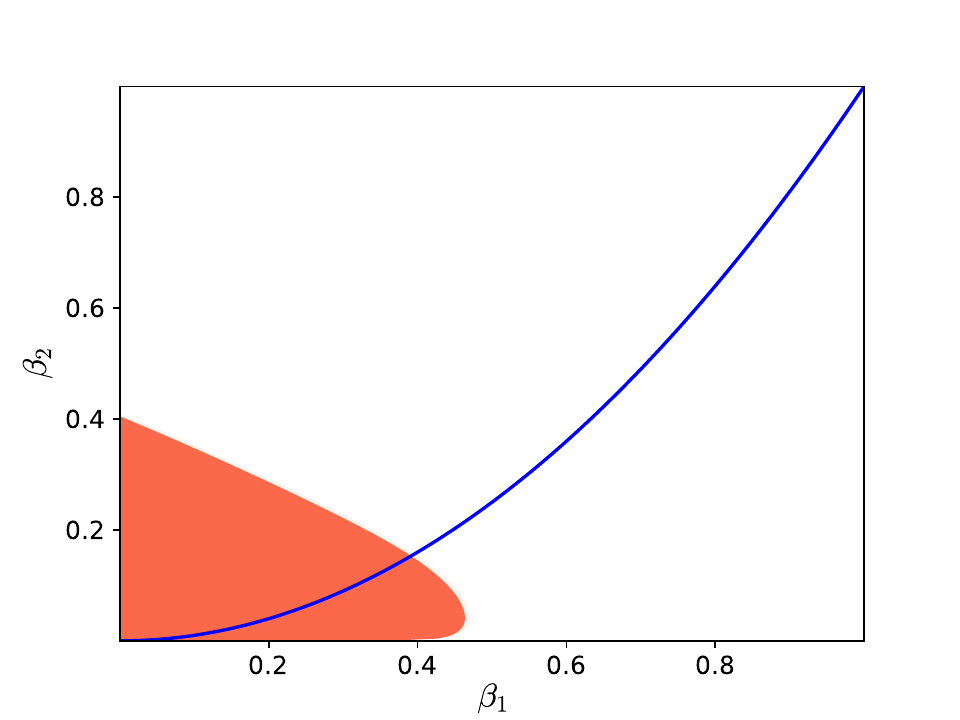}
        \end{minipage}%
        }%
        \subfigure[Divergence region when $n=50$]{
        \begin{minipage}[t]{0.32\linewidth}
        \centering
        \includegraphics[width=\linewidth]{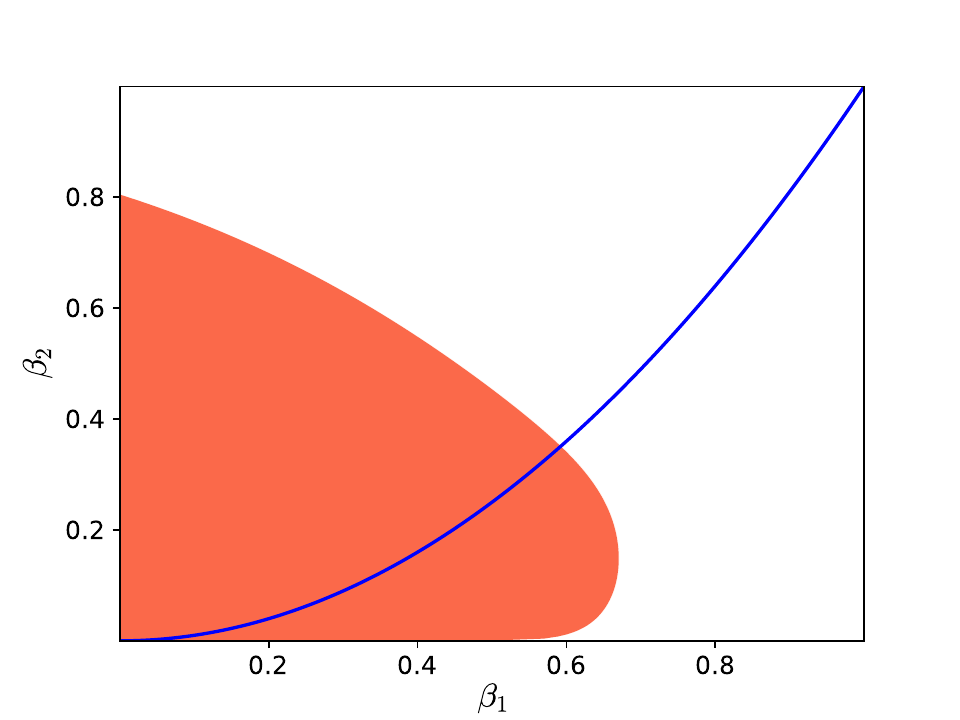}
        \end{minipage}%
        }%
    \subfigure[The effect of small $\beta_2$]{
    \begin{minipage}[t]{0.32\linewidth}
    \centering
    \includegraphics[width=\linewidth]{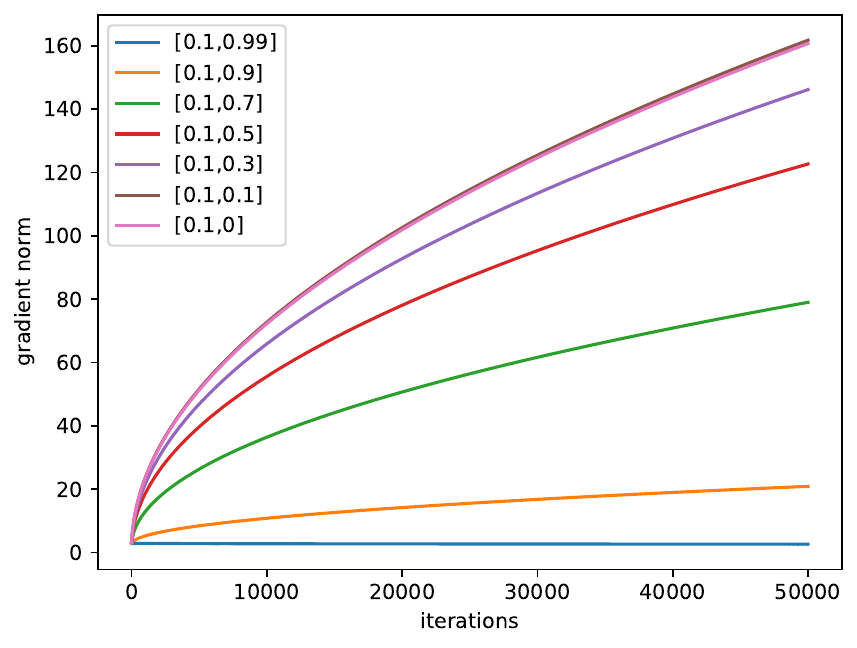}
    \end{minipage}%
    }%
\caption{{\small 
{\bf (a,b):} On function \eqref{counterexample1} with  $a= 1$ and $n= 10, 50 $, Adam diverges in the colored region. The region expands with $n$. The region is plotted by solving condition \eqref{diverge_c1} and \eqref{diverge_c2} in \texttt{NumPy}.
The blue curve satisfies  $\beta_1 = \sqrt{\beta_2}$.  {\bf (c):} When $\beta_1$ and $\beta_2$ are chosen in the colored region, Adam's gardient norm diverges. The labels in (c) stand for $[\beta_1,\beta_2]$. 
}
}
\label{fig:counter_boundary_paper}
\vspace{-10mm}
\end{center}
\end{figure}
\yushunrevisediverge{
\begin{theorem}\label{thm_diverge} 
Consider Adam with stepsize $\eta_{k}=\frac{\eta_{0}}{\sqrt{k}}$.
Let $\mathcal R_n$ be the region of hyperparameter pairs
$\betabeta$ for which the conditions \eqref{diverge_c1} and \eqref{diverge_c2} in Section \ref{appendix:thm_diverge} hold.
For any problem class $\functionclass$ with $n \geq 3$, $D_1 \geq \frac{19}{8}n^2$, and $D_0 \geq 0$, there exists an $f(x) \in \functionclass$, s.t.
for every
$\betabeta\in\mathcal R_n$,
for any sufficiently small initial stepsize $\eta_0$, we have
{\small \begin{equation}
    \|x_{k,i}\|, \|\nabla f(x_{k,i})\|, |f(x_{k,i})|  \to+\infty.
\end{equation}}
The size of the region $\mathcal R_n$ depends on $n$ and it expands  as $n$ grows.
\end{theorem}
}

\yushunrevisediverge{ The proof can be seen in Section \ref{appendix:thm_diverge}. The divergence region $\mathcal R_n$ is visualized in Figure \ref{fig:counter_boundary_paper} (a,b).  The divergence region is problem-dependent and it expands with $n$. Moreover, it
 always stays below the convergence threshold $\gamma_1(n)$ in Theorem \ref{thm_wr}, so the two results are self-consistent (see the final remark in Section \ref{appendix:thm_diverge}).  }
\begin{figure}[t]
\begin{center}
\vspace{0.5cm}
      \subfigure[$\beta_1 =0, \beta_2 = 0$]{
        \begin{minipage}[t]{0.45\linewidth}
        \centering
        \includegraphics[width=\linewidth]{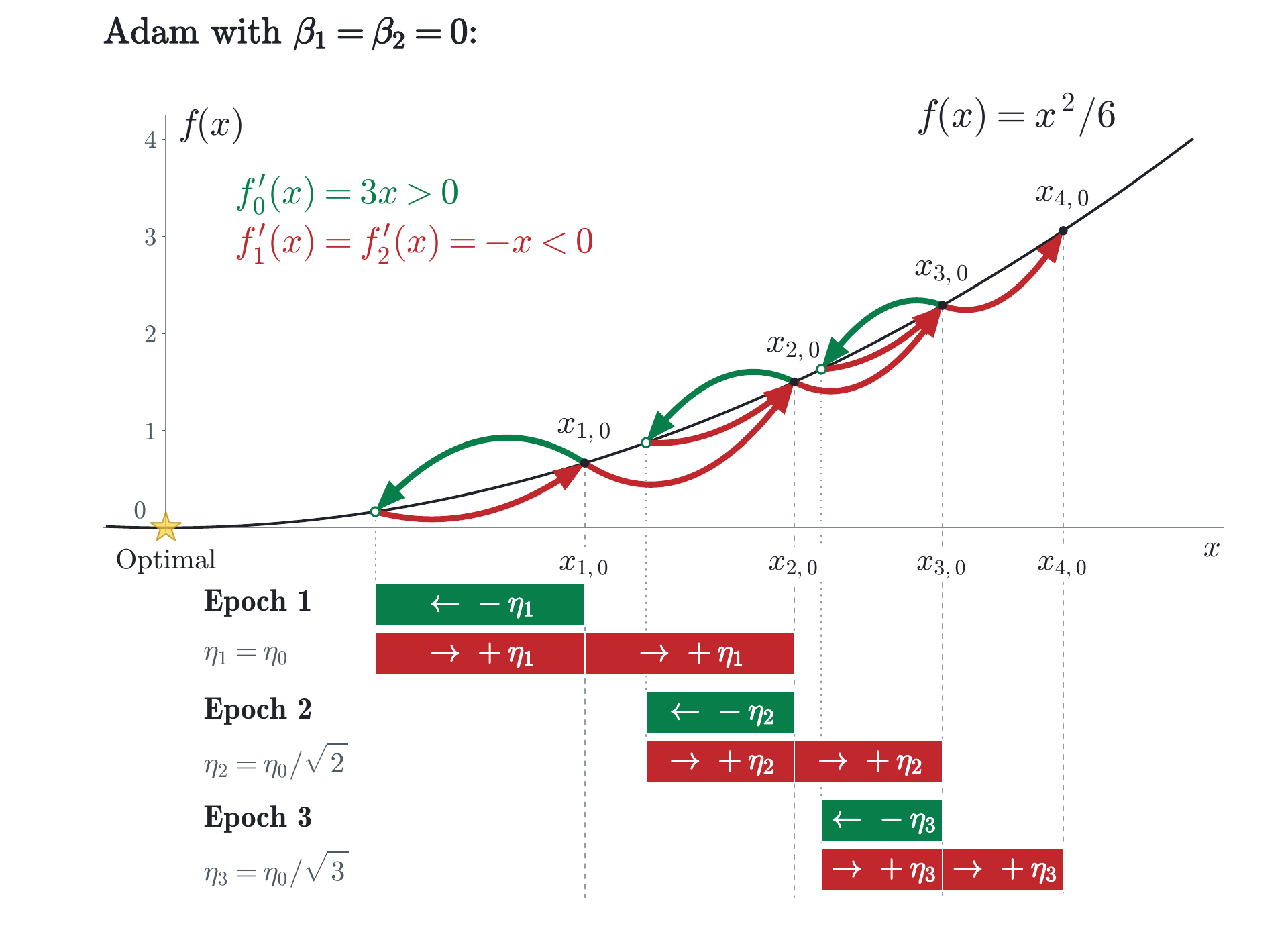}
        \end{minipage}%
        }%
        \subfigure[$\beta_1 =0, \beta_2 = 0.999$]{
        \begin{minipage}[t]{0.45\linewidth}
        \centering
        \includegraphics[width=\linewidth]{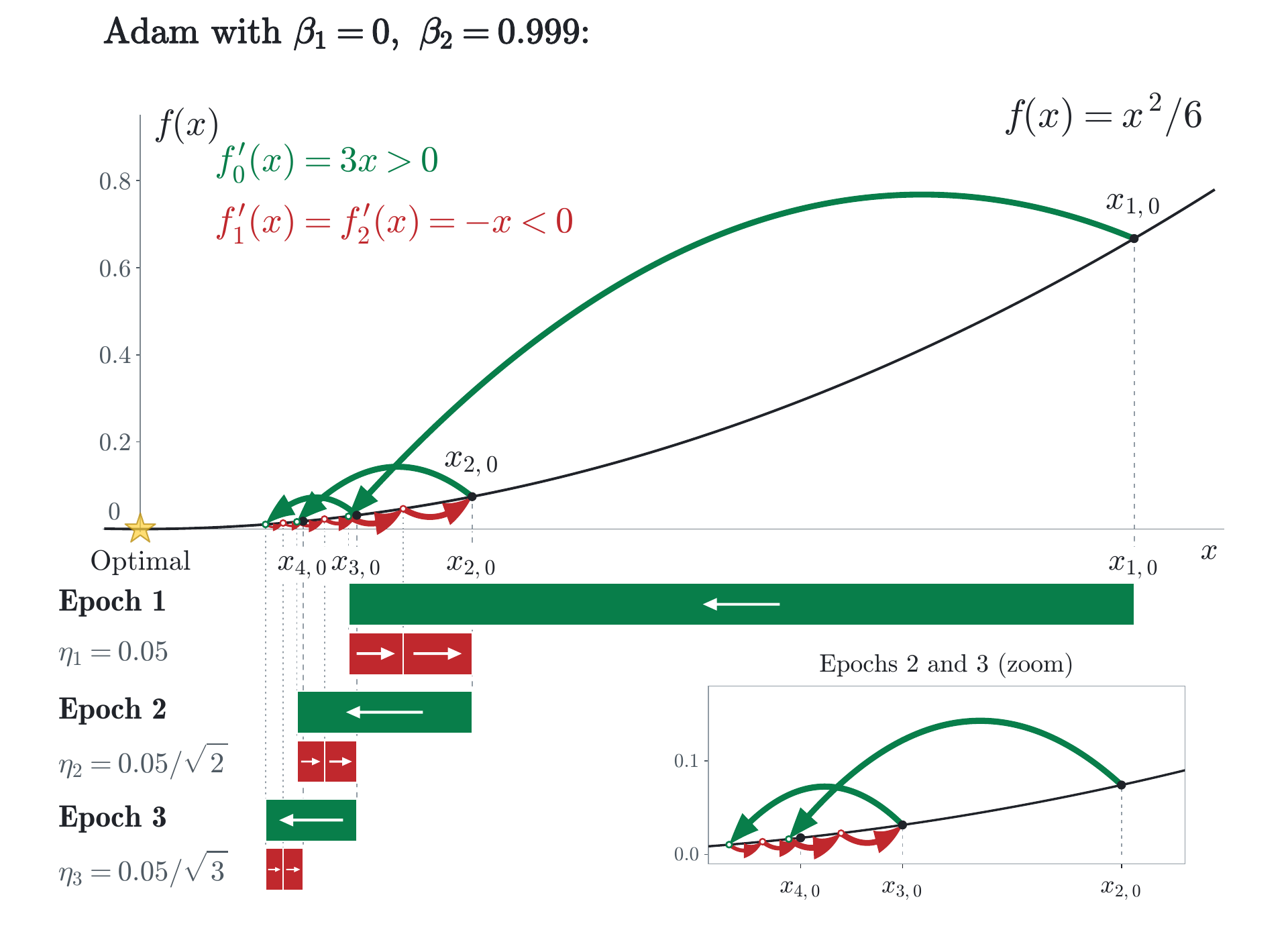}
        \end{minipage}%
        }%
        \\
        \subfigure[$\beta_1 =0.5, \beta_2 = 0.999$]{
        \begin{minipage}[t]{0.45\linewidth}
        \centering
        \includegraphics[width=\linewidth]{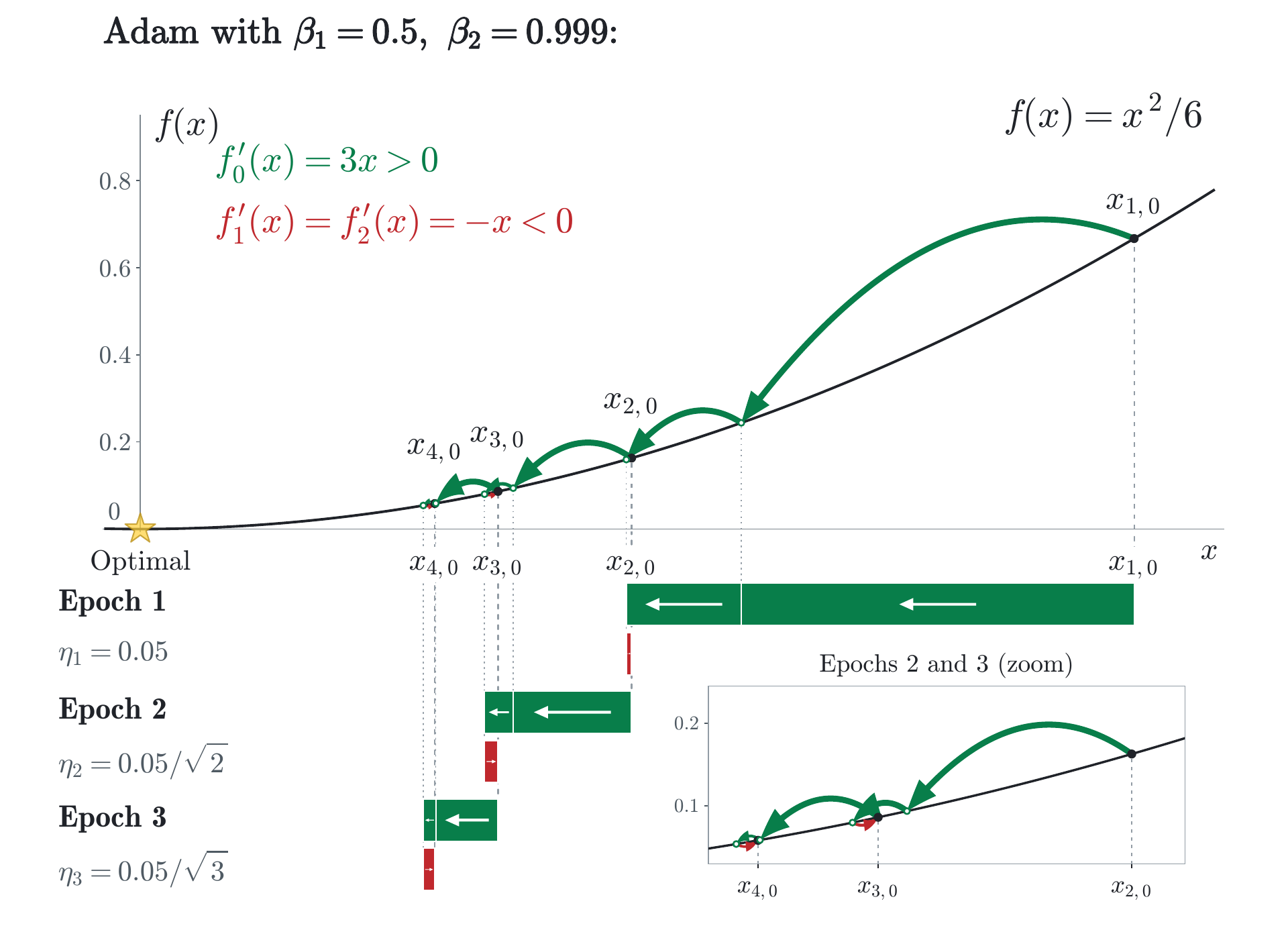}
        \end{minipage}%
        }%
        \subfigure[$\beta_1 =0.9, \beta_2 = 0.999$]{
        \begin{minipage}[t]{0.45\linewidth}
        \centering
        \includegraphics[width=\linewidth]{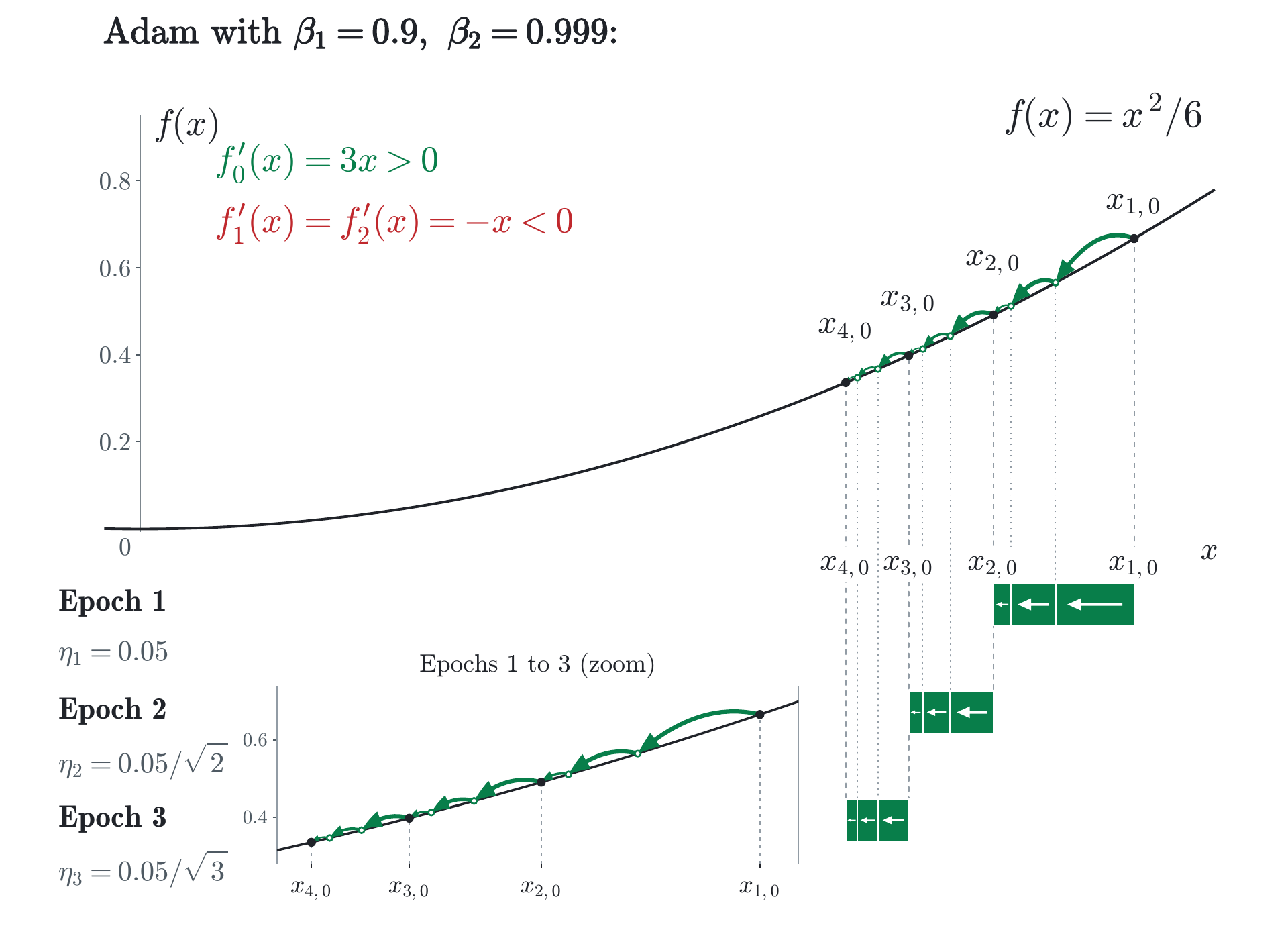}
        \end{minipage}%
        }%
\caption{
{\small 
{\bf (a,b,c,d):} The dynamics of Adam on counterexample \eqref{counterexample1} with $a = 1$ and $n =3$ under different $\beta_1$ and $\beta_2$. The optimal $x^* =0$.  {\bf (a)}: Adam diverges. {\bf (b, c, d)}: Adam converges. The divergence happens when the red forces outweigh the green forces, which happens when $\beta_2$ is small or equals 0 in {\bf (a)}.  The green forces become dominant when $\beta_2$ is large in {\bf (b)}.  Compared with {\bf (b)},    $\beta_1 = 0.5$  in {\bf (c, d)} further helps shrink the red forces.
}
}
\label{fig:divergence_figure_paper}
\vspace{-10mm}
\end{center}
\end{figure}

\yushunrevisediverge{To flush out the insight behind the counterexample,   we present the dynamics of Adam on counterexample \eqref{counterexample1} in Figure \ref{fig:divergence_figure_paper}. Here, the divergence happens when the red forces (towards right) outweigh the green forces (towards left), and the convergence happens if otherwise. When both $\beta_1$ and $\beta_2$ are small as in (a), the fast movement of $v_k$ provides up-to-date normalization, which weakens the strength of the green forces, causing divergence. The divergence can be avoided by using a larger $\beta_2$ as in (b), where the changes of $v_k$ is slowed down, and green forces dominate the red forces. In (c, d), a positive $\beta_1$ such as  $0.5$ or $0.9$ further helps shrink the red forces. The convergence in (b) and (c, d) align with our convergence theory in Section \ref{sec_conv_wr}.}

\yushunrevisediverge{Now we report the numerical behavior of Adam under small $\beta_1$ and $\beta_2$.  As shown in Figure \ref{fig:counter_boundary_paper} (c), when both $\beta_1$ and $\beta_2$ are small, the gradient norm of Adam increases monotonically over iterations, causing divergence. 
  In Figure \ref{fig:toy_gradient}, we provide $\betabeta$ grid-search on the counter-example \eqref{counterexample1}.  When both $\beta_1$ and $\beta_2$ are small, the optimality gap and gradient norm become significantly larger, consistent with the divergence observed in Figure \ref{fig:counter_boundary_paper} (c).  The divergence region expands with $n$. 
  }

\begin{figure*}[t!]
      \vspace{-0.5cm}
        \centering
          \vspace{-2mm}
    \centering
    \subfigure[$n=5$]{
    \begin{minipage}[t]{0.24\linewidth}
    \centering
    \includegraphics[width=\linewidth]{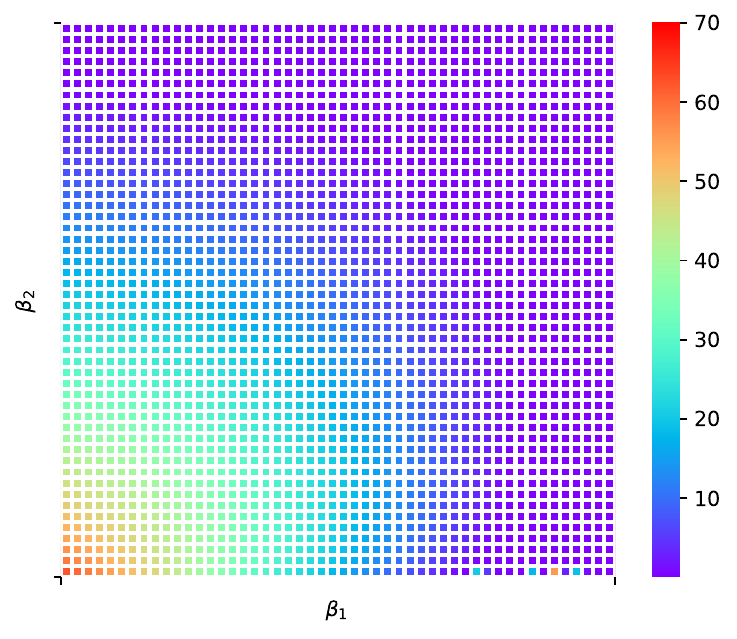}
    \end{minipage}%
    }%
    \subfigure[$n=10$]{
      \begin{minipage}[t]{0.24\linewidth}
      \centering
    \includegraphics[width=\linewidth]{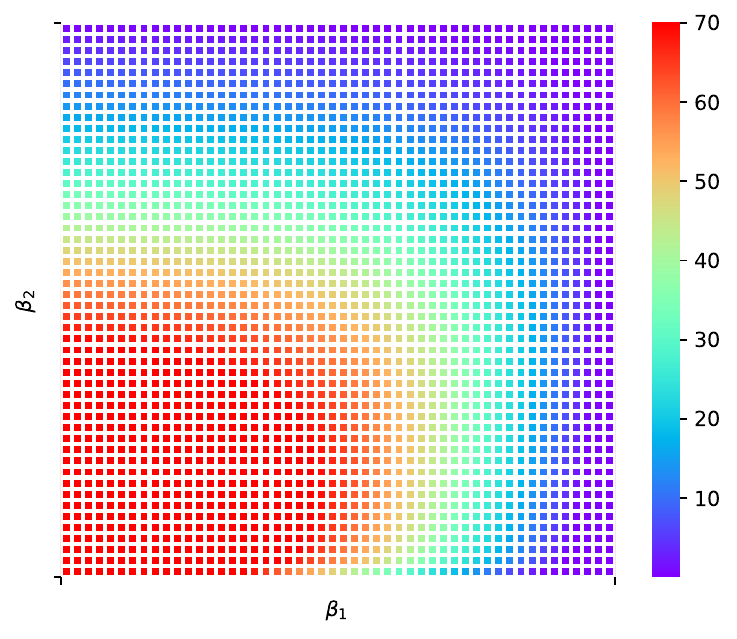}
      \end{minipage}%
      }%
    \subfigure[$n=15$]{
      \begin{minipage}[t]{0.24\linewidth}
      \centering
    \includegraphics[width=\linewidth]{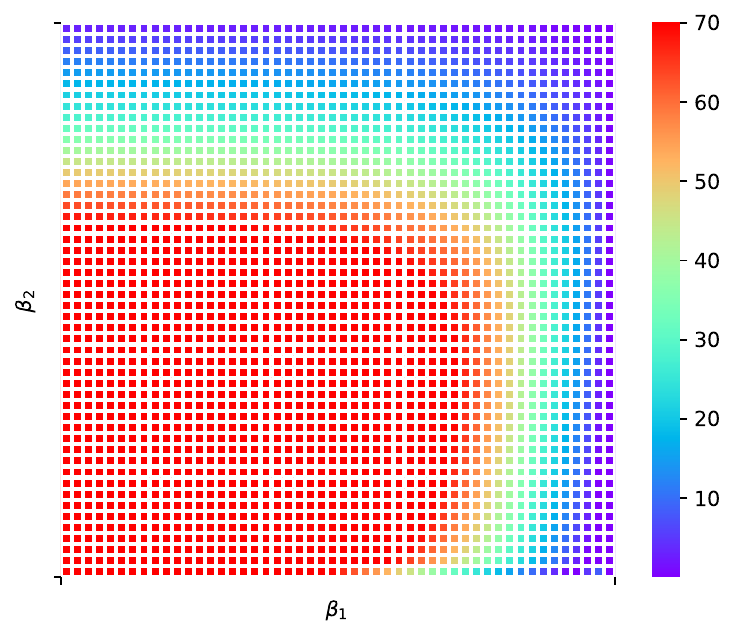}
      \end{minipage}%
      }%
    \subfigure[$n=20$]{
      \begin{minipage}[t]{0.24\linewidth}
  \includegraphics[width=\linewidth]{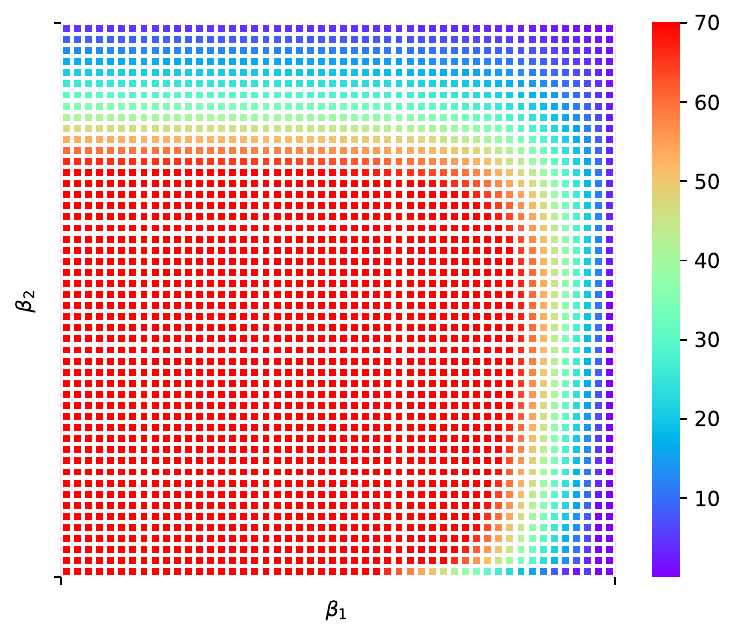}
      \end{minipage}%
      }%
      \\
        \subfigure[$n=5$]{
        \begin{minipage}[t]{0.24\linewidth}
        \centering
        \includegraphics[width=\linewidth]{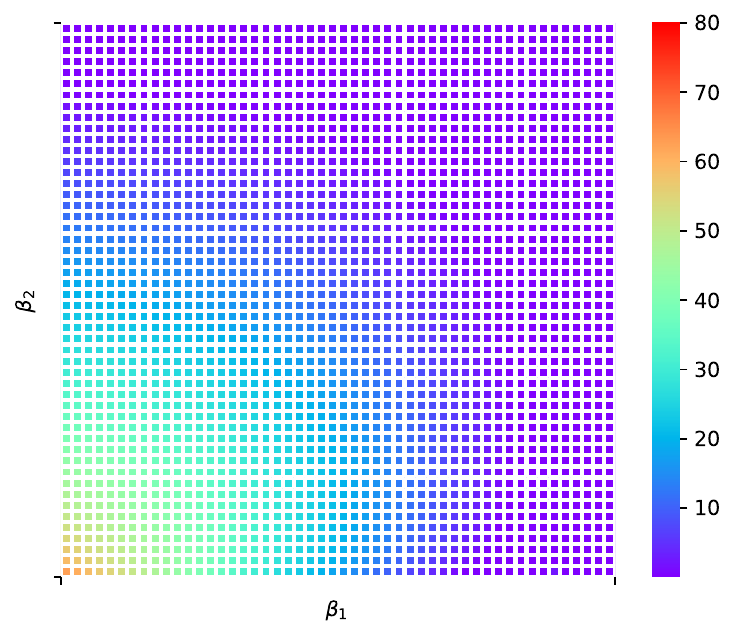}
        \end{minipage}%
        }%
        \subfigure[$n=10$]{
          \begin{minipage}[t]{0.24\linewidth}
          \centering
        \includegraphics[width=\linewidth]{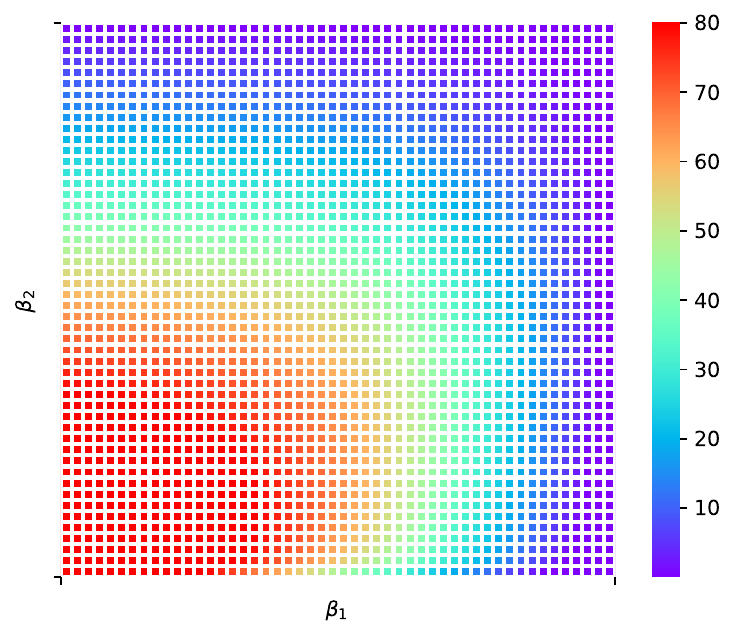}
          \end{minipage}%
          }%
        \subfigure[$n=15$]{
          \begin{minipage}[t]{0.24\linewidth}
          \centering
        \includegraphics[width=\linewidth]{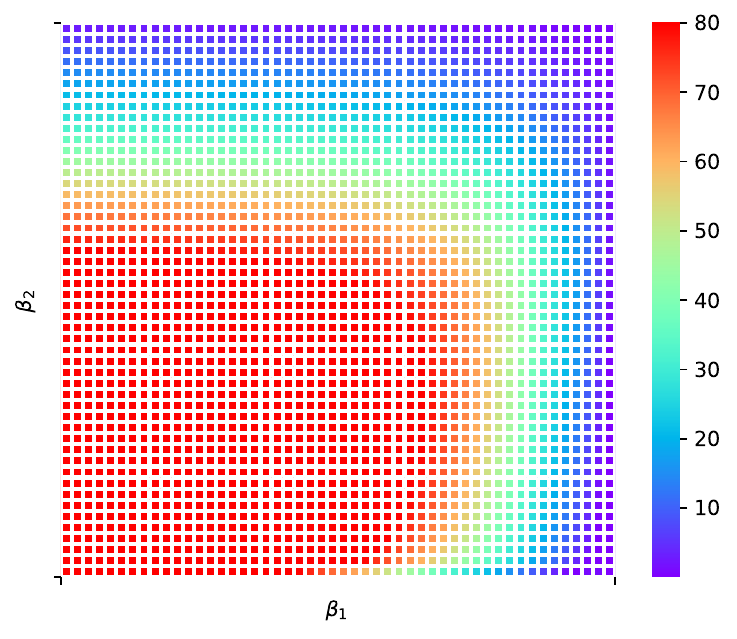}
          \end{minipage}%
          }%
        \subfigure[$n=20$]{
          \begin{minipage}[t]{0.24\linewidth}
          \centering
        \includegraphics[width=\linewidth]{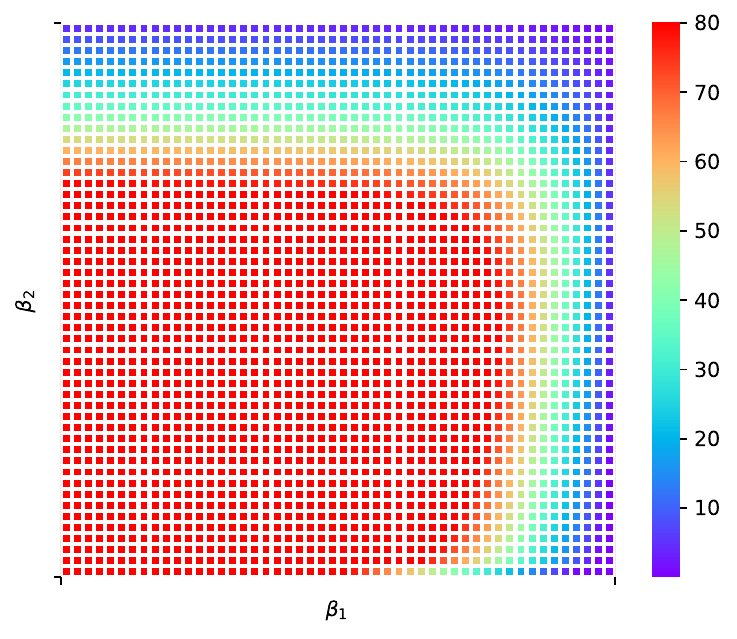}
          \end{minipage}%
          }%
        \centering
        \vspace{-3mm}
        \caption{{\small  The performance of Adam on the counterexample \eqref{counterexample1} with  $a = 1$. {\bf (a) - (d):}   The optimality gap $|x-x^*|
        $ after running 50k iterations. {\bf (e-h):} The gradient norm after running 50k iterations. }}
        \vspace{-5mm}
        \label{fig:toy_gradient}
\end{figure*}

Finally, we emphasize that the orange region in Figure \ref{fig:counter_boundary_paper} (a, b) is {\it not} discussed in \citep{reddi2019convergence} because we consider $n$ fixed while they allow $n$ changing. 
Combining the convergence result (Theorem \ref{thm_wr}) and divergence result (Theorem  \ref{thm_diverge}), we establish a clearer image on the relation between $\betabeta$ and Adam's qualitative behavior.

\paragraph{Remark 7: on the divergence of SignSGD.} Theorem \ref{thm_diverge} implies that SignSGD (Adam with $\beta_1 = \beta_2 = 0$) can diverge. The divergence can be avoided by modifying the algorithm, such as adding momentum \citep{sun2023momentum,jiang2025improved}; or by introducing stronger assumptions like the bounded variance condition in \eqref{eq_a2_bdd_variance} and increasing batch size \citep{bernstein2018signsgd}. In fact, the bounded variance assumption alone is sufficient to exclude our counter-example from the considered function class, and SignSGD will converge to the neighborhood of critical points with $\eta_k = \eta_0 /\sqrt{k}$ in that case \footnote{We have not recognized a clean convergence proof of SignSGD under bounded variance  \eqref{eq_a2_bdd_variance} without increasing batch size in the literature, but it is expected to be a simple proof. The convergence fails when relaxing \eqref{eq_a2_bdd_variance} to our Assumption \ref{assum2}.}. However, the analysis under bounded variance does not reveal the true divergent behavior of SignSGD. Our counter-example, which satisfies Assumption \ref{assum2} (with $D_1 \geq  2n^2$) but not bounded variance \eqref{eq_a2_bdd_variance} (with $D_1 = n$), provides additional motivation to relax \eqref{eq_a2_bdd_variance}. \yushunrevisediverge{An intriguing open question is whether there exist more counter-examples with $D_1 \in (n, \frac{19}{8}n^2)$}. We leave it for future investigation.

\paragraph{Remark 8: more fine-grained characterization on the critical boundary.} Integrating the divergence and convergence theories discussed above, one can conclude that there exists (at least) one critical boundary $(\beta_1^*, \beta_2^*)$ that demarcates the di-convergence phase transition. Since both  the divergence region and convergence region depend on $n$ (or equivalently, depend on batch size $\frac{\mathcal{D}}{n}$),  the critical boundary $(\beta_1^*, \beta_2^*)$ must also be dependent on the batch size and is located in the white region in Fig. \ref{fig:intro_paper} (b).  Note that our theory implies the existence of the critical boundary, whereas we have not fully determined the precise number and the shape of the boundar(-ies). Nevertheless, our experimental results in Fig. \ref{fig:intro_paper} (c, d) suggest that there exists only one boundary, which likely resembles the geometry of the blue region in  Fig. \ref{fig:intro_paper} (b).  We only point out its existence here and leave a more precise characterization as a future direction.

\section{Key Lemmas for the Convergence Result in  Theorem \ref{thm_wr}}
\label{sec:proofidea_wr}

Here, we summarize the key challenges in the proof of Theorem \ref{thm_wr}. The proof of Theorem \ref{thm_wr} can be seen in Section \ref{appendix_main_body_wr}. The proof of Theorem \ref{thm_rr}
follows a similar idea and we relegate the proof to Appendix \ref{appendix:thm_rr}.

\paragraph{Additional notations.}
For $a \in \mathbb{R}$, we use $\lceil a \rceil$ to denote taking the ceiling (rounding $a$ up). For $a \in \mathbb{N}$, we denote $[a]$ as the index set $\{1,\cdots, a\}$. For the function indices $n$ in \eqref{finite_sum}, we abuse the notation  $[n]=\{0, \cdots, n-1\}$. 
We denote $\partial_l f(x)$  and $\partial_l f_i(x)$  as the $l$-th component of $\nabla f(x)$ and  $\nabla f_i(x)$, respectively. We use $x_{l,k}$ to denote the $l$-th component of $x$ at the $k$-th iteration, i.e., $x_k$. Similarly for $v_{l,k}$ and $m_{l,k}$.  We use $\mathbb{I}(\cdot)$ as the indicator function. $\mathbb{E}(\cdot)$ means taking expectation over the whole trajectory.  
\RED{For $k\ge 1$, $\mathbb E_k[\cdot]$ denotes the conditional expectation given the entire history up to the end of iteration $k-1$
(i.e., given $x_1,m_0,v_0$ and all past sampled indices $\tau_1,\ldots,\tau_{k-1}$), but excluding the fresh sample $\tau_k$.}
\brrevise{Throughout the paper, whenever an $\arg\max$ is not unique, we select the smallest maximizing index.} We abuse the notation of $\alpha$ as follows: for $\partial_l f_{\alpha}(x)$, we  define $\alpha:= \arg \max_i|\partial_l f_i(x)|$; for $\partial_\alpha f(x)$,  we define $\alpha:= \arg \max_l|\partial_l f(x)|$. Similarly for $m_{\alpha, k}$ and $v_{\alpha,k}$. 
We use ``i.i.d.'' and ``r.v.'' as abbreviations for ``independent and identically distributed'' and ``random variable'', respectively.

The core of the proof  is to characterize when  the expected update direction of Adam $\Ex \left(\frac{m_k}{\sqrt{v_k}}
\right)$ constitutes a descent direction, i.e., a direction that lies in the dual cone of gradient $\nabla f(x_k)$. There are at least twofold challenges. {\bf First},  $v_k$ is a random variable and it appears in the denominator.  This makes the entire system a stochastic non-linear dynamic system, which is difficult to analyze in general.  Further, $v_k$ can potentially hit 0, which imposes extra difficulties.  {\bf Second}, $m_{k}$ contains heavy historical signals, which distort the trajectory from the gradient direction.  

\begin{figure}[h]
\begin{center}
\vspace{-1mm}
\centerline{\includegraphics[width=0.45\textwidth]{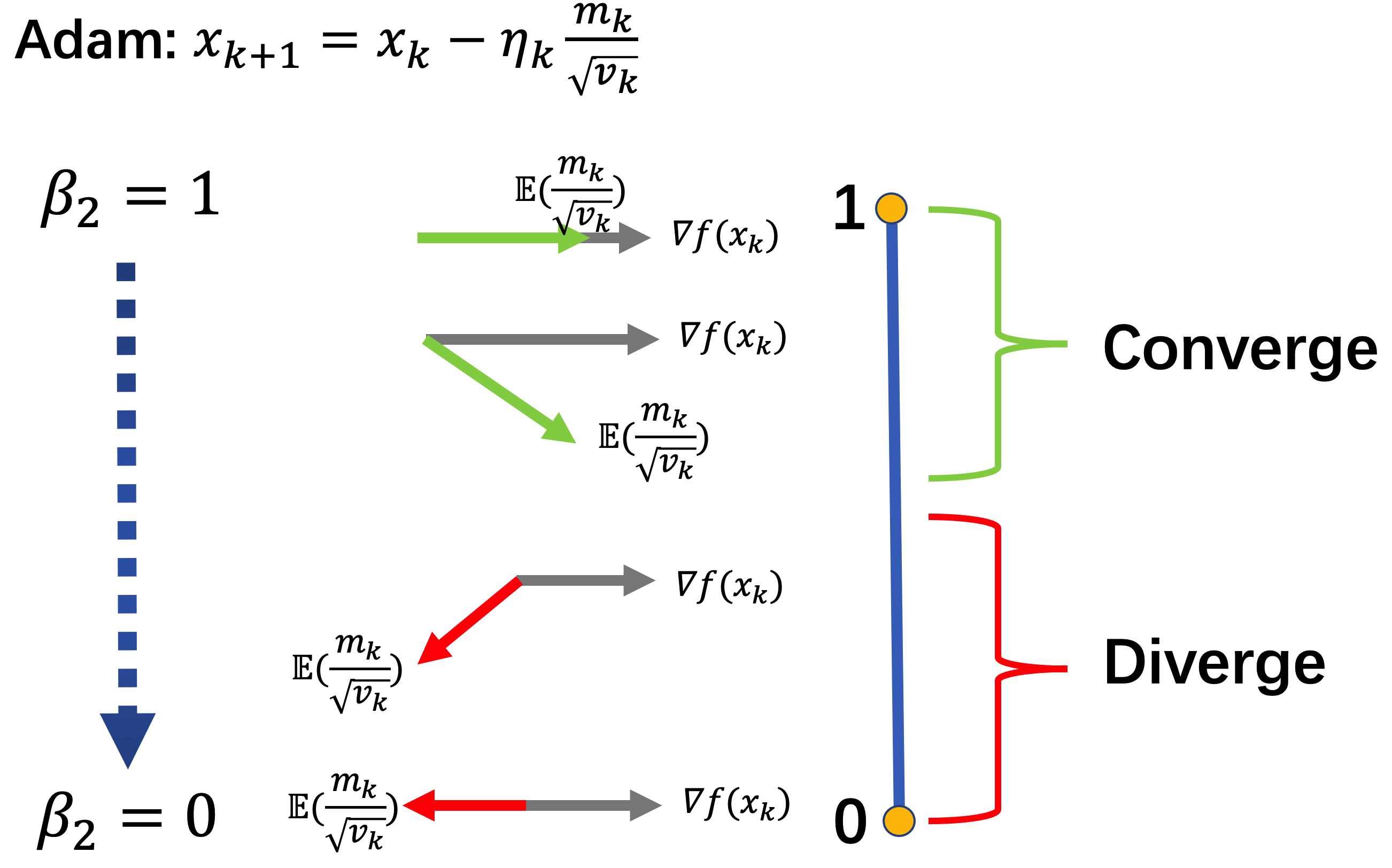}}
\caption{{\small An illustration of the changes of Adam's update direction when $\beta_2$ changes. } }
\label{fig:update_direction}
\end{center}
\vspace{-5mm}
\end{figure}

Our key insights are illustrated in Figure \ref{fig:update_direction}: we find that Adam's update direction  $\Ex \left(\frac{m_k}{\sqrt{v_k}}\right)$ is close to $\nabla f(x_k)$ when $\beta_2$ is large (leading to convergence) and starts to deviate to the opposite direction when $ \beta_2$ is small (causing divergence). We now introduce the proof for our convergence results. The proof for the divergence result is shown later in Section \ref{sec_div}.

\paragraph{Step I: Concentration effects of $\frac{1}{\sqrt{v_k}}$ when $\beta_2$ is large.}
We find that $\frac{1}{\sqrt{v_k}}$ concentrates around $\frac{1}{\sqrt{\Exk (v_k)}}$ when $\beta_2$ is large. In particular, we prove that:
{\small $$
\frac{1}{\sqrt{v_k}} \approx  \frac{1}{\sqrt{\Exk (v_k)}},  \text{ w.h.p., when $\beta_2$ is large}
$$}
With this concentration property, $\frac{1}{\sqrt{v_k}}$ is stabilized. {\bf Why does large $\beta_2$ help?} Intuitively, this is because large $\beta_2$ slows down the changes of $v_k$, and its behavior will become largely predictable. As a result, $-\frac{\nabla f_{\tau_k}(x_k)}{\sqrt{v_k}}$ becomes a descent direction in this case, i.e.,
{\small $$\Exk\left(\frac{\nabla f_{\tau_k}(x_k)}{\sqrt{v_k}}\right) \approx \frac{\Exk \left(\nabla f_{\tau_k}(x_k)\right)}{\sqrt{\Exk (v_k)}} = \frac{\nabla f(x_k)}{\sqrt{\Exk (v_k)}}, $$}
which lies in the dual cone of the gradient direction $\nabla f(x_k)$.

 It requires substantial effort to compare $ \frac{1}{\sqrt{v_k}}$ and $ \frac{1}{\sqrt{\Exk (v_k)}}$. In statistical theory, it is common to show the concentration of a random variable $x$ around its mean $\Ex (x)$, yet it is much less common  to characterize how $\frac{1}{\sqrt{x}}$ deviates from $\frac{1}{\sqrt{\Ex(x)}}$, especially when $x$ can be arbitrarily close to 0  and no extra boundedness condition is imposed.  Without any specific property on  r.v. $x$,   every small deviation in  $\sqrt{x}$ will be amplified  in $\frac{1}{\sqrt{x}}$  and it can behave very badly near zero and become arbitrarily far away from $\frac{1}{\sqrt{\Ex (x)}}$ \citep{durrett2019probability}.

In Lemma \ref{lemma_concentrate_v}, we  prove that $\frac{1}{\sqrt{v_k}}$ concentrates around $\frac{1}{\sqrt{\Exk (v_k)}}$ {\it when $\beta_2$ is large}, and the result holds {\it without} any boundedness condition on stochastic gradients. We prove this result by utilizing two special properties of  Adam: 1) stochastic gradients have a ``geometric sum'' structure in  $v_k$.  2) the index of stochastic gradients are uniformly sampled from a finite index set. With these two special properties, we find the behavior of $\frac{1}{\sqrt{v_k}}$ is largely predictable when $\beta_2$ is large. 

Technically, the concentration result is established via two steps: (i) map the dynamics of  $1/\sqrt{v_k}$ to the dynamics of a sequence of i.i.d. Bernoulli r.v.s, which is bounded and is much easier to analyze; (ii) map the dynamics back by several decoupling steps.  Step (i) and (ii) allow us to track the dynamics of possibly unbounded r.v. sequence using bounded Bernoulli proxies.  The rigorous concentration effect is shown in Lemma \ref{lemma_concentrate_v} as follows.  Lemma \ref{lemma_concentrate_v} can be used as a generic tool for analyzing Adam-type algorithms.

\begin{lemma}\label{lemma_concentrate_v}
Assume Algorithm \ref{algorithm_wr} satisfies: $\beta_2 \geq \frac{1}{2}$; $0\leq \beta_1<\sqrt{\beta_2}<1$; $\beta_2$ satisfies $\frac{1-\beta_2}{\beta_2^n} <  \frac{1}{8n} - \frac{\delta}{4}$; and $\eta_k = \frac{\eta_0}{\sqrt{k}}$. For any $f(x) \in \functionclass$, and for any $0< \delta\leq \frac{1}{4n}$, let $k\geq \lceil\frac{\log (n\delta)}{\log \beta_2} \rceil +n + 1$. 
Define the threshold $R_k := \frac{16 \sqrt{2} \Delta_1 }{ \sqrt{k}}\left(\lceil\frac{\log (n\delta)}{\log \beta_2}\rceil +n \right)$, which $\Delta_1$ is defined in \eqref{eq_constants_wr}.
\brrevise{With probability at least
{\small $$
1-n\exp\left(
-\frac{\delta^2}
{(1-\beta_2)(\frac{28}{3n}+\frac{8}{3}\delta)}
\right),
$$}
the following result holds for every $l\in[d]$: whenever
$\max_i|\partial_l f_i(x_k)|\geq R_k$, we have
{\small \begin{equation}
\label{eq_concentrate_v}
\frac{C_{\mathrm{lower}}}{\sqrt{\mathbb{E}_k(v_{l,k})}}
\leq
\frac{1}{\sqrt{v_{l,k}}}
\leq
\frac{C_{\mathrm{upper}}}{\sqrt{\mathbb{E}_k(v_{l,k})}}.
\end{equation}}}
where the constants are defined as:
{\small $$
C_{\text{lower}} := 1- (1-\beta_2) \frac{4n}{(1-2n\delta)\beta_2^n}, \quad 
C_{\text{upper}} := \left(1-(1-\beta_2)\frac{8n}{(1-2n\delta)\beta_2^n}\right)^{-1/2}.
$$}
\end{lemma}
The proof of Lemma \ref{lemma_concentrate_v} is shown in Section \ref{section:lemma_concentrate_v}. By Lemma \ref{lemma_concentrate_v}, we have  $\frac{1}{\sqrt{v_k}}\approx \frac{1}{\sqrt{\Exk (v_k)}}$ as $\beta_2 \rightarrow 1^{-}$.

\paragraph{Step II: Potential function to handle $m_k$.} Note that the concentration effect of $\frac{1}{\sqrt{v_k}} \approx \frac{1}{\sqrt{\Ex_k(v_k)}}$ alone is not enough to establish the convergence of Adam. This is because $m_k$ also contains heavy historical gradient signals which further distort the update direction from the gradient direction. In other words, it is still unclear whether  $\frac{\Ex_k(m_k)}{\sqrt{\Ex_k(v_k)}}$ is close to $\nabla f(x_k)$.

We use a potential function $f(z_k)$ to offset the effect of $m_k$, where $f(\cdot)$ is the original function in the problem  \eqref{finite_sum} and $z_k := \frac{x_k - \beta_1^n x_{k-n}}{1-\beta_1^n}$. The auxiliary sequence  $z_k$ can help cancel out all historical signals up to $(k-n)$-th iteration, which provides more convenience to the convergence analysis.
{\small \begin{eqnarray}
    z_{k}-z_{k+1} &=& \frac{x_k  - x_{k+1} -  \beta_1^n \left(x_{k-n} - x_{k+1-n} \right)}{1-\beta_1^n} =  \frac{1}{{1-\beta_1^n}}\cdot\left(\eta_k \frac{m_k}{\sqrt{v_k}}  - \beta_1^n \eta_{k-n} \frac{m_{k-n}}{\sqrt{v_{k-n}}}\right)  \nonumber\\
    &\overset{\text{Lemma \ref{lemma_concentrate_v}}}{\approx}  &\frac{(1-\beta_1)\eta_k}{{1-\beta_1^n}} \cdot \frac{\nabla f_{\tau_k}(x_k) + \beta_1 \nabla f_{\tau_{k-1}}(x_{k-1}) + \cdots + \beta_1^{n-1} \nabla f_{\tau_{k-n+1}}(x_{k-n+1})}{\sqrt{\Ex_k(v_k)}}  \nonumber
\end{eqnarray}}
Note that the  potential function is inspired by the following $L_k$ from SGD analysis \citep{liu2020improved}: 
{\small $$
\vspace{-1mm}
L_k=\left(f\left(\widetilde{z}_k\right)-f^{\star}\right)+\sum_{i=1}^{k-1} c_i\left\|x_{k+1-i}-x_{k-i}\right\|^2, \quad \widetilde{z}_k = \frac{x_k -\beta_1 x_{k-1}}{1-\beta_1}, \quad c_i>0
\vspace{-1mm}
$$}
We find that this idea of constructing  $L^k$ and $\widetilde{z}_k $ is useful for Adam analysis, and we make the following changes. 
\begin{itemize}
    \item {\bf Change I:} For Adam, we find that $L_k$ can be simplified by setting $c_i = 0$. This is because the purpose of introducing $c_i$ is to help bound the update magnitude $\frac{m_k}{\sqrt{v_k}}$, which is inherently bounded for Adam (shown later in Lemma \ref{lemma_delta} in \brrevise{Appendix \ref{appendix_notations}}). %
    \item {\bf Change II:} We replace \( \widetilde{z}_k \) with \( z_k \), which differs by retaining \( (n - 1) \) additional historical gradient signals. This modification enlarges the convergence region in the \( \betabeta \) diagram: with \( \widetilde{z}_k \), convergence requires
$
\beta_2 \geq 1 - \mathcal{O}\left( \frac{1 - \beta_1}{n^{5}} \right),
$
whereas with \( z_k \),  the requirement is relaxed to 
$
\beta_2 \geq 1 - \mathcal{O}\left( \frac{1 - \beta_1^n}{n^{5}} \right),
$
which defines a strictly larger region.
\end{itemize}
With the help of Lemma \ref{lemma_concentrate_v} and  the potential function $f(z_k)$, we show that Adam's update direction is close to $\nabla f(x_k)$ when $\beta_2$ is large enough, which can lead to convergence.   We find that it is  rather non-trivial to apply the concentration results in Lemma \ref{lemma_concentrate_v} to the convergence proof due to the existence of momentum, possibly unbounded gradients, and multiple coupling effects. We propose a systematic procedure to overcome them in Lemma \ref{lemma_descent_WR}.

\begin{lemma}
\label{lemma_descent_WR}
\brrevise{Consider the same conditions as Lemma \ref{lemma_concentrate_v} with $\delta = 1 / 4n$.
When $D_1>0$, let $k\in\mathbb{N}$ satisfy
{\small
$k\geq
\lceil\frac{\log(1/4)}{\log\beta_2}\rceil+2n+1$},
{\small
$\delta_{12}(\beta_2)\leq
\frac{1}{4nd\sqrt{5D_1n}}$},
{\small
$\frac{C_7}{\sqrt{k}}\leq
\frac{1}{4nd\sqrt{5D_1n}}$},
and
{\small
$\beta_1^{k-n}\leq\frac{1}{\sqrt{k}}$}.
When $D_1=0<D_0$, let $k$ satisfy the conditions in
\eqref{eq:D1_zero_k_conditions}. When $D_0=D_1=0$, let
$k\in\mathbb{N}$ be arbitrary. Then
{\small \begin{equation}
\label{eq:descent_all_parameter_cases}
\begin{split}
\Ex\langle\nabla f(z_k),z_k-z_{k+1}\rangle
\geq
\eta_k\Bigg\{&
\frac{1}{nd\sqrt{5nd}}\,
\Ex\left[
\mathcal{M}_{D_0,D_1}
\bigl(\|\nabla f(x_k)\|_2\bigr)
\right] -\delta_{D_0,D_1}(\beta_2)\sqrt{D_0}
-\frac{C_{D_0,D_1}}{\sqrt{k}}
\Bigg\}.
\end{split}
\end{equation}}
Here $\delta_{12}(\beta_2)$ and $C_7$ are defined in
\eqref{eq_delta_1_beta2} and \eqref{def_C_7}, respectively;
$\mathcal{M}_{D_0,D_1}$ is defined in
\eqref{eq:stationarity_measure}; and
$\bigl(C_{D_0,D_1},\delta_{D_0,D_1}(\beta_2)\bigr)$ is defined in
\eqref{def:casewise_descent_constants}.}
\end{lemma}
The Proof of Lemma \ref{lemma_descent_WR} is given in Section
\ref{appendix:lemma_descent_WR}. The complete proof of Theorem
\ref{thm_wr} is given in Section \ref{appendix_main_body_wr}.

\section{Proof of Lemma \ref{lemma_concentrate_v}}
\label{section:lemma_concentrate_v}

The proof proceeds in three steps. In Step I, we derive deterministic algebraic bounds relating $\frac{1}{\sqrt{v_{l,k}}}$ to $\frac{1}{\sqrt{\Exk(v_{l,k})}}$ via the ratio $\frac{\max_i (\partial_l f_i(x_k))^2}{\Exk(v_{l,k})}$. In Step II, we establish a lower bound on $\Exk(v_{l,k})$ using the geometric-sum structure of Adam, which yields an upper bound on this ratio. In Step III, we apply Bernstein's inequality to show that the lower bound on the geometric sum holds with high probability, thereby establishing the concentration result.

\paragraph{Step I: Deterministic algebraic bounds relating $1/\sqrt{v_{l,k}}$ to $1/\sqrt{\Exk(v_{l,k})}$.}
We begin by establishing the relationship between $\frac{1}{\sqrt{v_{l,k}}}$ and $\frac{1}{\sqrt{\Exk\left(v_{l,k}\right)}}$. 

For the lower bound, we have
{\small \begin{eqnarray}
    \frac{1}{\sqrt{v_{l,k}}} 
    &=& \frac{1}{\sqrt{\Exk\left(v_{l,k}\right) + \left(v_{l,k} - \Exk\left(v_{l,k}\right)\right)}} \nonumber\\
    &\geq& \frac{1}{\sqrt{\Exk\left(v_{l,k}\right)}} \cdot \frac{1}{\sqrt{1+\frac{\left|v_{l,k}-\Exk\left(v_{l,k}\right)\right|}{\Exk\left(v_{l,k}\right)}}}  \nonumber\\
    &\geq& \frac{1}{\sqrt{\Exk\left(v_{l,k}\right)}}\left(1- \frac{\left|v_{l,k} - \Exk\left(v_{l,k}\right)\right|}{2\Exk\left(v_{l,k}\right)} \right), \label{eq_lower_bound_preliminary}
\end{eqnarray}}
where the last inequality uses $\frac{1}{\sqrt{1+x}} \geq 1-\frac{x}{2}$ for $x\geq 0$.

To proceed, we bound $\left|v_{l,k} - \Exk\left(v_{l,k}\right)\right|$. By the update rule of $v_{l,k}$, we have
{\small \begin{equation}
\label{eq_vlk_deviation}
\left|v_{l,k} - \Exk\left(v_{l,k}\right)\right| = (1-\beta_2)\left|\left(\partial_l f_{\tau_k}(x_k)\right)^2 - \Exk\left[\left(\partial_l f_{\tau_k}(x_k)\right)^2\right]\right| \leq 2 (1-\beta_2) \max_i\left(\partial_l f_i(x_k)\right)^2.
\end{equation}}
Substituting \eqref{eq_vlk_deviation} into \eqref{eq_lower_bound_preliminary}, we obtain
{\small \begin{equation}
\label{lower_bound_of_one_over_sqrt_v}
\frac{1}{\sqrt{v_{l,k}}} \geq \frac{1}{\sqrt{\Exk\left(v_{l,k}\right)}} \left(1- (1-\beta_2)  \frac{\max_i\left(\partial_l f_i(x_k)\right)^2}{\Exk\left(v_{l,k}\right)}  \right).
\end{equation}}

For the upper bound, we proceed similarly:
{\small \begin{eqnarray}
    \frac{1}{\sqrt{v_{l,k}}} &=& \frac{1}{\sqrt{v_{l,k} - \Exk\left(v_{l,k}\right) + \Exk\left(v_{l,k}\right)}} \nonumber \\
     &=& \frac{1}{\sqrt{\Exk\left(v_{l,k}\right)}}\cdot\frac{1}{\sqrt{1+\frac{v_{l,k}-\Exk\left(v_{l,k}\right)}{\Exk\left(v_{l,k}\right)}}}, \label{upper_bound_of_one_over_sqrt_v_deterministic}
\end{eqnarray}}
To obtain a meaningful upper bound from \eqref{upper_bound_of_one_over_sqrt_v_deterministic}, we need to ensure that the denominator remains positive and bounded away from zero. In particular, we want to have
{\small \begin{equation}
\label{condition_for_upper_bound}
    \left|\frac{v_{l,k}-\Exk\left(v_{l,k}\right)}{\Exk\left(v_{l,k}\right)}\right| < 1.
\end{equation}}
By \eqref{eq_vlk_deviation}, condition \eqref{condition_for_upper_bound} holds if the following holds
{\small \begin{equation}
\label{condition_for_upper_bound_explicit}
   2 (1-\beta_2) \frac{\max_i \left(\partial_l f_i(x_k)\right)^2 }{\Exk\left(v_{l,k}\right)} < 1.
\end{equation}}
Assuming \eqref{condition_for_upper_bound_explicit} holds, we have $1+\frac{v_{l,k}-\Exk\left(v_{l,k}\right)}{\Exk\left(v_{l,k}\right)} \geq 1 - \left|\frac{v_{l,k}-\Exk\left(v_{l,k}\right)}{\Exk\left(v_{l,k}\right)}\right| > 0$. This allows us to bound the reciprocal square root:
{\small \begin{eqnarray}
    \frac{1}{\sqrt{1+\frac{v_{l,k}-\Exk\left(v_{l,k}\right)}{\Exk\left(v_{l,k}\right)}}} 
    &\leq& \frac{1}{\sqrt{1 - \left|\frac{v_{l,k}-\Exk\left(v_{l,k}\right)}{\Exk\left(v_{l,k}\right)}\right|}} \nonumber \\
    &\overset{\eqref{eq_vlk_deviation}}{\leq}& \frac{1}{\sqrt{ 1-2(1-\beta_2)\frac{\max_i \left(\partial_l f_i(x_k)\right)^2 }{\Exk\left(v_{l,k}\right)}}}. \label{upper_bound_of_one_over_sqrt_v_intermediate}
\end{eqnarray}}
Combining \eqref{upper_bound_of_one_over_sqrt_v_deterministic} and \eqref{upper_bound_of_one_over_sqrt_v_intermediate}, we obtain the following (assuming \eqref{condition_for_upper_bound_explicit} holds):
{\small \begin{equation}
\label{upper_bound_of_one_over_sqrt_v}
    \frac{1}{\sqrt{v_{l,k}}} \leq \frac{1}{\sqrt{\Exk\left(v_{l,k}\right)}} \cdot\frac{1}{\sqrt{ 1-2(1-\beta_2)\frac{\max_i \left(\partial_l f_i(x_k)\right)^2 }{\Exk\left(v_{l,k}\right)}}}.
\end{equation}}

The bounds \eqref{lower_bound_of_one_over_sqrt_v} and \eqref{upper_bound_of_one_over_sqrt_v} are useful only if we can bound the ratio $\frac{\max_i \left(\partial_l f_i(x_k)\right)^2 }{\Exk\left(v_{l,k}\right)}$ and verify condition \eqref{condition_for_upper_bound_explicit}. A trivial bound follows immediately:
{\small \begin{equation}
\label{eq:trivial}
    \frac{\max_i \left(\partial_l f_i(x_k)\right)^2 }{\Exk\left(v_{l,k}\right)} \leq \frac{n  \Exk\left[\left(\partial_l f_{\tau_k}(x_k)\right)^2\right]  } {\Exk\left(v_{l,k}\right)} \leq \frac{n  \Exk\left[\left( \partial_l f_{\tau_k}(x_k)\right)^2\right]  } { (1-\beta_2)  \Exk\left[\left(\partial_l f_{\tau_k}(x_k)\right)^2\right]}  = \frac{n}{1-\beta_2},
\end{equation}}
which diverges as $\beta_2 \to 1^-$. Substituting this into \eqref{condition_for_upper_bound_explicit} would require $(1-\beta_2) \cdot \frac{2n}{1-\beta_2} = 2n < 1$, which fails for any $n \geq 1$. Thus, the trivial bound is insufficient.

Our goal is to derive a refined bound such that 
{\small \begin{equation}
\label{goal_bound}
    (1-\beta_2)\frac{\max_i \left(\partial_l f_i(x_k)\right)^2 }{\Exk\left(v_{l,k}\right)} \rightarrow 0 \quad \text{as } \beta_2 \to 1^-,
\end{equation}}
which ensures that condition \eqref{condition_for_upper_bound_explicit} is satisfied for $\beta_2$ sufficiently close to 1, and that both \eqref{lower_bound_of_one_over_sqrt_v} and \eqref{upper_bound_of_one_over_sqrt_v} remain tight. Establishing such a bound requires substantial effort, which we undertake in the following Steps II and III.

\paragraph{Step II: Lower bounding $\Exk(v_{l,k})$ via the geometric-sum structure.}
Define  $\alpha := \arg \max_i |\partial_l f_i(x_k)|$. Expanding $\Exk\left(v_{l,k}\right)$, we have
{\small \begin{eqnarray}
    \Exk\left(v_{l,k}\right) &=& (1-\beta_2) \left( \Exk\left[\left(\partial_l f_{\tau_k} (x_k)\right)^2\right]  + \sum_{j = 1}^{k-1} \beta_2^{j} \left(\partial_l f_{\tau_{k-j}} (x_{k-j})\right)^2   \right) \nonumber\\
    &\geq& (1-\beta_2) \sum_{j=1}^{k-1} \beta_2^{j} \left(\partial_l f_{\tau_{k-j}} (x_{k-j})\right)^2 \nonumber\\ 
    &\geq&  (1-\beta_2) \sum_{j=1}^{k-1} \beta_2^{j} \left(\partial_l f_{\alpha} (x_{k-j})\right)^2 \mathbb{I}\left(\tau_{k-j} = \alpha \right). \label{eq_Exk_vlk_expansion}
\end{eqnarray}}

\brrevise{
When $k\geq\lceil\frac{\log(n\delta)}{\log\beta_2}\rceil+n+1$ and $\max_i|\partial_l f_i(x_k)|\geq R_k := \frac{16\sqrt{2}\Delta_1}{\sqrt{k}} \left(
\left\lceil\frac{\log(n\delta)}{\log\beta_2}\right\rceil+n \right)$, for every $1\leq j\leq\lceil\frac{\log(n\delta)}{\log\beta_2}\rceil+n$, we have
{\small $$ |\partial_l f_\alpha(x_k)| = \max_i|\partial_l f_i(x_k)|
\geq R_k 
\geq \frac{8\sqrt{2}j\Delta_1}{\sqrt{k}}. $$}
Therefore, the condition of Lemma \ref{lemma_k-j_k} is satisfied.
Applying Lemma \ref{lemma_k-j_k}, we obtain}

{\small \begin{eqnarray}
\label{v_k_fix_x}
    \Exk\left(v_{l,k}\right)
    &\geq&   
    \frac{\left(\partial_l f_{\alpha} (x_k)\right)^2}{ 2}(1-\beta_2)  \sum_{j=1}^{\lceil\frac{\log (n\delta)}{\log \beta_2}\rceil+n} \beta_2^{j}  \mathbb{I}\left(\tau_{k-j} = \alpha \right)\nonumber \\
    &\geq& \frac{\left(\partial_l f_{\alpha} (x_k)\right)^2}{ 2}(1-\beta_2)\beta_2^n  \sum_{t=1}^{\lceil\frac{\log (n\delta)}{\log \beta_2}\rceil} \beta_2^{t}  \mathbb{I}\left(\tau_{k-n-t} = \alpha \right).  \nonumber
\end{eqnarray}}

For $k^* < k-n$, define the auxiliary quantity
{\small \begin{equation}
\label{eq_def_Y_k}
    Y_{k^*,k}(i) := (1-\beta_2) \sum_{j=1}^{k^*} \beta_2^{j} \mathbb{I}\left(\tau_{k-n-j} = i\right), \quad \forall\, i \in \{0,1,\ldots,n-1\},
\end{equation}}
which represents a weighted count of how many times index $i$ is sampled within the $k^*$-step historical window preceding step $k-n$. Since $\tau_{k-n-j}$ is sampled uniformly from $\{0, 1, \ldots, n-1\}$, the sequence $\{\mathbb{I}(\tau_{k-n-j} = i)\}_{j=1}^{k^*}$ forms a sequence of i.i.d.\ Bernoulli random variables with success probability $\frac{1}{n}$. 

\paragraph{Step III: Concentration of $Y_{k^*,k}(i)$ via Bernstein's inequality and the choice of $k^*$.}

Our objective is to establish that for any $0 < \delta \leq \frac{1}{4n}$, the random variable $Y_{k^*,k}(\alpha)$ satisfies $Y_{k^*,k}(\alpha) \geq \frac{1}{2n}-\delta$ with probability approaching 1 as $\beta_2 \to 1^-$, where $\alpha := \arg \max_i |\partial_l f_i(x_k)|$. The concentration analysis is enabled by the geometric weighting structure: each term in $Y_{k^*,k}(\alpha)$ is scaled by the coefficient $(1-\beta_2)\beta_2^j$, whose magnitude is controlled by $(1-\beta_2)$. As $\beta_2 \to 1^-$, this coefficient vanishes, allowing us to apply Bernstein's inequality to get a sharp probability bound. However, direct analysis is complicated because the index $\alpha = \arg\max_i |\partial_l f_i(x_k)|$ depends on the random iterate $x_k$. Consequently, $\alpha$ is a random variable coupled with the optimization trajectory, making the distribution of $\mathbb{I}\left(\tau_{k-n-j} = \alpha \right)$ difficult to characterize. To circumvent this, we first establish concentration for an arbitrary fixed index $i \in \{0,1,\ldots,n-1\}$, then apply a union bound.

Define the events
{\small \begin{equation}
\label{eq_def_events}
    \mathcal{E}_k(i) := \left\{Y_{k^*,k}(i) \geq \frac{1}{2n} - \delta\right\}, \quad \forall\, i \in \{0,1,\ldots,n-1\},
\end{equation}}
and the global failure event
{\small \begin{equation}
\label{eq_def_Uk}
    \mathcal{U}_k := \bigcup_{i=0}^{n-1} \left\{Y_{k^*,k}(i) < \frac{1}{2n} - \delta\right\} = \bigcup_{i=0}^{n-1} \mathcal{E}_k(i)^c.
\end{equation}}
Our goal is to bound the probability $\mathbb{P}(\mathcal{E}_{k}(\alpha)^c)$. Since $\alpha$ takes some value in $[n]$, we observe that the failure event for $\alpha$ is contained within the global failure event:
{\small \begin{equation}
    \mathcal{E}_{k}(\alpha)^c= \left\{Y_{k^*,k}(\alpha) < \frac{1}{2n} - \delta\right\} \subseteq \mathcal{U}_k.  \nonumber
\end{equation}}
By the union bound,
{\small \begin{equation}
    \label{eq:union_bound_prelim}
    \mathbb{P}\left(\mathcal{E}_{k}(\alpha)^c\right) \leq \mathbb{P}\left(\mathcal{U}_k\right) \leq \sum_{i=0}^{n-1} \mathbb{P}\left(Y_{k^*,k}(i) < \frac{1}{2n} - \delta\right) \leq \sum_{i=0}^{n-1} \mathbb{P}\left(Y_{k^*,k}(i) \leq \frac{1}{2n} - \delta\right).
\end{equation}}
It thus suffices to bound $\mathbb{P}\left(Y_{k^*,k}(i) \leq \frac{1}{2n} - \delta\right)$ for an arbitrary fixed $i$.

For ease of presentation, we simplify notation by fixing an arbitrary index $i \in [n]$, drop the variable $k$ and replace the subscript $\tau_{k-n-j}$ with $\tau_{j}$ and redefine 
{\small \begin{equation*}
Y_{k^*} := (1-\beta_2) \sum_{j=1}^{k^*} \beta_2^{j} \mathbb{I}(\tau_{j} = i).
\end{equation*}}

To apply Bernstein's inequality, we introduce centered random variables. Define 
{\small $$
X_j := (1-\beta_2)\beta_2^j \left[\mathbb{I}(\tau_j = i) - \frac{1}{n}\right]
$$}
for $j = 1, \ldots, k^*$. Then $\mathbb{E}[X_j] = 0$ and $|X_j| \leq (1-\beta_2)\beta_2^j$ for all $j$. We can express
{\small \begin{equation}
    Y_{k^*} -\mathbb{E}\left[Y_{k^*}\right] = \sum_{j=1}^{k^*} X_j.  \nonumber
\end{equation}}

Before applying Bernstein's inequality, we first lower bound the expectation of $Y_{k^*}$ and upper bound its variance. When $\beta_2 \geq \frac{1}{2}$ and $k^*$ is sufficiently large such that $\beta_2^{k^*} \leq n\delta$, we have
{\small \begin{equation}
\label{eq_e_yk}
   \mathbb{E}\left[Y_{k^*}\right] = (1-\beta_2) \sum_{j=1}^{k^*} \beta_2^{j} \mathbb{E}\bigl[\mathbb{I}(\tau_{j} = i)\bigr] = \frac{\beta_2(1-\beta_2^{k^*})}{n} \geq \frac{1}{2n} - \frac{\delta}{2}.
\end{equation}}

For the variance, since the $X_j$ are independent, we have
{\small \begin{eqnarray}
    \operatorname{Var}(Y_{k^*}) &=& \sum_{j=1}^{k^*} \operatorname{Var}(X_j) = (1-\beta_2)^2 \sum_{j=1}^{k^*} \beta_2^{2j} \operatorname{Var}\bigl(\mathbb{I}(\tau_{j} = i)\bigr) \nonumber \\
    &\leq& \frac{(1-\beta_2)^2\beta_2^2}{1-\beta_2^2} \frac{1}{n}\left(1-\frac{1}{n}\right)
    \leq  \frac{(1-\beta_2)^2}{n(1-\beta_2^2)}
    \leq  \frac{1-\beta_2}{n}. \label{eq_var_yk}
\end{eqnarray}}

We now apply Bernstein's inequality. Note that for all $j \in \{1, \dots, k^*\}$, the random variables $X_j$ satisfy the uniform bound $|X_j| \leq (1-\beta_2)\beta_2^j \leq 1-\beta_2$. Applying Bernstein's inequality and using the variance bound \eqref{eq_var_yk}, we obtain for any $a > 0$:
{\small \begin{eqnarray*}
\label{bernstein_bound}
    \mathbb{P}\left(Y_{k^*} -\mathbb{E}\left[Y_{k^*}\right] \leq -a\right) &\leq& \exp\left(-\frac{a^2/2}{\operatorname{Var}(Y_{k^*}) + \frac{1}{3}(1-\beta_2) a}\right) \\
    &\leq& \exp\left(-\frac{a^2/2}{\frac{1-\beta_2}{n} + \frac{1}{3}(1-\beta_2)a}\right).
\end{eqnarray*}}

Let $a =\mathbb{E}\left[Y_{k^*}\right] - (\frac{1}{2n}-\delta)$ and notice that $a \leq \frac{1}{2n}+\delta$. When $a \geq \frac{\delta}{2}$, which holds for $\beta_2 \geq \frac{1}{2}$ and $k^* \geq \lceil\frac{\log(n\delta)}{\log \beta_2}\rceil$, we have
{\small \begin{eqnarray}
\label{eq:failure_1}
    \mathbb{P}\left(Y_{k^*} \leq \frac{1}{2n}-\delta\right) &\leq& \exp\left(-\frac{a^2/2}{\frac{1-\beta_2}{n} + \frac{1}{3}(1-\beta_2)a}\right) \nonumber \\
    &\leq& \exp\left(-\frac{(\delta/2)^2/2}{\frac{1-\beta_2}{n} + \frac{1}{3}(1-\beta_2)\left(\frac{1}{2n}+\delta\right)}\right) \nonumber \\
    &=& \exp\left(-\frac{\delta^2}{(1-\beta_2)\left(\frac{28}{3n}+\frac{8}{3}\delta\right)}\right).
\end{eqnarray}}
This concentration bound holds for any fixed index $i \in [n]$. Substituting \eqref{eq:failure_1} into the union bound \eqref{eq:union_bound_prelim}, we obtain
{\small \begin{equation}
    \mathbb{P}(\mathcal{E}_{k}(\alpha)^c) \leq n\exp\left(-\frac{\delta^2}{(1-\beta_2)\left(\frac{28}{3n}+\frac{8}{3}\delta\right)}\right).  \nonumber
\end{equation}}
We obtain that with probability at least $1-n\exp\left(-\frac{\delta^2}{(1-\beta_2)\left(\frac{28}{3n}+\frac{8}{3}\delta\right)}\right)$, the inequality $Y_{k^*,k}(\alpha) \geq \frac{1}{2n} - \delta$ holds. 
 
 Setting $k^* = \lceil\frac{\log(n\delta)}{\log \beta_2}\rceil$ and substituting this result into the lower bound for $\Exk\left(v_{l,k}\right)$ \eqref{v_k_fix_x}, we conclude that with probability at least $1-n\exp\left(-\frac{\delta^2}{(1-\beta_2)\left(\frac{28}{3n}+\frac{8}{3}\delta\right)}\right)$,
{\small \begin{equation}\label{eq_lemma_concentrate_v_lower_bd}
    \Exk\left(v_{l,k}\right) \geq \frac{(\partial_l f_{\alpha}(x_k))^2}{2} \cdot \beta_2^n \cdot \left(\frac{1}{2n} - \delta\right).
\end{equation}}
Recalling that $0 < \delta \leq \frac{1}{4n}$, we equivalently obtain 
{\small \begin{equation}
\label{upperbound_of_f_over_v}
    \frac{\max_i\bigl(\partial_l f_i(x_k)\bigr)^2}{\Exk\left(v_{l,k}\right)} \equiv \frac{(\partial_l f_\alpha(x_k))^2}{\Exk\left(v_{l,k}\right)} \leq \frac{4n}{(1-2n\delta)\beta_2^n}.
\end{equation}}
Finally, we verify that condition \eqref{condition_for_upper_bound_explicit} is satisfied for $\beta_2$ sufficiently close to 1. Substituting the bound \eqref{upperbound_of_f_over_v} into \eqref{condition_for_upper_bound_explicit}, we require
{\small \begin{equation}
\label{verify_condition}
    2(1-\beta_2) \frac{\max_i \left(\partial_l f_i(x_k)\right)^2 }{\Exk\left(v_{l,k}\right)} \leq (1-\beta_2) \frac{8n}{(1-2n\delta)\beta_2^n} < 1.
\end{equation}}
Rearranging \eqref{verify_condition}, we obtain the sufficient condition
{\small \begin{equation}
\label{beta2_range}
    \frac{1-\beta_2}{\beta_2^n} < \frac{1-2n\delta}{8n} = \frac{1}{8n} - \frac{\delta}{4}.
\end{equation}}
Since $0 < \delta \leq \frac{1}{4n}$, the right-hand side of \eqref{beta2_range} is at least $\frac{1}{16n} > 0$. The left-hand side $\frac{1-\beta_2}{\beta_2^n}$ is strictly decreasing in $\beta_2$ and vanishes as $\beta_2 \to 1^-$. Therefore, \eqref{beta2_range} holds when $\beta_2$ sufficiently close to 1 (e.g., when we choose $1-\beta_2=\mathcal{O}(1/n^{1.5})$), validating condition \eqref{condition_for_upper_bound_explicit} and the upper bound \eqref{upper_bound_of_one_over_sqrt_v}.

Substituting \eqref{upperbound_of_f_over_v} into the concentration bounds, we have with probability at least $1-n \exp\left(-\frac{\delta^2}{(1-\beta_2)\left(\frac{28}{3n}+\frac{8}{3}\delta\right)}\right)$, 
{\small \begin{equation*}
    \frac{1}{\sqrt{\Exk\left(v_{l,k}\right)}} \left(1 - (1-\beta_2)\frac{4n}{(1-2n\delta)\beta_2^n}\right) \leq \frac{1}{\sqrt{v_{l,k}}} \leq \frac{1}{\sqrt{\Exk\left(v_{l,k}\right)}} \left(\frac{1}{\sqrt{1-(1-\beta_2)\frac{8n}{(1-2n\delta)\beta_2^n}}}\right).
\end{equation*}}
Since the choice of $l$ was arbitrary, this bound holds for all $l \in [d]$.
This completes the proof of Lemma \ref{lemma_concentrate_v}. \qed

\section{Proof of Lemma \ref{lemma_descent_WR}}
\label{appendix:lemma_descent_WR}

\brrevise{We first prove Lemma \ref{lemma_descent_WR} in the interior case $D_0>0$ and $D_1>0$. The boundary cases are proved separately in Appendix \ref{appendix:D0_D1_boundary}.}

To prove Lemma \ref{lemma_descent_WR}, we need to prove that the update direction of Adam lies in the dual cone of the negative gradient direction, i.e., it is a descent direction. Recall that
$
z_k := \frac{x_k - \beta_1^n x_{k-n}}{1-\beta_1^n}.
$
We aim to establish a lower bound on $\Ex\left\langle \nabla f(z_k),z_k-z_{k+1}\right\rangle$, which we decompose as follows:
{\small \begin{eqnarray*}
\Ex\underbrace{\left\langle \nabla f(z_k)-\nabla f(x_k),z_k-z_{k+1}\right\rangle}_{:=(a)}
+\Ex\underbrace{\left\langle \nabla f(x_k),z_k-z_{k+1}\right\rangle}_{:=(b)}.
\end{eqnarray*}}
Now we bound both terms. $\Ex(a)$ can be simply bounded by using Lemma \ref{lemma_delta}.
{\scriptsize\setlength{\arraycolsep}{3pt}
\begin{eqnarray}
    \Ex(a)&\overset{\text{Cauchy-Schwarz}}{\geq}& - \frac{1}{1-\beta_1^n}\|\nabla f(z_k)-\nabla f(x_k)\|_2\|x_k-x_{k+1}-\beta_1^n (x_{k-n}-x_{k-n+1})\|_2 \nonumber\\
    &\overset{\text{Assumption }\ref{assum1}}{\geq}& -\frac{L}{1-\beta_1^n}\left\| \frac{x_k-\beta_1^n x_{k-n}}{1-\beta_1^n}-x_k\right\|_2 \left(\|x_k-x_{k+1}\|_2+\beta_1^n \|x_{k-n}-x_{k-n+1}\|_2\right)\nonumber \\
    &\geq &-\frac{L\beta_1^n}{1-\beta_1^n}\left(\left\|  x_{k}-x_{k-1}\right\|_2+\cdots+\left\|  x_{k-n+1}-x_{k-n}\right\|_2 \right)\left(\|x_k-x_{k+1}\|_2+\beta_1^n \|x_{k-n}-x_{k-n+1}\|_2\right)\nonumber\\
    &\overset{\text{\eqref{eq:m_over_v}}}{\geq}& -\frac{L\beta_1^n}{(1-\beta_1^n)^2}\left( \frac{d(1-\beta_1)}{\sqrt{1-\beta_2}} \frac{1}{1-\frac{\beta_1}{\sqrt{\beta_2}}}\right)^2 \left(\eta_{k-1}+\eta_{k-2}+\cdots+\eta_{k-n}\right)\left( \eta_k+\beta_1^n\eta_{k-n} \right)\nonumber \\
    &\overset{k\geq n+1}{\geq}&-\frac{L\beta_1^n}{(1-\beta_1^n)^2}\left( \frac{d(1-\beta_1)}{\sqrt{1-\beta_2}} \frac{1}{1-\frac{\beta_1}{\sqrt{\beta_2}}}\right)^2 n\sqrt{n+1}\eta_{k} \left( \eta_k+\beta_1^n\sqrt{n+1}\eta_{k}\right)\nonumber\\
    &\overset{}{\geq}& - \eta_k\frac{1}{\sqrt{k}}
    \frac{L\eta_0 \beta_1^n d^2 (1-\beta_1)^2 (1 + \beta_1^n\sqrt{n+1}) n \sqrt{n+1}}{(1- \beta_1^n)^2(1-\beta_2)(1-\frac{\beta_1}{\sqrt{\beta_2}})^2} :=- \eta_k\frac{C_1}{\sqrt{k}}.\label{eq:a}
\end{eqnarray}
}

To handle term $(b)$, we first divide the analysis into two cases for each dimension $l$: 
\begin{itemize}
    \item {\bf Case 1:} the \textit{bounded} partial gradient event  $B_{l,k}$, defined as $\max_i |\partial_l f_i (x_k)| \leq Q_k$; The definition of constant $Q_k$ can be seen later in \eqref{eq_constants_wr}. 
    \item {\bf Case 2:} the complement of Case 1, the \textit{unbounded}  partial gradient event $B_{l,k}^c$. 
\end{itemize}
Note that the former case is easier due to the bounded gradient condition. The latter case is more involved, and we will handle it using Lemma \ref{lemma_concentrate_v}. We first state two observations, which would help the subsequent analysis.

\begin{itemize}
    \item \textbf{Observation 1:}  When applied with $\delta = \frac{1}{4n}$, Lemma \ref{lemma_concentrate_v} guarantees that the following two key inequalities (\eqref{eq_concentrate_v_epsilon_given} and \eqref{upperbound_of_f_over_v_epsilon}) hold with high probability for integer $\bar{k} \in [k-n, k]$, provided that the condition $\max_i |\partial_l f_i (x_{\bar{k}})| \geq R_{\bar{k}}$ is met:
{\small
\begin{eqnarray}
\left(1- (1-\beta_2) \frac{8n}{\beta_2^n} \right) \frac{1}{\sqrt{\mathbb{E}_{\bar{k}}(v_{l,\bar{k}})}} &\leq & \frac{1}{\sqrt{v_{l,\bar{k}}}} \leq\left(\frac{1}{\sqrt{1-(1-\beta_2)\frac{16n}{\beta_2^n}}}\right) \frac{1}{\sqrt{\mathbb{E}_{\bar{k}}(v_{l,\bar{k}})}}.\label{eq_concentrate_v_epsilon_given} \\
\frac{\max_i(\partial_l f_i(x_{\bar{k}}))^2}{\mathbb{E}_{\bar{k}}(v_{l,\bar{k}})} &\leq& \frac{4n}{(1-2n\delta)\beta_2^n} = \frac{8n}{\beta_2^n},\label{upperbound_of_f_over_v_epsilon}
\end{eqnarray}
}
\item \textbf{Observation 2:}  When $B_{l,k}^c$ occurs,  \eqref{eq_concentrate_v_epsilon_given} and \eqref{upperbound_of_f_over_v_epsilon} hold for all $\bar{k} \in [k-n, k]$ with a probability of at least $1- n(n+1) \exp\left(-\frac{1}{160(1-\beta_2)n}\right)$. 
\end{itemize}

Observation 1 can be seen directly from the statement of Lemma \ref{lemma_concentrate_v} and its proof \eqref{upperbound_of_f_over_v}. Observation 2 can be shown as follows:

First, by Lemma \ref{lemma_delta} and the definition of $Q_k$, one can easily see that: under event $B_{l,k}^c$, the conditions of Lemma \ref{lemma_concentrate_v} are automatically satisfied {\it for all} integers $\bar{k} \in [k-n, k]$.

Second,  under the conditions of  Lemma \ref{lemma_concentrate_v}, the event of  ``inequalities \eqref{eq_concentrate_v_epsilon_given} and \eqref{upperbound_of_f_over_v_epsilon} are satisfied  for a given $\bar{k}$'' is equivalent to the 
event $\mathcal{E}_{\bar{k}}(\alpha(\bar{k})):=\{Y_{k^*,\bar{k}}(\alpha(\bar{k})) \geq \frac{1}{4n}\}$.   This can be easily seen from the proof of  Lemma \ref{lemma_concentrate_v}.  As such, to prove observation 2, one need to calculate the probability of event $\bigcap_{\bar{k}=k-n}^k \mathcal{E}_{\bar{k}}(\alpha(\bar{k}))$.

Third, instead of  directly working on $\bigcap_{\bar{k}=k-n}^k \mathcal{E}_{\bar{k}}(\alpha(\bar{k}))$, we consider a slightly more restricted event:
{\small \begin{equation}
\label{eq_A_k}
    A_k := \bigcap_{i=0}^{n-1} \bigcap_{\bar{k}=k-n}^k \left\{Y_{k^*,\bar{k}}(i) \geq \tfrac{1}{4n}\right\}.
\end{equation}}
The definition of $A_k$ avoids events' dependency on $\alpha(\bar{k})$,  which will benefit the subsequent analysis. The probability of the tail event $A_k^c$ can be easily calculated by applying the union bound and \eqref{eq:failure_1}, and one can get $\mathbb{P}(A_k^c)\leq n(n+1)\exp\left(-\frac{1}{160(1-\beta_2)n}\right)$.  

Finally, event $A_k$  implies event $\bigcap_{\bar{k}=k-n}^k \mathcal{E}_{\bar{k}}(\alpha(\bar{k}))$, which is further equivalent to \eqref{eq_concentrate_v_epsilon_given} and \eqref{upperbound_of_f_over_v_epsilon} under $B_{l,k}^c$.  Therefore, under $B_{l,k}^c$, \eqref{eq_concentrate_v_epsilon_given} and \eqref{upperbound_of_f_over_v_epsilon} hold for all $\bar{k} \in [k-n, k]$ with a probability of at least $1- n(n+1) \exp\left(-\frac{1}{160(1-\beta_2)n}\right)$. This concludes the proof for observation 2.

Our subsequent analysis is centered on Observations 1 and 2. We first decompose the summation in term $(b)$ into three distinct parts to facilitate the analysis of its expectation:
{\small
\begin{eqnarray}
    (b)
    &=&\frac{1}{1-\beta_1^n}\sum_{l=1}^{d}\left[
    \eta_k \partial_lf(x_k)\frac{m_{l,k}}{\sqrt{v_{l,k}}}-\beta_1^n\eta_{k-n} \partial _lf(x_k)\frac{m_{l,k-n}}{\sqrt{v_{l,k-n}}}\right] \nonumber\\
    &=&\underbrace{ \frac{1}{1-\beta_1^n} \sum_{l=1}^{d}\left[\left(
    \eta_k \partial_lf(x_k)\frac{m_{l,k}}{\sqrt{v_{l,k}}}-\beta_1^n\eta_{k-n} \partial _lf(x_k)\frac{m_{l,k-n}}{\sqrt{v_{l,k-n}}}\right)
    \mathbb{I}(B_{l,k}) \right]}_{(c)} \nonumber\\
    &&+\underbrace{\frac{1}{1-\beta_1^n}\sum_{l=1}^{d}\left[\left(
    \eta_k \partial_lf(x_k)\frac{m_{l,k}}{\sqrt{v_{l,k}}}-\beta_1^n\eta_{k-n} \partial _lf(x_k)\frac{m_{l,k-n}}{\sqrt{v_{l,k-n}}}\right)
   \mathbb{I}(A_k \cap B_{l,k}^c)\right]}_{(d)} \nonumber\\
    &&+\underbrace{\frac{1}{1-\beta_1^n}\sum_{l=1}^{d}\left[\left(
    \eta_k \partial_lf(x_k)\frac{m_{l,k}}{\sqrt{v_{l,k}}}-\beta_1^n\eta_{k-n} \partial _lf(x_k)\frac{m_{l,k-n}}{\sqrt{v_{l,k-n}}}\right)
    \mathbb{I}(A_k^c \cap B_{l,k}^c)\right]}_{(e)}, \label{eq:b_unsimplified}
\end{eqnarray}
}
where $B_{l,k}, A_k$ and their complements  $B_{l,k}^c, A_k^c$ are defined in the previous paragraphs and \eqref{eq_A_k}.

\paragraph{Road map for the rest of the proof.} We will provide a lower bound for each term. We will show that Term $(d)$ constitutes the descent direction, and $(c)$ \& $(e)$ are either smaller than $(d)$ or vanishing with iteration $k$.
Among the three terms, term $(c)$ and $(e)$ are easier to handle: term $(c)$ addresses the bounded gradient scenario, and term $(e)$ accounts for the tail event $A_k^c$. The term $(d)$ will be more involved, and it will be further divided into more sub-terms and we will resort to Lemma \ref{lemma_concentrate_v}.

\paragraph{Lower bound of $\Ex(c)$.} We start with $\Ex(c)$. Based on the proof of Lemma \ref{lemma_delta}, we have
{\small \begin{eqnarray}
     \Ex(c) &\geq& - \frac{1}{1-\beta_1^n}\Ex\left[\sum_{l=1}^{d} \left(
     \eta_k \vert \partial_lf(x_k)\vert \frac{\vert m_{l,k}\vert}{\sqrt{v_{l,k}}} + \beta_1^n\eta_{k-n}  \vert \partial_l f(x_k)\vert \frac{\vert m_{l,k-n}\vert}{\sqrt{v_{l,k-n}}}\right) \mathbb{I}(\text{$B_{l,k}$})\right]\nonumber\\
     &\overset{\eqref{eq:m_over_v}}{\geq}& -\frac{1}{1-\beta_1^n}\Ex\left[\sum_{l=1}^{d} \left(
     \eta_k Q_k \frac{(1-\beta_1)}{\sqrt{1-\beta_2}} \frac{1}{1-\frac{\beta_1}{\sqrt{\beta_2}}}  +  \beta_1^n\eta_{k-n}  Q_k \frac{(1-\beta_1)}{\sqrt{1-\beta_2}} \frac{1}{1-\frac{\beta_1}{\sqrt{\beta_2}}}\right) \mathbb{I}(\text{$B_{l,k}$})\right]\nonumber\\
     & \geq&  - \frac{1}{1-\beta_1^n} d\eta_k  Q_k \frac{(1-\beta_1)}{\sqrt{1-\beta_2}} \frac{1}{1-\frac{\beta_1}{\sqrt{\beta_2}}} 
     \left(1+\beta_1^n \frac{\sqrt{k}}{\sqrt{k-n}}\right)\nonumber \\
     & \overset{k\geq n+1}{\geq}&  - \frac{\eta_k}{1-\beta_1^n} d Q_k \frac{(1-\beta_1)}{\sqrt{1-\beta_2}} \frac{1}{1-\frac{\beta_1}{\sqrt{\beta_2}}} 
     \left(1+\sqrt{n+1}\beta_1^n\right):=-\eta_k\frac{C_2}{\sqrt{k}}.\label{eq:c}
 \end{eqnarray}}

\paragraph{Lower bound of $\Ex(e)$.}  Now we derive a lower bound for the term $\Ex(e)$. By applying \eqref{eq:m_over_v} and using $\mathbb{I}(A_k^c\cap B_{l,k}^c)\le \mathbb{I}(A_k^c)$, we obtain
{\small \begin{eqnarray}
    \Ex(e)
    &=&\frac{1}{1-\beta_1^n}\Ex\Bigl[\sum_{l=1}^{d}\Bigl(
    \eta_k\partial_l f(x_k)\frac{m_{l,k}}{\sqrt{v_{l,k}}} \nonumber\\
    &&\quad-\beta_1^n\eta_{k-n}\partial_l f(x_k)
    \frac{m_{l,k-n}}{\sqrt{v_{l,k-n}}}\Bigr)
    \mathbb{I}(A_k^c\cap B_{l,k}^c)\Bigr]\nonumber\\
    &\overset{\eqref{eq:m_over_v}}{\geq}&
    -\frac{1}{1-\beta_1^n}\Ex\Bigl[\sum_{l=1}^{d}\Bigl(
    \eta_k\vert\partial_l f(x_k)\vert\frac{1-\beta_1}{\sqrt{1-\beta_2}}
    \frac{1}{1-\frac{\beta_1}{\sqrt{\beta_2}}} \nonumber\\
    &&\quad+\beta_1^n\eta_{k-n}\vert\partial_l f(x_k)\vert
    \frac{1-\beta_1}{\sqrt{1-\beta_2}}
    \frac{1}{1-\frac{\beta_1}{\sqrt{\beta_2}}}\Bigr)
    \mathbb{I}(A_k^c\cap B_{l,k}^c)\Bigr]\nonumber\\
    &\overset{k\geq n+1}{\geq}&
    -\eta_k\frac{1}{1-\beta_1^n}\left(1+\beta_1^n\sqrt{n+1}\right)\frac{1}{\sqrt{1-\beta_2}}\frac{1}{1-\frac{\beta_1}{\sqrt{\beta_2}}}
    \Ex \left[\sum_{l=1}^{d}\vert\partial_l f(x_k)\vert\mathbb{I}(A_k^c\cap B_{l,k}^c)\right]\nonumber\\
    &\geq &
    -\eta_k\frac{1}{1-\beta_1^n}\left(1+\beta_1^n\sqrt{n+1}\right)\frac{1}{\sqrt{1-\beta_2}}\frac{1}{1-\frac{\beta_1}{\sqrt{\beta_2}}}
    \Ex\left[\sum_{l=1}^{d}\vert\partial_l f(x_k)\vert\mathbb{I}(A_k^c) \right].
    \label{eq_e_without_blk}
\end{eqnarray}}
Here, the penultimate inequality holds because $k\geq n+1$, which ensures $\frac{\eta_{k-n}}{\eta_k}=\sqrt{\frac{k}{k-n}}\leq\sqrt{n+1}$.

To upper bound $\Ex\sum_{l=1}^d \vert\partial_l f(x_k)\vert\mathbb{I}(A_k^c)$, we observe that the event $A_k^c$ depends on the sampled indices within a historical window $\{\tau_j\}_{j=k-2n-k^*}^{k-n-1}$. These indices are correlated with $x_k$ through the algorithm's update recursion, thereby inducing a dependence between $\vert\partial_l f(x_k)\vert$ and $\mathbb{I}(A_k^c)$. To get a desired upper bound, we perform the following decoupling  procedures. First, we introduce a lag parameter

{\small \begin{equation}\label{eq:def_h_e}
h := 2n+k^*,
\qquad\text{where}\quad
k^*=\left\lceil\frac{\log(1/4)}{\log\beta_2}\right\rceil.
\end{equation}}

We now shift from $x_k$ to the earlier iterate $x_{k-h}$. By the condition of Lemma \ref{lemma_descent_WR}, $k-h$ is greater than 0, and thus it is a valid index.  Further, the event $A_k^c$ depends on the indices
{\small $$
\{\tau_j\}_{j=k-2n-k^*}^{k-n-1}=\{\tau_j\}_{j=k-h}^{k-n-1},
$$}
which are sampled independently after $x_{k-h}$ and thus are independent of $\mathcal{F}_{k-h}$. Consequently, $x_{k-h}$ and $\mathbb{I}(A_k^c)$ are independent random variables, yielding
{\small \begin{equation}\label{eq:independence_e}
\Ex\left[\vert\partial_l f(x_{k-h})\vert\mathbb{I}(A_k^c)\right]
= \Ex\left[\vert\partial_l f(x_{k-h})\vert\right]\mathbb{P}(A_k^c).
\end{equation}}

Using Lemma~\ref{lemma_delta} to relate gradients at $x_k$ and $x_{k-h}$, we obtain
{\small \begin{eqnarray}
\Ex\left[\sum_{l=1}^{d}\vert\partial_l f(x_k)\vert\mathbb{I}(A_k^c)\right]
&\overset{\text{Lemma~\ref{lemma_delta}}}{\leq}&
\Ex\left[\sum_{l=1}^{d}\left(\vert\partial_l f(x_{k-h})\vert+h\Delta_{k-h}\right)\mathbb{I}(A_k^c)\right]\nonumber\\
&\overset{\eqref{eq:independence_e}}{=}&
\Ex\left[\sum_{l=1}^{d}\left(\vert\partial_l f(x_{k-h})\vert+h\Delta_{k-h}\right)\right]\cdot \mathbb{P}(A_k^c)\nonumber\\
&\overset{\text{Lemma~\ref{lemma_delta}}}{\leq}&
\Ex\left[\sum_{l=1}^{d}\vert\partial_l f(x_k)\vert+2d h \,\Delta_{k-h}\right]\cdot \mathbb{P}(A_k^c),\label{eq:e_decouple_mid}
\end{eqnarray}}
\begingroup\emergencystretch=2em
where in the last inequality we have applied Lemma~\ref{lemma_delta} once again to convert
$\Ex\sum_{l=1}^d\vert\partial_l f(x_{k-h})\vert$ back to the current iterate $x_k$, incurring an additive error of magnitude $h\Delta_{k-h}$.
\par\endgroup

Combining \eqref{eq_e_without_blk} and \eqref{eq:e_decouple_mid}, and using
$\sum_{l=1}^d \vert\partial_l f(x_k)\vert \le d\,\vert\partial_\alpha f(x_k)\vert$, where   $\partial_\alpha f(x):= \max_l|\partial_l f(x)|$, together with
$\Delta_{k-h}\le \sqrt{h+1}\Delta_1/\sqrt{k}$, we obtain
{\small \begin{eqnarray}
\Ex(e)
&\geq&
-\eta_k\frac{1}{1-\beta_1^n}\left(1+\beta_1^n\sqrt{n+1}\right)\frac{\mathbb{P}(A_k^c)}{\sqrt{1-\beta_2}}\frac{1}{1-\frac{\beta_1}{\sqrt{\beta_2}}}
\Ex\left[d\vert\partial_\alpha f(x_k)\vert+2dh\,\Delta_{k-h}\right]\nonumber\\
&\geq&
-\eta_k\,\delta_1(\beta_2)\,\Ex\left[\vert\partial_\alpha f(x_k)\vert\right]
-\eta_k\frac{C_3}{\sqrt{k}},
\label{eq:e}
\end{eqnarray}}
where we define
{\small \begin{equation}\label{eq:def_delta1_C3_e}
\begin{aligned}
\delta_1(\beta_2)&:=\frac{d\left(1+\beta_1^n\sqrt{n+1}\right)}{1-\beta_1^n}
\frac{\mathbb{P}(A_k^c)}{\sqrt{1-\beta_2}}\frac{1}{1-\frac{\beta_1}{\sqrt{\beta_2}}},\\
C_3&:=\frac{2\,dh\sqrt{h+1}\left(1+\beta_1^n\sqrt{n+1}\right)}{1-\beta_1^n}
\frac{\mathbb{P}(A_k^c)}{\sqrt{1-\beta_2}}\frac{1}{1-\frac{\beta_1}{\sqrt{\beta_2}}}\Delta_1 .
\end{aligned}
\end{equation}}
Recall from the union bound and \eqref{eq:failure_1} that 
{\small \begin{equation}
\label{eq_p_ak_c}
\mathbb{P}(A_k^c)\leq n(n+1)\exp\left(-\frac{1}{160(1-\beta_2)n}\right).
\end{equation}}
As $\beta_2\to 1^-$, the probability $\mathbb{P}(A_k^c)$ decays exponentially at rate $O\left(\exp\left(-\frac{1}{160(1-\beta_2)n}\right)\right)$, which dominates the polynomial growth of $\frac{1}{\sqrt{1-\beta_2}}$. Consequently, the factor
$\frac{\mathbb{P}(A_k^c)}{\sqrt{1-\beta_2}}$ vanishes as $\beta_2\to 1^-$, ensuring that $\delta_1(\beta_2)\to 0$ as $\beta_2\to 1^-$.
This concludes the lower bound for $\Ex(e)$.

\paragraph{Lower bound of $\Ex(d)$.}  The lower bound of $(d)$ is more involved due to the existence of momentum and unbounded gradients. We first convert all $v_{l,k-n}$'s to $v_{l,k}$'s. We discuss two cases.

\paragraph{Case (1): when $\partial_lf(x_k)m_{l,k-n}\leq 0$.} Recall {\small $\frac{1}{\sqrt{v_{l,\bar{k}-n}}}\geq
\frac{1}{\sqrt{v_{l,\bar{k}}}}\sqrt{\beta_2^n}$ }, we have 

 {\small $$-
\partial _lf(x_k)\frac{m_{l,k-n}}{\sqrt{v_{l,k-n}}}\geq 
-
\partial _lf(x_k)\frac{m_{l,k-n}}{\sqrt{v_{l,k}}}\sqrt{\beta_2^n}.$$}

\paragraph{Case (2): when $\partial_lf(x_k)m_{l,k-n}\geq 0$.} In this case, we first provide a  lower bound  on $v_{l,\bar{k}-1}$ for the integer $\bar{k}\in [k-n+1,k]$:
{\small \begin{eqnarray}
v_{l,\bar{k}-1} &\geq& v_{l,\bar{k}} \left(1 - \frac{\vert v_{l,\bar{k}-1}-v_{l,\bar{k}} \vert}{v_{l,\bar{k}}}\right) \nonumber\\
&\overset{}{\geq} &v_{l,\bar{k}}\left(1-(1-\beta_2)\frac{\max_i\vert\partial_l f_i(x_{\bar{k}}) \vert^2 +v_{l,\bar{k}-1}}{v_{l,\bar{k}}}\right)\nonumber \\   &\overset{(\ref{eq_concentrate_v_epsilon_given}) (\ref{upperbound_of_f_over_v_epsilon})}{\geq}& v_{l,\bar{k}}\left(1-(1-\beta_2)\left(\frac{8 n}{\beta_2^n-(1-\beta_2)16n}+\frac{1}{\beta_2}\right)\right).\label{eq:v_k-1_geq_v_k}
\end{eqnarray}}
Then we have 
 {\small \begin{equation*}
-\partial _lf(x_k)\frac{m_{l,k-n}}{\sqrt{v_{l,k-n}}}\overset{(\ref{eq:v_k-1_geq_v_k})}{\geq} -\partial _lf(x_k)\frac{m_{l,k-n}}{\sqrt{v_{l,k}}}\left(1-(1-\beta_2)\left(\frac{8 n}{\beta_2^n-(1-\beta_2)16n}+\frac{1}{\beta_2}\right)\right)^{-\frac{n}{2}}.\end{equation*}}

Combining the two cases, we have
{\small \begin{eqnarray}\label{eq:-f_m_over_v_for_term_d}
-\partial_lf(x_k)\frac{m_{l,k-n}}{\sqrt{v_{l,k-n}}}%
&\geq& -\partial_l f(x_k)\frac{m_{l,k-n}}{\sqrt{v_{l,k}}}-\delta_2(\beta_2)
\vert \partial_lf(x_k)\vert \frac{\vert m_{l,k-n}\vert}{\sqrt{v_{l,k}}},
\end{eqnarray}}
where 
{\small \begin{equation}
    \label{eq_delta_2}
    \delta_2(\beta_2):=\left\vert \sqrt{\beta_2^n}-1 \right\vert
 +\left\vert \left(1-(1-\beta_2)\left(\frac{8 n}{\beta_2^n-(1-\beta_2)16n}+\frac{1}{\beta_2}\right)\right)^{-\frac{n}{2}}-1\right\vert.
\end{equation}}
We briefly comment on the order of $\delta_2(\beta_2)$ w.r.t. $n$.   Assume we are using $1-\beta_2 = \mathcal{O}(n^{-a})$ with $a >0$, then by basic calculation, one can show that the 1st term is $\mathcal{O}(n^{-(a-1)})$ and the 2nd term is $\mathcal{O}(n^{-(a-2)})$.

With the inequality \eqref{eq:-f_m_over_v_for_term_d} above, we have converted all $v_{l,k-n}$'s  in term $(d)$ to $v_{l,k}$'s. Now we further decompose $(d)$ into three terms $(d_1)$, $(d_2)$ and $(d_3)$:
{\fontsize{6.5pt}{8pt}\selectfont\setlength{\arraycolsep}{0pt}
\begin{eqnarray}
    (d)&\geq& \frac{1}{1-\beta_1^n}\sum_{l=1}^{d}\left[\left(
    \eta_k \partial_lf(x_k)\frac{m_{l,k}}{\sqrt{v_{l,k}}}-\beta_1^n\eta_{k-n} \partial _lf(x_k)\frac{m_{l,k-n}}{\sqrt{v_{l,k}}}-\beta_1^n\eta_{k-n}\delta_2(\beta_2)\vert \partial_lf(x_k)\vert \frac{\vert m_{l,k-n}\vert}{\sqrt{v_{l,k}}}\right)
    \mathbb{I}(A_k \cap B_{l,k}^c)\right] \nonumber \\
       &=&  \frac{1}{1-\beta_1^n} \sum_{l=1}^{d}\left[\left(
    \eta_k \frac{\partial_lf(x_k)}{\sqrt{v_{l,k}}}(m_{l,k}-\frac{\sqrt{k}}{\sqrt{k-n}}\beta_1^n m_{l,k-n})-\beta_1^n\eta_{k-n}\delta_2(\beta_2)\vert \partial_lf(x_k)\vert \frac{\vert m_{l,k-n}\vert}{\sqrt{v_{l,k}}}\right)
    \mathbb{I}(A_k \cap B_{l,k}^c)\right] \nonumber \\
    &=&  \frac{1}{1-\beta_1^n} \sum_{l=1}^{d}\Bigl[\Bigl(
    \eta_k \frac{\partial_lf(x_k)}{\sqrt{v_{l,k}}}\left((m_{l,k}-\beta_1^n m_{l,k-n}) + (1-\frac{\sqrt{k}}{\sqrt{k-n}})\beta_1^n m_{l,k-n})\right)\nonumber\\
    &&\qquad-\beta_1^n\eta_{k-n}\delta_2(\beta_2)\vert \partial_lf(x_k)\vert \frac{\vert m_{l,k-n}\vert}{\sqrt{v_{l,k}}}
    \Bigr)
    \mathbb{I}(A_k \cap B_{l,k}^c)\Bigr] \nonumber \\
      &\overset{\text{(i)}}{\geq}& \eta_k \frac{(1-\beta_1)}{1-\beta_1^n} \sum_{l=1}^{d}\left(\frac{\partial_lf(x_k)}{\sqrt{v_{l,k}}} \partial_l f_{\taut}(x_k) +\underbrace{\frac{\partial_lf(x_k)}{\sqrt{v_{l,k}}}}_{:=F_k} \underbrace{\left(\beta_1\partial_l f_{\tau_{k-1}}(x_{k-1})+\dots+\beta_1^{n-1}\partial_l f_{\tau_{k-n+1}}(x_{k-n+1})\right)}_{:=G}\right)\mathbb{I}(A_k \cap B_{l,k}^c)\nonumber \\
          &&\underbrace{-\left(\eta_k\frac{n+1}{\sqrt{k}}\frac{\beta_1^n}{1-\beta_1^n} +\eta_{k}\frac{\sqrt{n+1}\beta_1^n}{1-\beta_1^n}\delta_2(\beta_2)\right)\sum_{l=1}^{d}  \frac{\vert \partial_l f(x_k) \vert\vert m_{l,k-n}\vert}{\sqrt{v_{l,k}}}\mathbb{I}(A_k \cap B_{l,k}^c)}_{:=(d_3)}\nonumber\\
    &\overset{}{=}& \underbrace{\eta_k \frac{(1-\beta_1)}{1-\beta_1^n} \sum_{l=1}^{d}\frac{   \partial_l f(x_k) \partial_l f_{\taut}(x_k)}{\sqrt{v_{l,k}}} \cdot \mathbb{I}(A_k \cap B_{l,k}^c)}_{:=(d_1)}
    +\underbrace{\eta_k \frac{(1-\beta_1)}{1-\beta_1^n}\sum_{l=1}^{d}F_k \cdot G\cdot \mathbb{I}(A_k \cap B_{l,k}^c)}_{:=(d_2)}+(d_3).\label{eq:d_unsimplified}
\end{eqnarray}
}
where (i) is due to $|1-\frac{\sqrt{k}}{\sqrt{k-n}}|\leq \frac{n+1}{\sqrt{k}}$. 

In the following, we will show that $(d_1)$ constitutes the descent direction, whereas $(d_2)$ and $(d_3)$ are error terms. 

\paragraph{Lower bound of $\Ex(d_1)$.} We first bound $\mathbb{E}(d_1)$. We first derive a lower bound for the conditional expectation $\mathbb{E}_k(d_1)$ and subsequently take the outer expectation. By Lemma \ref{lemma_concentrate_v}, we control the quantity $\frac{1}{\sqrt{v_{l,k}}}$ via $\frac{1}{\sqrt{\mathbb{E}_k(v_{l,k})}}$, which yields:
{\footnotesize
\begin{eqnarray}
  \Exk(d_1)
  &=& \eta_k \frac{(1-\beta_1)}{1-\beta_1^n} \Exk \left( \sum_{l =1}^d \frac{   \partial_l f(x_k) \partial_l f_{\taut}(x_k)}{\sqrt{v_{l,k}}} \mathbb{I}(A_k \cap B_{l,k}^c)  \right) \nonumber\\
  &\overset{\text{(i)}}{\geq}&  \eta_k\left[ \frac{(1-\beta_1)}{1-\beta_1^n} \Exk \left( \sum_{l =1}^d  \frac{   \partial_l f(x_k) \partial_l f_{\taut}(x_k)}{\sqrt{\Exk(v_{l,k})}} \mathbb{I}(A_k \cap B_{l,k}^c)  \right) -\delta_3(\beta_2)\left| \partial_\alpha f(x_k) \right| \right] \nonumber \\
  &=&  \eta_k \left[\frac{(1-\beta_1)}{1-\beta_1^n}\Exk \left( \sum_{l =1}^d  \frac{   \partial_l f(x_k) \partial_l f_{\taut}(x_k)}{\sqrt{\Exk(v_{l,k})}}  \mathbb{I}(A_k \cap B_{l,k}^c)  \right)   - \delta_3(\beta_2)\left| \partial_\alpha f(x_k) \right| \right. \nonumber\\
  && \left. +  \frac{(1-\beta_1)}{1-\beta_1^n} \Exk \left( \sum_{l =1}^d  \frac{   \partial_l f(x_k) \partial_l f_{\taut}(x_k)}{\sqrt{\Exk(v_{l,k})}} \mathbb{I}(A_k \cap B_{l,k})  \right)  \right. \nonumber \\
  && \left. -  \frac{(1-\beta_1)}{1-\beta_1^n} \Exk \left( \sum_{l =1}^d  \frac{   \partial_l f(x_k)   \partial_l f_{\taut}(x_k) }{\sqrt{\Exk(v_{l,k})}} \mathbb{I}(A_k \cap B_{l,k})  \right)  \right] \nonumber \\
  &\overset{\text{(ii)}}{\geq}& \eta_k\left[  \frac{(1-\beta_1)}{1-\beta_1^n}\Exk \left( \sum_{l =1}^d  \frac{   \partial_l f(x_k) \partial_l f_{\taut}(x_k)}{\sqrt{\Exk(v_{l,k})}}  \mathbb{I}(A_k)  \right)   - \delta_3(\beta_2)\left| \partial_\alpha f(x_k) \right| - \frac{(1-\beta_1)}{1-\beta_1^n} \frac{2\sqrt{2n} d}{\sqrt{\beta_2^n}} Q_k\right] \nonumber\\
  &\overset{\text{(iii)}}{\geq}&  \eta_k\left[ \frac{   (\partial_\alpha f(x_k))^2 }{n\sqrt{\Exk(v_{\alpha,k})}} - \sum_{l =1}^d  \frac{(\partial_l f(x_k))^2 }{n \sqrt{\Exk\left(v_{l,k}\right)}}\mathbb{I}(A_k^c)     -\delta_3(\beta_2)\left| \partial_\alpha f(x_k) \right|-\frac{(1-\beta_1)}{1-\beta_1^n} \frac{2\sqrt{2n} d}{\sqrt{\beta_2^n}} Q_k\right],\label{eq:d1_exactly}
\end{eqnarray}
}
where $\alpha := \arg \max_l |\partial_l f(x_k)|$ and  $v_{\alpha, k}:= \max_lv_{l, k}$. The steps (i), (ii), (iii) are justified as follows:

\paragraph{Justification for (i) in \eqref{eq:d1_exactly}:}
We apply the concentration inequality \eqref{eq_concentrate_v_epsilon_given} from Lemma \ref{lemma_concentrate_v} to split the term into two parts based on the sign of $\partial_l f(x_k) \partial_l f_{\taut}(x_k)$:
\begin{itemize}
    \item When $\partial_l f(x_k) \partial_l f_{\taut}(x_k) \geq 0$, we use the lower bound in \eqref{eq_concentrate_v_epsilon_given}:
    {\small $$
    \frac{1}{\sqrt{v_{l,k}}} \geq \left(1- (1-\beta_2) \frac{8n}{\beta_2^n} \right) \frac{1}{\sqrt{\Exk(v_{l,k})}}.
    $$}
    \item When $\partial_l f(x_k) \partial_l f_{\taut}(x_k) < 0$, we use the upper bound in \eqref{eq_concentrate_v_epsilon_given}:
    {\small $$
    \frac{1}{\sqrt{v_{l,k}}} \leq \left(\frac{1}{\sqrt{1-(1-\beta_2)\frac{16n}{\beta_2^n}}}\right) \frac{1}{\sqrt{\Exk(v_{l,k})}}.
    $$}
\end{itemize}
Combining these two cases and collecting error terms, we obtain the coefficient $\delta_3(\beta_2)$ defined as:
{\small \begin{equation}
\label{eq_delta_3}
    \delta_3(\beta_2) = \frac{(1-\beta_1)}{1-\beta_1^n} \left(\frac{(1-\beta_2)16\sqrt{2n}nd}{\sqrt{\beta_2^{3n}}}   +   \left(\frac{1}{\sqrt{1-(1-\beta_2)\frac{16n}{\beta_2^n}}}-1 \right) \frac{2\sqrt{2n} d}{\sqrt{\beta_2^n}}  \right).
\end{equation}}

The proof follows a similar strategy as in \eqref{eq:-f_m_over_v_for_term_d}, where we decomposed the term based on the sign of the product.

\paragraph{Justification for (ii) in \eqref{eq:d1_exactly}:}
Under the event $B_{l,k}$, we have $\partial_l f(x_k) \leq  \max_i |\partial_l f_i(x_k)| \leq  Q_k$. Additionally, under the event $A_k$, Lemma \ref{lemma_concentrate_v} (specifically inequality \eqref{upperbound_of_f_over_v_epsilon}) ensures that:
{\small $$
\frac{\partial_l f_{\taut}(x_k)}{\sqrt{\Exk(v_{l,k})}} \leq \frac{2\sqrt{2n}}{\sqrt{\beta_2^n}}.
$$}
Therefore, under $A_k \cap B_{l,k}$, we have:
{\small $$
\left|\frac{\partial_l f(x_k) \partial_l f_{\taut}(x_k)}{\sqrt{\Exk(v_{l,k})}}\right| \leq  Q_k \cdot \frac{2\sqrt{2n}}{\sqrt{\beta_2^n}} = \frac{2\sqrt{2n} Q_k}{\sqrt{\beta_2^n}}.
$$}

\paragraph{Justification for (iii) in \eqref{eq:d1_exactly}:}
This step exploits two key facts:

\begin{enumerate}[label=(\roman*)]
    \item \textbf{Fact 1:} By the inequality $\frac{1-\beta_1}{1-\beta_1^n} \geq \frac{1}{n}$, we can lower bound the coefficient.
    \item \textbf{Fact 2:} The event $A_k$ depends only on the history up to step $k-n$, and hence is independent of the current gradient sample $\partial_l f_{\taut}(x_k)$ conditional on $\mathcal{F}_k$. Therefore:
    {\small \begin{eqnarray}
    \Exk \left[ \sum_{l =1}^d  \frac{\partial_l f(x_k) \partial_l f_{\taut}(x_k)}{\sqrt{\Exk(v_{l,k})}}  \mathbb{I}(A_k)  \right]  
    &=&   \sum_{l =1}^d  \frac{\Exk \left( \partial_l f(x_k) \partial_l f_{\taut}(x_k)  \right)}{\sqrt{\Exk(v_{l,k})}}    \mathbb{I}(A_k) \nonumber \\
    &=&  \sum_{l =1}^d  \frac{(\partial_l f(x_k))^2 }{\sqrt{\Exk(v_{l,k})}} \left(1 -  \mathbb{I}(A_k^c) \right)\nonumber\\
    &\geq& \frac{(\partial_\alpha f(x_k))^2 }{\sqrt{\Exk(v_{\alpha,k})}}  - \sum_{l =1}^d  \frac{(\partial_l f(x_k))^2 }{\sqrt{\Exk(v_{l,k})}}\mathbb{I}(A_k^c).\label{eq_justification3_fact2}
    \end{eqnarray}}
\end{enumerate}

This completes the derivation of the lower bound for $\Exk(d_1)$ in \eqref{eq:d1_exactly}.

Now we proceed to handle \eqref{eq:d1_exactly}. Taking the expectation over \eqref{eq:d1_exactly}, we now upper bound $\mathbb{E} \left[ \frac{(\partial_{l}f(x_k))^2}{n\sqrt{\mathbb{E}_k (v_{l,k})}}\mathbb{I}(A_k^c)\right]$. We will use the upper bound \eqref{eq:trivial} and the same decoupling strategy in \eqref{eq:e}:

{\small \begin{eqnarray*}
\mathbb{E} \left[ \sum_{l =1}^d\frac{(\partial_{l}f(x_k))^2}{n\sqrt{\mathbb{E}_k (v_{l,k})}}\mathbb{I}(A_k^c)\right]
&\leq&  \mathbb{E} \left[ \sum_{l =1}^d\frac{\left(\max_i \vert \partial_{l}f_i(x_k)\vert\right) }{n\sqrt{\mathbb{E}_k (v_{l,k})}}\vert \partial_{l}f(x_k)\vert\mathbb{I}(A_k^c)\right]
\\ &\overset{\eqref{eq:trivial}}{\leq}& \mathbb{E}\left[\sum_{l =1}^d\frac{1}{n}\cdot\frac{\sqrt{n}}{\sqrt{(1-\beta_2)}} \vert \partial_{l}f(x_k) \vert\mathbb{I}(A_k^c)\right]\\
&\leq& \frac{\mathbb{P}(A_k^c)dh}{\sqrt{n(1-\beta_2)}} \mathbb{E}\vert \partial_{\alpha}f(x_k) \vert +\frac{\mathbb{P}(A_k^c)2dh}{\sqrt{n(1-\beta_2)}}\sqrt{h+1}\frac{\Delta_1}{\sqrt{k}}.
\end{eqnarray*}}
Combining the results above, we have
{\small \begin{eqnarray}
    \Ex(d_1) = \Ex \Exk(d_1)
 &\geq& \eta_k\left[\Ex\frac{   (\partial_\alpha f(x_k))^2 }{n\sqrt{\Exk(v_{\alpha,k})}}-\delta_4(\beta_2)\Ex\left| \partial_\alpha f(x_k) \right|-\frac{C_4}{\sqrt{k}}\right],
 \label{eq:d1}
\end{eqnarray}}
where
{\small \begin{eqnarray}
\delta_4(\beta_2)&=& \frac{\mathbb{P}(A_k^c)dh}{\sqrt{n(1-\beta_2)}} + \delta_3(\beta_2).\label{eq_delta_4} \\
\quad C_4&=&\frac{\mathbb{P}(A_k^c)2dh}{\sqrt{n(1-\beta_2)}}\sqrt{h+1}\Delta_1 + \frac{(1-\beta_1)}{1-\beta_1^n}\frac{2\sqrt{2n} d}{\sqrt{\beta_2^n}} Q_1, \label{def_C_4} 
\end{eqnarray}}
and $\delta_3(\beta_2)$ and $Q_1$ are defined in \eqref{eq_delta_3}  and \eqref{eq_constants_wr}, respectively. \brrevise{Recall that the probability $\mathbb{P}(A_k^c)$ decays exponentially at rate $O\left(\exp\left(-\frac{1}{160(1-\beta_2)n}\right)\right)$. So the factor
$\frac{\mathbb{P}(A_k^c)}{\sqrt{1-\beta_2}}$ vanishes as $\beta_2\to 1^-$, ensuring that $\delta_4(\beta_2)\to 0$ as $\beta_2\to 1^-$.}

\paragraph{Lower bound of $\Ex(d_2)$.} To bound the term $\Ex(d_2)$ in \eqref{eq:d_unsimplified}, we convert all the variables $x$'s to $x_{k-n}$ and then take the expectation conditioned on the history up to $\mathcal{F}_{k-n}$.
We denote our target terms as 
{\small \begin{equation}
\label{eq_G_k_n}
F_{k-n}:=\frac{\partial_lf(x_{k-n})}{\sqrt{v_{l,k-n}}} \quad \text{and} \quad G_{k-n}:=\beta_1 \partial_l f_{\tau_{k-1}} (x_{k-n}) + \cdots + \beta_1^{n-1} \partial_l f_{\tau_{k-n+1}} (x_{k-n}).
\end{equation}}
We now convert $F_k$ in \eqref{eq:d_unsimplified} to $F_{k-n}$ using the following Lemma \ref{lemma_k-k-n_WR}.

\begin{lemma} \label{lemma_k-k-n_WR}
Consider Algorithm \ref{algorithm_wr}. Assume that   $k \geq \lceil\frac{\log (1/4) }{\log \beta_2} \rceil +2n+ 1$, and for all integers $\bar{k} \in [k-n,k]$, $\max_i |\partial_l f_i (x_{\bar{k}})| \geq R_{\bar{k}}$, and the joint event $A_k \cap B_{l,k}^c$ occurs.
Then, for any $f \in \functionclass$, we have
{\small \begin{equation}
\label{eq_k-k-1_WR}
    \left| \frac{\partial_l f(x_{k})}{\sqrt{v_{l,k}}} -   \frac{\partial_l f(x_{k-n})}{\sqrt{v_{l,k-n}}} \right|  \leq  \frac{n}{\sqrt{\beta_2^{n}}} \frac{\Delta_{k-n}}{\sqrt{v_{l,k-n}}} + n\delta_5(\beta_2),
\end{equation}}
where $\Delta_{k-n}$ is defined as in Lemma \ref{lemma_delta}, and $\delta_5(\beta_2)$ is defined as follows.
{\small \begin{equation}
    \label{eq_delta_5}
    \delta_5(\beta_2) = \delta_{5,1}(\beta_2)\delta_{5,2}(\beta_2)\delta_{5,3}(\beta_2),
\end{equation}}
with $\delta_{5,1}(\beta_2) = \left(\left[1- \left(1-(1-\beta_2)\left(\frac{8 n}{\beta_2^n-(1-\beta_2)16n}+\frac{1}{\beta_2}\right)\right)^{\frac{1}{2}}\right] +\frac{1}{\sqrt{\beta_2}}-1  \right)$,\\
$\delta_{5,2}(\beta_2) = \left(1-(1-\beta_2)\left(\frac{8 n}{\beta_2^n-(1-\beta_2)16n}+\frac{1}{\beta_2}\right)\right)^{-\frac{1}{2}}$ and
{\small$\delta_{5,3}(\beta_2) = \frac{2\sqrt{2n}}{\sqrt{\beta_2^n-(1-\beta_2)16n}}$}.
\end{lemma}

The proof of Lemma \ref{lemma_k-k-n_WR} can be seen in Appendix \ref{appendix:lemma_k-k-n}. 

With Lemma \ref{lemma_k-k-n_WR} and Lemma \ref{lemma_delta}, we have the following results ($G, F_{k-n}, G_{k-n}$ are defined in \eqref{eq:d_unsimplified} and \eqref{eq_G_k_n}).
{\small $$\vert F_k-F_{k-n}\vert \overset{\text{Lemma }\ref{lemma_k-k-n_WR}}{\leq} \frac{n}{\sqrt{\beta_2^n}} \frac{\Delta_{k-n}}{\sqrt{v_{l,k-n}}} + n \delta_5(\beta_2):=C_{F,k},$$}
{\small $$\vert G_{k-n}-G\vert  \overset{\text{Lemma \ref{lemma_delta}}}{\leq} \sum_{m=1}^{n-1} \beta_1^m \left( \sum_{\bar{k}=m}^{n-1} \Delta_{k-\bar{k}} \right)\leq n^2\Delta_{k-n}\leq n^2\Delta_k\sqrt{n+1}:=C_{G,k}.$$}
Then,
{\footnotesize
\begin{eqnarray}
(d_2) 
&=& \sum_{l=1}^d\eta_k \frac{(1-\beta_1)}{1-\beta_1^n} (F_k \cdot G) \cdot\mathbb{I}(A_k \cap B_{l,k}^c)\nonumber\\
&=& \sum_{l=1}^d\eta_k \frac{(1-\beta_1)}{1-\beta_1^n} F_k \cdot[G_{k-n} - (G_{k-n} - G)] \cdot\mathbb{I}(A_k \cap B_{l,k}^c)\nonumber\\
&\geq & \sum_{l=1}^d\eta_k \frac{(1-\beta_1)}{1-\beta_1^n} \left[F_k \cdot G_{k-n} - \vert F_k \vert \cdot\vert G_{k-n} - G\vert\right] \cdot\mathbb{I}(A_k \cap B_{l,k}^c)\nonumber\\
&\geq& \sum_{l=1}^d\eta_k \frac{(1-\beta_1)}{1-\beta_1^n} \left[ F_{k-n} \cdot G_{k-n} - \vert F_{k-n} - F_k \vert \cdot \vert G_{k-n}\vert - \vert F_k \vert \cdot \vert G_{k-n} - G\vert\right]  \cdot\mathbb{I}(A_k \cap B_{l,k}^c)\nonumber\\
&\geq& \underbrace{\sum_{l=1}^d\eta_k \frac{(1-\beta_1)}{1-\beta_1^n} \left[ \frac{\partial_l f(x_{k-n})}{\sqrt{v_{l,k-n}}} \left( \beta_1 \partial_l f_{\tau_{k-1}} (x_{k-n}) + \cdots + \beta_1^{n-1} \partial_l f_{\tau_{k-n+1}} (x_{k-n}) \right)\right]\cdot\mathbb{I}(A_k \cap B_{l,k}^c)}_{:=(d_{2,1})} \nonumber\\
    && -\sum_{l=1}^d\eta_k \frac{(1-\beta_1)}{1-\beta_1^n}\bigg[\left| C_{F,k}\cdot G_{k-n} \right| +\left| F_{k}\cdot C_{G,k} \right| \bigg] \cdot\mathbb{I}(A_k \cap B_{l,k}^c).\label{eq:d2_unsimplified}
\end{eqnarray}
}
We can bound $\Ex(d_{2,1})$ by following the same procedure used to derive the lower bound of $\Ex(d_1)$ in \eqref{eq:d1_exactly} and \eqref{eq:d1}.
We first note the following fact (similar to the fact \eqref{eq_justification3_fact2}): Since event $A_k$ depends only on the history up to the $(k-n)$-th iteration and $\tau_{\bar{k}}$ is independent of $\mathcal{F}_{k-n}$ for each $\bar{k}\in [k-n,k]$, we have
{\small \begin{equation}
\mathbb{E}_{k-n}\left[ \partial_l f(x_{k-n})  \partial_l f_{\tau_{\bar{k}}} (x_{k-n})\mathbb{I}(A_k)\right]= \partial_l f(x_{k-n}) \mathbb{E}_{k-n} [\partial_l f_{\tau_{\bar{k}}} (x_{k-n})]\mathbb{I}(A_k)=(\partial_l f(x_{k-n}))^2\mathbb{I}(A_k).
\label{eq:d_2_1_nonnegative}
\end{equation}}
Proceeding as before, we obtain the following bound using the fact that $\mathbb{I}(A_k ) = \mathbb{I}(A_k \cap B_{l,k}) + \mathbb{I}(A_k \cap B_{l,k}^c)$:
{\small \begin{eqnarray}
     \Ex(d_{2,1}) = \Ex\mathbb{E}_{k-n}(d_{2,1})&\geq& \sum_{i=1}^{n-1}\beta_1^i\eta_k \left[\Ex\frac{   \left( \partial_\alpha f(x_{k-n}) \right) ^2 }{n\sqrt{\Ex_{k-n}(v_{\alpha,k-n})}} - \delta_4(\beta_2)   \Ex\left| \partial_\alpha f(x_{k-n}) \right| -\frac{C_4}{\sqrt{k-n}} \right]\nonumber\\
&\geq& -\eta_k \frac{1}{1-\beta_1}\left[ \delta_4(\beta_2)   \Ex \left| \partial_\alpha f(x_{k-n}) \right| +\frac{C_4\sqrt{n+1}}{\sqrt{k}} \right] \nonumber\\
  &\overset{}{\geq}& -\eta_k \left[\delta_6(\beta_2)\Ex \left|\partial_{\alpha} f\left(x_{k}\right)\right|+\delta_7(\beta_2)\sqrt{D_0}+\frac{C_5}{\sqrt{k}}\right],
  \label{eq:EEk-nd_2,1}
\end{eqnarray}}
where {\small \begin{equation}
\begin{aligned}
\delta_6(\beta_2)&=\frac{\delta_4(\beta_2)}{1-\beta_1}nd\sqrt{D_1}\sqrt{n}d,\\
\delta_7(\beta_2)&=\frac{\delta_4(\beta_2)}{1-\beta_1}nd\sqrt{n}\sqrt{d},\\
C_5&=\frac{\delta_4(\beta_2)}{1-\beta_1}n^2d\Delta_1\sqrt{n+1}+\frac{C_4\sqrt{n+1}}{1-\beta_1}.
\end{aligned}\label{eq_delta_6_7_C_5}
\end{equation}}

We deal with the rest of the terms in  (\ref{eq:d2_unsimplified}) similarly, using Lemma \ref{lemma_fi_f}, (\ref{eq_concentrate_v_epsilon_given}) and (\ref{upperbound_of_f_over_v_epsilon}).
{\small \begin{eqnarray}
|C_{F,k}\cdot G_{k-n}|\mathbb{I}(A_k\cap B_{l,k}^c)
&=&\Bigl(\frac{n}{\sqrt{\beta_2^n}}\frac{\Delta_{k-n}}{\sqrt{v_{l,k-n}}}+n\delta_5(\beta_2)\Bigr) \nonumber\\
&&\quad\cdot\Bigl(\beta_1|\partial_l f_{\tau_{k-1}}(x_{k-n})|
+\cdots+\beta_1^{n-1}|\partial_l f_{\tau_{k-n+1}}(x_{k-n})|\Bigr)
\mathbb{I}(A_k\cap B_{l,k}^c)\nonumber\\
&\leq&\frac{1}{\sqrt{k}}\frac{1-\beta_1^n}{1-\beta_1}
\frac{\Delta_1 2\sqrt{2n}n\sqrt{n+1}}{\sqrt{\beta_2^n}\sqrt{\beta_2^n-(1-\beta_2)16n}}\nonumber\\
&&\quad+\delta_5(\beta_2)\frac{1-\beta_1^n}{1-\beta_1}
n\sqrt{D_1}\sqrt{n}d|\partial_\alpha f(x_k)|\nonumber\\
&&  +  \frac{1}{\sqrt{k}}\delta_5(\beta_2)\frac{(1-\beta_1^n)}{1-\beta_1}n^2\sqrt{n+1}\Delta_1
+ \delta_5(\beta_2)\frac{(1-\beta_1^n)}{1-\beta_1}n \sqrt{n} \sqrt{d} \sqrt{D_0}.\nonumber \\ \label{eq:C_F*Gbase}
\end{eqnarray}}
And
{\small \begin{equation}\label{eq:F_K*C_G}
\begin{aligned}
|F_k\cdot C_{G,k}|\mathbb{I}(A_k\cap B_{l,k}^c)
&=\frac{|\partial_l f(x_k)|}{\sqrt{v_{l,k}}}n^2\Delta_k\sqrt{n+1}
\mathbb{I}(A_k\cap B_{l,k}^c)\\
&\overset{\eqref{eq_concentrate_v_epsilon_given}\eqref{upperbound_of_f_over_v_epsilon}}{\le}
\frac{2\sqrt{2n}}{\sqrt{\beta_2^n-(1-\beta_2)16n}}n^2\Delta_k\sqrt{n+1}.
\end{aligned}
\end{equation}}
Plug \eqref{eq:EEk-nd_2,1}, \eqref{eq:C_F*Gbase}, \eqref{eq:F_K*C_G} into \eqref{eq:d2_unsimplified} and take expectation,

{\small\begin{eqnarray*}
\Ex(d_2) &\overset{\eqref{eq:d2_unsimplified}}{\geq}&\Ex(d_{2,1})-\sum_{l=1}^d\eta_k \frac{(1-\beta_1)}{1-\beta_1^n}\bigg[\left| C_{F,k}\cdot G_{k-n} \right| +\left| F_{k} \cdot C_{G,k} \right| \bigg] \mathbb{I}(A_k \cap B_{l,k}^c)\\
&\overset{\eqref{eq:EEk-nd_2,1}(\ref{eq:C_F*Gbase})(\ref{eq:F_K*C_G})}{\ge}&-\eta_k\left[ \delta_8(\beta_2)\Ex\left[\left|\partial_{\alpha} f\left(x_{k}\right)\right|\right]+\delta_9(\beta_2)\sqrt{D_0}+\frac{C_6}{\sqrt{k}}\right],
\end{eqnarray*}}
where the constant terms are defined as follows ($\delta_5(\beta_2)$, $\delta_6(\beta_2)$ and $\delta_7(\beta_2)$ are defined in \eqref{eq_delta_5} and \eqref{eq_delta_6_7_C_5}):
{\small \begin{equation}\label{eq_delta_8}
    \delta_8(\beta_2)= \delta_6(\beta_2) + d\delta_5(\beta_2)n \sqrt{D_1}\sqrt{n} d,
\end{equation}}
{\small \begin{equation}\label{eq_delta_9}
   \delta_9(\beta_2)=\delta_7(\beta_2) + d\delta_5(\beta_2)n \sqrt{n} \sqrt{d},
\end{equation}}

{\small \begin{equation}\label{def_C_6}
C_6=C_5 +d\frac{ \Delta_1 2\sqrt{2n}n\sqrt{n+1}}{\sqrt{\beta_2^n}\sqrt{\beta_2^n-(1-\beta_2) 16 n}}
+ d\delta_5(\beta_2)\frac{(1-\beta_1^n)}{1-\beta_1}n^2\sqrt{n+1}\Delta_1
+\frac{(1-\beta_1)}{1-\beta_1^n} d\frac{\Delta_12\sqrt{2n} n^2 \sqrt{n+1} }{ \sqrt{\beta_2^n-(1-\beta_2)16n}}.
\end{equation}}

\paragraph{Lower bound of $\Ex(d_3)$.} Now we provide a lower bound of  $\Ex(d_3)$ in \eqref{eq:d_unsimplified}.
Since $\vert \partial_lf(x_k) \vert \leq  \max_i\vert \partial_lf_i(x_k) \vert$, we have
{\small \begin{eqnarray}
    -\sum_{l=1}^{d}  \frac{\vert \partial_lf(x_k) \vert\vert m_{l,k-n}\vert}{\sqrt{v_{l,k}}}\mathbb{I}(A_k \cap B_{l,k}^c)\ &\overset{(\ref{eq_concentrate_v_epsilon_given}) (\ref{upperbound_of_f_over_v_epsilon})}{\geq}& -\frac{2\sqrt{2n}}{\sqrt{\beta_2^n-
    (1-\beta_2)16n}} \sum_{l=1}^{d} \vert m_{l,k-n}\vert. \label{eq:c_sum_m}
\end{eqnarray}}
We further bound $\sum_{l=1}^d |m_{l,k-n}|$, 

{\scriptsize
\begin{eqnarray}
\sum_{l=1}^d |m_{l,k-n}| &\leq& (1-\beta_1) \sum_{l=1}^d [|\partial_l f_{\tau_{k-n}}(x_{k-n})| + \beta_1|\partial_l f_{\tau_{k-n-1}}(x_{k-n-1})| + \cdots] + \beta_1^{k-n} \sum_{l=1}^d |m_{l,0}| \nonumber\\
& \overset{\text{Lemma \ref{lemma_fi_f}} }{\leq} & (1-\beta_1) \sum_{q=1}^{k-n} \beta_1^{q-1} \left[\sqrt{D_1}\sqrt{n} d \left(|\partial_\alpha f(x_{k})| + \sqrt{\frac{D_0}{D_1 d}}\right) + \frac{2\sqrt{n+1}q d \Delta_1}{\sqrt{k}} \right]+\beta_1^{k-n} \|m_0\|_1 \nonumber \\
&\overset{\text{Lemma \ref{lemma_beta}}}{\leq}& \sqrt{D_1}\sqrt{n} d \left(|\partial_\alpha f(x_{k})| + \sqrt{\frac{D_0}{D_1 d}}\right) + \frac{1}{1-\beta_1} \frac{2\sqrt{n+1}d \Delta_1}{\sqrt{k}} + \beta_1^{k-n} \|m_0\|_1 \nonumber \\
&\overset{\text{(i)}}{\leq}& \sqrt{D_1}\sqrt{n} d \left(|\partial_\alpha f(x_{k})| + \sqrt{\frac{D_0}{D_1 d}}\right) + \frac{1}{1-\beta_1} \frac{2\sqrt{n+1}d \Delta_1}{\sqrt{k}} + \frac{1}{\sqrt{k}} \|m_0\|_1 ,\label{eq:sum_m}
\end{eqnarray}
}
where (i) holds for $k$ large enough such that $\beta_1^{k-n}\leq \frac{1}{\sqrt{k}}$.
Plug (\ref{eq:sum_m}) into (\ref{eq:c_sum_m}),
{\small
\begin{eqnarray}  && -\sum_{l=1}^{d}  \frac{\vert \partial_lf(x_k) \vert\vert m_{l,k-n}\vert}{\sqrt{v_{l,k}}}\mathbb{I}(A_k \cap B_{l,k}^c) \nonumber \\
&\overset{(\ref{eq:c_sum_m})}{\geq} &-\frac{2\sqrt{2n}}{\sqrt{\beta_2^n-
    (1-\beta_2)16n}} \sum_{l=1}^{d} \vert m_{l,k-n}\vert \nonumber\\
    &\overset{(\ref{eq:sum_m})}{\geq} &-\frac{2\sqrt{2n}}{\sqrt{\beta_2^n-
    (1-\beta_2)16n}}\left(\sqrt{D_1}\sqrt{n} d\left(\vert\partial_\alpha f(x_k)\vert +\sqrt{\frac{D_0}{D_1d}}\right)+\frac{2\sqrt{n+1}d\Delta_1}{(1-\beta_1)\sqrt{k}}+\frac{\|m_0\|_1}{\sqrt{k}}\right).\nonumber \\
    \label{eq:d3_fm_over_v}
\end{eqnarray}
}
Plug inequality (\ref{eq:d3_fm_over_v}) into the term $(d_3)$ in \eqref{eq:d_unsimplified} and take expectation,

{\small \begin{eqnarray}
\Ex(d_3)
    &\geq& -\eta_k\left[ \left(\delta_{10}(\beta_2)+\frac{C_7}{\sqrt{k}}\right)\Ex\vert\partial_\alpha f(x_k)\vert+\delta_{11}(\beta_2)\sqrt{D_0}+\frac{C_8}{\sqrt{k}}\right],
\end{eqnarray}}

where the constant terms are defined as follows ($\delta_2(\beta_2)$ is defined in \eqref{eq_delta_2}.):
{\small \begin{eqnarray}
    \delta_{10}(\beta_2)&=&\frac{\sqrt{n+1}\beta_1^n}{1-\beta_1^n}\delta_2(\beta_2)\frac{2\sqrt{2n}}{\sqrt{\beta_2^n-
    (1-\beta_2)16n}}\sqrt{D_1}\sqrt{n} d,\label{eq_delta_10} \\
    \delta_{11}(\beta_2)&=&\frac{\sqrt{n+1}\beta_1^n}{1-\beta_1^n}\delta_2(\beta_2)\frac{2\sqrt{2n}}{\sqrt{\beta_2^n-
    (1-\beta_2)16n}}\sqrt{n}\sqrt{d} ,\label{eq_delta_11} \\
    C_7&=&(n+1)\frac{\beta_1^n}{1-\beta_1^n}\frac{2\sqrt{2n}}{\sqrt{\beta_2^n-
    (1-\beta_2)16n}}\sqrt{D_1}\sqrt{n} d, \label{def_C_7}\\
    C_8&=&(n+1)\frac{\beta_1^n}{1-\beta_1^n}\frac{2\sqrt{2n}}{\sqrt{\beta_2^n-
    (1-\beta_2)16n}}\sqrt{n}\sqrt{D_0d}\nonumber\\
    &&+
    \frac{\sqrt{n+1}\beta_1^n}{1-\beta_1^n}\delta_2(\beta_2)\frac{2\sqrt{2n}}{\sqrt{\beta_2^n-
    (1-\beta_2)16n}}\left(\frac{2\sqrt{n+1}d\Delta_1}{(1-\beta_1)}+\|m_0\|_1\right). \label{def_C_8}
\end{eqnarray}}

\paragraph{Summary of the analysis above.} Now we summarize the analysis above and provide a lower bound of $\Ex(b) \geq\Ex(c)+\Ex(d_1) + \Ex(d_2)+\Ex(d_3)+\Ex(e)$ in \eqref{eq:d_unsimplified} by grouping the terms of $(\partial_{\alpha}f(x_k))^2$, $\partial_{\alpha}f(x_k)$ and gradient-independent errors.
{\small
\begin{eqnarray}
  \Ex (b)
  &\overset{}{\geq} &\eta_k\Bigg(\Ex\frac{(\partial_{\alpha}f(x_k))^2}{n\sqrt{\Exk (v_{\alpha,k})}}
  -\Ex \vert\partial_\alpha f(x_k)\vert \cdot \left(\delta_4(\beta_2)+\delta_8(\beta_2)+\delta_{10}(\beta_2)+\delta_{1}(\beta_2)+\frac{C_7}{\sqrt{k}}\right)\nonumber\\
  &&-(\delta_9(\beta_2)+\delta_{11}(\beta_2))\sqrt{D_0}-\frac{C_2+C_3+C_4+C_6+C_8}{\sqrt{k}}\Bigg)\nonumber\\
    &\overset{}{:=}& \eta_k \left\{ \Ex\left[\frac{(\partial_{\alpha}f(x_k))^2}{n\sqrt{\Exk (v_{\alpha,k})}}-\left(\delta_{12}(\beta_2)+\frac{C_7}{\sqrt{k}}\right)\vert\partial_\alpha f(x_k)\vert\right]-(\delta_9(\beta_2)+\delta_{11}(\beta_2))\sqrt{D_0}-\frac{C_9}{\sqrt{k}}
    \right\}, \nonumber \\ \label{eq:b}
\end{eqnarray}
}
where 
{\small \begin{eqnarray}
 \delta_{12}(\beta_2)&=&\delta_4(\beta_2)+\delta_8(\beta_2)+\delta_{10}(\beta_2)+\delta_{1}(\beta_2),   \nonumber \\
C_9&=&C_2+C_3+C_4+C_6+C_8, \label{eq_delta_1_beta2} 
\end{eqnarray}}
and $\delta_4(\beta_2)$, $\delta_8(\beta_2)$,  $\delta_{10}(\beta_2)$, and $\delta_1(\beta_2)$ are defined in \eqref{eq_delta_4}, \eqref{eq_delta_8}, \eqref{eq_delta_10}, \eqref{eq:def_delta1_C3_e}. 
Note that $\delta_{12}(\beta_2)+\frac{C_7}{\sqrt{k}}$ vanishes as $\beta_2\rightarrow 1$ and $k\rightarrow \infty$.  $C_2, C_3, C_4, C_6, C_8$ are defined in \eqref{eq:c},  \eqref{eq:def_delta1_C3_e}, \eqref{def_C_4}, \eqref{def_C_6}, \eqref{def_C_8}, respectively. 

Based on \eqref{eq:b},  one can show that the following \eqref{eqn:min_gradient_norm} holds when  $\beta_2$ is large enough such that $\delta_{12}(\beta_2)\leq\frac{1}{4n d \sqrt{5D_1 n}}$ and (iii) $k$ is large enough such that $\frac{C_7}{\sqrt{k}}\leq \frac{1}{4n d \sqrt{5D_1 n}}$ and \brrevise{$\beta_2^k \|v_0\|_\infty\leq D_0nd$}. 
Since  $\delta_{12}(\beta_2)$ is a polynomial of $n$,  condition (i) can be achieved by setting $1-\beta_2 = \mathcal{O}(1/n^5)$.
{\small \begin{equation}\label{eqn:min_gradient_norm}
\begin{aligned}
\Ex\left[\frac{(\partial_{\alpha}f(x_k))^2}{n\sqrt{\Exk(v_{\alpha,k})}}
-\left(\delta_{12}(\beta_2)+\frac{C_7}{\sqrt{k}}\right)\vert\partial_\alpha f(x_k)\vert\right]
&\geq{}\\
&\Ex\left[\min\left\{\frac{\|\nabla f(x_k)\|_2^2}{nd\sqrt{5D_0nd}},
\frac{\|\nabla f(x_k)\|_2}{2nd^2\sqrt{5D_1n}}\right\}\right]\\
&\quad-\delta_{12}(\beta_2)\sqrt{\frac{D_0}{D_1d}}-\frac{C_{11}}{\sqrt{k}},
\end{aligned}
\end{equation}}
where $C_{11}=\frac{2\sqrt{2}\Delta_1}{(1-\beta_2)D_1n^3d^2 \sqrt{5}} + \frac{8\sqrt{2}\Delta_1(\delta_{12}(\beta_2)+C_7)}{(1-\beta_2)\sqrt{D_1 n }d} + C_7\sqrt{\frac{D_0}{D_1 d}}$ is defined later in \eqref{def_C_11}.

The proof of \eqref{eqn:min_gradient_norm} will be presented in the immediate future.  Finally, we combine the lower bounds for term (a) (\ref{eq:a}) and term (b) (\ref{eq:b}) to get:
{\small \begin{equation}\label{eq:interior_descent}
\begin{aligned}
\Ex\langle\nabla f(z_k),z_k-z_{k+1}\rangle
&=\Ex[(a)+(b)]\\
&\overset{\eqref{eq:a}\eqref{eq:b}}{\geq}
\eta_k\Bigl\{\Ex\Bigl[\min\Bigl\{
\frac{\|\nabla f(x_k)\|_2^2}{nd\sqrt{5D_0nd}},
\frac{\|\nabla f(x_k)\|_2}{2nd^2\sqrt{5D_1n}}
\Bigr\}\Bigr]\\
&\qquad-\delta_{\mathrm{int}}(\beta_2)\sqrt{D_0}
-\frac{C_{\mathrm{int}}}{\sqrt{k}}\Bigr\},
\end{aligned}
\end{equation}}
where
{\small \begin{equation}
\label{def_delta_C}
\brrevise{\delta_{\mathrm{int}}(\beta_2)
:=\frac{\delta_{12}(\beta_2)}{\sqrt{D_1d}}
+\delta_9(\beta_2)
+\delta_{11}(\beta_2),
\qquad
C_{\mathrm{int}} 
:= C_1+C_9+C_{11}.}
\end{equation}}
and $\delta_{\mathrm{int}}(\beta_2)$ approaches 0 as $\beta_2$ approaches 1. The constants $\delta_{12}(\beta_2)$ and $\delta_{11}(\beta_2)$ are defined in \eqref{eq_delta_1_beta2} and \eqref{eq_delta_11}, respectively; the constants $C_1$ and $C_9$ are defined in \eqref{eq:a} and \eqref{eq_delta_1_beta2}, respectively;  $C_{11}$ is defined later in  \eqref{def_C_11}. 
\brrevise{Recall that the measure $\mathcal{M}_{D_0,D_1}(\cdot)$ is defined in \eqref{eq:stationarity_measure}. Since
{\small \[
\frac{1}{nd\sqrt{5nd}}\,
\mathcal{M}_{D_0,D_1}(u)
=
\min\left\{
\frac{u^2}{nd\sqrt{5D_0nd}},
\frac{u}{2nd^2\sqrt{5D_1n}}
\right\}
\]}
when $D_0>0$ and $D_1>0$, inequality
\eqref{eq:interior_descent} proves the interior branch of
Lemma \ref{lemma_descent_WR}. }

\paragraph{A sufficient range of $\beta_2$.}
\brrevise{Combining the condition in \eqref{eq_delta_1_beta2} with the requirement
in \eqref{beta2_range}, there exists a constant $c_\gamma\in(0,1]$,
independent of $n$, $\beta_1$, $k$, and $T$, such that both conditions
are satisfied whenever
{\small \begin{equation}
\label{eq:beta2_sufficient_range}
\beta_2
\geq
\gamma_1(n)
:=
1-c_\gamma\frac{1-\beta_1^n}{n^5}.
\end{equation}}
Consequently,
$1-\gamma_1(n)
=\mathcal{O}\bigl((1-\beta_1^n)/n^5\bigr)$.
}

Now we prove \eqref{eqn:min_gradient_norm}. We will derive a lower bound for {\small$\frac{(\partial_{\alpha}f(x_k))^2}{n\sqrt{\Exk (v_{\alpha,k})}}-\left(\delta_{12}(\beta_2)+\frac{C_7}{\sqrt{k}}\right)\vert\partial_\alpha f(x_k)\vert$}, where $\delta_{12}(\beta_2)$ and $C_7$ are as defined earlier in \eqref{eq_delta_1_beta2}.  Note that $\delta_{12}(\beta_2)$ approaches 0 as $\beta_2$ approaches 1, and $\frac{C_7}{\sqrt{k}}$ vanishes as $k$ increases.

We consider two cases:
\paragraph{Case (a): When {\small$\vert\partial_\alpha f(x_k)\vert \geq \frac{4\sqrt{2}\Delta_k}{(1-\beta_2)\sqrt{D_1n}d}$}:}

\brrevise{
Let $V_0:=\|v_0\|_\infty$. }

In this case, we have

\brrevise{
{\scriptsize
\begin{eqnarray}
  \frac{(\partial_\alpha f(x_k))^2}{\sqrt{\Exk(v_{\alpha,k})}}
   &=&  \frac{(\partial_\alpha f(x_k))^2}
  {\sqrt{  \beta_2^k v_{\alpha,0}
  +(1-\beta_2)\left( \frac{1}{n}\sum_{i=0}^{n-1} (\partial_{\alpha}f_i(x_k))^2
  +\sum_{j=1}^{k-1}  \left(\partial_{\alpha}  f_{\tau_{k-j}}(x_{k-j})\right)^2\beta_2^j
  \right)}}\nonumber\\
  &\overset{\text{\eqref{lemma_fi_f_2nd}}}{\geq}&
  \frac{(\partial_\alpha f(x_k))^2}  {\sqrt{  \beta_2^kV_0
  +(1-\beta_2)\left(\begin{aligned}
  &\left(  \sqrt{  \left|\partial_{\alpha}f(x_k)\right|^2
  +\frac{D_0}{D_1d}}  \sqrt{D_1n}d  \right)^2\\
  &+\sum_{j=1}^{k-1}
  \left(  \sqrt{  \left|\partial_{\alpha}f(x_k)\right|^2
  +\frac{D_0}{D_1d}}
  \sqrt{D_1n}d
  +\sum_{t=1}^j\Delta_{k-t}
  \right)^2\beta_2^j
  \end{aligned}\right)}}\nonumber\\
  &\geq&
  \frac{(\partial_\alpha f(x_k))^2}
  {\sqrt{
  \beta_2^kV_0
  +(1-\beta_2)
  \left(  \sum_{j=0}^{\infty}\beta_2^j
  \left(\begin{aligned}
  &\left(  \left|\partial_{\alpha}f(x_k)\right|^2
  +\frac{D_0}{D_1d}  \right)D_1nd^2\\
  &+4\sqrt{2}j\Delta_k  \sqrt{  \left|\partial_{\alpha}f(x_k)\right|^2
  +\frac{D_0}{D_1d}}  \sqrt{D_1n}d
  +8j^2\Delta_k^2
  \end{aligned}\right)
  \right)}}\nonumber\\
  &\overset{\text{Lemma \ref{lemma_beta}}}{\geq}&
  \frac{(\partial_\alpha f(x_k))^2}
  {\sqrt{
  \beta_2^kV_0
  +D_1nd^2
  \left(\begin{aligned}
  &\left|\partial_{\alpha}f(x_k)\right|^2+\frac{D_0}{D_1d}\\
  &+\sqrt{\left|\partial_{\alpha}f(x_k)\right|^2+\frac{D_0}{D_1d}}
  \frac{4\sqrt{2}\Delta_k}{(1-\beta_2)\sqrt{D_1n}d}\\
  &+\frac{16\Delta_k^2}{(1-\beta_2)^2D_1nd^2}
  \end{aligned}\right)}}.\label{eq:case_a}
\end{eqnarray}
}
}

We consider two sub-cases according to the relative size of $(\partial_\alpha f(x_k))^2$:

\paragraph{Sub-case 1: When $(\partial_\alpha f(x_k))^2 \leq \frac{D_0}{D_1d}$:}

\brrevise{
By \eqref{eq:case_a} and
$(\partial_\alpha f(x_k))^2\leq\frac{D_0}{D_1d}$, when $k$ is large enough such that $
\beta_2^kV_0\leq D_0nd$
we have
{\small \begin{eqnarray}
\Exk(v_{\alpha,k})
&\leq&
\beta_2^kV_0
+D_1nd^2
\left(
|\partial_\alpha f(x_k)|^2
+\frac{D_0}{D_1d}
+|\partial_\alpha f(x_k)|
\sqrt{
|\partial_\alpha f(x_k)|^2
+\frac{D_0}{D_1d}}
+\frac{|\partial_\alpha f(x_k)|^2}{2}
\right)\nonumber\\
&\leq&
D_0nd+
\left(\frac{5}{2}+\sqrt{2}\right)D_0nd \nonumber\\
&=&
\left(\frac{7}{2}+\sqrt{2}\right)D_0nd
\leq 5D_0nd.
\label{eq:subcase1}
\end{eqnarray}}
}

Then by \eqref{eq:case_a} and \eqref{eq:subcase1}, we have
{\small \begin{eqnarray*}
\frac{(\partial_{\alpha} f(x_k))^2}{n\sqrt{\Exk(v_{\alpha,k})}} - \left(\delta_{12}(\beta_2)+\frac{C_7}{\sqrt{k}}\right)\vert\partial_\alpha f(x_k)\vert
&\geq& \frac{(\partial_{\alpha} f(x_k))^2}{n\sqrt{5D_0 nd}} - \left(\delta_{12}(\beta_2)+\frac{C_7}{\sqrt{k}}\right)\sqrt{\frac{D_0}{D_1d}}. \\
\end{eqnarray*}}

\paragraph{Sub-case 2: When $(\partial_\alpha f(x_k))^2 > \frac{D_0}{D_1d}$:}

\brrevise{
Similarly, since
$(\partial_\alpha f(x_k))^2>\frac{D_0}{D_1d}$, we have
{\small \begin{eqnarray}
\Exk(v_{\alpha,k})
&\leq&
\beta_2^kV_0
+D_1nd^2
\left(
|\partial_\alpha f(x_k)|^2
+\frac{D_0}{D_1d}
+|\partial_\alpha f(x_k)|
\sqrt{
|\partial_\alpha f(x_k)|^2
+\frac{D_0}{D_1d}}
+\frac{|\partial_\alpha f(x_k)|^2}{2}
\right) \nonumber\\
&\leq&
D_1nd^2|\partial_\alpha f(x_k)|^2
+
\left(\frac{5}{2}+\sqrt{2}\right)
D_1nd^2|\partial_\alpha f(x_k)|^2 \nonumber\\
&=&
\left(\frac{7}{2}+\sqrt{2}\right)
D_1nd^2|\partial_\alpha f(x_k)|^2 \nonumber\\
&\leq&
5D_1nd^2|\partial_\alpha f(x_k)|^2.
\label{eq:subcase2}
\end{eqnarray}}
}

Then by \eqref{eq:case_a} and \eqref{eq:subcase2}, we have
{\small \begin{eqnarray*}
\frac{(\partial_{\alpha} f(x_k))^2}{n\sqrt{\Exk(v_{\alpha,k})}} - \left(\delta_{12}(\beta_2)+\frac{C_7}{\sqrt{k}}\right)\vert\partial_\alpha f(x_k)\vert
&\geq& \frac{\vert\partial_\alpha f(x_k)\vert}{n\sqrt{5D_1 n d^2}} - \left(\delta_{12}(\beta_2)+\frac{C_7}{\sqrt{k}}\right)\vert\partial_\alpha f(x_k)\vert \\
&=& \vert\partial_\alpha f(x_k)\vert \left(\frac{1}{n d \sqrt{5D_1 n}} - \delta_{12}(\beta_2) - \frac{C_7}{\sqrt{k}}\right).
\end{eqnarray*}}

\brrevise{
Combining these sub-cases, when 
$\beta_2$ is large enough such that
$\delta_{12}(\beta_2)\leq
\frac{1}{4nd\sqrt{5D_1n}}$,
and $k$ is large enough such that
$\frac{C_7}{\sqrt{k}}\leq
\frac{1}{4nd\sqrt{5D_1n}}$ and $
\beta_2^kV_0\leq D_0nd$, we have
}

{\scriptsize\begin{eqnarray*}
\frac{(\partial_{\alpha} f(x_k))^2}{n\sqrt{\Exk(v_{\alpha,k})}} - \left(\delta_{12}(\beta_2)+\frac{C_7}{\sqrt{k}}\right)\vert\partial_\alpha f(x_k)\vert
&\geq& \min\left\{\frac{(\partial_{\alpha} f(x_k))^2}{n\sqrt{5D_0 n d}}, \frac{\vert\partial_\alpha f(x_k)\vert}{2n d \sqrt{5D_1 n}}\right\} - \delta_{12}(\beta_2)\sqrt{\frac{D_0}{D_1d}} - \frac{C_7}{\sqrt{k}}\sqrt{\frac{D_0}{D_1d}}.
\end{eqnarray*}}

\paragraph{Case (b): When $\vert\partial_\alpha f(x_k)\vert < \frac{4\sqrt{2}\Delta_k}{(1-\beta_2)\sqrt{D_1n}d}$:}

In this case:
{\small \begin{eqnarray*}
\frac{(\partial_{\alpha} f(x_k))^2}{n\sqrt{\Exk(v_{\alpha,k})}} - \left(\delta_{12}(\beta_2)+\frac{C_7}{\sqrt{k}}\right)\vert\partial_\alpha f(x_k)\vert
&\geq& \frac{\vert\partial_\alpha f(x_k)\vert}{2n d \sqrt{5D_1 n}}  - \frac{2\sqrt{2}\Delta_1}{(1-\beta_2)D_1n^2d^2 \sqrt{5}}\frac{1}{\sqrt{k}} \\
&&- \frac{8\sqrt{2}\Delta_1\delta_{12}(\beta_2)}{(1-\beta_2)\sqrt{D_1 n }d}\frac{1}{\sqrt{k}} - \frac{8\sqrt{2}\Delta_1C_7}{(1-\beta_2)\sqrt{D_1 n }d}\frac{1}{k} \\
&\geq& \min\left\{\frac{(\partial_{\alpha} f(x_k))^2}{n\sqrt{5D_0 n d}}, \frac{\vert\partial_\alpha f(x_k)\vert}{2n d \sqrt{5D_1 n}}\right\} - \frac{C_{10}}{\sqrt{k}},
\end{eqnarray*}}
where we define:

{\small \begin{equation}\label{def_C_10}
    C_{10} = \frac{2\sqrt{2}\Delta_1}{(1-\beta_2)D_1n^2d^2 \sqrt{5}} + \frac{8\sqrt{2}\Delta_1(\delta_{12}(\beta_2)+C_7)}{(1-\beta_2)\sqrt{D_1 n }d}.
\end{equation}}

Taking expectations and combining cases (a) and (b), we obtain
{\small \begin{align*}
\Ex\left[\frac{(\partial_{\alpha}f(x_k))^2}{n\sqrt{\Exk(v_{\alpha,k})}}
-\left(\delta_{12}(\beta_2)+\frac{C_7}{\sqrt{k}}\right)\vert\partial_\alpha f(x_k)\vert\right]
&\geq{}\\
&\Ex\left[\min\left\{\frac{(\partial_{\alpha} f(x_k))^2}{n\sqrt{5D_0nd}},
\frac{\vert\partial_\alpha f(x_k)\vert}{2nd\sqrt{5D_1n}}\right\}\right]\\
&\quad-\delta_{12}(\beta_2)\sqrt{\frac{D_0}{D_1d}}-\frac{C_{11}}{\sqrt{k}},
\end{align*}}
where 
{\small \begin{equation}\label{def_C_11}
    C_{11}=C_{10}+ C_7\sqrt{\frac{D_0}{D_1d}}, 
\end{equation}}
and $C_{10}$, $C_7$ are defined in \eqref{def_C_10}, \eqref{def_C_7}, respectively.

Using the relation $d(\partial_{\alpha}f(x_k))^2\geq \|\nabla f(x_k)\|^2_2$ and $d|\partial_{\alpha}f(x_k)|\geq \|\nabla f(x_k)\|_1\geq\|\nabla f(x_k)\|_2$, we have
{\small\begin{align*}
\Ex\left[\frac{(\partial_{\alpha}f(x_k))^2}{n\sqrt{\Exk(v_{\alpha,k})}}
-\left(\delta_{12}(\beta_2)+\frac{C_7}{\sqrt{k}}\right)\vert\partial_\alpha f(x_k)\vert\right]
&\geq{}\\
&\Ex\left[\min\left\{\frac{\|\nabla f(x_k)\|_2^2}{nd\sqrt{5D_0nd}},
\frac{\|\nabla f(x_k)\|_2}{2nd^2\sqrt{5D_1n}}\right\}\right]\\
&\quad-\delta_{12}(\beta_2)\sqrt{\frac{D_0}{D_1d}}-\frac{C_{11}}{\sqrt{k}},
\end{align*}}
where $C_{11}$ and $\delta_{12}(\beta_2)$ are defined in \eqref{def_C_11} and \eqref{eq_delta_1_beta2}, respectively.
This completes the proof for \eqref{eqn:min_gradient_norm}.

\brrevise{
The argument above proves the interior branch $D_0>0$, $D_1>0$ of
Lemma \ref{lemma_descent_WR}. The boundary branches, in which at least
one of $D_0$ and $D_1$ is zero, are proved separately in Appendix
\ref{appendix:D0_D1_boundary}.}

\brrevise{For the branch $D_1=0<D_0$, Assumption \ref{assum2} reduces to the
bounded second-moment condition. The corresponding proof uses the
bounded-gradient estimates derived in Appendix
\ref{appendix:D0_D1_boundary}. In this branch, it suffices that $k$
satisfies
{\small \begin{equation}
\label{eq:D1_zero_k_conditions}
k\geq
\left\lceil\frac{\log(1/4)}{\log\beta_2}\right\rceil+2n+1,
\qquad
\beta_1^{k-n}\leq\frac{1}{\sqrt{k}},
\qquad
\beta_2^k\|v_0\|_\infty\leq D_0.
\end{equation}}}

The constants in \eqref{eq:descent_all_parameter_cases} are defined
casewise by
{\small \begin{equation}
\label{def:casewise_descent_constants}
\bigl(C_{D_0,D_1},\delta_{D_0,D_1}(\beta_2)\bigr)
:=
\begin{cases}
\bigl(C_{\mathrm{int}},\delta_{\mathrm{int}}(\beta_2)\bigr),
& D_0>0,\ D_1>0,\\
\bigl(C^{(0)},0\bigr),
& D_0=0<D_1,\\
\bigl(C_{\mathrm{bd}},\delta_{\mathrm{bd}}(\beta_2)\bigr),
& D_1=0<D_0,\\
(0,0),
& D_0=D_1=0.
\end{cases}
\end{equation}}
Here
$\bigl(C_{\mathrm{int}},\delta_{\mathrm{int}}(\beta_2)\bigr)$ is defined
in \eqref{def_delta_C}, $C^{(0)}$ is defined in
\eqref{def:D0_zero_constant}, and
$\bigl(C_{\mathrm{bd}},\delta_{\mathrm{bd}}(\beta_2)\bigr)$ is defined
in \eqref{def:D1_zero_constants}. \brrevise{ We remark that those constant term are independent of time ($k$ and $T$) and are fixed up to $\beta_2$}

Together with the boundary arguments in Appendix
\ref{appendix:D0_D1_boundary}, this completes the proof of Lemma
\ref{lemma_descent_WR}.
\qed

\section{Proof of the Convergence Results in Theorem \ref{thm_wr}}
\label{appendix_main_body_wr}

Now we present the proof of the interior case for Theorem \ref{thm_wr}. Based on the descent Lemma and the telescoping sum, by Lemma \ref{lemma_descent_WR} we have the following relation. Let $T_{\min}$ be sufficiently large and satisfy the conditions of Lemma \ref{lemma_descent_WR}, we have:
{\footnotesize
\begin{eqnarray}
\sum_{k=T_{\min}}^{T} \eta_k \left\{\begin{aligned}
&\Ex \left[\min\left\{\frac{\|\nabla f(x_k)\|_2^2}{nd\sqrt{5D_0 n d}}, \frac{\|\nabla f(x_k)\|_2}{2nd^2\sqrt{5D_1 n}}\right\} \right]\\
&- \delta_{\mathrm{int}}(\beta_2)\sqrt{D_0} - \frac{C_{\mathrm{int}}}{\sqrt{k}}
\end{aligned}\right\}
&\overset{}{\leq}& \sum_{k=T_{\min}}^{T} \Ex\langle \nabla f(z_k), z_{k}-z_{k+1}\rangle\nonumber\\
&\overset{}{\leq}& \Ex f(z_{T_{\min}}) - \Ex f(z_{T+1}) + \sum_{k=T_{\min}}^{T} \frac{L}{2}\Ex\|z_{k+1}-z_k\|_2^2\nonumber\\
&\leq& \Ex f(z_{T_{\min}}) - f^* + \sum_{k=T_{\min}}^{T} \frac{L}{2}\Ex\|z_{k+1}-z_k\|_2^2.\label{eq:descent_unsimplified}
\end{eqnarray}
}
where $C$ is defined in \eqref{def_delta_C}.

From \eqref{eq:m_over_v} and the relation $z_{k+1}-z_k = \frac{x_{k+1}-x_{k}-\beta_1^n(x_{k-n+1}-x_{k-n})}{1-\beta_1^n}$, we have
{\small \begin{eqnarray}
\Vert z_{k+1} - z_k \Vert_2^2 
&\le& \frac{1}{(1-\beta_1^n)^2} \left( \Vert x_{k+1}-x_k \Vert_2 + \beta_1^n \Vert x_{k-n+1}-x_{k-n} \Vert_2 \right)^2 \nonumber\\
&\le& \frac{1}{(1-\beta_1^n)^2} \cdot \left( \frac{d(1-\beta_1)}{\sqrt{1-\beta_2}(1-\frac{\beta_1}{\sqrt{\beta_2}})} \right)^2 \cdot \left(\eta_k^2  + \beta_1^{2n} \eta_{k-n}^2+2\beta_1^{n} \eta_{k}\eta_{k-n} \right) \nonumber\\
&\le& \frac{\eta_k^2}{(1-\beta_1^n)^2}\cdot  \frac{d^2(1-\beta_1)^2}{(1-\beta_2)\left(1-\frac{\beta_1}{\sqrt{\beta_2}}\right)^2} \cdot \left( 1 + \beta_1^{2n} (n+1) + 2\beta_1^{n}\sqrt{n+1} \right) \nonumber\\
&=& \frac{1}{k} \cdot \frac{d^2(1-\beta_1)^2 \eta_0^2\left( 1 + \beta_1^{2n} (n+1) + 2\beta_1^{n}\sqrt{n+1} \right) }{(1-\beta_1^n)^2 (1-\beta_2) \left(1-\frac{\beta_1}{\sqrt{\beta_2}}\right)^2}:=\frac{C_z}{k},\label{eq:z_k+1-z_k}
\end{eqnarray}}
where the second last inequality holds because for $k\geq n+1$, $\frac{1}{\sqrt{k-n}}\leq \frac{\sqrt{n+1}}{\sqrt{k}}$.

Plug \eqref{eq:z_k+1-z_k} into \eqref{eq:descent_unsimplified}, we get
{\small \begin{eqnarray}
\sum_{k=T_{\min}}^{T} \frac{\eta_0}{\sqrt{k}} 
\left\{ 
\Ex \left[\min\left\{\frac{\|\nabla f(x_k)\|_2^2}{nd\sqrt{5D_0 n d}}, \frac{\|\nabla f(x_k)\|_2}{2nd^2\sqrt{5D_1 n}}\right\}\right] 
- \delta_{\mathrm{int}}(\beta_2)\sqrt{D_0} 
- \frac{C_{\mathrm{int}}}{\sqrt{k}}
\right\}
\nonumber\\
\leq \Ex f(z_{T_{\min}}) - f^* + \sum_{k=T_{\min}}^{T} \frac{L}{2}\cdot\frac{C_z}{k}.
\label{eq:thm_wr_telescoping_bound}
\end{eqnarray}}

Using $\sum_{k=T_{\min}}^{T} \frac{1}{\sqrt{k}} \geq 2(\sqrt{T} - \sqrt{T_{\min}})$ and $\sum_{k=T_{\min}}^{T} \frac{1}{k} \leq \log\frac{T+1}{T_{\min}}$, we have
{\small \begin{gather*}
\min_{k \in [T_{\min},T]} \Ex \left[
    \min\left\{\frac{\|\nabla f(x_k)\|_2^2}{nd\sqrt{5D_0 n d}},
    \frac{\|\nabla f(x_k)\|_2}{2nd^2\sqrt{5D_1 n}}\right\}
\right]
\\
\leq \delta_{\mathrm{int}}(\beta_2)\sqrt{D_0}
 + \frac{\left(C_{\mathrm{int}}\eta_0 + \frac{LC_z}{2}\right)
      \log\frac{T+1}{T_{\min}} + \Ex f(z_{T_{\min}}) - f^*}
    {\sum_{k=T_{\min}}^{T} \frac{\eta_0}{\sqrt{k}}}
\\
\leq \left(C_{\mathrm{int}}\eta_0 + \frac{LC_z}{2}\right)
 \frac{\log T}{\sqrt{T}} + \delta_{\mathrm{int}}(\beta_2)\sqrt{D_0}.
\end{gather*}}
for sufficiently large $T$.

After appropriate scaling, we have the following relation.
{\scriptsize
\begin{eqnarray}
\min_{k \in [1, T]} \Ex\left[\min  \left\{ \frac{\|\nabla f(x_{k})\|_2^2 }{\sqrt{D_{0}} }  , \frac{\|\nabla f(x_k)\|_2}{2\sqrt{d D_1 }}\right\}  \right]
&\leq&  n^{1.5} d^{1.5}\sqrt{5}\left(C_{\mathrm{int}}\eta_0 + \frac{LC_z}{2}\right) \frac{\log T }{\sqrt{T}}  + n^{1.5} d^{1.5}\sqrt{5} \delta_{\mathrm{int}}(\beta_2)\sqrt{D_0},\nonumber\\
&=&\mathcal{O}\left(\frac{C_{\mathrm{int}}\log T }{\sqrt{T}} \right)  + \mathcal{O}(\delta_{\mathrm{in
t}}(\beta_2)\sqrt{D_0})\\
&=& \mathcal{O}\left(\frac{\log T }{\brrevise{(1-\beta_2)^2}\sqrt{T}} \right)
  +\mathcal{O}(\brrevise{(1-\beta_2)}\sqrt{D_0}).
\label{eq_logT_sqrtT}
\end{eqnarray}
}
because it can be shown that the constant $C_{\mathrm{int}}$ is $\mathcal{O}(\frac{1}{(1-\beta_2)^2})$ and the constant $\delta_{\mathrm{int}}$ is $\mathcal{O}(\frac{1}{(1-\beta_2)})$.

\brrevise{Together with the boundary arguments in Appendix
\ref{appendix:D0_D1_boundary}}, we concludes the proof for Theorem \ref{thm_wr}. 
\qed

\section{\yushunrevisediverge{Proof of the Divergence Results in Theorem  \ref{thm_diverge} }}
\label{appendix:thm_diverge}

We now present the proof for Theorem \ref{thm_diverge}. We denote $x_{k,i}$ as the value of $x$ at the $k$-th outer loop and
$i$-th inner loop. We consider cyclic update ordering where the component
functions are sampled in the order of $f_0(x),f_1(x),\ldots,f_{n-1}(x)$ within
the $k$-th outer loop. Here, we only present the proof for Adam under cyclic update
ordering, which helps reveal the key insights. The proof for random sampling
follows the same procedure and gives a similar conclusion, which we omit for
brevity. 

We consider
$a=1$ in the proof below. The proof for other choices of  $a$ follows the same procedure,  and we omit for brevity.  For $a=1$, consider the following component functions:
{\small \begin{equation}
 f_0(x)=\frac n2x^2,
 \qquad
 f_i(x)=-\frac12x^2,
 \quad i=1,\ldots,n-1.
 \label{counterexample1_a_equals_1}
\end{equation}}
Summing up all the $f_i(x)$, we have:
{\small \begin{equation}
 f(x)=\frac{x^2}{2n},
 \qquad
 \nabla f(x)=\frac{x}{n}.
 \label{eq:gradient_objective_c1c2}
\end{equation}}

Firstly, we prove the following claim:

\begin{center}
  {\bf Claim:}  For any $n\geq3$, there exists a region in the
  $(\beta_1,\beta_2)$ plane described by ${\bf (C1)}$ and ${\bf (C2)}$
  defined later in \eqref{diverge_c1} and \eqref{diverge_c2}, such that, as long as Adam uses a sufficiently small initial stepsize dependent on $(\beta_1,\beta_2)$ and initial $x_{1,0}$ (shown in  ${\bf (C3)}$  in \eqref{diverge_c3}), we have
  $|x_{k+1,0}|>|x_{k,0}|>0$ for every outer loop. Consequently,  
  {\small $$|x_{k,i}|, |\nabla f(x_{k,i})|, |f(x_{k,i})|  \to+\infty.$$}
\end{center}

Now we prove the claim. We will consider a positive initial point $x_{1,0}>0$ and show that $x_{k,i} \rightarrow + \infty$. The proof for  $x_{k,i} \rightarrow - \infty$ under negative initial point $x_{1,0}<0$  follows the same procedure due to the symmetry of $f(x)$.

Consider Adam  with stepsize $\eta_k=\eta_0/\sqrt{k}$ and $\epsilon=0$. The first- and second-order momentum are
{\small \begin{equation*}
 m_{k,i}=\beta_1m_{k,i-1}+(1-\beta_1)\nabla f_i(x_{k,i}),
 \qquad
 v_{k,i}=\beta_2v_{k,i-1}+(1-\beta_2)\nabla f_i(x_{k,i})^2.
\end{equation*}}

As long as the
trajectory stays on the right side of the minimizer, the component gradients are
{\small \begin{equation*}
 \nabla f_0(x_{k,0})=nx_{k,0},
 \qquad
 \nabla f_i(x_{k,i})=-x_{k,i},\quad i=1,\ldots,n-1.
\end{equation*}}

The update of the variable $x$ is therefore
{\small \begin{eqnarray}
 x_{k,1}&=&x_{k,0}+\delta_{k,0}, \nonumber\\
 \delta_{k,0}&=&\frac{\eta_0}{\sqrt{k}}
 \left(\frac{-n(1-\beta_1)x_{k,0}-\beta_1m_{k-1,n-1}}
 {\sqrt{(1-\beta_2)n^2x_{k,0}^2+\beta_2v_{k-1,n-1}}}\right),
 \nonumber\\
 x_{k,i+1}&=&x_{k,i}+\delta_{k,i},\qquad i=1,\ldots,n-1,
 \nonumber %
\end{eqnarray}}
where
{\small
\begin{equation}
 \delta_{k,i}=\frac{\eta_0}{\sqrt{k}}
 \left(
 \frac{(1-\beta_1)\displaystyle\sum_{j=1}^{i}
 \beta_1^{i-j}x_{k,j}
 -\beta_1^{i+1}m_{k-1,n-1}
 -n(1-\beta_1)\beta_1^i x_{k,0}}
 {\sqrt{(1-\beta_2)x_{k,i}^2+\beta_2v_{k,i-1}}}
 \right).
 \label{eq:gradient_delta_c1c2}
\end{equation}}

We decompose the total movement of $x$ in the $k$-th outer loop into three
terms as follows:
{\scriptsize
\begin{flalign*}
&x_{k+1,0}-x_{k,0}=\sum_{i=0}^{n-1}\delta_{k,i} &&\\
&\overset{\eqref{eq:gradient_delta_c1c2}}{=}
\frac{\eta_0}{\sqrt{k}}
 \underbrace{\left(\begin{aligned}
 &-\frac{\beta_1m_{k-1,n-1}}
 {\sqrt{(1-\beta_2)n^2x_{k,0}^2+\beta_2v_{k-1,n-1}}}
 -\frac{\beta_1^2m_{k-1,n-1}}
 {\sqrt{(1-\beta_2)x_{k,1}^2+\beta_2v_{k,0}}}\\
 &-\cdots
 -\frac{\beta_1^nm_{k-1,n-1}}
 {\sqrt{(1-\beta_2)x_{k,n-1}^2+\beta_2v_{k,n-2}}}
 \end{aligned}\right)}_{\mathrm{(I)}} &&\\
&\qquad+\frac{\eta_0}{\sqrt{k}}
 \underbrace{\left(\begin{aligned}
 &\frac{(1-\beta_1)x_{k,1}}
 {\sqrt{(1-\beta_2)x_{k,1}^2+\beta_2v_{k,0}}}
 +\frac{(1-\beta_1)(x_{k,2}+\beta_1x_{k,1})}
 {\sqrt{(1-\beta_2)x_{k,2}^2+\beta_2v_{k,1}}}\\
 &+\cdots
 +\qquad\qquad
\frac{(1-\beta_1)\displaystyle\sum_{j=1}^{n-1}
 \beta_1^{n-1-j}x_{k,j}}
 {\sqrt{(1-\beta_2)x_{k,n-1}^2+\beta_2v_{k,n-2}}}
 \end{aligned}\right)}_{\mathrm{(II)}} &&\\
&\qquad+\frac{\eta_0}{\sqrt{k}}
 \underbrace{\left(\begin{aligned}
 &-\frac{n(1-\beta_1)x_{k,0}}
 {\sqrt{(1-\beta_2)n^2x_{k,0}^2+\beta_2v_{k-1,n-1}}}
 -\frac{n(1-\beta_1)\beta_1x_{k,0}}
 {\sqrt{(1-\beta_2)x_{k,1}^2+\beta_2v_{k,0}}}\\
 &-\cdots
 -\frac{n(1-\beta_1)\beta_1^{n-1}x_{k,0}}
 {\sqrt{(1-\beta_2)x_{k,n-1}^2+\beta_2v_{k,n-2}}}
 \end{aligned}\right)}_{\mathrm{(III)}}. &&
\end{flalign*}
}

In the following, we will show that for some $\beta_1$ and $\beta_2$: term ${\mathrm{(I)}}$ is nonnegative,
term ${\mathrm{(II)}}$ is positive, and term ${\mathrm{(III)}}$ is
negative. In addition, ${\mathrm{(I)}}$ and ${\mathrm{(II)}}$ outweigh
${\mathrm{(III)}}$, causing $x_{k,0}$ to diverge.

We start with  term ${\mathrm{(I)}}$.  We  will show  that the $m_{k,n-1}$ stays nonpositive under some $\beta_1$. Expanding
the recursion of $m_{k,n-1}$ gives
{\small \begin{eqnarray*}
 m_{k,n-1}
 &=&\beta_1^nm_{k-1,n-1}
 +n(1-\beta_1)\beta_1^{n-1}x_{k,0}
 -(1-\beta_1)\sum_{j=1}^{n-1}
 \beta_1^{n-1-j}x_{k,j}.
\end{eqnarray*}}

We now introduce an auxiliary constant $\rho\in(0,1)$. Its value will be chosen sufficiently small after
$(\beta_1,\beta_2)$ is fixed. For the moment, we temporarily assume
that the iterates in the current outer loop stay in the relative band
{\small \begin{equation}
 (1-\rho)x_{k,0}\leq x_{k,i}\leq(1+\rho)x_{k,0},
 \qquad i=0,\ldots,n.
 \label{eq:relative_band_c1c2}
\end{equation}}
Note that the condition \eqref{eq:relative_band_c1c2} can always be achieved by a sufficiently small initial stepsize $\eta_0$. To meet \eqref{eq:relative_band_c1c2}, the choice of $\eta_0$ would depend on $\rho$. As we will choose $\rho$ based on $(\beta_1,\beta_2)$, $\eta_0$ will also depend on $(\beta_1,\beta_2)$. The specific requirement of $\eta_0$  will be verified later in  the
stepsize condition ${\bf (C3)}$ in \eqref{diverge_c3}.
For the moment, we temporarily assume that the iterates in the current
outer loop stay in the relative band \eqref{eq:relative_band_c1c2}.

When the iterates in the current outer loop satisfy
$x_{k,j}\geq(1-\rho)x_{k,0}$, we have
{\small \begin{eqnarray*}
 m_{k,n-1}
 &\leq&\beta_1^nm_{k-1,n-1}
 +x_{k,0}\left[n(1-\beta_1)\beta_1^{n-1}
 -(1-\rho)(1-\beta_1^{n-1})\right].
\end{eqnarray*}}

Now we require the fixed auxiliary constant $\rho$ to be sufficiently small so
that
{\small \begin{equation}
 (1-\rho)(1-\beta_1^{n-1})
 >n(1-\beta_1)\beta_1^{n-1}.
 \label{eq:c2_rho_c1c2}
\end{equation}}
Consequently, $m_{k,n-1}<0$ whenever $m_{k-1,n-1}\leq0$. Since
$m_{1,-1}=0$, this proves inductively that
$m_{k-1,n-1}\leq0$, and hence ${\mathrm{(I)}}\geq0$.

Now we derive a lower bound for ${\mathrm{(II)}}$.
Under the temporary band assumption, the recursion on $v_{k,i }$ gives
{\small \[
 v_{k,-1}\leq n^2(1+\rho)^2x_{k,0}^2.
\]}
The second-order momentum within the current epoch therefore satisfies
{\small \begin{equation*}
 v_{k,i}\leq(1+\rho)^2x_{k,0}^2
 \left[1+(n^2-1)\beta_2^{i}\right],
 \qquad i=0,\ldots,n-1,
\end{equation*}}
where we use $v_{k,-1}\leq n^2(1+\rho)^2x_{k,0}^2$ and the fact that the
first gradient in the epoch has magnitude $nx_{k,0}$. Therefore,
{\small \begin{eqnarray*}
 {\mathrm{(II)}}
 &\overset{\text{(a)}}{\geq}&\frac{1-\rho}{1+\rho}
 \sum_{i=1}^{n-1}
 \frac{1-\beta_1^i}
 {\sqrt{1+(n^2-1)\beta_2^i}}\\
 &>&\frac{1-\rho}{1+\rho}
 \left[
 \frac{1-\beta_1}{\sqrt{1+n^2\beta_2}}
 +\frac{1-\beta_1^2}{\sqrt{1+n^2\beta_2^2}}
 +\cdots
 +\frac{1-\beta_1^{n-1}}
 {\sqrt{1+n^2\beta_2^{n-1}}}
 \right].
\end{eqnarray*}}
The strict inequality follows from $n^2-1<n^2$ and
$0<\beta_2<1$. The first inequality (a) is due to the following relation:
{\small \begin{equation}
 (1-\beta_1)\sum_{j=1}^{i}\beta_1^{i-j}x_{k,j}
 \geq(1-\rho)x_{k,0}(1-\beta_1^i).
 \label{eq:II_numerator_c1c2}
\end{equation}}

We keep the summands with
$j\geq\log_{\beta_2}(0.1/n^2)$ and drop the rest. This is valid because $\beta_2^j n^2$ is small when $\beta_2$ is small and $j$ is close
to $n$.  With basic calculus, we have 
{\small \begin{eqnarray*}
 \sum_{i=1}^{n-1}
 \frac{1-\beta_1^i}{\sqrt{1+n^2\beta_2^i}}
 &\geq&
 \left(n-1-\min\left\{n-1,
 \log_{\beta_2}\left(\frac{1}{10n^2}\right)\right\}\right)\\
 &&\qquad\cdot
 \frac{1-\beta_1^{\min\left\{n-1,
 \log_{\beta_2}(1/(10n^2))\right\}}}
 {\sqrt{1+\max\{0.1,\beta_2^{n-1}n^2\}}}.
 \end{eqnarray*}}

Now we derive an upper bound for $|{\mathrm{(III)}}|$. The first denominator
is lower bounded by $\sqrt{1-\beta_2}\,nx_{k,0}$, while all later denominators
are lower bounded by
$\sqrt{1-\beta_2}(1-\rho)x_{k,0}$. Hence
{\small \begin{eqnarray*}
 |{\mathrm{(III)}}|
 &\leq&\frac{1-\beta_1}{\sqrt{1-\beta_2}}
 +\frac{n(1-\beta_1)}{(1-\rho)\sqrt{1-\beta_2}}
 \sum_{j=1}^{n-1}\beta_1^j\\
 &\leq&\frac{1-\beta_1}{\sqrt{1-\beta_2}}
 +\frac{n\beta_1}{(1-\rho)\sqrt{1-\beta_2}}.
\end{eqnarray*}}

Now we characterize when ${\mathrm{(I)}}$ and ${\mathrm{(II)}}$ would outweigh
${\mathrm{(III)}}$. We will first analyze the limiting case when $\rho=0$ and then extend to $\rho >0$ using continuity. For $\rho = 0$, ${\mathrm{(II)}} \geq |{\mathrm{(III)}} |$ happens when the following ${\bf (C1)}$ condition holds: 
{\small \begin{equation}\label{diverge_c1}
    \begin{aligned}
    {\bf (C1):}\quad &
    \left(n-1-\min \left\{n-1,
    \log_{\beta_2} \left(\frac{1}{10n^2}\right)\right\}\right)
    \frac{1-\beta_1^{\min \left\{n-1,
    \log_{\beta_2} (1/(10n^2))\right\}}}
    {\sqrt{1+\max\{0.1,\beta_2^{n-1}n^2\}}}
    \\
    &\geq
    \frac{1-\beta_1}{\sqrt{1-\beta_2}}
    +\frac{n\beta_1}{\sqrt{1-\beta_2}}.
    \end{aligned}
\end{equation}}

By setting $\rho=0$ in \eqref{eq:c2_rho_c1c2}, the following  condition  guarantees that $m_{k,n-1} < 0$ whenever $m_{k-1,n-1}\leq0$, and  hence ${\mathrm{(I)}}\geq0$:
{\small \begin{equation}\label{diverge_c2}
      {\bf (C2):}\quad
      1-\beta_1^{n-1}
      > (1-\beta_1)\beta_1^{n-1}n.
\end{equation}}

Combining together,  ${\bf (C1)}$  and ${\bf (C2)}$ gives us  ${\mathrm{(I)}} + {\mathrm{(II)}} \geq |{\mathrm{(III)}} |$ when $\rho = 0$.
By continuity, the same strict inequality remains true
for all sufficiently small $\rho>0$.

The order of hyperparameter selection is as follows. We first take
$(\beta_1,\beta_2)$ in the interior of the region satisfying
${\bf (C1)}$ and ${\bf (C2)}$. We then choose and fix the sufficiently
small $\rho>0$ so that \eqref{eq:c2_rho_c1c2} holds and the strict
margin between ${\mathrm{(I)}} + {\mathrm{(II)}}$ and $|{\mathrm{(III)}}|$ is preserved.
After this $\rho$ has been fixed, recall $x_{1,0}>0$ and 
$m_{1,-1}=v_{1,-1}=0$, we choose the following initial stepsize, which makes relative band \eqref{eq:relative_band_c1c2} hold:  (note that $\rho$ is a function of $(\beta_1,\beta_2)$.)
{\small \begin{equation}\label{diverge_c3}
    {\bf (C3):}\quad
    0<\eta_0\leq
    \frac{\rho(1-\rho)x_{1,0}\sqrt{1-\beta_2}}
    {n^2(1+\rho)}.
\end{equation}}

We now verify that the choice of stepsize  in  ${\bf (C3)}$ keeps
all the iterates positive and keeps the relative band
\eqref{eq:relative_band_c1c2} valid.  We prove it by induction over the outer loops.
At the beginning of the $k$-th outer loop, assume that
$x_{1,0}\leq\cdots\leq x_{k,0}$, $m_{k-1,n-1}\leq0$, and all previously evaluated
gradients have been evaluated at points with
$0<x\leq(1+\rho)x_{k,0}$. Within the current outer loop, we use an
inner-loop induction. The base case $i=0$ holds with equality. Assuming
that the relative band holds through $x_{k,i}$, we prove that it also
holds for $x_{k,i+1}$. Since the
recursion on $m_{k,i}$ is an exponential average with nonnegative
coefficients, we have
{\small \begin{equation*}
 |m_{k,i}|\leq n(1+\rho)x_{k,0}.
\end{equation*}}
Moreover,
{\small \begin{equation*}
 v_{k,i}\geq(1-\beta_2)(1-\rho)^2x_{k,0}^2.
\end{equation*}}
Consequently,
{\small \begin{equation}
 |\delta_{k,i}|
 \leq\frac{\eta_0}{\sqrt{k}}
 \frac{n(1+\rho)}{(1-\rho)\sqrt{1-\beta_2}}.
 \label{eq:delta_upper_c1c2}
\end{equation}}
For every inner index $i$, telescoping over at most $n$ inner updates and
using ${\bf (C3)}$, we obtain
{\small \begin{eqnarray*}
 |x_{k,i+1}-x_{k,0}|
 &\overset{\eqref{eq:delta_upper_c1c2}}{\leq}&\frac{n^2\eta_0(1+\rho)}
 {(1-\rho)\sqrt{1-\beta_2}}\\
 &\leq&\rho x_{1,0}\leq\rho x_{k,0}.
\end{eqnarray*}}
This proves \eqref{eq:relative_band_c1c2} for the whole current epoch. In
particular, all iterates remain positive, or equivalently $x_{k,i}>0$. 

Using the lower bound for ${\mathrm{(I)}}$ and ${\mathrm{(II)}}$, the upper bound for
$|{\mathrm{(III)}}|$, and the strict margin ensured by ${\bf (C1)}$, there
exists a constant $c>0$, independent of $k$, such that
{\small \begin{equation}
 x_{k+1,0}-x_{k,0}
 \geq c\frac{\eta_0}{\sqrt{k}}>0,
 \label{eq:positive_drift_c1c2}
\end{equation}}
Thus $x_{k+1,0}>x_{k,0}$. Due to  \eqref{eq:c2_rho_c1c2}, we also have
{\small \begin{eqnarray*}
 m_{k,n-1}
 &\leq&x_{k,0}\left[n(1-\beta_1)\beta_1^{n-1}
 -(1-\rho)(1-\beta_1^{n-1})\right]<0,
\end{eqnarray*}}
Therefore, all the induction hypotheses hold for the next outer loop.

Finally, summing \eqref{eq:positive_drift_c1c2} gives
{\small \begin{equation*}
 x_{k,0}\geq x_{1,0}+c\eta_0\sum_{j=1}^{k-1}\frac1{\sqrt{j}}
 \longrightarrow+\infty.
\end{equation*}}
By \eqref{eq:relative_band_c1c2}, every inner iterate satisfies
{\small \begin{equation*}
 x_{k,i}\geq(1-\rho)x_{k,0}\longrightarrow+\infty.
\end{equation*}}
Since $x_{k,i}\to+\infty$, we have $x_{k,i}\to+\infty$. Finally,
using \eqref{eq:gradient_objective_c1c2},
{\small \begin{equation*}
 |\nabla f(x_{k,i})|=\frac{x_{k,i}}n\longrightarrow+\infty,
 \qquad
 f(x_{k,i})=\frac{x_{k,i}^2}{2n}\longrightarrow+\infty.
\end{equation*}}
This concludes the proof of the gradient divergence result.
\qed

With the help of \texttt{NumPy}, we visualize the region where
${\bf (C1)}$ and ${\bf (C2)}$ hold. The results are shown in
Figure~\ref{fig:counter_boundary}. We use orange color to indicate the
region where ${\bf (C1)}$ holds. White color is used for the counterpart.
As for ${\bf (C2)}$, we use the gray vertical line to indicate the line
where $(1-\beta_1^{n-1})=(1-\beta_1)\beta_1^{n-1}n$. Note that there are
two solutions to this equation: one solution is $\beta_1=1$ and the other
solution lies in $0<\beta_1<1$, which is why there are two vertical lines in
the figure. ${\bf (C2)}$ holds on the left-hand side of the left gray
vertical line.

\begin{figure}[t]
  \vspace{-0.5cm}
    \subfigure[$n=5$]{
      \begin{minipage}[t]{0.24\linewidth}
      \centering
      \includegraphics[width=\linewidth]{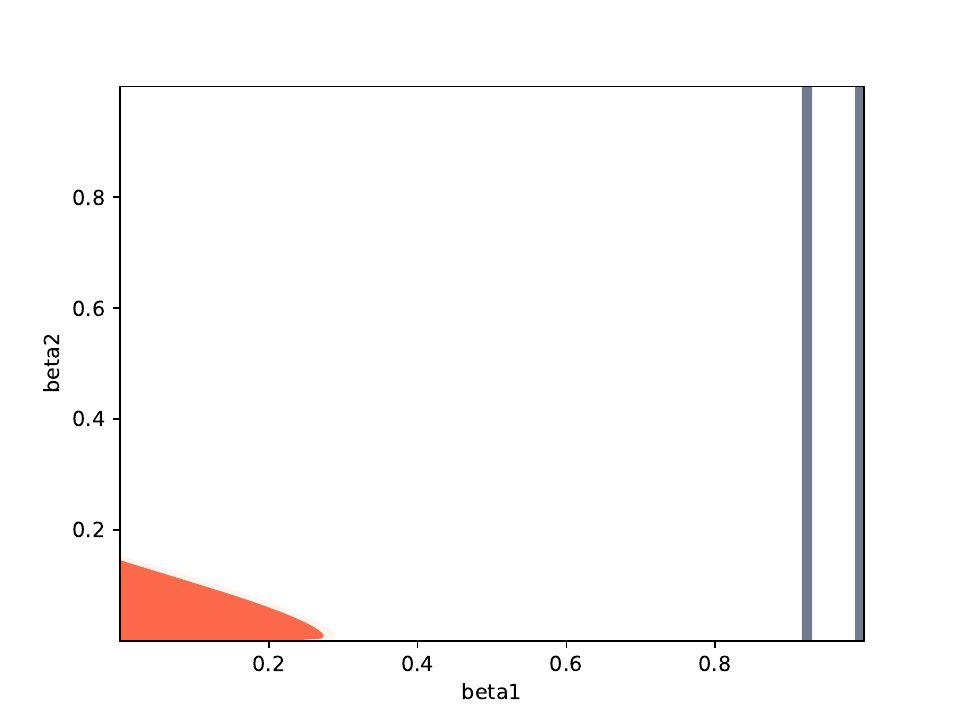}
      \end{minipage}%
      }%
      \subfigure[$n=10$]{
        \begin{minipage}[t]{0.24\linewidth}
        \centering
        \includegraphics[width=\linewidth]{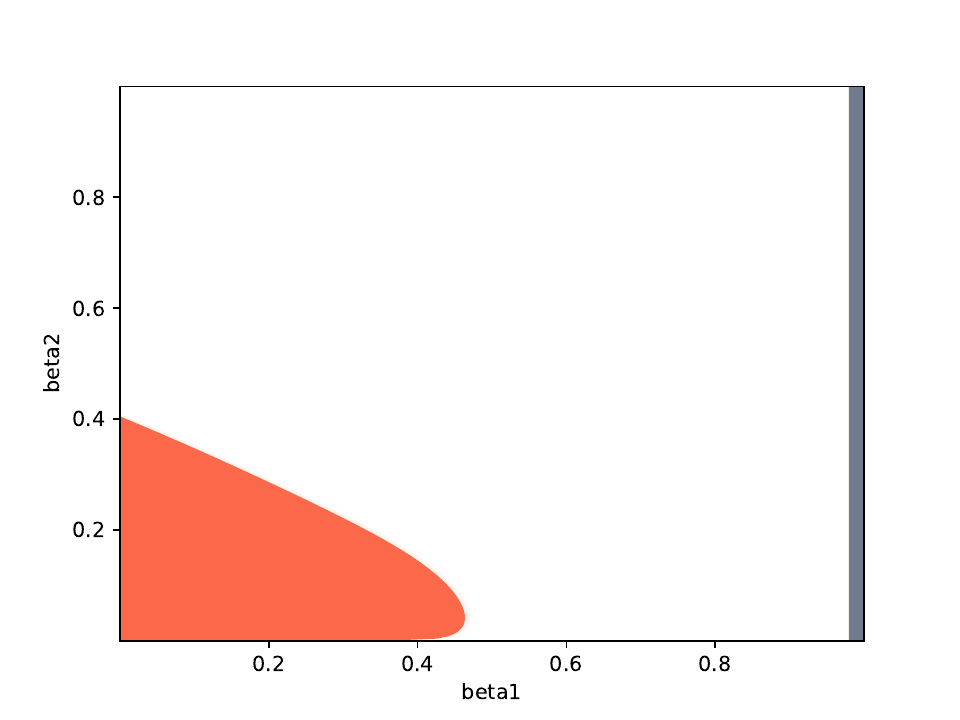}
        \end{minipage}%
        }%
        \subfigure[$n=50$]{
        \begin{minipage}[t]{0.24\linewidth}
        \centering
        \includegraphics[width=\linewidth]{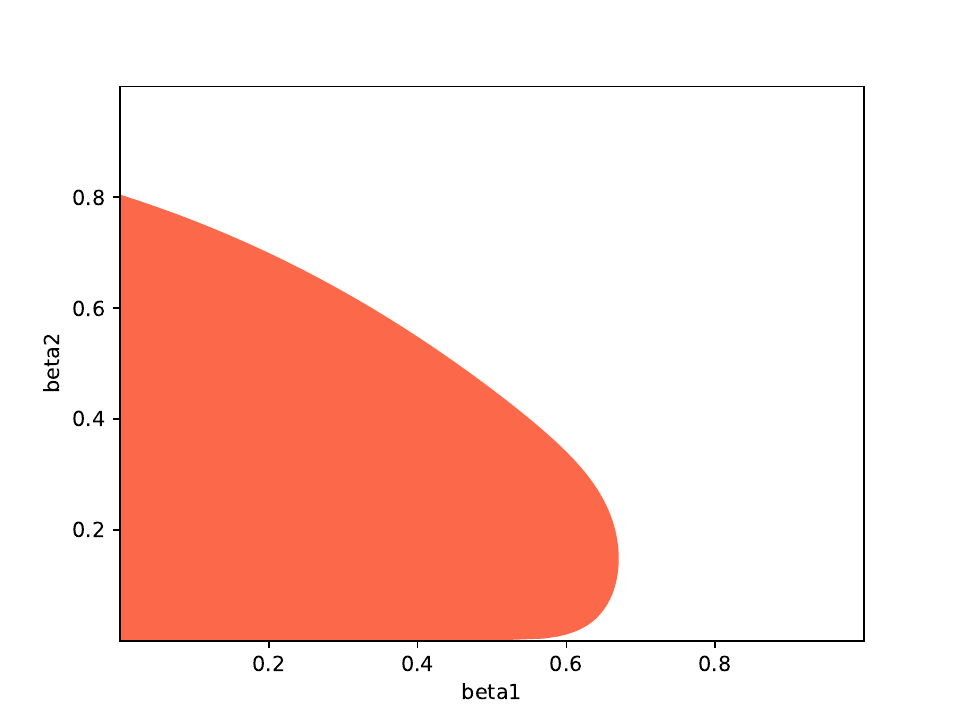}
        \end{minipage}%
        }%
        \subfigure[$n=100$]{
          \begin{minipage}[t]{0.24\linewidth}
          \centering
          \includegraphics[width=\linewidth]{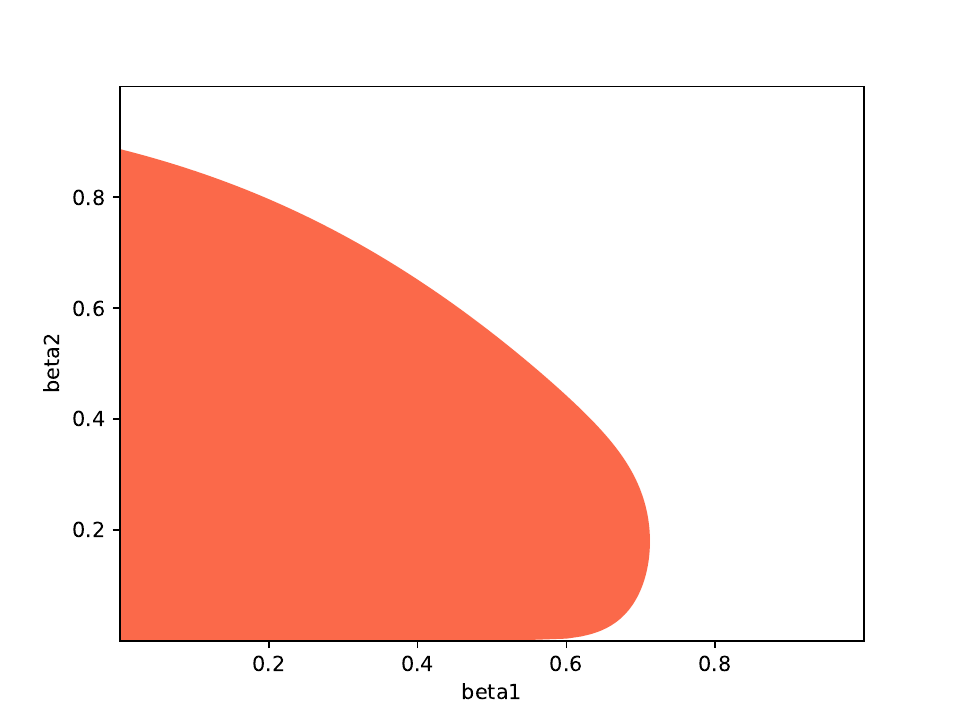}
          \end{minipage}%
          }%
    \caption{ {\small This figure illustrates the region where both
    ${\bf (C1)}$ and ${\bf (C2)}$ hold. The orange color indicates the
    region where ${\bf (C1)}$ holds. White color is used for the counterpart.
    The gray vertical lines are used to indicate the boundary of
    ${\bf (C2)}$. Note that there are two solutions to the equation in
    ${\bf (C2)}$: one solution is $\beta_1=1$ and the other solution lies
    in $0<\beta_1<1$, this is why there are two vertical lines in the
    figure. ${\bf (C2)}$ holds on the left hand side of the left gray
    vertical line. The region is plotted by solving conditions ${\bf (C1)}$ and 
    ${\bf (C2)}$ in \texttt{NumPy}.} }
    \label{fig:counter_boundary}
    \vspace{-1mm}
  \end{figure}

The intersection of the two regions is the region where Adam diverges, which is
the orange region in Figure~\ref{fig:counter_boundary}. The size of the
divergence region increases with $n$.

\paragraph{Remark on the role of ${\bf (C3)}$.} We remark that ${\bf (C3)}$ is a
technical condition used to prevent the trajectory from crossing the minimizer
$x^*=0$. It guarantees the relative band
\eqref{eq:relative_band_c1c2}, and hence keeps all iterates on the same side
of $x^*$ throughout the proof. This fixes the signs of the cyclic component
gradients and allows the decomposition into ${\mathrm{(I)}}$,
${\mathrm{(II)}}$, and ${\mathrm{(III)}}$ to be used without additional case
splits. The condition may be removable with a separate analysis of trajectories
that cross $x^*$ and possibly alternate between the two sides. We leave it for future investigation.

\paragraph{Relation with $\gamma_1(n)$ in Theorem \ref{thm_wr}.}  According to Theorem \ref{thm_wr}, $\gamma_1(n) $ is at least in the order of $1-\mathcal{O}(n^{-5})$. Combining with Figure \ref{fig:counter_boundary}. It is not hard to see that $\gamma_1(n)$ is always larger than the upper boundary of the orange region, so there is no contradiction.

\section{Limitations and Future Directions}
\begin{itemize}
    \item {\bf More fine-grained characterization on the critical boundary.} Integrating our divergence and convergence theories, one can conclude that there exists (at least) one critical boundary $(\beta_1^*, \beta_2^*)$ that demarcates the di-convergence phase transition. However,  we have not fully determined the precise number and the shape of the boundar(-ies). Our experimental results in Fig. \ref{fig:intro_paper} (c, d) suggest that there exists only one boundary, which likely resembles the shape of the blue region in  Fig. \ref{fig:intro_paper} (b).  We only point out its existence here and leave a more precise characterization as a future direction. 
    \item {\bf Adam v.s. SGD.} In this work, we focus on the fundamental issue of convergence. One intriguing question is to
verify the advantage of Adam over SGD. Identifying when and why Adam converges faster serves as an independent research topic, and we leave as future investigation.
\item {\bf Generalized Lipschitz conditions.} A recent line of theoretical work relaxes the standard Lipschitz condition in Assumption  \ref{assum1} to some generalized Lipschitz conditions (e.g., \citep{zhang2019gradient,  li2023convergence, wang2024provable}). As mentioned earlier in Section \ref{section:preliminaries}, the generalized Lipschitz conditions primarily help refine the {\it quantitative} analysis of Adam, e.g., yielding a sharper convergence rate. Meanwhile,  the core focus of this work is a more basic topic: characterizing the {\it qualitative} behavior of Adam---such as phase transition from divergence to convergence. For this purpose, the standard Lipschitz condition is already adequate. Nevertheless,  generalizing the Lipschitz condition is a theoretically intriguing direction for future research. 
\item {\bf Tighter bounds.} Finally, we note that in our convergence upper bounds, the $\mathcal{O}\left(\frac{\log T }{\sqrt{T}} \right)$ term has multiplicative constants of order  $\mathcal{O}(d^4 n^{4.5} )$, where $d$ is the problem dimension $d$ and $n$ is the number of mini-batches. Similarly, the threshold of $\beta_2$ is in the order of $n^5$ and $n^{5.5}$. These powers are not claimed tight, and it is possible to reduce these dependencies via a more fine-grained analysis. We leave it as an interesting future direction. 
\item {\bf Optimal combination of $\beta_1$ and $\beta_2$.} Our results point out the existence of a safe region of $\betabeta$, but we have not yet identified which combination of $\beta_1$ and $\beta_2$ yields the optimal performance within the safe region. Recent works \citep{zhao2024deconstructing,orvieto2025search} report that $\beta_1 =\beta_2$ (when both are large enough) usually brings optimal performance, and they provide an initial explanation. This finding is further investigated in recent works \citep{fernandez2026adam,cattaneo2026effect}. Identifying the optimal hyperparameter is an independent topic that requires substantially more effort. We leave it as an important future direction.
\end{itemize}

\section{Conclusions}
\label{sec_conclusion}
In this work, we explore the convergence of Adam. When $\beta_2$ is large, we  
prove that Adam converges 
 with any $\beta_1 < \sqrt{\beta_2}$.
When $\beta_2$ is small, we further show that Adam can diverge to infinity for  a wide range of $\beta_1$. The critical boundary of the phase transition $(\beta_1^*, \beta_2^*)$ is problem-dependent, and in particular, depends on batch size. Our results provide rigorous theoretical groundings for Adam optimizer. These results also provide
practical suggestions on how to tune $\beta_1$ and $\beta_2$.

\acks{
Yushun Zhang would like to thank Naichen Shi, Bohan Wang, and anonymous NeurIPS 2022 reviewers for their valuable discussion and contributions to the conference version of the script.
Yushun Zhang would like to thank Prof. Anthony Man-cho So for the valuable discussion on the bounded variance condition. Yushun Zhang would like to thank Prof. Lexing Ying for the discussion on the relation between the divergence of  SignSGD and batch size.

\paragraph{Funding.} This research is supported in part by NSFC (No. 12326608); Hetao Shenzhen-Hong Kong Science and Technology Innovation Cooperation Zone Project (No.HZQSWS-KCCYB-2024016); Guangdong Provincial Key Laboratory of Mathematical Foundations for Artificial Intelligence (2023B1212010001).

\paragraph{Competing Interests.} The authors have no competing interests to declare that are relevant to the content of this article.

\paragraph{Declaration of AI Usage.} AI tools were used for language polishing, figure refinement, and auxiliary checking of the correctness of proof details. AI tools did not participate in generating the main theorem proofs of this work.
All AI-assisted content was reviewed and verified by the authors, who take full responsibility for the final manuscript.

}

\clearpage
\appendix

\section*{Table of Contents for the Appendix}

\startcontents[sections]
\printcontents[sections]{l}{1}{\setcounter{tocdepth}{2}}

\section{More Discussions}

\subsection{More Related Works}
\label{appendix:more_related}

\paragraph{On the Variants of Adam.} Ever since \citet{reddi2019convergence} pointed out the divergence issue of Adam, one active line of work has tried to design new variants of Adam that can be proved to converge. For instance,  \citet{zou2019sufficient,gadat2020asymptotic,chen2018convergence,chen2021towards}
replace the constant hyperparameters by iterate-dependent ones e.g. $\beta_{1k}$ or $\beta_{2k}$. AMSGrad \citep{reddi2019convergence} and AdaFom \citep{chen2018convergence} modify $\{v_k\}$ to be a non-decreasing sequence. \citet{iiduka2022theoretical} further analyze the convergence of AMSGrad by relaxing the Lipschitz-gradient condition.  However, their analysis requires extra conditions on both bounded gradient and bounded domain. 
AdaBound \citep{luo2018adaptive} imposes lower and upper bounds on $\{v_k\}$ to prevent the effective stepsize from  vanishing or exploding.   \citet{zhou2018adashift} also adopt a new estimate of  $v_k$ to correct the  bias.
There are also attempts to combine Adam with Nesterov momentum \citep{dozat2016incorporating} as well as warm-up techniques \citep{Liu2020On}.  Padam \citep{chen2018closing} also introduce a partial adaptive parameter to improve the generalization performance. 
There are also some works providing theoretical analysis on the variants of Adam.  For instance, \citet{zhou2018convergence} study the convergence of  AdaGrad and AMSGrad under bounded gradient condition. \citet{gadat2020asymptotic} study the asymptotic behavior of a subclass of adaptive gradient methods from landscape point of view. Their analysis applies to a sub-class of  Adam variants with $\beta_1 =0$ and $\beta_2$ increasing along the iterates (it could also be understood as RMSProp with increasing $\beta_2$). \citet{alacaoglu2020new} study AMSGrad and two other variants of Adam. Their analysis requires both bounded gradient and bounded domain assumptions, and does not cover the original Adam.  \citet{iiduka2022theoretical} analyze the convergence of AMSGrad by relaxing the Lipschitz-gradient condition. However, their analysis requires extra conditions on both bounded gradient and bounded domain.

Another line of theoretical work imposes extra conditions on Adam's trajectory, which implicitly modifies the update rules of Adam.
 \citet{de2018convergence} analyze RMSProp and Adam, but they assume the sign of all stochastic gradients to keep the same. This requires additional projection steps in every iteration, and the resulting algorithm is no longer vanilla Adam. \citet{barakat2021convergence} provide two proofs for Adam. The first proof considers $(\beta_1,\beta_2)$ close to $ (1,1)$ and it  requires bounded gradient and bounded iterates assumptions. As discussed in the paper, we need {\it ``a projection step on a compact set ... to ensure the boundedness of the estimates''}. This projection step changes the update rules of Adam, and more importantly, it eliminates the possibility of divergence a priori, and thus did not fully capture the behaviors of vanilla Adam. The second proof also requires bounded gradient assumption and it replaces the constant hyperparameters by iterate-dependent ones $\beta_{1k}$ or $\beta_{2k}$ under certain rules.  This also changes the update rules of Adam.
  \citet{huang2021super} and \citet{guo2021novel}  propose novel and simple frameworks to analyze Adam-family with large $\beta_1$. 
Yet, they require the effective stepsize of Adam to be bounded in certain interval, i.e.,  $\frac{1}{\sqrt{v_k}+ \epsilon} \in [C_l, C_u]$. This boundedness condition is also imposed for RMSProp in  \citep{Manzil2018adaptive}. This boundedness condition changes Adam into AdaBound \citep{luo2018adaptive}, and thus they cannot explain the divergence-convergence phase transition on the original Adam.

To summarize, these works study the modified versions of Adam or implicitly impose extra operations on Adam (such as projection), which changes the update rules of Adam.
Additionally, they all (including those for new variants) require bounded gradient assumptions, which prevents the potential divergence a priori. In contrast,  we do not require such assumptions and prove the existence of divergence-convergence phase transition when changing $\betabeta$.

\paragraph{Understanding Adam under specialized settings.}  There is a line of work on understanding Adam in specialized settings such as quadratic functions (e.g., \citep{da2020general,das2024towards,zhang2024transformers}), highly degenerate polynomials \citep{bai2026towards}, strongly convex functions \citep{dereich2025adam}, and deterministic and locally strongly convex functions \citep{dereich2025sharp}. \citet{jiang2023does} study the evolution of a newly defined condition number along Adam's trajectory on large-batch two-layer linear network. There are also theoretical works analyzing the dynamics of Adam under continuous-time approximation via  ODE or SDE methods (e.g., \citep{da2020general,malladi2022sdes, dereich2024convergence,kunstner2024heavy,cohen2024understanding,dereich2025ode,li2025adam}).
 Different from these works, we do not adopt the continuous approximation and analyze the original (and thus discrete) Adam on generic non-convex functions.
There is also excellent work on understanding the benefit of AdamW over Adam under deterministic (full-batch) settings \citep{xie2024implicit}, which serves as an orthogonal research topic to this work.

\subsection{More Discussions on the Bounded Variance Condition \eqref{eq_a2_bdd_variance}}
\label{appendix_bdd_variance}

Here, we provide more discussion on why the classical ``bounded variance'' condition
\eqref{eq_a2_bdd_variance} is too restricted, and why it is meaningful to relax it to a more general form, such as Assumption \ref{assum2}. Consider a convex quadratic minimization problem \citep{bottou2018optimization}:
{\small \begin{equation}
\label{eq_a2_quadratic_example_appendix}
     \operatorname*{minimize}_{x\in\mathbb{R}^d} \quad  f(x) = \frac{1}{2n}\|Ax-b\|_2^2
    = \frac{1}{2n}\sum_{i=0}^{n-1}(a_i^\top x-b_i)^2
    := \frac{1}{n}\sum_{i=0}^{n-1} f_i(x),
\end{equation}}
where $A \in \mathbb{R}^{d\times d}$ and $a_i^\top$ denotes the $i$-th row of $A$ and we assume $A$ has rank at least $2$. For \eqref{eq_a2_quadratic_example_appendix},  we will show that bounded variance condition \eqref{eq_a2_bdd_variance} fails, but Assumption \ref{assum2} holds.

\paragraph{Bounded variance condition \eqref{eq_a2_bdd_variance} fails on \eqref{eq_a2_quadratic_example_appendix}.}
Let $r(x):=Ax-b\in\mathbb{R}^d$ with entries $r_i(x)=a_i^\top x-b_i$. For the decomposition in
\eqref{eq_a2_quadratic_example}, we have
{\small \[
\nabla f_i(x)=(a_i^\top x-b_i)a_i = r_i(x)a_i,
\qquad
\nabla f(x)=\frac{1}{n} \sum_{i=0}^{n-1}\nabla f_i(x)= \frac{1}{n} A^\top r(x).
\]}
Plugging into \eqref{eq_a2_bdd_variance} gives
{\small \begin{equation}
\label{eq:var-expand}
\frac{1}{n}\sum_{i=0}^{n-1}\left\|\nabla f_{i}(x) - \nabla f(x)\right\|_{2}^{2}
=
\frac{1}{n}\sum_{i=0}^{n-1}\|r_i(x)a_i\|^2
-\frac{1}{n^2}\|A^\top r(x)\|^2.
\end{equation}}

Assume $\mathrm{rank}(A)\ge 2$, so there exist two indices $p\neq q$ such that $a_p$ and $a_q$ are not collinear.
Pick $u\in\mathbb{R}^d$ such that
{\small \begin{equation}
\label{eq:u-choice}
a_p^\top u = 1,\qquad a_q^\top u = -1,
\end{equation}}
which is feasible since $a_p,a_q$ are linearly independent. Consider the ray $x(t)=x_0+t u$ for any fixed $x_0$.
Then
{\small \[
r_p(x(t))=r_p(x_0)+t,\qquad r_q(x(t))=r_q(x_0)-t,
\]}
and for every $i\notin\{p,q\}$, $r_i(x(t))=r_i(x_0)+t\,a_i^\top u$ grows at most linearly in $t$.

Using \eqref{eq:var-expand} and keeping only the two terms $p,q$ yields the lower bound
{\small \begin{align}
\frac{1}{n} \sum_{i=0}^{n-1}\left\|\nabla f_{i}(x(t)) - \nabla f(x(t))\right\|_{2}^{2}
&\ge
\frac{1}{n} \left(
\|r_p(x(t))a_p\|^2+\|r_q(x(t))a_q\|^2-\frac{1}{n}\|A^\top r(x(t))\|^2  \right)\nonumber\\
&= \frac{1}{n} \left(
(r_p(x_0)+t)^2\|a_p\|^2+(r_q(x_0)-t)^2\|a_q\|^2-\frac{1}{n}\|A^\top r(x(t))\|^2 \right). 
\label{eq:var-lb}
\end{align}}
\begingroup\emergencystretch=2em
The first two terms in \eqref{eq:var-lb} grow like $(\|a_p\|^2+\|a_q\|^2)t^2$ as $|t|\to\infty$.
Meanwhile, $\|A^\top r(x(t))\|^2=\|A^\top(Ax(t)-b)\|^2$ is a quadratic function of $t$ as well, but the
coefficient in front of $t^2$ is $\|A^\top A u\|^2$.
Because $a_p^\top u=1$ and $a_q^\top u=-1$, we have $A u$ has at least two nonzero coordinates of opposite sign,
and in particular $\|A u\|^2\ge 2$.
\par\endgroup

Therefore, along $x(t)$ the variance expression in \eqref{eq:var-expand} is a
nonconstant quadratic polynomial in $t$ with a strictly positive leading coefficient, and hence is unbounded above as $|t|\to\infty$. Consequently, there is no
finite constant $D_0$ such that \eqref{eq_a2_bdd_variance} holds for all $x\in\mathbb{R}^d$.

To sum up, for the quadratic example \eqref{eq_a2_quadratic_example} with $\mathrm{rank}(A)\ge 2$, the left-hand side of
\eqref{eq_a2_bdd_variance} can grow on the order of $\|x\|^2$, so the bounded variance condition fails in general.

\paragraph{Assumption~\ref{assum2} holds with finite constants.}
We next show that Assumption~\ref{assum2} can still hold for \eqref{eq_a2_quadratic_example}. Let
$W:=\mathrm{diag}(\|a_0\|^2,\dots,\|a_{n-1}\|^2)$. Since $\nabla f_i(x)=r_i(x)a_i$,
{\small \begin{equation}
\label{eq:sum-grad-i}
\sum_{i=0}^{n-1}\|\nabla f_i(x)\|^2
=\sum_{i=0}^{n-1} r_i(x)^2\|a_i\|^2
= r(x)^\top W r(x).
\end{equation}}
Moreover,
{\small \begin{equation}
\label{eq:grad-full}
\|\nabla f(x)\|^2 = \frac{1}{n^2}  \|A^\top r(x)\|^2 =   \frac{1}{n^2}   r(x)^\top (AA^\top)\,r(x). 
\end{equation}}

Decompose $r(x)=r_\parallel(x)+r_\perp$ with
$r_\parallel(x)\in\mathrm{range}(A)$ and $r_\perp\in\mathrm{null}(A^\top)$. Since $r(x)=Ax-b$ and $Ax\in\mathrm{range}(A)$,
the orthogonal component is constant in $x$:
{\small \[
r_\perp = -\Pi_{\mathrm{null}(A^\top)}b,
\qquad
A^\top r_\perp = 0.
\]}
Using $(u+v)^\top W(u+v)\le 2u^\top Wu+2v^\top Wv$ and \eqref{eq:sum-grad-i} and \eqref{eq:grad-full}, we obtain
{\small \begin{align}
\sum_{i=0}^{n-1}\|\nabla f_i(x)\|^2
= r^\top W r
&\le 2 r_\parallel^\top W r_\parallel + 2 r_\perp^\top W r_\perp
\label{eq:split}\\
&\le 2c \, r_\parallel^\top (AA^\top) r_\parallel + 2 r_\perp^\top W r_\perp
= 2c n^2 \|\nabla f(x)\|^2 + 2 r_\perp^\top W r_\perp,
\nonumber
\end{align}}
where
{\small \begin{equation}
\label{eq:c-def}
c := \sup_{u\in \mathrm{range}(A),\,u\neq 0}\frac{u^\top W u}{u^\top (AA^\top)u} \;<\;\infty.
\end{equation}}
Finiteness of $c$ follows because $AA^\top$ is positive definite on $\mathrm{range}(A)$, hence the generalized
Rayleigh quotient in \eqref{eq:c-def} is bounded.

Thus Assumption~\ref{assum2} holds with
{\small \begin{equation}
\label{eq:D1D0}
D_1 = 2c n^2  ,
\qquad
D_0 = 2\, r_\perp^\top W r_\perp
= 2\,\big(\Pi_{\mathrm{null}(A^\top)}b\big)^\top
W\,\big(\Pi_{\mathrm{null}(A^\top)}b\big),
\end{equation}}
both finite constants depending only on $(A,b)$. We also mention a special case here: if $b\in\mathrm{range}(A)$, then $r_\perp=0$ and \eqref{eq:D1D0} gives $D_0=0$.

In conclusion, for the  quadratic problem \eqref{eq_a2_quadratic_example}, Assumption~\ref{assum2} can hold with finite $(D_1,D_0)$, while the bounded variance condition \eqref{eq_a2_bdd_variance} fails.

\section{Some More Notations and Useful Lemmas}
\label{appendix_notations}

\paragraph{Some More Notations.} For Adam under random shuffling (Algorithm \ref{algorithm_rr}),  we denote $x_{k, i}, m_{k, i}, v_{k, i} \in \mathbb{R}^{d}$ as the value of $x, m, v$ in the $k$-th outer loop and $i$-th inner loop. Further, we denote  $x_{l, k, i}, m_{l, k, i}, v_{l, k, i} \in \mathbb{R}$ as the $l$-th component of $x_{k, i}, m_{k, i}, v_{k, i}$.  Further, we will use $\tau_{k,i}$ to index  the $i$-th randomly chosen batch in the $k$-th epoch. In this sense,  we denote  $ \partial_l f_{\tau_{k,i}}(x)$ as $\frac{\partial}{\partial x_{l}} f_{\tau_{k,i}}(x)$.
\RED{For each epoch $k\ge 1$, we use $\mathbb E_k[\cdot]$ as a shorthand for the conditional expectation given the history up to (and including) the epoch-start iterate $x_{k,0}$, but excluding the permutation $\{\tau_{k,0},\ldots,\tau_{k,n-1}\}$ sampled in epoch $k$. In other words, the expectation is conditioned on the initial states ($m_{1,-1}$, $v_{1,-1}$, and $x_{1,0}$), and all past permutations $\{\tau_{s,0},\ldots,\tau_{s,n-1}\}$ for $s=1,\ldots,k-1$.
Similarly, $\mathbb E_{k-1}[\cdot]$ denotes the conditional expectation given the entire history up to (and including) $x_{k-1,0}$, but excluding the permutation in epoch $k-1$.}
We abuse the notation of $\alpha$ as follows: for $\partial_l f_{\alpha}(x)$, we  define $\alpha:= \arg \max_i|\partial_l f_i(x)|$; for $\partial_\alpha f(x)$,  we define $\alpha:= \arg \max_l|\partial_l f(x)|$.
Similarly for $m_{\alpha, k,i}$ and $v_{\alpha,k,i}$.

We further define the following constants, which will be repeatedly used in the analysis of Adam under both with-replacement sampling (Algorithm \ref{algorithm_wr}) and random shuffling (Algorithm \ref{algorithm_rr}).

  {\small \begin{eqnarray}
 \label{eq_constants_wr}
        \Delta_k :=\frac{\eta_{0} }{\sqrt{k}}\frac{L\sqrt{d}}{\sqrt{1-\beta_{2}}} \frac{1-\beta_{1} }{1-\frac{\beta_{1}}{\sqrt{\beta_{2}}}},& \quad Q_k:=\frac{32 (n+1)  \Delta_1}{\sqrt{k}}\left(\lceil\frac{\log (1/2)}{\log \beta_2}\rceil+n\right).
\end{eqnarray}}

\paragraph{Some Useful Lemmas for  Theorem \ref{thm_wr} and \ref{thm_rr}.} We now prove the following Lemma \ref{lemma_delta}, Lemma \ref{lemma_k-j_k} and Lemma \ref{lemma_fi_f}, which will be repeatedly used later.

\begin{lemma} \label{lemma_beta}
  For any $\beta \in   \brrevise{[0, 1)}$, we have
  {\small $$\left(1-\beta\right) \sum_{j=1}^{\infty} \beta^{j-1}=1, \quad  \left(1-\beta\right) \sum_{j=1}^{\infty} j \beta^{j-1}=\frac{1}{1-\beta}, \quad \left(1-\beta\right) \sum_{j=1}^{\infty} j^{2} \beta^{j-1}=\frac{1+\beta}{\left(1-\beta\right)^{2}}.$$}
\end{lemma}

\begin{proof}
  The proof only involves basic calculation, we omit the proof here.
\end{proof}

\begin{lemma}\label{lemma_delta}
Consider Algorithm \ref{algorithm_wr},  for any $f \in \functionclass $, $l\in [d]$,  $i \in [n]$, if $0\leq \beta_1<\sqrt{\beta_2}<1$, then we have
{\small \begin{equation}
   \label{eq_delta_wr}
    \left| \partial_{l} f_{i}(x_{k+1}) -\partial_{l} f_{i}(x_{k})\right| \leq \frac{\eta_{0} }{\sqrt{k}}\frac{L\sqrt{d}}{\sqrt{1-\beta_{2}}} \frac{1-\beta_{1} }{1-\frac{\beta_{1}}{\sqrt{\beta_{2}}}} := \Delta_{k}.
  \end{equation}}

Similarly, for Algorithm \ref{algorithm_rr}, we have the following results under the same condition:
   {\small \begin{equation}
  \label{eq_delta_rr}
    \left| \partial_{l} f_{i}(x_{k,i+1}) -\partial_{l} f_{i}(x_{k,i})\right| \leq  \Delta_{k}.
  \end{equation}}

\end{lemma}

\begin{proof}
Here, we present the proof of \eqref{eq_delta_wr}. The proof of \eqref{eq_delta_rr} is done following the same procedure. \yushunrevise{In the following, we will use the condition that $m_0 = 0$. Extending to non-zero initialization only involves a change of several constant terms.}

  {\small \begin{eqnarray}
    \left| x_{l,k+1} - x_{l,k}\right|  &=& \eta_k \frac{|m_{l,k}|}{\sqrt{v_{l,k}}}  \nonumber \\
    &\leq& \eta_k (1-\beta_1) \sum_{j=0}^{k} \beta_1^j\frac{|\partial_l f_{\tau_{k-j}}(x_{k-j})|}{\sqrt{v_{l,k}}} \label{eq:m_over_v_1} \\
    &\overset{\text{(i)}}{\leq }&\eta_k (1-\beta_1) \sum_{j=0}^{k} \beta_1^j\frac{|\partial_l f_{\tau_{k-j}}(x_{k-j})|}   {|\partial_l f_{\tau_{k-j}}(x_{k-j})|\sqrt{ (1-\beta_2) \beta_2^j }} \nonumber \\
    &\overset{\text{(ii)}}{=} & \eta_k \frac{(1-\beta_1)  }{\sqrt{1-\beta_2}} \sum_{j=0}^{k} \left(\frac{\beta_1}{\sqrt{\beta_2}}\right)^j \leq \eta_k \frac{(1-\beta_1)  }{\sqrt{1-\beta_2}} \sum_{j=0}^{\infty} \left(\frac{\beta_1}{\sqrt{\beta_2}}\right)^j  \nonumber \\
    &=& \eta_k\frac{(1-\beta_1)}{\sqrt{1-\beta_2}} \frac{1}{1-\frac{\beta_1}{\sqrt{\beta_2}}}.  \label{eq:m_over_v}
  \end{eqnarray}}
Note that (i) holds due to $v_{l,k} \geq (1-\beta_2) \beta_2^j \partial_l f_{\tau_{k-j}}(x_{k-j})^2$.  (ii) clearly holds when $|\partial_l f_{\tau_{k-j}}(x_{k-j})| >0$.  
When $|\partial_l f_{\tau_{k-j}}(x_{k-j}) | =0$, the summand in \eqref{eq:m_over_v_1} equals 0 under non-zero initialization $v \succ 0$ or $\epsilon >0$, which keeps Adam well-defined.  Hence, the inequality in (ii) holds trivially. 
In either case, we arrive at the final result \eqref{eq:m_over_v}, and the final result is independent of initialization $v$ or $\epsilon$. We further remark that this is the only place in the analysis where we use the condition $v \succ 0$.

 We conclude the proof of Lemma \ref{lemma_delta} by applying the Lipschitz condition in Assumption \ref{assum1}.
\end{proof}

\begin{lemma} \label{lemma_k-j_k}
Consider Algorithm \ref{algorithm_wr}, for any $f \in \functionclass $, $0\leq \beta_1<\sqrt{\beta_2}<1$, $l\in [d]$,  $i \in [n]$ and $j \in [k-2]$, when 
{\small$|\partial_l f_i (x_k)|  \geq  \frac{8\sqrt{2} j \Delta_1}{\sqrt{k}}$}, we have 
{\small \begin{equation}
\label{eq_f_k-j_k_wr}
           (\partial_l f_i (x_{k-j}))^2 \geq  \frac{(\partial_l f_i (x_k))^2 }{2}.
\end{equation}}

Similarly, for Algorithm \ref{algorithm_rr}, when 
{\small $|\partial_l f_i (x_{k,0})|  \geq  \frac{8\sqrt{2} j n\Delta_1}{\sqrt{k}}$}, for any $i^\prime \in [n]$, we have

{\small \begin{equation}
\label{eq_f_k-j_k_rr}
          ( \partial_l f_i (x_{k-j,i^\prime}))^2 \geq  \frac{(\partial_l f_i (x_{k,0}))^2 }{2}.
\end{equation}}

\end{lemma}

\begin{proof}
Here, we present the proof of \eqref{eq_f_k-j_k_wr}. The proof of \eqref{eq_f_k-j_k_rr} is done following the same procedure.

{\small \begin{eqnarray*}
           (\partial_l f_i (x_{k-j}))^2 & = &  
            (\partial_l f_i (x_{k-j}) - \partial_l f_i (x_{k}) + \partial_l f_i (x_{k})   )^2 \\
            &\geq& (\partial_l f_i (x_k))^2   -2 |\partial_l f_i (x_{k-j}) - \partial_l f_i (x_{k}) | |\partial_l f_i (x_{k})| \\
            &= & (\partial_l f_i (x_k))^2 \left(1-   \frac{2|\partial_l f_i (x_{k-j}) - \partial_l f_i (x_{k}) |}{     |\partial_l f_i (x_k)|  }\right).
\end{eqnarray*}}

By Lemma \ref{lemma_delta}, we have 

{\small \begin{eqnarray}
          |  \partial_l f_i (x_{k-j})  -  \partial_l f_i (x_{k}) | 
          &\overset{\text{Lemma \ref{lemma_delta}}}{\leq} &\sum_{i=1}^j \Delta_{k-i}.\nonumber\\
        &\overset{\text{(i)}}{\leq}& \frac{2j\Delta_1}{\sqrt{k-1}} \overset{}{\leq} \frac{2j\sqrt{2}\Delta_1}{\sqrt{k}},\label{eq:sum_delta}
    \end{eqnarray}}
where (i) is due to 
$\sum_{i=1}^j \frac{1}{\sqrt{k-i}} \leq \int_{k-j-1}^{k-1} \frac{d t}{\sqrt{t}}=\frac{2 j}{\sqrt{k-j-1}+\sqrt{k-1}} \leq \frac{2 j}{\sqrt{k-1}}.$ Therefore, we have
{\small \begin{eqnarray*}
           (\partial_l f_i (x_{k-j}))^2 &\geq&   (\partial_l f_i (x_k))^2 \left(1-   \frac{4j \sqrt{2}\Delta_1}{     |\partial_l f_i (x_k)|    \sqrt{k}} \right).
\end{eqnarray*}}
So when {\small $|\partial_l f_i (x_k)|  \geq  \frac{8\sqrt{2} j \Delta_1}{\sqrt{k}}$}, we have  {\small $ (\partial_l f_i (x_{k-j}))^2 \geq\frac{(\partial_l f_i (x_k))^2 }{2}$}, which  concludes the proof.\end{proof}

 \begin{lemma} \label{lemma_fi_f}
 Consider Algorithm \ref{algorithm_wr}, for any $f \in \functionclass $, $0\leq \beta_1<\sqrt{\beta_2}<1$, $l\in [d]$,  $i \in [n]$ and $j \in [k-1]$,  we have the following two results:
 {\small \begin{equation}\label{eq_fi_f_wr}
   |\partial_l f_{\tau_{k-j}}(x_{k-j})|  \leq  \sum_{t =1}^j \Delta_{k-t} + \sqrt{D_{1}} \sqrt{n} d\left(\left|\partial_{\alpha} f\left(x_{k}\right)\right|+\sqrt{\frac{D_{0}}{D_{1} d}}\right)  ,  
 \end{equation}}
 where $\left|\partial_{\alpha} f\left(x_{k}\right)\right| := \max_{l\in [d]} \left|\partial_{l} f\left(x_{k}\right)\right|$.

Similarly, for Algorithm \ref{algorithm_rr}, we have the following results under the same condition.

{\small \begin{equation}\label{eq_fi_f_rr}
  |\partial_l f_{\tau_{k,i}}(x_{k,i})|  \leq i \Delta_{k} + \sqrt{D_{1}} \sqrt{n} d\left(\left|\partial_{\alpha} f\left(x_{k, 0}\right)\right|+\sqrt{\frac{D_{0}}{D_{1} d}}\right),  
\end{equation}}
where $\left|\partial_{\alpha} f\left(x_{k, 0}\right)\right| := \max_{l\in [d]} \left|\partial_{l} f\left(x_{k, 0}\right)\right|$.

\end{lemma}

\begin{proof}
Here, we present the proof of \eqref{eq_fi_f_wr}. The  proof of  \eqref{eq_fi_f_rr} is done via the same procedure.
  {\small \begin{eqnarray}
    |\partial_l f_{\tau_{k-j}}(x_{k-j}) |  &\overset{\text{Lemma \ref{lemma_delta}}}{\leq}& \sum_{t =1}^j \Delta_{k-t} +  |\partial_l f_{\tau_{k-j}}(x_{k}) |  \nonumber \\
    &\leq & \sum_{t =1}^j \Delta_{k-t} + \sum_{l^{\prime}=1}^d \sum_{i=0}^{n-1} |\partial_{l^{\prime}} f_{i}(x_{k}) |  \nonumber \\  
    &\overset{\text{Cauchy–Schwarz inequality}}{\leq } &\sum_{t =1}^j \Delta_{k-t} +  \sqrt{n} \sum_{l^{\prime}=1}^d \sqrt{\sum_{i=0}^{n-1}\left|\partial_{l^{\prime}} f_i (x_{k}) \right|^{2}}\nonumber  \\
    & \overset{\text{Cauchy–Schwarz inequality}}{\leq} & \sum_{t =1}^j \Delta_{k-t} + \sqrt{n} \sqrt{d} \sqrt{ \sum_{l^{\prime}=1}^d \sum_{i=0}^{n-1}\left|\partial_{l^{\prime}} f_i (x_{k}) \right|^{2}} \nonumber  \\
    &\overset{\text{Assumption \ref
    {assum2}}}{\leq }& \sum_{t =1}^j \Delta_{k-t} + \sqrt{n} \sqrt{d}  \sqrt{D_1 \|\nabla f(x_{k})\|_2^2 + D_0} \nonumber \\
    &  \overset{ }{\leq}   & \sum_{t =1}^j \Delta_{k-t} + \sqrt{D_{1}} \sqrt{n} d \sqrt{\left|\partial_{\alpha} f\left(x_{k}\right)\right|^{2}+\frac{D_{0}}{D_{1} d}} \label{lemma_fi_f_2nd} \\
    & \overset{}{\leq} & \sum_{t =1}^j \Delta_{k-t} + \sqrt{D_{1}} \sqrt{n} d\left(\left|\partial_{\alpha} f\left(x_{k}\right)\right|+\sqrt{\frac{D_{0}}{D_{1} d}}\right).  \label{lemma_fi_f_2nd_2} 
  \end{eqnarray}}
  The proof is complete. As a side remark, we further note that if $D_1 = 0$ (i.e., assume  bounded 2nd-order moment),  \eqref{lemma_fi_f_2nd} can be further simplified using $\sum_{l^{\prime}=1}^d \sum_{i=0}^{n-1}\left|\partial_{l^{\prime}} f_i (x_{k}) \right|^{2} \leq D_0$, and the final result  would be $ |\partial_l f_{\tau_{k-j}}(x_{k-j}) |  \leq \sum_{t =1}^j \Delta_{k-t} +  \sqrt{n d D_0}$. This is a special case of  \eqref{lemma_fi_f_2nd_2} and it will make the subsequent analysis strictly simpler since we get rid of the dependency on $\left|\partial_{\alpha} f\left(x_{k}\right)\right|$. The simplified analysis will have strictly fewer error terms and can be done by following the same procedure as the current proof. In our analysis, we do not restrict $D_1$ to be $0$.
  
\end{proof}

\begin{lemma} \label{lemma_upper_m}
Consider Algorithm \ref{algorithm_rr}, for any $f \in \functionclass $, $0\leq \beta_1<\sqrt{\beta_2}<1$, $l\in [d]$,  $i \in [n]$ and $k\geq 4$ large enough such that $\frac{1-\beta_1^n}{1-\beta_1}\beta_1^{(k-1)n+1}\leq \frac{1}{\sqrt{k}}$,  we have 
{\small \begin{equation}
    \label{eq_upper_m_k}
    \sum_{i =0}^{n-1}|m_{l,k,i}| \leq  C_{m,1}\left(\vert \partial_\alpha f(x_{k,0})\vert  + \sqrt{\frac{D_0}{D_1d } }\right) + \frac{C_{m,2}}{\sqrt{k}},
\end{equation}}

{\small \begin{equation}
    \label{eq_upper_m_k-1}
    \sum_{i =0}^{n-1}|m_{l,k-1,i}| \leq  \tilde{C}_{m,1}\left(\vert \partial_\alpha f(x_{k,0})\vert  + \sqrt{\frac{D_0}{D_1d } }\right) + \frac{\tilde{C}_{m,2}}{\sqrt{k}},
\end{equation}}
where 
{{\small $$C_{m,1} = 2n , \quad C_{m,2} = \frac{n(n+1)\Delta_1}{2(1-\beta_1)\sqrt{n}}+\beta_1\frac{1-\beta_1^n}{(1-\beta_1)^2}\frac{2\sqrt{2}\Delta_1}{\sqrt{n}}+\sum_{i=1}^{n-1}\|\nabla f_i(x_{1,0})\|_1,$$}}

{\small$$\tilde{C}_{m,1}  = n, \quad \tilde{C}_{m,2} = \frac{n}{1-\beta_1}\frac{2\sqrt{2}\Delta_1}{\sqrt{n}}+\frac{\sqrt{2}n(n-1)\Delta_1}{\sqrt{n}}+\sum_{i=1}^{n-1}\|\nabla f_i(x_{1,0})\|_1.$$}

\end{lemma}

\begin{proof}
Lemma  \ref{lemma_upper_m} can be proved by repeatedly applying Lemma \ref{lemma_fi_f}. The proof is straightforward and we omit the proof for brevity.
\end{proof}

\section{Proof of Lemma \ref{lemma_k-k-n_WR} }
\label{appendix:lemma_k-k-n}

\begin{proof}

We now prove Lemma \ref{lemma_k-k-n_WR}. We discuss the following two cases.

\paragraph{Case 1: when $\frac{\partial_l f(x_{k})}{\sqrt{v_{l,k}}} \geq \frac{\partial_l f(x_{k-1})}{\sqrt{v_{l,k-1}}} $: } when $\partial_l f(x_{k})\leq 0$, we have

{\small $$\frac{\partial_l f(x_{k})}{\sqrt{v_{l,k}}}  \overset{(\ref{eq:v_k-1_geq_v_k})}{\leq} \frac{\partial_l f(x_{k})}{\sqrt{v_{l,k-1}}} \left(1-(1-\beta_2)\left(\frac{8 n}{\beta_2^n-(1-\beta_2)16n}+\frac{1}{\beta_2}\right)\right)^{\frac{1}{2}}.   $$}
When $\partial_l f(x_{k}) > 0$, we have
{\small $$\frac{\partial_l f(x_{k})}{\sqrt{v_{l,k}}}  \overset{}{\leq}  \frac{\partial_l f(x_{k})}{\sqrt{v_{l,k-1}}} \frac{1}{\sqrt{\beta_2}}  .$$}
In conclusion, we have 
{\footnotesize \begin{eqnarray*}
  \frac{\partial_l f(x_{k})}{\sqrt{v_{l,k}}} &\leq &
  \max\left\{\begin{aligned}
  &\frac{\partial_l f(x_{k})}{\sqrt{v_{l,k-1}}}
   \left(1-(1-\beta_2)\left(\frac{8 n}{\beta_2^n-(1-\beta_2)16n}+\frac{1}{\beta_2}\right)\right)^{\frac{1}{2}},\\
  &\frac{\partial_l f(x_{k})}{\sqrt{v_{l,k-1}}}\frac{1}{\sqrt{\beta_2}}
  \end{aligned}\right\}\\
  &\leq &\frac{\partial_l f(x_{k})}{\sqrt{v_{l,k-1}}} \nonumber\\
  &&\quad+\max\left\{\begin{aligned}
  &\frac{\partial_l f(x_{k})}{\sqrt{v_{l,k-1}}}
   \left[\left(1-(1-\beta_2)\left(\frac{8 n}{\beta_2^n-(1-\beta_2)16n}+\frac{1}{\beta_2}\right)\right)^{\frac{1}{2}}-1\right],\\
  &\frac{\partial_l f(x_{k})}{\sqrt{v_{l,k-1}}}\left(\frac{1}{\sqrt{\beta_2}}-1\right)
  \end{aligned}\right\}\\
  &\leq &\frac{\partial_l f(x_{k})}{\sqrt{v_{l,k-1}}} \nonumber\\
  &&\quad+\frac{\left|\partial_l f(x_{k})\right|}{\sqrt{v_{l,k-1}}}
  \underbrace{\left(\left[1-\left(1-(1-\beta_2)\left(\frac{8 n}{\beta_2^n-(1-\beta_2)16n}+\frac{1}{\beta_2}\right)\right)^{\frac{1}{2}}\right]
  +\frac{1}{\sqrt{\beta_2}}-1\right)}_{\delta_{5,1}(\beta_2)}\\
  &\overset{\text{(\ref{eq:v_k-1_geq_v_k})}}{\leq}&
  \frac{\partial_l f(x_{k})}{\sqrt{v_{l,k-1}}} \nonumber\\
  &&\quad+\frac{\left|\partial_l f(x_{k})\right|}{\sqrt{v_{l,k}}}
  \underbrace{\left(1-(1-\beta_2)\left(\frac{8 n}{\beta_2^n-(1-\beta_2)16n}+\frac{1}{\beta_2}\right)\right)^{-\frac{1}{2}}}_{\delta_{5,2}(\beta_2)}
  \delta_{5,1}(\beta_2)\\
    &\overset{\text{(\ref{eq_concentrate_v_epsilon_given}) (\ref{upperbound_of_f_over_v_epsilon})}}{\leq}&\frac{\partial_l f(x_{k})}{\sqrt{v_{l,k-1}}} + \underbrace{\frac{2\sqrt{2n}}{\sqrt{\beta_2^n-(1-\beta_2)16n}}}_{\delta_{5,3}(\beta_2)}
    \delta_{5,1}(\beta_2)\delta_{5,2}(\beta_2).
\end{eqnarray*}}
Denote $\delta_5(\beta_2) = \delta_{5,1}(\beta_2)\delta_{5,2}(\beta_2)\delta_{5,3}(\beta_2)$.
{\small \begin{eqnarray*}
  \frac{\partial_l f(x_{k})}{\sqrt{v_{l,k}}} -   \frac{\partial_l f(x_{k-1})}{\sqrt{v_{l,k-1}}} &\leq & \frac{\partial_l f(x_{k})- \partial_l f(x_{k-1}) }{\sqrt{v_{l,k-1}}} + \delta_5(\beta_2) \\
  &\overset{\text{Lemma \ref{lemma_delta}}}{ \leq } & \frac{\Delta_{k-1}}{\sqrt{v_{l,k-1}}} + \delta_5(\beta_2).
\end{eqnarray*}}

\paragraph{Case 2: when $\frac{\partial_l f(x_{k})}{\sqrt{v_{l,k}}} \leq \frac{\partial_l f(x_{k-1})}{\sqrt{v_{l,k-1}}} $: } when $\partial_l f(x_{k})\geq 0$, we have

{\small $$\frac{\partial_l f(x_{k})}{\sqrt{v_{l,k}}}  \overset{\text{\eqref{eq:v_k-1_geq_v_k}}}{\geq} \frac{\partial_l f(x_{k})}{\sqrt{v_{l,k-1}}} \left(1-(1-\beta_2)\left(\frac{8 n}{\beta_2^n-(1-\beta_2)16n}+\frac{1}{\beta_2}\right)\right)^{\frac{1}{2}}.   $$}
When $\partial_l f(x_{k}) < 0$, we have

{\small $$\frac{\partial_l f(x_{k})}{\sqrt{v_{l,k}}}  \overset{}{\geq} \frac{\partial_l f(x_{k})}{\sqrt{v_{l,k-1}}} \frac{1}{\sqrt{\beta_2}}  .$$}
Following the same strategy as in Case 1, we can show that 
{\small \begin{eqnarray*}
  \frac{\partial_l f(x_{k})}{\sqrt{v_{l,k}}} &\geq & 
  \frac{\partial_l f(x_{k})}{\sqrt{v_{l,k-1}}} - \delta_5(\beta_2),
\end{eqnarray*}}
which further implies 
{\small \begin{eqnarray*}
  \frac{\partial_l f(x_{k})}{\sqrt{v_{l,k}}} -   \frac{\partial_l f(x_{k-1})}{\sqrt{v_{l,k-1}}} &\geq & \frac{\partial_l f(x_{k})- \partial_l f(x_{k-1}) }{\sqrt{v_{l,k-1}}} -\delta_5(\beta_2)\\
  &\overset{\text{Lemma \ref{lemma_delta}}}{ \geq } & - \frac{\Delta_{k-1}}{\sqrt{v_{l,k-1}}} - \delta_5(\beta_2). 
\end{eqnarray*}}
Combining \textbf{Case 1} and \textbf{Case 2} together, we have 

{\small $$ \left| \frac{\partial_l f(x_{k})}{\sqrt{v_{l,k}}} -   \frac{\partial_l f(x_{k-1})}{\sqrt{v_{l,k-1}}} \right| \leq \frac{\Delta_{k-1}}{\sqrt{v_{l,k-1}}} + \delta_5(\beta_2) . $$}
Based on the above inequality, we have

{\small \begin{eqnarray*}
  \left| \frac{\partial_l f(x_{k})}{\sqrt{v_{l,k}}} -   \frac{\partial_l f(x_{k-n})}{\sqrt{v_{l,k-n}}} \right| &\leq& \left( \frac{\Delta_{k-1}}{\sqrt{v_{l,k-1}}} + \frac{\Delta_{k-2}}{\sqrt{v_{l,k-2}}} + \cdots + \frac{\Delta_{k-n}}{\sqrt{v_{l,k-n}}}\right) + n \delta_5(\beta_2)  \\
    &\leq &  \left( \frac{\Delta_{k-1}}{\sqrt{\beta_2^{n}v_{l,k-n}}} + \frac{\Delta_{k-2}}{\sqrt{\beta_2^{n}v_{l,k-n}}} + \cdots + \frac{\Delta_{k-n}}{\sqrt{\beta_2^{n}v_{l,k-n}}}\right)  +  n \delta_5(\beta_2)   \\
    &\leq & \frac{n}{\sqrt{\beta_2^{n}}} \frac{\Delta_{k-n}}{\sqrt{v_{l,k-n}}} + n \delta_5(\beta_2).
\end{eqnarray*}}
The proof of Lemma \ref{lemma_k-k-n_WR} is now complete.

\end{proof}

\section{Proof of Theorem \ref{thm_rr} }
\label{appendix:thm_rr}

\brrevise{
As in the proof of Theorem \ref{thm_wr}, we first establish the interior
case $D_0>0$ and $D_1>0$. The boundary cases $D_0=0$ or $D_1=0$ follow
from the corresponding arguments in Appendix
\ref{appendix:D0_D1_boundary}, except that we need to replace the iteration-level
quantities of Algorithm \ref{algorithm_wr} by their epoch-level
counterparts for Algorithm \ref{algorithm_rr}; and the
initial second-moment contribution $\beta_2^k v_0$ needs to be replaced by
$\beta_2^{n(k-1)}v_{1,-1}$.}

Before delving into the proof of Theorem \ref{thm_rr}, we first present some additional technical lemmas that will be used in the proof.  \yushunrevise{We recall that $v_{k,-1} = v_{k-1,n-1}$, as defined in Algorithm \ref{algorithm_rr}.}

Similarly to the proof of Theorem \ref{thm_wr}, we focus on the
interior case $D_0>0$ and $D_1>0$ in the proof below. When $D_0=0$ or
$D_1=0$, the proof follows the same procedure as the corresponding
boundary cases for Theorem \ref{thm_wr} in Appendix
\ref{appendix:D0_D1_boundary}, subject to the notational changes for
Algorithm \ref{algorithm_rr}. We omit the repeated details.

\begin{lemma}
\label{lemma_f_over_v}
    \begingroup\emergencystretch=2em Consider Algorithm \ref{algorithm_rr}. Assume $k \geq \lceil \frac{\log (1/2)}{n \log\beta_2} \rceil +1 $  and $\partial_l f_\alpha (x_{k,0})  \allowbreak := \max_i \vert \partial_l f_i (x_{k,0}) \vert \allowbreak \geq  \frac{8\sqrt{2} n \Delta_1}{\sqrt{k}} \frac{\log (1/2)}{n \log\beta_2} $. Then for any  $f \in \functionclass$, $0\leq \beta_1<\sqrt{\beta_2}<1$, we have \par\endgroup
   {\small \begin{equation}
   \label{eq_f_over_v}
      \frac{\left(\partial_l f_\alpha\left(x_{k, 0}\right)\right)^2}{v_{l,k,\yushunrevise{-1}}} \leq \frac{4 n}{\beta_2^n}.
   \end{equation}}
\end{lemma}
The proof of Lemma \ref{lemma_f_over_v} is in Appendix \ref{appendix:lemma_f_over_v}.  Now we present the concentration inequalities of $\frac{1}{\sqrt{v_{l,k,i}}}$.

\begin{lemma}
\label{lemma_concentrate_v_k}
  \begingroup\emergencystretch=2em Consider Algorithm \ref{algorithm_rr}. Assume $k \geq \lceil \frac{\log (1/2)}{n \log\beta_2} \rceil +1 $  and $\partial_l f_\alpha (x_{k,0})  \allowbreak := \max_i \vert \partial_l f_i (x_{k,0}) \vert \allowbreak \geq  \frac{8\sqrt{2} n \Delta_1}{\sqrt{k}} \frac{\log (1/2)}{n \log\beta_2}$. Then for any  $f \in \functionclass$, $0\leq \beta_1<\sqrt{\beta_2}<1$, and for any $i \in [n]$,  we have \par\endgroup
   {\small \begin{equation}
   \label{eq_concentrate_v_k}
      \left( 1 - \RED{(1-\beta_2) n}(\frac{8n}{\beta_2^n} + \frac{1}{2} ) \right) \frac{1}{\sqrt{v_{l,k,\yushunrevise{-1}}}} \leq \frac{1}{\sqrt{v_{l,k,i}}}
\leq \frac{1}{\sqrt{\beta_2^n}}\frac{1}{\sqrt{v_{l,k,\yushunrevise{-1}}}},
   \end{equation}}
\end{lemma}

The proof of Lemma \ref{lemma_concentrate_v_k} can be seen in Appendix \ref{appendix:lemma_concentrate_v_k}.

\begin{lemma}
\label{lemma_concentrate_v_k-1}
  Consider Algorithm \ref{algorithm_rr}. Assume $k \geq \lceil \frac{\log (1/2)}{n \log\beta_2} \rceil +1 $  and $\max_i \left| \partial_l f_{i}(x_{k,0}) \right| \geq \frac{16 n \Delta_1}{\sqrt{k}} \frac{\log (1/2)}{n \log\beta_2} $. Then for any  $f \in \functionclass$, $0\leq \beta_1<\sqrt{\beta_2}<1$, and  for any \yushunrevise{$i = -1, 0,  \cdots, n-2$},  we have
      {\small \begin{equation}
   \label{eq_concentrate_v_k-1}
\sqrt{\beta_2^n}\frac{1}{\sqrt{v_{l,k,\yushunrevise{-1}}}}
\leq \frac{1}{\sqrt{v_{l,k-1,i}}}
\leq  \left(1 - \RED{(1-\beta_2) n}\left( \yushunrevise{\frac{16n +1}{\beta_2^n} } \right)\right)^{-\frac{1}{2}}  \frac{1}{\sqrt{v_{l,k,\yushunrevise{-1}}}}  
   \end{equation}}
\end{lemma}

The proof of Lemma \ref{lemma_concentrate_v_k-1} can be seen in Appendix \ref{appendix:lemma_concentrate_v_k-1}.

\begin{lemma} \label{lemma_k-k-1}
  \begingroup\emergencystretch=2em Consider Algorithm \ref{algorithm_rr}.  Assume  $\partial_l f_\alpha (x_{k,0})  \allowbreak := \max_i \vert \partial_l f_i (x_{k,0}) \vert \allowbreak \geq  \frac{16 n \Delta_1}{\sqrt{k}} \frac{\log (1/2)}{n \log\beta_2}$, $\partial_l f_\alpha (x_{k-1,0})  \allowbreak := \max_i \vert \partial_l f_i (x_{k-1,0}) \vert \allowbreak \geq  \frac{16 n \Delta_1}{\sqrt{k-1}}\frac{\log (1/2)}{n \log\beta_2}$, ...,  $\partial_l f_\alpha (x_{k-j,0})  \allowbreak := \max_i \vert \partial_l f_i (x_{k-j,0}) \vert  \geq  \frac{16 n \Delta_1}{\sqrt{k-j}} \frac{\log (1/2)}{n \log\beta_2}$. \yushunrevise{Assume $k \geq \lceil \frac{\log (1/2)}{n \log\beta_2} \rceil +1  +j$,} then for any  $f \in \functionclass$, $0\leq \beta_1<\sqrt{\beta_2}<1$, we have \par\endgroup
  {\small \begin{equation}
  \label{eq_k-k-1}
    \left| \frac{\partial_l f(x_{k,0})}{\sqrt{v_{l,k,\yushunrevise{-1}}}} -   \frac{\partial_l f(x_{k-j,0})}{\sqrt{v_{l,k-j,\yushunrevise{-1}}}} \right|  \leq  \frac{j}{\sqrt{\beta_2^{nj}}} \frac{n\Delta_{(k-j)}}{\sqrt{v_{l,k-j,\yushunrevise{-1}}}} +  j \tilde{\delta}_1(\beta_2),
\end{equation}}
    where {\small $\tilde{\delta}_1(\beta_2) = 4n \left(\frac{1}{\sqrt{\beta_2^n}}-1  +\left[1-\left(1 - (1-\beta_2) n \left( \yushunrevise{\frac{16n+1}{\beta_2^n}} \right)\right)^{\frac{1}{2}} \right]\right)  \left(1 - (1-\beta_2) n\left( \yushunrevise{\frac{16n + 1}{\beta_2^n}} \right)\right)^{-\frac{1}{2}} $} is a constant that goes to 0 when $\beta_2$ goes to 1.
\end{lemma}

The proof of Lemma \ref{lemma_k-k-1} can be seen in Appendix \ref{appendix:lemma_k-k-1}.  The main body of the proof of Theorem \ref{thm_rr} is then presented in Appendix \ref{appendix_main_body_rr}.

\subsection{Proof of Lemma \ref{lemma_f_over_v} }
\label{appendix:lemma_f_over_v}

We start with a lower bound for $ v_{l,k,\yushunrevise{-1}}$.
{\small \begin{eqnarray*}
   v_{l,k,\yushunrevise{-1}} &=& (1 - \beta_2) \left(
 \left( \partial_l f_{\tau_{k-1,n-1}}(x_{k-1,n-1}) \right)^2
+ \cdots
+ \beta_2^{n\yushunrevise{(k-1)}} \left( \partial_l f_{\tau_{1,0}}(x_{1,0}) \right)^2
\right) \\
&\geq& (1 - \beta_2) \left( 
\sum_{j=1}^{k-1} \beta_2^{n j} 
\sum_{i=0}^{n-1}\left( \partial_l f_{\tau_{k-j,i}}(x_{k-j,i}) \right)^2 
\right)
\\
&\geq& (1 - \beta_2) \left( 
\sum_{j=1}^{k-1} \beta_2^{n j} 
\sum_{i=0}^{n-1}\left( \partial_l f_{\tau_{k-j,i}}(x_{k-j,i}) \right)^2 
\cdot \mathbb{I}(\tau_{k-j,i} = \alpha) 
\right)
\end{eqnarray*}}
where $ \alpha = \arg\max_{i} \left| \partial_l f_{i}(x_{k,0}) \right|$. Since $ k \geq \lceil \frac{\log (1/2)}{n \log\beta_2} \rceil +1$,  we further have:

{\small \begin{eqnarray*}
   v_{l,k,\yushunrevise{-1}} 
&\geq& (1 - \beta_2) \left( 
\sum_{j=1}^{\lceil \frac{\log (1/2)}{n \log\beta_2} \rceil} \beta_2^{n j} 
\sum_{i=0}^{n-1}\left( \partial_l f_{\tau_{k-j,i}}(x_{k-j,i}) \right)^2 
\cdot \mathbb{I}(\tau_{k-j,i} = \alpha) 
\right) \\
&\overset{\text{(i)}}{\geq}& \frac{(1 - \beta_2)}{2} \left( 
\sum_{j=1}^{\lceil \frac{\log (1/2)}{n \log\beta_2} \rceil  } \beta_2^{n j} 
\sum_{i=0}^{n-1}\left( \partial_l f_{\tau_{k-j,i}}(x_{k,0}) \right)^2 
\cdot \mathbb{I}(\tau_{k-j,i} = \alpha) 
\right) \\
&\geq & \frac{1 - \beta_2}{2} \cdot \beta_2^n \cdot
\frac{1 - \beta_2^{n\left\lceil \frac{\log (1/2)}{n \log \beta_2} \right\rceil }}{1 - \beta_2^n}
\left( \partial_l f_{\alpha}(x_{k,0}) \right)^2 \\
&\overset{\text{(ii)}}{\geq}& \frac{\beta_2^n}{4} \cdot   \frac{1 - \beta_2}{1 - \beta_2^n} \left( \partial_l f_{\alpha}(x_{k,0}) \right)^2\\
&\overset{\text{(iii)}}{\geq}& \frac{\beta_2^n}{4n} \left( \partial_l f_{\alpha}(x_{k,0}) \right)^2
\end{eqnarray*}}
where (i) is obtained by applying Lemma \ref{lemma_k-j_k}, whose condition $\partial_l f_\alpha (x_{k,0}) := \max_i \vert \partial_l f_i (x_{k,0}) \vert \allowbreak \geq \frac{8\sqrt{2} n \Delta_1}{\sqrt{k}} \left(\frac{\log (1/2)}{n \log\beta_2}\right)$ is satisfied.
(ii) is due to $1 - \beta_2^{n\left\lceil \frac{\log (1/2)}{n \log \beta_2} \right\rceil } \geq \frac{1}{2}$; (iii) is due to $\frac{1 - \beta_2}{1 - \beta_2^n}  \geq \frac{1}{n}$. Rearrange the inequality and we conclude the proof for Lemma \ref{lemma_f_over_v}.

\yushunrevise{We further comment on one difference between Lemma \ref{lemma_f_over_v} and the concentration results for the with-replacement sampling version, i.e., Lemma \ref{lemma_concentrate_v}. Here, under random shuffling, index $\alpha$ will occur once and only once within one epoch, while this is not guaranteed under with-replacement sampling. As a result, the conclusion of Lemma \ref{lemma_f_over_v} is deterministic, while the conclusion in Lemma \ref{lemma_concentrate_v} is not. }

\qed

\subsection{Proof of Lemma \ref{lemma_concentrate_v_k} }
\label{appendix:lemma_concentrate_v_k}

The upper bound in \eqref{eq_concentrate_v_k} is straightforward due to  $v_{l,k,i} \geq  \beta_2^n v_{l,k,\yushunrevise{-1}}$. We now prove the lower bound in \eqref{eq_concentrate_v_k}. 

{\small \begin{eqnarray*}
    \frac{1}{\sqrt{v_{l,k,i}}} &\geq&   \frac{1}{\sqrt{v_{l,k,\yushunrevise{-1}}}} \frac{1}{ \sqrt{1+ \frac{|v_{l,k,i} - v_{l,k,\yushunrevise{-1}}|} {{v_{l,k,\yushunrevise{-1}}}  }}}\\   
    &\overset{\text{(i)}}{\geq} & \frac{1}{\sqrt{v_{l,k,\yushunrevise{-1}}}}\left(1 - \frac{|v_{l,k,i} - v_{l,k,\yushunrevise{-1}}|} {2v_{l,k,\yushunrevise{-1}}}  \right),
 \end{eqnarray*}}
where $(i):  \frac{1}{\sqrt{1 +x}} \geq 1 -\frac{x}{2}$. Now we provide an upper bound of $|v_{l,k,i} - v_{l,k,\yushunrevise{-1}}|$.  

{\small
\begin{eqnarray*}
|v_{l,k,i}-v_{l,k,\yushunrevise{-1}}|
&=&\Bigl|(1-\beta_2)\Bigl(
\bigl(\partial_l f_{\tau_{k,i}}(x_{k,i})\bigr)^2 \nonumber\\
&&\quad+\beta_2\bigl(\partial_l f_{\tau_{k,i-1}}(x_{k,i-1})\bigr)^2+\cdots \nonumber\\
&&\quad+\beta_2^{\yushunrevise{i}}
\bigl(\partial_l f_{\tau_{k,\yushunrevise{0}}}(x_{k,\yushunrevise{0}})\bigr)^2\Bigr)
-(1-\beta_2^{\yushunrevise{i+1}})v_{l,k,\yushunrevise{-1}}\Bigr|\\
&=&\Bigl|(1-\beta_2)\Bigl(
\bigl(\partial_l f_{\tau_{k,i}}(x_{k,i})\bigr)^2-v_{l,k,\yushunrevise{-1}} \nonumber\\
&&\quad+\beta_2\Bigl(\bigl(\partial_l f_{\tau_{k,i-1}}(x_{k,i-1})\bigr)^2-v_{l,k,\yushunrevise{-1}}\Bigr)+\cdots \nonumber\\
&&\quad+\beta_2^{\yushunrevise{i}}
\Bigl(\bigl(\partial_l f_{\tau_{k,\yushunrevise{0}}}(x_{k,\yushunrevise{0}})\bigr)^2-v_{l,k,\yushunrevise{-1}}\Bigr)\Bigr)\Bigr|\\
&\leq&(1-\beta_2)\Bigl(
\Bigl|\bigl(\partial_l f_{\tau_{k,i}}(x_{k,i})\bigr)^2-v_{l,k,\yushunrevise{-1}}\Bigr| \nonumber\\
&&\quad+\beta_2\Bigl|\bigl(\partial_l f_{\tau_{k,i-1}}(x_{k,i-1})\bigr)^2-v_{l,k,\yushunrevise{-1}}\Bigr|+\cdots \nonumber\\
&&\quad+\beta_2^{\yushunrevise{i}}
\Bigl|\bigl(\partial_l f_{\tau_{k,\yushunrevise{0}}}(x_{k,\yushunrevise{0}})\bigr)^2-v_{l,k,\yushunrevise{-1}}\Bigr|\Bigr)
\end{eqnarray*}
}

Recall 
{\small \begin{eqnarray*}
    \left| \partial_l f_{\tau_{k,i}}(x_{k,i}) \right|
&\leq& \left| \partial_l f_{\tau_{k,i}}(x_{k,i}) - \partial_l f_{\tau_{k,i}}(x_{k,0}) \right| 
+ \left| \partial_l f_{\tau_{k,i}}(x_{k,0}) \right| \\
 &\overset{\text{Lemma \ref{lemma_delta}}}{\leq} & n \Delta_k + \left| \partial_l f_{\tau_{k,i}}(x_{k,0}) \right|,
\end{eqnarray*}}
so we have 
{\small \begin{eqnarray*}
    \left| \partial_l f_{\tau_{k,i}}(x_{k,i}) \right|^2 
&\leq& (n \Delta_k)^2 + 2n \Delta_k \left| \partial_l f_{\tau_{k,i}}(x_{k,0}) \right| 
+ \left| \partial_l f_{\tau_{k,i}}(x_{k,0}) \right|^2
\\
&\overset{\text{(i)}}{\leq}& 4 \max_i \left| \partial_l f_i(x_{k,0}) \right|^2,
\end{eqnarray*}}
where (i) is due to $\max_i \left| \partial_l f_{i}(x_{k,0}) \right| \geq \frac{8\sqrt{2} n \Delta_1}{\sqrt{k}}  \frac{\log (1/2)}{n \log\beta_2}> n \Delta_k$. Therefore, we have
{\small \begin{eqnarray*}
|v_{l,k,i}-v_{l,k,\yushunrevise{-1}}|
&\leq&(1-\beta_2)\Bigl(
\Bigl|\bigl(\partial_l f_{\tau_{k,i}}(x_{k,i})\bigr)^2-v_{l,k,\yushunrevise{-1}}\Bigr| \nonumber\\
&&\quad+\beta_2\Bigl|\bigl(\partial_l f_{\tau_{k,i-1}}(x_{k,i-1})\bigr)^2-v_{l,k,\yushunrevise{-1}}\Bigr|+\cdots \nonumber\\
&&\quad+\beta_2^{\yushunrevise{i}}
\Bigl|\bigl(\partial_l f_{\tau_{k,1}}(x_{k,1})\bigr)^2-v_{l,k,\yushunrevise{-1}}\Bigr|\Bigr)\\
&\leq&(1-\beta_2)\Bigl(
\Bigl|\bigl(\partial_l f_{\tau_{k,i}}(x_{k,i})\bigr)^2+v_{l,k,\yushunrevise{-1}}\Bigr| \nonumber\\
&&\quad+\beta_2\Bigl|\bigl(\partial_l f_{\tau_{k,i-1}}(x_{k,i-1})\bigr)^2+v_{l,k,\yushunrevise{-1}}\Bigr|+\cdots \nonumber\\
&&\quad+\beta_2^{\yushunrevise{i}}
\Bigl|\bigl(\partial_l f_{\tau_{k,1}}(x_{k,1})\bigr)^2+v_{l,k,\yushunrevise{-1}}\Bigr|\Bigr)\\
&\leq& (1-\beta_2) n \left( 4 \max_i \left| \partial_l f_i(x_{k,0}) \right|^2   + v_{l,k,\yushunrevise{-1}}\right).
\end{eqnarray*}}

Finally, we have

{\small \begin{eqnarray*}
    \frac{1}{\sqrt{v_{l,k,i}}} &\geq&   \frac{1}{\sqrt{v_{l,k,\yushunrevise{-1}}}}\left(1 - \frac{|v_{l,k,i} - v_{l,k,\yushunrevise{-1}}|} {2v_{l,k,\yushunrevise{-1}}}  \right) \\
    &\geq &\frac{1}{\sqrt{v_{l,k,\yushunrevise{-1}}}}\left(1 - \frac{(1-\beta_2) n \left( 4 \max_i \left| \partial_l f_i(x_{k,0}) \right|^2   + v_{l,k,\yushunrevise{-1}}\right)}{2 v_{l,k,\yushunrevise{-1}}} \right) \\
    &\overset{\text{Lemma \ref{lemma_f_over_v}}}{\geq} & \frac{1}{\sqrt{v_{l,k,\yushunrevise{-1}}}}\left(1 - (1-\beta_2) n (\frac{8n}{\beta_2^n} + \frac{1}{2} )\right).
 \end{eqnarray*}}
This concludes the proof of Lemma \ref{lemma_concentrate_v_k}. \qed

\subsection{Proof of Lemma \ref{lemma_concentrate_v_k-1} }
\label{appendix:lemma_concentrate_v_k-1}

The lower bound in \eqref{eq_concentrate_v_k-1} is straightforward due to  $ v_{l,k-1,i} \leq  \frac{v_{l,k,\yushunrevise{-1}}}{\beta_2^n}$. Regarding the upper bound in \eqref{eq_concentrate_v_k}, we consider

{\small \begin{eqnarray*}
   v_{l,k-1,i} & = &   v_{l,k,\yushunrevise{-1}} \cdot \frac{v_{l,k-1,i}}{v_{l,k,\yushunrevise{-1}}}\\   
    &=& v_{l,k,\yushunrevise{-1}} \cdot \frac{ v_{l,k,\yushunrevise{-1}} + v_{l,k-1,i} - v_{l,k,\yushunrevise{-1}}}{v_{l,k,\yushunrevise{-1}}} \\
   &\geq&v_{l,k,\yushunrevise{-1}} \cdot \left(1- \frac{|v_{l,k-1,i} - v_{l,k,\yushunrevise{-1}}| }{v_{l,k,\yushunrevise{-1}}} \right).
 \end{eqnarray*}}

The rest of the proof follows the same procedure as the proof of  Lemma \ref{lemma_concentrate_v_k}. Firstly,  one can show that: when $ \max_i \left| \partial_l f_{i}(x_{k,0}) \right| \geq \frac{16 n \Delta_1}{\sqrt{k}}  \frac{\log (1/2)}{n \log\beta_2}$, we have 

{\small \begin{equation}
\label{eq_k-1_to_k}
    \left(\partial_l f_{\tau_{k-1,j}} (x_{k-1,j})\right)^2 \leq  4 \max_i \left(\partial_l f_i (x_{k,0})\right)^2.
\end{equation}}

Therefore, we have \yushunrevise{the following results for $i = -1, 0, \cdots ,n-2$}
    
{\small \begin{eqnarray*}
\frac{|v_{l,k,\yushunrevise{-1}}-v_{l,k-1,i}|}{v_{l,k,\yushunrevise{-1}}}
&=&\Bigl|\frac{1-\beta_2}{v_{l,k,\yushunrevise{-1}}}\Bigl[
\Bigl(\bigl(\partial_l f_{\yushunrevise{\tau_{k-1,n-1}}}(x_{\yushunrevise{k-1,n-1}})\bigr)^2-v_{l,k-1,i}\Bigr) \nonumber\\
&&\quad+\cdots \nonumber\\
&&\quad+\beta_2^{n-i\yushunrevise{-2}}
\Bigl(\bigl(\partial_l f_{\tau_{k-1,i+1}}(x_{k-1,i+1})\bigr)^2-v_{l,k-1,i}\Bigr)\Bigr]\Bigr|\\
&\leq&\frac{1-\beta_2}{v_{l,k,\yushunrevise{-1}}}\Bigl[
\Bigl(\bigl(\partial_l f_{\yushunrevise{\tau_{k-1,n-1}}}(x_{\yushunrevise{k-1,n-1}})\bigr)^2+v_{l,k-1,i}\Bigr) \nonumber\\
&&\quad+\cdots \nonumber\\
&&\quad+\beta_2^{n-i\yushunrevise{-2}}
\Bigl(\bigl(\partial_l f_{\tau_{k-1,i+1}}(x_{k-1,i+1})\bigr)^2+v_{l,k-1,i}\Bigr)\Bigr]\\
&\overset{\eqref{eq_k-1_to_k}}{\leq}&   \frac{(1 - \beta_2)}{v_{l,k,\yushunrevise{-1}}} (1 + \cdots 
+ \beta_2^{n - i \RED{- 2}}) \left( 4\max_i \left| \partial_l f_i(x_{k,0}) \right|^2 + v_{l,k-1,i} \right) \\
&\overset{}{\leq}&  \RED{(1-\beta_2) n} \left( \frac{16n}{\beta_2^n} +\frac{1}{\beta_2^n} \right),
\end{eqnarray*}}
where the last inequality is due to  $v_{l,k,\yushunrevise{-1}}\geq \beta_2^{n\yushunrevise{-1}-i} v_{l,k-1,i}$ \yushunrevise{for  $i = -1, 0, \cdots ,n-2$,}  and Lemma \ref{lemma_f_over_v} (note that the condition of Lemma \ref{lemma_f_over_v} is satisfied  when $\max_i \left| \partial_l f_{i}(x_{k,0}) \right| \geq \frac{16 n \Delta_1}{\sqrt{k}}  \frac{\log (1/2)}{n \log\beta_2}$). In summary, we have

{\small \begin{eqnarray*}
   v_{l,k-1,i} 
   &\geq&v_{l,k,\yushunrevise{-1}} \left(1- \frac{|v_{l,k-1,i} - v_{l,k,\yushunrevise{-1}}| }{v_{l,k,\yushunrevise{-1}}} \right) \\
   &\geq &v_{l,k,\yushunrevise{-1}}  \left(1 - (1-\beta_2) n\left( \frac{16n + 1}{\beta_2^n}  \right)\right).
 \end{eqnarray*}}

 Rearrange the inequality and we conclude the proof for Lemma \ref{lemma_concentrate_v_k-1}. \qed
 
\subsection{Proof of Lemma \ref{lemma_k-k-1} }
\label{appendix:lemma_k-k-1}

 We first prove the result for $j=1$. We discuss the following two cases.

\paragraph{Case 1: when $\frac{\partial_l f(x_{k,0})}{\sqrt{v_{l,k,\yushunrevise{-1}}}} \geq \frac{\partial_l f(x_{k-1,0})}{\sqrt{v_{l,k-1,\yushunrevise{-1}}}} $:     } when $\partial_l f(x_{k,0})\leq 0$, we have

{\small $$\frac{\partial_l f(x_{k,0})}{\sqrt{v_{l,k,\yushunrevise{-1}}}}  \overset{\text{Lemma \ref{lemma_concentrate_v_k-1}}}{\leq} \frac{\partial_l f(x_{k,0})}{\sqrt{v_{l,k-1,\yushunrevise{-1}}}} \left(1 - (1-\beta_2) n\left( \RED{\frac{16n + 1}{\beta_2^n}}\right)\right)^{\frac{1}{2}} .   $$}

When $\partial_l f(x_{k,0}) > 0$, we have

{\small $$\frac{\partial_l f(x_{k,0})}{\sqrt{v_{l,k,\yushunrevise{-1}}}}  \overset{\text{Lemma \ref{lemma_concentrate_v_k-1}}}{\leq}  \frac{\partial_l f(x_{k,0})}{\sqrt{v_{l,k-1,\yushunrevise{-1}}}} \frac{1}{\sqrt{\beta_2^n}}.  $$}

In conclusion, we have 

{\scriptsize
\begin{eqnarray*}
  &\frac{\partial_l f(x_{k,0})}{\sqrt{v_{l,k,\yushunrevise{-1}}}} \\
  &\leq &\max\left\{\begin{aligned}
  &\frac{\partial_l f(x_{k,0})}{\sqrt{v_{l,k-1,\yushunrevise{-1}}}}
   \left(1-(1-\beta_2)n\left(\RED{\frac{16n+1}{\beta_2^n}}\right)\right)^{\frac12},\\
  &\frac{\partial_l f(x_{k,0})}{\sqrt{v_{l,k-1,\yushunrevise{-1}}}}
   \frac{1}{\sqrt{\beta_2^n}}
  \end{aligned}\right\}\\
  &\leq &\frac{\partial_l f(x_{k,0})}{\sqrt{v_{l,k-1,\yushunrevise{-1}}}} \nonumber\\
  &&\quad+\max\left\{\begin{aligned}
  &\frac{\partial_l f(x_{k,0})}{\sqrt{v_{l,k-1,\yushunrevise{-1}}}}
   \left[\left(1-(1-\beta_2)n\left(\RED{\frac{16n+1}{\beta_2^n}}\right)\right)^{\frac12}-1\right],\\
  &\frac{\partial_l f(x_{k,0})}{\sqrt{v_{l,k-1,\yushunrevise{-1}}}}
   \left(\frac{1}{\sqrt{\beta_2^n}}-1\right)
  \end{aligned}\right\}\\
  &\leq &\frac{\partial_l f(x_{k,0})}{\sqrt{v_{l,k-1,\yushunrevise{-1}}}} \nonumber\\
  &&\quad+\frac{\left|\partial_l f(x_{k,0})\right|}{\sqrt{v_{l,k-1,\yushunrevise{-1}}}}
  \left(\frac{1}{\sqrt{\beta_2^n}}-1
  +\left[1-\left(1-(1-\beta_2)n\left(\RED{\frac{16n+1}{\beta_2^n}}\right)\right)^{\frac12}\right]\right)\\
  &\overset{\text{Lemma \ref{lemma_concentrate_v_k-1}}}{\leq}&
  \frac{\partial_l f(x_{k,0})}{\sqrt{v_{l,k-1,\yushunrevise{-1}}}} \nonumber\\
  &&\quad+\frac{\left|\partial_l f(x_{k,0})\right|}{\sqrt{v_{l,k,\yushunrevise{-1}}}}
  \left(\frac{1}{\sqrt{\beta_2^n}}-1
  +\left[1-\left(1-(1-\beta_2)n\left(\RED{\frac{16n+1}{\beta_2^n}}\right)\right)^{\frac12}\right]\right)
  \left(1-(1-\beta_2)n\left(\RED{\frac{16n+1}{\beta_2^n}}\right)\right)^{-\frac12}
\end{eqnarray*}
}
{\footnotesize
\begin{eqnarray*}
    &\overset{\text{Lemma \ref{lemma_f_over_v}}}{\leq}&\frac{\partial_l f(x_{k,0})}{\sqrt{v_{l,k-1,\yushunrevise{-1}}}} + \sqrt{\frac{4n}{\beta_2^n}} \left(\frac{1}{\sqrt{\beta_2^n}}-1  +\left[1-\left(1 - (1-\beta_2) n\left( \RED{\frac{16n + 1}{\beta_2^n}}\right)\right)^{\frac{1}{2}} \right]\right)  \left(1 - (1-\beta_2) n\left( \RED{\frac{16n + 1}{\beta_2^n}}\right)\right)^{-\frac{1}{2}}  \\
    &:=&\frac{\partial_l f(x_{k,0})}{\sqrt{v_{l,k-1,\yushunrevise{-1}}}} +  \tilde{\delta}_1(\beta_2),
\end{eqnarray*}
}
where {\small $\tilde{\delta}_1(\beta_2) := \sqrt{\frac{4n}{\beta_2^n}} \left(\frac{1}{\sqrt{\beta_2^n}}-1  +\left[1-\left(1 - (1-\beta_2) n \left( \yushunrevise{\frac{16n+1}{\beta_2^n}}\right)\right)^{\frac{1}{2}} \right]\right)  \left(1 - (1-\beta_2) n \left( \yushunrevise{\frac{16n+1}{\beta_2^n}}\right)\right)^{-\frac{1}{2}} $} is a constant that goes to 0 when $\beta_2$ goes to 1. Therefore, we have

{\small \begin{eqnarray*}
  \frac{\partial_l f(x_{k,0})}{\sqrt{v_{l,k,\yushunrevise{-1}}}} -   \frac{\partial_l f(x_{k-1,0})}{\sqrt{v_{l,k-1,\yushunrevise{-1}}}} &\leq & \frac{\partial_l f(x_{k,0})- \partial_l f(x_{k-1,0}) }{\sqrt{v_{l,k-1,\yushunrevise{-1}}}} + \tilde{\delta_1}(\beta_2) \\
  &\overset{\text{Lemma \ref{lemma_delta}}}{ \leq } & \frac{n\Delta_{(k-1)}}{\sqrt{v_{l,k-1,\yushunrevise{-1}}}} + \tilde{\delta_1}(\beta_2).
\end{eqnarray*}}

\paragraph{Case 2: when $\frac{\partial_l f(x_{k,0})}{\sqrt{v_{l,k,\yushunrevise{-1}}}} \leq \frac{\partial_l f(x_{k-1,0})}{\sqrt{v_{l,k-1,\yushunrevise{-1}}}} $:} when $\partial_l f(x_{k,0})\geq 0$, we have

{\small $$\frac{\partial_l f(x_{k,0})}{\sqrt{v_{l,k,\yushunrevise{-1}}}}  \overset{\text{\eqref{lemma_concentrate_v_k-1}}}{\geq} \frac{\partial_l f(x_{k,0})}{\sqrt{v_{l,k-1,\yushunrevise{-1}}}} \left(1- (1-\beta_2) n \left( \yushunrevise{\frac{16n+1}{\beta_2^n}} \right)\right)^{\frac{1}{2}}.   $$}

When $\partial_l f(x_{k,0}) < 0$, we have

{\small $$\frac{\partial_l f(x_{k,0})}{\sqrt{v_{l,k,\yushunrevise{-1}}}}  \overset{}{\geq} \frac{\partial_l f(x_{k,0})}{\sqrt{v_{l,k-1,\yushunrevise{-1}}}} \frac{1}{\sqrt{\beta_2^n}} . $$}

Following the same strategy as in Case 1, we can show that

{\small \begin{eqnarray*}
  \frac{\partial_l f(x_{k,0})}{\sqrt{v_{l,k,\yushunrevise{-1}}}} &\geq & 
  \frac{\partial_l f(x_{k,0})}{\sqrt{v_{l,k-1,\yushunrevise{-1}}}} - \tilde\delta_1(\beta_2),
\end{eqnarray*}}

which further implies 

{\small \begin{eqnarray*}
  \frac{\partial_l f(x_{k,0})}{\sqrt{v_{l,k,\yushunrevise{-1}}}} -   \frac{\partial_l f(x_{k-1,0})}{\sqrt{v_{l,k-1,\yushunrevise{-1}}}} &\geq & \frac{\partial_l f(x_{k,0})- \partial_l f(x_{k-1,0}) }{\sqrt{v_{l,k-1,\yushunrevise{-1}}}} -\tilde\delta_1(\beta_2)\\
  &\overset{\text{Lemma \ref{lemma_delta}}}{ \geq } & - \frac{n\Delta_{(k-1)}}{\sqrt{v_{l,k-1,\yushunrevise{-1}}}} - \tilde\delta_1(\beta_2). 
\end{eqnarray*}}

Combining Case 1 and Case 2 together, we have 

{\small $$ \left| \frac{\partial_l f(x_{k,0})}{\sqrt{v_{l,k,\yushunrevise{-1}}}} -   \frac{\partial_l f(x_{k-1,0})}{\sqrt{v_{l,k-1,\yushunrevise{-1}}}} \right| \leq \frac{n\Delta_{(k-1)}}{\sqrt{v_{l,k-1,\yushunrevise{-1}}}} + \tilde\delta_1(\beta_2) . $$}

Now we consider the case when $j>1$. Based on the above inequality, we have

{\footnotesize \begin{eqnarray*}
  \left| \frac{\partial_l f(x_{k,0})}{\sqrt{v_{l,k,\yushunrevise{-1}}}} -   \frac{\partial_l f(x_{k-j,0})}{\sqrt{v_{l,k-j,\yushunrevise{-1}}}} \right| &\leq& \left( \frac{n\Delta_{(k-1)}}{\sqrt{v_{l,k-1,\yushunrevise{-1}}}} + \frac{n\Delta_{(k-2)}}{\sqrt{v_{l,k-2,\yushunrevise{-1}}}} + \cdots + \frac{n\Delta_{(k-j)}}{\sqrt{v_{l,k-j,\yushunrevise{-1}}}}\right) + j \tilde\delta_1(\beta_2)  \\
    &\leq &  \left( \frac{n\Delta_{(k-1)}}{\sqrt{\beta_2^{nj}v_{l,k-j,\yushunrevise{-1}}}} + \frac{n\Delta_{(k-2)}}{\sqrt{\beta_2^{nj}v_{l,k-j,\yushunrevise{-1}}}} + \cdots + \frac{n\Delta_{(k-j)}}{\sqrt{\beta_2^{nj}v_{l,k-j,\yushunrevise{-1}}}}\right)  +  j \tilde\delta_1(\beta_2)   \\
    &\leq & \frac{j}{\sqrt{\beta_2^{nj}}} \frac{n\Delta_{(k-j)}}{\sqrt{v_{l,k-j,\yushunrevise{-1}}}} +  j \tilde\delta_1(\beta_2).
\end{eqnarray*}}
This concludes the proof for Lemma \ref{lemma_k-k-1}. \qed

\subsection{Proof of Lemma \ref{lemma_II_2}}
\label{appendix:lemma_II_2}

Now we prove Lemma \ref{lemma_II_2}, which appears later in Appendix \ref{appendix_main_body_rr}. \yushunrevise{Note that  $\vert \partial_lf(x_{k,0})\vert \geq Q_k $, implies $\max_i \vert \partial_lf_i(x_{k,0})\vert \geq Q_k$, where $Q_k$ is defined in \eqref{eq_constants_wr}.} Therefore, Lemma \ref{lemma_f_over_v}  and \ref{lemma_concentrate_v_k} can be applied here. We discuss the following two cases.

\paragraph{Case 1:} When  $\partial_{l} f\left(x_{k, 0}\right)\mlki \geq 0$,  we have

{\small \begin{eqnarray*}
\partial_l f(x_{k,0})\frac{m_{l,k,i}}{\sqrt{v_{l,k,i}}}  \overset{\text{Lemma \ref{lemma_concentrate_v_k}}}{\geq}  \partial_l f(x_{k,0})\frac{m_{l,k,i}}{\sqrt{v_{l,k,\yushunrevise{-1}}}}  \left(1-(1-\beta_2) n(\frac{8n}{\beta_2^n} + \frac{1}{2} )\right).
\end{eqnarray*}}

\paragraph{Case 2:} When  $\partial_{l} f\left(x_{k, 0}\right)\mlki < 0$,  we have

{\small \begin{eqnarray*}
\partial_l f(x_{k,0})\frac{m_{l,k,i}}{\sqrt{v_{l,k,i}}}  \overset{\text{Lemma \ref{lemma_concentrate_v_k}}}{\geq}  \partial_l f(x_{k,0})\frac{m_{l,k,i}}{\sqrt{v_{l,k,\yushunrevise{-1}}}} \frac{1}{\sqrt{\beta_2^n}}.
\end{eqnarray*}}

\RED{Let $i+ := \{\, i\in\{0,\ldots,n-1\} : \partial_l f(x_{k,0})\, m_{l,k,i} \ge 0 \,\}$ and $i- := \{\, i\in\{0,\ldots,n-1\} : \partial_l f(x_{k,0})\, m_{l,k,i} < 0 \,\}$.
Accordingly, we write $\sum_{i+} (\cdot):=\sum_{i\in i+}(\cdot)$ and $\sum_{i-}(\cdot):=\sum_{i\in i-}(\cdot)$.} Combining \textbf{Case 1} and \textbf{Case 2} together, we have the following results for $l$ with $\max_i \vert \partial_lf_i(x_{k,0})\vert \geq Q_k$:

{\footnotesize \begin{eqnarray*}
    \sum_{i=0}^{n-1}\partial_l f(x_{k,0})\frac{m_{l,k,i}}{\sqrt{v_{l,k,i}}}
    &=&\sum_{i+}\partial_l f(x_{k,0})\frac{m_{l,k,i}}{\sqrt{v_{l,k,i}}}\nonumber\\
    &&\quad+\sum_{i-}\partial_l f(x_{k,0})\frac{m_{l,k,i}}{\sqrt{v_{l,k,i}}}\\
    &\geq&\sum_{i+}\partial_l f(x_{k,0})
    \frac{m_{l,k,i}}{\sqrt{v_{l,k,\yushunrevise{-1}}}}
    \left(1-(1-\beta_2)n\left(\frac{8n}{\beta_2^n}+\frac12\right)\right)\nonumber\\
    &&\quad+\sum_{i-}\partial_l f(x_{k,0})
    \frac{m_{l,k,i}}{\sqrt{v_{l,k,\yushunrevise{-1}}}}
    \frac{1}{\sqrt{\beta_2^n}}\\
    &=&\sum_{i=0}^{n-1}\partial_l f(x_{k,0})
    \frac{m_{l,k,i}}{\sqrt{v_{l,k,\yushunrevise{-1}}}}\nonumber\\
    &&\quad+\sum_{i+}\partial_l f(x_{k,0})
    \frac{m_{l,k,i}}{\sqrt{v_{l,k,\yushunrevise{-1}}}}
    \left(-(1-\beta_2)n\left(\frac{8n}{\beta_2^n}+\frac12\right)\right)\\
    &&+ \sum_{i-} \partial_l f(x_{k,0})\frac{m_{l,k,i}}{\sqrt{v_{l,k,\yushunrevise{-1}}}}  \left(\frac{1}{\sqrt{\beta_2^n}} -1\right) \\
    &\geq&\sum_{i=0}^{n-1}\partial_l f(x_{k,0})
    \frac{m_{l,k,i}}{\sqrt{v_{l,k,\yushunrevise{-1}}}}\nonumber\\
    &&\quad-\tilde{\delta}_2(\beta_2)\sum_{i=0}^{n-1}
    \left\vert\frac{\partial_l f(x_{k,0})}{\sqrt{v_{l,k,\yushunrevise{-1}}}}\right\vert
    \left\vert m_{l,k,i}\right\vert\\
    &\overset{\text{Lemma \ref{lemma_f_over_v} and \ref{lemma_upper_m}}}{\geq}&
    \sum_{i=0}^{n-1}\partial_l f(x_{k,0})
    \frac{m_{l,k,i}}{\sqrt{v_{l,k,\yushunrevise{-1}}}}\nonumber\\
    &&\quad-\tilde{\delta}_2(\beta_2)2\sqrt{\frac{n}{\beta_2^n}}
    \left(C_{m,1}\left(\vert\partial_\alpha f(x_{k,0})\vert
    +\sqrt{\frac{D_0}{D_1d}}\right)+\frac{C_{m,2}}{\sqrt{k}}\right),
\end{eqnarray*}}
where $\tilde{\delta}_2(\beta_2 ) = (1-\beta_2) n(\frac{8n}{\beta_2^n} + \frac{1}{2} ) +  \left(\frac{1}{\sqrt{\beta_2^n}} -1\right)$ and $C_{m,1}, C_{m,2}$ are defined in Lemma \ref{lemma_upper_m}. We have

{\footnotesize \begin{eqnarray*}
  \text{(II-2)}  & = &\frac{\eta_k}{1 - \beta_1^n} \sum_{l=1}^d  
 \sum_{i=0}^{n-1} \partial_l f(x_{k,0})\frac{m_{l,k,i}}{\sqrt{v_{l,k,i}}}
\yushunrevise{\mathbb{I} \left( \vert \partial_lf(x_{k,0})\vert \geq Q_k \right)}  \\
&\geq & \frac{\eta_k}{1 - \beta_1^n} \sum_{l=1}^d\sum_{i=0}^{n-1}  \partial_l f(x_{k,0})\frac{m_{l,k,i}}{\sqrt{v_{l,k,\yushunrevise{-1}}}} \yushunrevise{\mathbb{I} \left( \vert \partial_lf(x_{k,0})\vert \geq Q_k \right)} \\
&&- \frac{d\eta_k}{1 - \beta_1^n} \tilde{\delta}_2(\beta_2) 2\sqrt{\frac{n}{\beta_2^n}} \left( C_{m,1}\left(\vert \partial_\alpha f(x_{k,0})\vert  + \sqrt{\frac{D_0}{D_1d } }\right) + \frac{C_{m,2}}{\sqrt{k}}\right)\\
&:=& \frac{\eta_k}{1 - \beta_1^n} \sum_{l=1}^d\sum_{i=0}^{n-1}  \partial_l f(x_{k,0})\frac{m_{l,k,i}}{\sqrt{v_{l,k,\yushunrevise{-1}}}} \yushunrevise{\mathbb{I} \left( \vert \partial_lf(x_{k,0})\vert \geq Q_k \right)} - \eta_k \tilde{\delta}_3(\beta_2)\left(\vert \partial_\alpha f(x_{k,0})\vert  + \sqrt{\frac{D_0}{D_1d } }\right) -\frac{\widetilde{C}_{4}}{k},
\end{eqnarray*}}
where 
{\small $$\tilde{\delta}_3(\beta_2) =\frac{d}{1 - \beta_1^n} \tilde{\delta}_2(\beta_2) 2\sqrt{\frac{n}{\beta_2^n}} \cdot  C_{m,1}, \quad \widetilde{C}_{4} = \frac{d\eta_0}{1 - \beta_1^n}  \tilde{\delta}_2(\beta_2) 2\sqrt{\frac{n}{\beta_2^n}} \cdot  C_{m,2},$$}
where $C_{m,1}$ and $C_{m,2}$ are defined in Lemma \ref{lemma_upper_m}. Note that  $\tilde{\delta}_3(\beta_2)$ is a constant that approaches 0 when $\beta_2$ approaches 1.  This concludes the proof for (II-2). The proof of (II-3) follows the exact same procedure using Lemma \ref{lemma_concentrate_v_k-1} and we omit the proof for brevity. \qed

\subsection{Proof of Lemma \ref{lemma_III}}
\label{appendix:lemma_III}

Now we prove Lemma \ref{lemma_III}, which appears later in Appendix \ref{appendix_main_body_rr}. We consider those $l \in [d]$ such that  $ \vert \partial_lf(x_{k,0})\vert \geq Q_k$. \yushunrevise{Note that 
$\vert \partial_lf(x_{k,0})\vert \geq Q_k $,  implies $\max_i \vert \partial_lf_i(x_{k,0})\vert \geq Q_k$, where $Q_k$ is defined in \eqref{eq_constants_wr}.} Therefore, Lemma \ref{lemma_f_over_v}, \ref{lemma_concentrate_v_k}, and \ref{appendix:lemma_k-k-1} can be applied here.  We have:
{\small
\begin{align*}
\mathbb{E}_{k-1}\mathbb{E}_k[\text{(III)}]
&=\frac{\eta_k}{1-\beta_1^n}\sum_{l=1}^d
\mathbb{E}_{k-1}\mathbb{E}_k\Bigl[\\
&\quad\frac{\partial_l f(x_{k,0})}{\sqrt{v_{l,k,\yushunrevise{-1}}}}
\Bigl(\sum_{i=0}^{n-1}(m_{l,k,i}-\beta_1^n m_{l,k-1,i})\Bigr)
\yushunrevise{\mathbb{I}\left(\vert\partial_lf(x_{k,0})\vert\geq Q_k\right)}\Bigr]\\
&=\text{(III-1)}+\text{(III-2)}.
\end{align*}
{\small \begin{align*}
\text{(III-1)}
&:=\frac{\eta_k}{1-\beta_1^n}\sum_{l=1}^d
\mathbb{E}_{k-1}\mathbb{E}_k\Bigl[\\
&\quad\frac{\partial_l f(x_{k,0})}{\sqrt{v_{l,k,\yushunrevise{-1}}}}
\sum_{i=0}^{n-1}(1-\beta_1)
\Bigl(\partial_l f_{\tau_{k,i}}(x_{k,i})
+\beta_1\partial_l f_{\tau_{k,i-1}}(x_{k,i-1})\\
&\qquad+\cdots+\beta_1^i\partial_l f_{\tau_{k,0}}(x_{k,0})\Bigr)
\yushunrevise{\mathbb{I}\left(\vert\partial_lf(x_{k,0})\vert\geq Q_k\right)}\Bigr],\\
\text{(III-2)}
&:=\frac{\eta_k}{1-\beta_1^n}\sum_{l=1}^d
\mathbb{E}_{k-1}\mathbb{E}_k\Bigl[\\
&\quad\frac{\partial_l f(x_{k,0})}{\sqrt{v_{l,k,\yushunrevise{-1}}}}
\sum_{i=0}^{n-2}(1-\beta_1)\beta_1^{i+1}
\Bigl(\partial_l f_{\tau_{k-1,n-1}}(x_{k-1,n-1})+\cdots\\
&\qquad+\beta_1^{n-i-2}\partial_l f_{\tau_{k-1,i+1}}(x_{k-1,i+1})\Bigr)
\yushunrevise{\mathbb{I}\left(\vert\partial_lf(x_{k,0})\vert\geq Q_k\right)}\Bigr].
\end{align*}}
}

We first provide a lower bound for (III-1).
{\small
\begin{flalign*}
&{\text{(III-1)}}\overset{\text{Lemma \ref{lemma_delta}}}{\geq}
\frac{\eta_{k}}{1-\beta_1^n}\sum_{l=1}^d\mathbb{E}_{k-1}\mathbb{E}_k\biggl[
\frac{\partial_l f(x_{k,0})}{\sqrt{v_{l,k,\yushunrevise{-1}}}}
\sum_{i=0}^{n-1}(1-\beta_1) &&\\
&\qquad\left(
\partial_l f_{\tau_{k,i}}(x_{k,0})
+\beta_1\partial_l f_{\tau_{k,i-1}}(x_{k,0})
+\cdots+\beta_1^i\partial_l f_{\tau_{k,0}}(x_{k,0})
\right)\yushunrevise{\mathbb{I}\left(\vert\partial_lf(x_{k,0})\vert\geq Q_k\right)}\biggr] &&\\
&\qquad-\frac{\eta_{k}}{1-\beta_1^n}\sum_{l=1}^d\mathbb{E}_{k-1}\mathbb{E}_k\left[
\left|\frac{\partial_l f(x_{k,0})}{\sqrt{v_{l,k,\yushunrevise{-1}}}}\right|
\yushunrevise{\mathbb{I}\left(\vert\partial_lf(x_{k,0})\vert\geq Q_k\right)}
\right]\cdot n^3\Delta_{k} &&\\
&\overset{\text{Lemma \ref{lemma_f_over_v}}}{\geq}
\frac{\eta_{k}}{1-\beta_1^n}\sum_{l=1}^d\mathbb{E}_{k-1}\mathbb{E}_k\biggl[
\frac{\partial_l f(x_{k,0})}{\sqrt{v_{l,k,\yushunrevise{-1}}}}
\sum_{i=0}^{n-1}(1-\beta_1) &&\\
&\qquad\left(
\partial_l f_{\tau_{k,i}}(x_{k,0})
+\beta_1\partial_l f_{\tau_{k,i-1}}(x_{k,0})
+\cdots+\beta_1^i\partial_l f_{\tau_{k,0}}(x_{k,0})
\right)\yushunrevise{\mathbb{I}\left(\vert\partial_lf(x_{k,0})\vert\geq Q_k\right)}\biggr] &&\\
&\qquad-\frac{\eta_{k}}{1-\beta_1^n}\cdot d\cdot 2\sqrt{\frac{n}{\beta_2^n}}\cdot n^3\Delta_{k} &&
\end{flalign*}
{\tiny \begin{eqnarray}
&\overset{\text{(i)}}{=}& \frac{\eta_{k}}{1 - \beta_1^n} \sum_{l=1}^d \mathbb{E}_{k-1} \left[\frac{\partial_l f(x_{k,0})}{\sqrt{v_{l,k,\yushunrevise{-1}}}}
\cdot \partial_l f(x_{k,0})
\sum_{i=0}^{n-1} (1 - \beta_1) (1 + \beta_1 + \cdots + \beta_1^i) \yushunrevise{\mathbb{I} \left( \vert \partial_lf(x_{k,0})\vert \geq Q_k \right)} \right]
- \frac{\eta_{k}}{1 - \beta_1^n} \cdot d \cdot 2\sqrt{\frac{n}{\beta_2^n}} \cdot n^3 \Delta_{k} \nonumber
\\
&\geq& \frac{\eta_{k}}{1 - \beta_1^n} \sum_{l=1}^d\mathbb{E}_{k-1} \left[\frac{(\partial_l f(x_{k,0}))^2}{ \sqrt{v_{l,k,\yushunrevise{-1}}}}
\cdot (1 - \beta_1)(1 + \beta_1 + \cdots + \beta_1^{n-1})\yushunrevise{\mathbb{I} \left( \vert \partial_lf(x_{k,0})\vert \geq Q_k \right)} \right]
- \frac{\eta_{k}}{1 - \beta_1^n} \cdot d \cdot 2\sqrt{\frac{n}{\beta_2^n}} \cdot n^3 \Delta_{k} \nonumber \\ 
&= & \eta_k \sum_{l=1}^d\mathbb{E}_{k-1} \left[\frac{(\partial_l f(x_{k,0}))^2}{ \sqrt{v_{l,k,\yushunrevise{-1}}}}\yushunrevise{\mathbb{I} \left( \vert \partial_lf(x_{k,0})\vert \geq Q_k \right)} \right]
-  \frac{\eta_{k}}{1 - \beta_1^n} \cdot d  \cdot 2\sqrt{\frac{n}{\beta_2^n}} \cdot n^3 \Delta_{k}, \label{eq_lower_bound_III_1}
\end{eqnarray}}
}
\RED{where (i) follows by conditioning on the entire history up to $x_{k,0}$ while excluding the fresh random permutation of indices within epoch $k$. Under $\mathbb E_k[\cdot]$, the quantities $\partial_l f(x_{k,0})$ and $v_{l,k,-1}$ are treated as fixed, so we have
{\small \[
\mathbb E_k\!\left[\frac{\partial_l f(x_{k,0})}{\sqrt{v_{l,k,-1}}}\,\partial_l f_{\tau_{k,i}}(x_{k,0}) \yushunrevise{\mathbb{I} \left( \vert \partial_lf(x_{k,0})\vert \geq Q_k \right)} \right]
=
\frac{\partial_l f(x_{k,0})}{\sqrt{v_{l,k,-1}}}\,\yushunrevise{\mathbb{I} \left( \vert \partial_lf(x_{k,0})\vert \geq Q_k \right)} 
\mathbb E_k[\partial_l f_{\tau_{k,i}}(x_{k,0})].
\]}
Moreover, conditional on this history, for any fixed $i \in[n]$, $\tau_{k,i}$ is uniformly distributed over $[n]$, hence
$\mathbb E_k[\partial_l f_{\tau_{k,i}}(x_{k,0})]=\partial_l f(x_{k,0})$,} which concludes the proof for (i).

Now we bound (III-2).

{\scriptsize
\begin{flalign}
&\text{(III-2)}=\frac{\eta_{k}}{1 - \beta_1^n} \sum_{l=1}^d \mathbb{E}_{k-1}\mathbb{E}_k \biggl[
\frac{\partial_l f(x_{k,0})}{\sqrt{v_{l,k,\yushunrevise{-1}}}}
\sum_{i=0}^{n-2} (1 - \beta_1)  \beta_1^{i+1} &&\nonumber\\
&\qquad\left(\partial_l f_{\tau_{k-1,n-1}}(x_{k-1,n-1})+\cdots
+ \beta_1^{n-i-2} \partial_l f_{\tau_{k-1,i+1}}(x_{k-1,i+1})\right)
\yushunrevise{\mathbb{I} \left( \vert \partial_lf(x_{k,0})\vert \geq Q_k \right)}\biggr] &&\nonumber\\
&\overset{\text{Lemma \ref{lemma_delta}}}{\geq}\frac{\eta_{k}}{1 - \beta_1^n}\sum_{l=1}^d\mathbb{E}_{k-1}\mathbb{E}_k\biggl[
\frac{\partial_l f(x_{k,0})}{\sqrt{v_{l,k,\yushunrevise{-1}}}}
(1 - \beta_1) \sum_{i=0}^{n-2} \beta_1^{i+1} &&\nonumber\\
&\qquad\left(\partial_l f_{\tau_{k-1,n-1}}(x_{k,0})+\cdots
+ \beta_1^{n-i-2} \partial_l f_{\tau_{k-1,i+1}}(x_{k,0})\right)
\yushunrevise{\mathbb{I} \left( \vert \partial_lf(x_{k,0})\vert \geq Q_k \right)}\biggr] &&\nonumber\\
&\qquad- \frac{\eta_{k}}{1 - \beta_1^n} \sum_{l=1}^d  \mathbb{E}_{k-1} \mathbb{E}_k \left[
\left| \frac{\partial_l f(x_{k,0})}{\sqrt{v_{l,k,\yushunrevise{-1}}}} \right| \yushunrevise{\mathbb{I} \left( \vert \partial_lf(x_{k,0})\vert \geq Q_k \right)}
\right] \cdot n^3 \Delta_{k-1} &&\nonumber\\
&\overset{\text{Lemma \ref{lemma_f_over_v}}}{\geq}
\frac{\eta_{k}}{1 - \beta_1^n}\sum_{l=1}^d\mathbb{E}_{k-1}\mathbb{E}_k\biggl[ &&\nonumber\\
&\qquad\underbrace{\frac{\partial_l f(x_{k,0})}{\sqrt{v_{l,k,\yushunrevise{-1}}}}
(1 - \beta_1) \sum_{i=0}^{n-2} \beta_1^{i+1}
\left(\partial_l f_{\tau_{k-1,n-1}}(x_{k,0})+\cdots
+\beta_1^{n-i-2}\partial_l f_{\tau_{k-1,i+1}}(x_{k,0})\right)}_{\text{(III-3)}}
\yushunrevise{\mathbb{I} \left( \vert \partial_lf(x_{k,0})\vert \geq Q_k \right)}\biggr] &&\nonumber\\
&\qquad-\frac{\eta_{k}}{1 - \beta_1^n}\cdot d\cdot 2\sqrt{\frac{n}{\beta_2^n}}\cdot n^3\Delta_{k-1}. &&\label{eq_lower_bound_III_2}
\end{flalign}
}

We further have:
{\scriptsize\setlength{\arraycolsep}{0pt}
\begin{eqnarray}
    &&\mathbb{E}_{k-1}  \mathbb{E}_k \left[   \text{(III-3)} \yushunrevise{\mathbb{I} \left( \vert \partial_lf(x_{k,0})\vert \geq Q_k \right)}  \right] \nonumber \\
    & = &   \mathbb{E}_{k-1}  \mathbb{E}_k \left[   \text{(III-3)} \yushunrevise{\mathbb{I} \left( \vert \partial_lf(x_{k,0})\vert \geq Q_k, \text{ and } \vert \partial_lf(x_{k-1,0})\vert \geq 2Q_{k} \right)}  \right] \nonumber\\
    && + \mathbb{E}_{k-1}  \mathbb{E}_k  \left[   \text{(III-3)} \yushunrevise{\mathbb{I} \left( \vert \partial_lf(x_{k,0})\vert \geq Q_k, \text{ and }  \vert\partial_lf(x_{k-1,0})\vert < 2Q_{k} \right)}  \right] \nonumber \\
    & \overset{\substack{\textbf{Lemma}\\\textbf{\ref{lemma_f_over_v}}}}{\geq}&    \mathbb{E}_{k-1}  \mathbb{E}_k \left[   \text{(III-3)} \yushunrevise{\mathbb{I} \left( \vert \partial_lf(x_{k,0})\vert \geq Q_k, \text{ and } \vert \partial_lf(x_{k-1,0})\vert \geq 2Q_{k} \right)}  \right] \nonumber\\
    && - 2\sqrt{\frac{n}{\beta_2^n}} \cdot (1 - \beta_1) \sum_{i=0}^{n-2} \beta_1^{i+1}
\left(1+ \cdots + \beta_1^{n-i-2}
\right)    \mathbb{E}_{k-1}  \mathbb{E}_k \left[\begin{aligned}
&\vert \partial_lf(x_{k,0})\vert\\
&\quad\yushunrevise{ \cdot \mathbb{I} \left( \vert \partial_lf(x_{k,0})\vert \geq Q_k, \text{ and }  \vert\partial_lf(x_{k-1,0})\vert < 2Q_{k} \right)}
\end{aligned}\right]\nonumber \\
& \overset{\text{(ii)}}{\geq} &\mathbb{E}_{k-1}  \mathbb{E}_k \left[   \text{(III-3)} \yushunrevise{\mathbb{I} \left( \vert \partial_lf(x_{k,0})\vert \geq Q_k, \text{ and } \vert \partial_lf(x_{k-1,0})\vert \geq 2Q_{k} \right)}  \right] \nonumber\\
 && - 2\sqrt{\frac{n}{\beta_2^n}} \cdot (1 - \beta_1) \sum_{i=0}^{n-2} \beta_1^{i+1}
\left(1+ \cdots + \beta_1^{n-i-2} 
\right) \cdot 3Q_k, \label{eq_lower_bound_III_3_1}
\end{eqnarray}
}
where (ii) is because  $\vert\partial_lf(x_{k-1,0})\vert < 2Q_{k}$ implies $\vert\partial_lf(x_{k,0})\vert < 3Q_{k}$ by Lemma \ref{lemma_delta}. Now we further deal with (III-3) under the condition that  $\vert \partial_lf(x_{k,0})\vert \geq Q_k, \text{ and } \vert \partial_lf(x_{k-1,0})\vert \geq 2Q_{k}$.

{\scriptsize
\begin{eqnarray*}
    \text{(III-3)} 
&\overset{\text{Lemma \ref{lemma_delta}}}{\geq}&
\frac{\partial_l f(x_{k,0})}{\sqrt{v_{l,k,\yushunrevise{-1}}}}
(1 - \beta_1) \sum_{i=0}^{n-2} \beta_1^{i+1}
\left(
\partial_l f_{\tau_{k-1,n-1}}(x_{k-1,0}) + \cdots + \beta_1^{n-i-2} \partial_l f_{\tau_{k-1,i+1}}(x_{k-1,0})
\right)
\\
&&-
\left| \frac{\partial_l f(x_{k,0})}{\sqrt{v_{l,k,\yushunrevise{-1}}}} \right|
\cdot n^3 \Delta_{k-1}
\\
&\geq&
\frac{\partial_l f(x_{k-1,0})}{\sqrt{v_{l,k-1,\yushunrevise{-1}}}}
(1 - \beta_1) \sum_{i=0}^{n-2} \beta_1^{i+1}
\left(
\partial_l f_{\tau_{k-1,n-1}}(x_{k-1,0}) + \cdots + \beta_1^{n-i-2} \partial_l f_{\tau_{k-1,i+1}}(x_{k-1,0})
\right)
\\
&&- \left|\begin{aligned}
&\frac{\partial_l f(x_{k-1,0})}{\sqrt{v_{l,k-1,\yushunrevise{-1}}}}\\
&- \frac{\partial_l f(x_{k,0})}{\sqrt{v_{l,k,\yushunrevise{-1}}}}
\end{aligned}\right|
 \left|(1 - \beta_1)\sum_{i=0}^{n-2} \beta_1^{i+1}
\left(
\partial_l f_{\tau_{k-1,n-1}}(x_{k-1,0}) + \cdots + \beta_1^{n-i-2} \partial_l f_{\tau_{k-1,i+1}}(x_{k-1,0})
\right) \right|
\\
&&-
\left| \frac{\partial_l f(x_{k,0})}{\sqrt{v_{l,k,\yushunrevise{-1}}}} \right| \cdot n^3 \Delta_{k-1} 
\end{eqnarray*}
}
{\scriptsize
\begin{eqnarray*}
&\overset{\text{Lemma \ref{lemma_k-k-1}}}{\geq}&  
\frac{\partial_l f(x_{k-1,0})}{\sqrt{v_{l,k-1,\yushunrevise{-1}}}}
(1 - \beta_1) \sum_{i=0}^{n-2} \beta_1^{i+1}
\left(
\partial_l f_{\tau_{k-1,n-1}}(x_{k-1,0}) + \cdots + \beta_1^{n-i-2} \partial_l f_{\tau_{k-1,i+1}}(x_{k-1,0})
\right)
\\
&& - \frac{1}{\sqrt{\beta_2^n}} \cdot \frac{\Delta_{k-1}} {\sqrt{v_{l,k-1,\yushunrevise{-1}}}}
(1 - \beta_1)\sum_{i=0}^{n-2} \beta_1^{i+1}
\left(
\left|\partial_l f_{\tau_{k-1,n-1}}(x_{k-1,0})\right| + \cdots + \left|\beta_1^{n-i-2} \partial_l f_{\tau_{k-1,i+1}}(x_{k-1,0})\right| 
\right) \\
&& - \tilde{\delta}_1(\beta_2)
  (1 - \beta_1)\sum_{i=0}^{n-2} \beta_1^{i+1}
\left(
\left|\partial_l f_{\tau_{k-1,n-1}}(x_{k-1,0})\right| + \cdots + \left|\beta_1^{n-i-2} \partial_l f_{\tau_{k-1,i+1}}(x_{k-1,0})\right| 
\right)\\
&&-
\left| \frac{\partial_l f(x_{k,0})}{\sqrt{v_{l,k,\yushunrevise{-1}}}} \right|
\cdot n^3 \Delta_{k-1} 
\end{eqnarray*}
}
{\small
\begin{eqnarray}
&\overset{\text{(i)}}{\geq}& 
\frac{\partial_l f(x_{k-1,0})}{\sqrt{v_{l,k-1,\yushunrevise{-1}}}}
(1 - \beta_1) \sum_{i=0}^{n-2} \beta_1^{i+1}
\left(
\partial_l f_{\tau_{k-1,n-1}}(x_{k-1,0}) + \cdots + \beta_1^{n-i-2} \partial_l f_{\tau_{k-1,i+1}}(x_{k-1,0})
\right)
 \nonumber \\
&&- \frac{1}{\sqrt{\beta_2^n}} \cdot \Delta_{k-1} \cdot (1-\beta_1)  \cdot 2\sqrt{\frac{n}{\beta_2^n}} \cdot n^2
\nonumber \\
&&- \tilde{\delta}_1(\beta_2) n^2 \left( n\Delta_{k} + \sqrt{D_{1}} \sqrt{n} d\left( \left|\partial_{\alpha} f\left(x_{k, 0}\right)\right|+\sqrt{\frac{D_{0}}{D_{1} d}}\right)\right)
\nonumber \\
&&- 2\sqrt{\frac{n}{\beta_2^n}} \cdot n^3 \cdot \Delta_{k-1}, \label{eq_lower_bound_III_3}
\end{eqnarray}
}
\begingroup\emergencystretch=2em where (i) is due to  \text{Lemma \ref{lemma_f_over_v} and  Lemma \ref{lemma_fi_f}}. Since  {\footnotesize $\yushunrevise{\mathbb{I} \left( \vert \partial_lf(x_{k,0})\vert \geq Q_k, \text{ and } \vert \partial_lf(x_{k-1,0})\vert \geq 2Q_{k} \right)}  \allowbreak = \yushunrevise{\mathbb{I} \left( \vert \partial_lf(x_{k-1,0})\vert \geq 2Q_{k} \right)} $}, we have\par\endgroup

{\small \begin{eqnarray*}
   && \frac{\eta_{k}}{1 - \beta_1^n}  \sum_{l=1}^d  \mathbb{E}_{k-1}  \mathbb{E}_k \left[   \text{(III-3)} \yushunrevise{\mathbb{I} \left( \vert \partial_lf(x_{k,0})\vert \geq Q_k, \text{ and } \vert \partial_lf(x_{k-1,0})\vert \geq 2Q_{k} \right)}  \right]  \\
   &= &   \frac{\eta_{k}}{1 - \beta_1^n}  \sum_{l=1}^d   \mathbb{E}_{k-1}  \mathbb{E}_k \left[   \text{(III-3)} \yushunrevise{\mathbb{I} \left( \vert \partial_lf(x_{k-1,0})\vert \geq 2 Q_k \right)}  \right]  \\
    &\overset{\eqref{eq_lower_bound_III_3}}{\geq}&
    \frac{\eta_k}{1-\beta_1^n}\sum_{l=1}^d
    \mathbb{E}_{k-1}\mathbb{E}_k\Bigl[ \nonumber\\
    &&\quad\frac{\partial_l f(x_{k-1,0})}{\sqrt{v_{l,k-1,\yushunrevise{-1}}}}
    (1-\beta_1)\sum_{i=0}^{n-2}\beta_1^{i+1}
    \Bigl(\partial_l f_{\tau_{k-1,n-1}}(x_{k-1,0}) \nonumber\\
    &&\quad+\cdots+\beta_1^{n-i-2}
    \partial_l f_{\tau_{k-1,i+1}}(x_{k-1,0})\Bigr)
    \yushunrevise{\mathbb{I}\left(\vert\partial_lf(x_{k-1,0})\vert\geq 2Q_k\right)}\Bigr]\\
&&-  \frac{\eta_{k}}{1 - \beta_1^n} \cdot d \cdot  \frac{1}{\sqrt{\beta_2^n}} \cdot \Delta_{k-1} \cdot (1-\beta_1)  \cdot 2\sqrt{\frac{n}{\beta_2^n}} \cdot n^2
\nonumber \\
&&-  \frac{\eta_{k}}{1 - \beta_1^n} \cdot d \cdot\tilde{\delta}_1(\beta_2) n^2 \left( n\Delta_{k} + \sqrt{D_{1}} \sqrt{n} d\left( \mathbb{E}_{k-1} \left[\left|\partial_{\alpha} f\left(x_{k, 0}\right)\right| \right]+\sqrt{\frac{D_{0}}{D_{1} d}}\right)\right)
\nonumber \\
&&-  \frac{\eta_{k}}{1 - \beta_1^n} \cdot d \cdot 2\sqrt{\frac{n}{\beta_2^n}} \cdot n^3 \cdot \Delta_{k-1} 
\end{eqnarray*}}
{\small \begin{eqnarray}
    &\overset{\text{(iii)}}{\geq}  &  \frac{\eta_{k}}{1 - \beta_1^n}  \sum_{l=1}^d \left[  \frac{\partial_l f(x_{k-1,0})^2}{\sqrt{v_{l,k-1,\yushunrevise{-1}}}}
(1 - \beta_1) \sum_{i=0}^{n-2} \beta_1^{i+1}
\left(1 + \cdots + \beta_1^{n-i-2} 
\right) \yushunrevise{\mathbb{I} \left( \vert \partial_lf(x_{k-1,0})\vert \geq 2 Q_k \right)}  \right] \nonumber \\
&&-  \frac{\eta_{k}}{1 - \beta_1^n} \cdot d \cdot  \frac{1}{\sqrt{\beta_2^n}} \cdot \Delta_{k-1} \cdot (1-\beta_1)  \cdot 2\sqrt{\frac{n}{\beta_2^n}} \cdot n^2
\nonumber \\
&&-  \frac{\eta_{k}}{1 - \beta_1^n} \cdot d \cdot\tilde{\delta}_1(\beta_2) n^2 \left( n\Delta_{k} + \sqrt{D_{1}} \sqrt{n} d\left( \mathbb{E}_{k-1} \left[\left|\partial_{\alpha} f\left(x_{k, 0}\right)\right| \right]+\sqrt{\frac{D_{0}}{D_{1} d}}\right)\right)
\nonumber \\
&&-  \frac{\eta_{k}}{1 - \beta_1^n} \cdot d \cdot 2\sqrt{\frac{n}{\beta_2^n}} \cdot n^3 \cdot \Delta_{k-1}  \nonumber \\
&\geq &-  \frac{\eta_{k}}{1 - \beta_1^n} \cdot d \cdot  \frac{1}{\sqrt{\beta_2^n}} \cdot \Delta_{k-1} \cdot (1-\beta_1)  \cdot 2\sqrt{\frac{n}{\beta_2^n}} \cdot n^2
\nonumber \\
&&-  \frac{\eta_{k}}{1 - \beta_1^n} \cdot d \cdot\tilde{\delta}_1(\beta_2) n^2 \left( n\Delta_{k} + \sqrt{D_{1}} \sqrt{n} d\left( \mathbb{E}_{k-1} \left[\left|\partial_{\alpha} f\left(x_{k, 0}\right)\right| \right]+\sqrt{\frac{D_{0}}{D_{1} d}}\right)\right)
\nonumber \\
&&-  \frac{\eta_{k}}{1 - \beta_1^n} \cdot d \cdot 2\sqrt{\frac{n}{\beta_2^n}} \cdot n^3 \cdot \Delta_{k-1} ,\label{eq_lower_bound_III_3_2}
\end{eqnarray}}
where (iii) is due to the property of random permutation: $\Ex_{k-1}(\partial_l f_{\tau_{k-1,i}}(x_{k-1,0})) = \partial_l f(x_{k-1,0})$ for any $i\in [n]$.

Plugging the above results back into (III-2), we have 

{\small \begin{eqnarray}
       \text{(III-2)}  &\overset{\eqref{eq_lower_bound_III_2}}{\geq}& \frac{\eta_{k}}{1 - \beta_1^n}  \sum_{l=1}^d  \mathbb{E}_{k-1}  \mathbb{E}_k \left[ \text{(III-3)} \yushunrevise{\mathbb{I} \left( \vert \partial_lf(x_{k,0})\vert \geq Q_k \right)} 
\right] \nonumber\\
&&- \frac{\eta_{k}}{1 - \beta_1^n} \cdot d\cdot 2\sqrt{\frac{n}{\beta_2^n}}  \cdot n^3 \Delta_{k-1} \nonumber\\
& \overset{\eqref{eq_lower_bound_III_3_1}}{\geq} & \frac{\eta_{k}}{1 - \beta_1^n}  \sum_{l=1}^d \mathbb{E}_{k-1}   \mathbb{E}_k \left[   \text{(III-3)} \yushunrevise{\mathbb{I} \left( \vert \partial_lf(x_{k,0})\vert \geq Q_k, \text{ and } \vert \partial_lf(x_{k-1,0})\vert \geq 2Q_{k} \right)}  \right] \nonumber\\ 
 && - \frac{\eta_{k}}{1 - \beta_1^n}  \cdot d \cdot 2\sqrt{\frac{n}{\beta_2^n}} \cdot (1 - \beta_1) \sum_{i=0}^{n-2} \beta_1^{i+1}
\left(1+ \cdots + \beta_1^{n-i-2} 
\right) \cdot 3Q_k \nonumber\\
&&- \frac{\eta_{k}}{1 - \beta_1^n} \cdot d\cdot 2\sqrt{\frac{n}{\beta_2^n}}  \cdot n^3 \Delta_{k-1} \nonumber \\
&\overset{\eqref{eq_lower_bound_III_3_2}}{\geq}&
-  \frac{\eta_{k}}{1 - \beta_1^n} \cdot d \cdot  \frac{1}{\sqrt{\beta_2^n}} \cdot \Delta_{k-1} \cdot (1-\beta_1)  \cdot 2\sqrt{\frac{n}{\beta_2^n}} \cdot n^2
\nonumber \\
&&-  \frac{\eta_{k}}{1 - \beta_1^n} \cdot d \cdot\tilde{\delta}_1(\beta_2) n^2 \left( n\Delta_{k} + \sqrt{D_{1}} \sqrt{n} d\left( \mathbb{E}_{k-1} \left[\left|\partial_{\alpha} f\left(x_{k, 0}\right)\right| \right]+\sqrt{\frac{D_{0}}{D_{1} d}}\right)\right)
\nonumber \\
&&-  \frac{\eta_{k}}{1 - \beta_1^n} \cdot d \cdot 2\sqrt{\frac{n}{\beta_2^n}} \cdot n^3 \cdot \Delta_{k-1} \nonumber \\
 && - \frac{\eta_{k}}{1 - \beta_1^n}  \cdot d \cdot 2\sqrt{\frac{n}{\beta_2^n}} \cdot (1 - \beta_1) \sum_{i=0}^{n-2} \beta_1^{i+1}
\left(1+ \cdots + \beta_1^{n-i-2} 
\right) \cdot 3Q_k \nonumber \\
&&- \frac{\eta_{k}}{1 - \beta_1^n} \cdot d\cdot 2\sqrt{\frac{n}{\beta_2^n}}  \cdot n^3 \Delta_{k-1}. \label{eq_lower_bound_III_2_final}
\end{eqnarray}}

Combining the  lower bounds of (III-1) and (III-2) together, we have
{\footnotesize
  \begin{eqnarray*}
      \Ex_{k-1} \Exk  \text{(III)} &=&   \Ex_{k-1} \Exk \left[\frac{\eta_{k}}{1 - \beta_1^n} \sum_{l=1}^d\sum_{i=0}^{n-1}  \frac{\partial_l f(x_{k,0})}{{\sqrt{v_{l,k,\yushunrevise{-1}}}}} \left(m_{l,k,i} - \beta_1^n m_{l,k-1,i}\right)\yushunrevise{\mathbb{I} \left( \vert \partial_lf(x_{k-1,0})\vert \geq  2Q_k \right)} \right] \\
&=& \text{(III-1)} + \text{(III-2)}  \\
      &\overset{\eqref{eq_lower_bound_III_1}, \eqref{eq_lower_bound_III_2_final}}{\geq}& \eta_k\mathbb{E}_{k-1} \left[ \sum_{l=1}^d\frac{(\partial_l f(x_{k,0}))^2}{ \sqrt{v_{l,k,\yushunrevise{-1}}}}\yushunrevise{\mathbb{I} \left( \vert \partial_lf(x_{k-1,0})\vert \geq 2Q_k \right)}
\right]-  \frac{d\eta_k}{1-\beta_1^n} 2\sqrt{\frac{n}{\beta_2^n}} \cdot n^3 \Delta_{k}  \\
&& -  \frac{d\eta_k}{1-\beta_1^n} \left(\tilde{\delta}_1(\beta_2) n^2 \left( n\Delta_{k} + \sqrt{D_{1}} \sqrt{n} d\left(  \Ex_{k-1} \left|\partial_{\alpha} f\left(x_{k, 0}\right)\right|+\sqrt{\frac{D_{0}}{D_{1} d}}\right)\right) \right)
\\
&&- \frac{d\eta_k}{1-\beta_1^n}\Delta_{k-1} \left( \frac{1}{\sqrt{\beta_2^n}} \cdot  (1-\beta_1)  \cdot 2\sqrt{\frac{n}{\beta_2^n}} \cdot n^2 + 4\sqrt{\frac{n}{\beta_2^n}} \cdot n^3 \right)  \\
 && - \frac{\eta_{k}}{1 - \beta_1^n}  \cdot d \cdot 2\sqrt{\frac{n}{\beta_2^n}} \cdot n^2\cdot 3Q_k \\
&\geq& \eta_k\mathbb{E}_{k-1} \left[\frac{(\partial_\alpha f(x_{k,0}))^2 }{ \sqrt{v_{\alpha,k,\yushunrevise{-1}}}}
\right] -  \eta_k \tilde{\delta}_6(\beta_2)\left(  \Ex_{k-1}\left|\partial_{\alpha} f\left(x_{k, 0}\right)\right|+\sqrt{\frac{D_{0}}{D_{1} d}}\right) -   \frac{\widetilde{C}_8}{k}, 
\end{eqnarray*}  
}

where
{\small \begin{eqnarray*}
    \tilde{\delta}_6(\beta_2)&=& \frac{1}{1-\beta_1^n} \tilde{\delta}_1(\beta_2)n^2 \sqrt{D_1} \sqrt{n} d^2,\\
   \widetilde{C}_8&=&\frac{\eta_0d\Delta_1}{1-\beta_1^n}\Bigl(
   2\sqrt{\frac{n}{\beta_2^n}}\cdot n^3+n^3\tilde{\delta}_1(\beta_2)\nonumber\\
   &&\quad+\sqrt{2}\Bigl(\frac{1}{\sqrt{\beta_2^n}}(1-\beta_1)
   \cdot2\sqrt{\frac{n}{\beta_2^n}}\cdot n^2\nonumber\\
   &&\qquad+4\sqrt{\frac{n}{\beta_2^n}}\cdot n^3\Bigr)\Bigr)\nonumber\\
   &&\quad+\frac{\eta_0d}{1-\beta_1^n}\cdot2\sqrt{\frac{n}{\beta_2^n}}
   \cdot n^2\cdot3Q_1.
\end{eqnarray*}}
This concludes the proof for Lemma \ref{lemma_III}.  \qed

\subsection{Main Body of the Proof}
\label{appendix_main_body_rr}
Now we are ready to prove Theorem \ref{thm_rr}. We apply the descent lemma on the auxiliary variable $z_k = \frac{x_{k,0} - \beta_1^n x_{k-1,0}}{1-\beta_1^n}$.

{\small \begin{eqnarray}
\label{eq_descent_lemma}
    f(z_{k+1}) &\leq& f(z_k) + \langle \nabla f(z_k), z_{k+1} - z_k \rangle + \frac{L}{2} \| z_{k+1} - z_k \|_2^2  
 \end{eqnarray}}

 We will provide an upper bound on $\frac{L}{2} \| z_{k+1} - z_k \|_2^2$ and a lower bound on $\langle \nabla f(z_k), z_{k} - z_{k+1} \rangle$.  We start with $\frac{L}{2} \| z_{k+1} - z_k \|_2^2$. 

 {\small \begin{eqnarray}
\left\| z_{k+1} - z_k \right\|_2^2 
&=& \frac{1}{(1 - \beta_1^n)^2} 
\left\| x_{k+1,0} - \beta_1^n x_{k,0} - x_{k,0} + \beta_1^n x_{k-1,0} \right\|_2^2
\nonumber\\
&\leq& \frac{1}{(1 - \beta_1^n)^2} 
\left( \left\| x_{k+1,0} - x_{k,0} \right\|_2 
+ \left\| x_{k,0} - x_{k-1,0} \right\|_2 \right)^2
\nonumber\\
&\leq& \frac{1}{(1 - \beta_1^n)^2} 
\left( \left\| x_{k+1,0} - x_{k,n-1} + x_{k,n-1} - x_{k,n-2} + \cdots + x_{k,1} - x_{k,0} \right\|_2\right.
\nonumber\\
&&+  
\left.\left\| x_{k,0} - x_{k-1,n-1} +x_{k-1,n-1} - \cdots + x_{k-1,1} - x_{k-1,0} \right\|_2\right)^2
\nonumber\\
&\overset{\eqref{eq:m_over_v}}{\leq}& 
\frac{1}{(1 - \beta_1^n)^2} 
\cdot \left( \frac{(1 - \beta_1)}{\sqrt{1 - \beta_2}} 
\cdot \frac{d}{1 - \frac{\beta_1}{\sqrt{\beta_2}}} \right)^2 \cdot
\left( n \eta_{k} + n \eta_{k-1} \right)^2 \nonumber\\
&\overset{\text{(i)}}{\leq}& \frac{1}{(1 - \beta_1^n)^2} 
\cdot \left( \frac{(1 - \beta_1)}{\sqrt{1 - \beta_2}} 
\cdot \frac{d}{1 - \frac{\beta_1}{\sqrt{\beta_2}}} \right)^2 \cdot 
\frac{(1+\sqrt{2})^2 n^2\eta_0^2 }{k} \nonumber\\
&:=& \frac{\widetilde{C}_1}{k}, \label{eq_upper_bound_2nd_term}
 \end{eqnarray}}
where (i) is due to $\frac{1}{\sqrt{k-1}} \leq \frac{\sqrt{2}}{{\sqrt{k}}}$ when $k \geq 2$.  We now provide a lower bound on $\langle \nabla f(z_k), z_{k} - z_{k+1} \rangle$. 

 {\small \begin{eqnarray*}
     \langle \nabla f(z_k), z_{k} - z_{k+1} \rangle &= & \left\langle \nabla f(z_k),\ 
\frac{x_{k,0} - \beta_1^n x_{k-1,0}}{1 - \beta_1^n} 
- \frac{x_{k+1,0} - \beta_1^n x_{k,0}}{1 - \beta_1^n} 
\right\rangle \\
&=& 
\underbrace{\frac{1}{1 - \beta_1^n} 
\left\langle 
\nabla f(z_k) - \nabla f(x_{k,0}),\ x_{k,0} - x_{k+1,0} - \beta_1^n (x_{k-1,0} - x_{k,0}) 
\right\rangle}_{\text{(I)}} \\
&&+ 
\underbrace{\frac{1}{1 - \beta_1^n} 
\left\langle 
\nabla f(x_{k,0}),\ x_{k,0} - x_{k+1,0} - \beta_1^n (x_{k-1,0} - x_{k,0}) 
\right\rangle}_{\text{(II)}}.  
 \end{eqnarray*}}

 We first prove a lower bound for (I):

 {\small \begin{eqnarray}
{\text{(I)}}&\geq&  -  \frac{1}{1 - \beta_1^n} \|\nabla f(z_k) - \nabla f(x_{k,0})\|_2 \|x_{k,0} - x_{k+1,0} - \beta_1^n (x_{k-1,0} - x_{k,0}) \|_2 \nonumber \\
&\geq&  -  \frac{L}{1 - \beta_1^n} \|z_k - x_{k,0}\|_2 \|x_{k,0} - x_{k+1,0} - \beta_1^n (x_{k-1,0} - x_{k,0}) \|_2 \nonumber\\
&\overset{\text{(i)}}{\geq}&- \frac{L}{1 - \beta_1^n} 
\cdot \frac{\beta_1^n}{1 - \beta_1^n} 
\cdot \frac{(1 - \beta_1)}{\sqrt{1 - \beta_2}} 
\cdot \frac{nd}{1 - \frac{\beta_1}{\sqrt{\beta_2}}} 
\cdot \eta_k  \cdot \|x_{k,0} - x_{k+1,0} - \beta_1^n (x_{k-1,0} - x_{k,0}) \|_2 \nonumber\\
&\geq &- \frac{L}{1 - \beta_1^n} 
\cdot \left(\frac{\beta_1^n}{1 - \beta_1^n} 
\cdot \frac{(1 - \beta_1)}{\sqrt{1 - \beta_2}} 
\cdot \frac{nd}{1 - \frac{\beta_1}{\sqrt{\beta_2}}}\right)^2 
\cdot  \left( n + \sqrt{2}n \beta_1^n \right)\cdot \left(\frac{\eta_0}{\sqrt{k}}\right)^2\nonumber \\
&:= & -\frac{\widetilde{C}_2}{k}, \label{eq_lower_bound_I}
 \end{eqnarray}}

where (i) is due to 
 {\small \begin{eqnarray*}
     \left\| z_k - x_{k,0} \right\|_2
&=& \left\| \frac{x_{k,0} - \beta_1^n x_{k-1,0}}{1 - \beta_1^n} - x_{k,0} \right\|_2
\\
&=& \left\| \frac{\beta_1^n (x_{k,0} - x_{k-1,0})}{1 - \beta_1^n} \right\|_2
\\
&\overset{\text{Lemma \ref{lemma_delta}}}{\leq}&
\frac{\beta_1^n}{1 - \beta_1^n} \cdot \frac{(1 - \beta_1)}{\sqrt{1 - \beta_2}} \cdot \frac{\sqrt{2}nd}{1 - \frac{\beta_1}{\sqrt{\beta_2}}} \cdot \eta_k. 
 \end{eqnarray*}}

We now provide a lower bound on (II), which is more involved. 

{\small \begin{eqnarray*}
\text{(II)} &=& \frac{1}{1 - \beta_1^n} \sum_{l=1}^d \partial_l f(x_{k,0}) \left[
\eta_k \sum_{i=0}^{n-1} \frac{m_{l,k,i}}{\sqrt{v_{l,k,i}}}
- \beta_1^n \eta_{k-1} \sum_{i=0}^{n-1} \frac{m_{l,k-1,i}}{\sqrt{v_{l,k-1,i}}}
\right]
\\
&=& \underbrace{\frac{1}{1 - \beta_1^n} \sum_{l=1}^d \partial_l f(x_{k,0}) \left[
\eta_k \sum_{i=0}^{n-1} \frac{m_{l,k,i}}{\sqrt{v_{l,k,i}}}
- \beta_1^n \eta_{k-1} \sum_{i=0}^{n-1} \frac{m_{l,k-1,i}}{\sqrt{v_{l,k-1,i}}}
\right]
\yushunrevise{\mathbb{I} \left( \vert \partial_lf(x_{k,0})\vert < Q_k \right)}}_{\text{(II-1)}} \\
&&+\underbrace{\frac{\eta_k}{1 - \beta_1^n} \sum_{l=1}^d  \left[
 \sum_{i=0}^{n-1} \partial_l f(x_{k,0})\frac{m_{l,k,i}}{\sqrt{v_{l,k,i}}}
\right]
\yushunrevise{\mathbb{I} \left( \vert \partial_lf(x_{k,0})\vert \geq Q_k \right)}}_{\text{(II-2)}} \\
&&\underbrace{ - \frac{\beta_1^n \eta_{k-1} }{1 - \beta_1^n} \sum_{l=1}^d  \left[
 \sum_{i=0}^{n-1} \partial_l f(x_{k,0})\frac{m_{l,k-1,i}}{\sqrt{v_{l,k-1,i}}}
\right]
\yushunrevise{\mathbb{I} \left( \vert \partial_lf(x_{k,0})\vert \geq Q_k \right)}}_{\text{(II-3)}}.
\end{eqnarray*}}

We start with (II-1). Note that (II-1) imposes bounded gradient condition.  (II-1) can be simply bounded as follows: 

{\small \begin{eqnarray}
\text{(II-1)} &\overset{\eqref{eq:m_over_v}}{\geq}& -\frac{\sum_{l=1}^d  Q_k }{1-\beta_1^n}\frac{\eta_0}{\sqrt{k}} \left( \frac{1 - \beta_1}{\sqrt{1 - \beta_2}} \cdot \frac{1}{1 - \frac{\beta_1}{\sqrt{\beta_2}}}
+ \sqrt{2} \cdot \frac{1 - \beta_1}{\sqrt{1 - \beta_2}} \cdot \frac{1}{1 - \frac{\beta_1}{\sqrt{\beta_2}}}
\right) \nonumber \\
&\geq& -\frac{dQ_k }{1-\beta_1^n}\frac{2\sqrt{2}n\eta_0}{\sqrt{k}} \left(\frac{1 - \beta_1}{\sqrt{1 - \beta_2}} \cdot \frac{1}{1 - \frac{\beta_1}{\sqrt{\beta_2}}}
\right)  \nonumber \\
&:= & -\frac{\widetilde{C}_{3}}{k} \label{eq_II-1}.
\end{eqnarray}}

Now we handle (II-2) and (II-3), which focus on the case when the gradient is unbounded. We  bound (II-2) and (II-3) in the following Lemma \ref{lemma_II_2}.

\begin{lemma}
\label{lemma_II_2}
      \begingroup\emergencystretch=2em Consider Algorithm \ref{algorithm_rr}. Assume $k \geq \lceil \frac{\log (1/2)}{n \log\beta_2} \rceil +1 $. Then for any  $f \in \functionclass$,  we have\par\endgroup

{\small
\begin{eqnarray}
\label{eq_II_2}
  \text{(II-2)} &= &  \!\!\frac{\eta_k}{1 - \beta_1^n} \sum_{l=1}^d  \left[
 \sum_{i=0}^{n-1} \partial_l f(x_{k,0})\frac{m_{l,k,i}}{\sqrt{v_{l,k,i}}}
\right]
\yushunrevise{\mathbb{I} \left( \vert \partial_lf(x_{k,0})\vert \geq Q_k \right)} \nonumber \\
&\geq&  \!\!\frac{\eta_k}{1 - \beta_1^n} \sum_{l=1}^d\sum_{i=0}^{n-1}  \partial_l f(x_{k,0})\frac{m_{l,k,i}}{\sqrt{v_{l,k,\yushunrevise{-1}}}} \yushunrevise{\mathbb{I} \left( \vert \partial_lf(x_{k,0})\vert \geq Q_k \right)}  \nonumber\\
&&- \eta_k \tilde{\delta}_3(\beta_2)\left(\vert \partial_\alpha f(x_{k,0})\vert  + \sqrt{\frac{D_0}{D_1d } }\right) -\frac{\widetilde{C}_{4}}{k},
\end{eqnarray}
}
where 
{\small \[
\begin{aligned}
\tilde{\delta}_2(\beta_2)&=\left(1-\beta_2\right)\RED{n}\left(\frac{8n}{\beta_2^n}+\frac{1}{2}\right)+\left(\frac{1}{\sqrt{\beta_2^n}}-1\right),\\
\tilde{\delta}_3(\beta_2)&=\frac{d}{1-\beta_1^n}\tilde{\delta}_2(\beta_2)2\sqrt{\frac{n}{\beta_2^n}}\cdot C_{m,1},\\
\widetilde{C}_{4}&=\frac{d\eta_0}{1-\beta_1^n}\tilde{\delta}_2(\beta_2)2\sqrt{\frac{n}{\beta_2^n}}\cdot C_{m,2},
\end{aligned}
\]}
and  $C_{m,1}$, and $C_{m,2}$ are constants defined in Lemma \ref{lemma_upper_m}. Note that  $\tilde{\delta}_2(\beta_2)$ is a constant that approaches 0 when $\beta_2$ approaches 1. Similarly, we have 
{\small \begin{eqnarray}
\label{eq_II_3}
  \text{(II-3)} &= & \!\! - \frac{\beta_1^n \eta_{k-1}}{1 - \beta_1^n} \sum_{l=1}^d  \left[
 \sum_{i=0}^{n-1} \partial_l f(x_{k,0})\frac{m_{l,k-1,i}}{\sqrt{v_{l,k-1,i}}}
\right]
\yushunrevise{\mathbb{I} \left( \vert \partial_lf(x_{k,0})\vert \geq Q_k \right)}  \nonumber\\
&\geq & \!\!-  \frac{\beta_1^n\eta_{k-1}}{1 - \beta_1^n} \sum_{l=1}^d\sum_{i=0}^{n-1}  \partial_l f(x_{k,0})\frac{m_{l,k-1,i}}{\sqrt{v_{l,k,\yushunrevise{-1}}}} \yushunrevise{\mathbb{I} \left( \vert \partial_lf(x_{k,0})\vert \geq Q_k \right)} \nonumber \\
&&- \eta_{k-1} \tilde{\delta}_4(\beta_2)\left(\vert \partial_\alpha f(x_{k,0})\vert  + \sqrt{\frac{D_0}{D_1d } }\right) -\frac{\widetilde{C}_{5}}{k},
\end{eqnarray}}
where 
{\small $$\tilde{\delta}_4(\beta_2) =\frac{d \beta_1^n}{1 - \beta_1^n} \tilde{\delta}_2(\beta_2) 2\sqrt{\frac{n}{\beta_2^n}} \cdot  \tilde{C}_{m,1}, \quad \widetilde{C}_{5} = \frac{d\beta_1^n\eta_0}{1 - \beta_1^n} \tilde{\delta}_2(\beta_2) 2\sqrt{\frac{n}{\beta_2^n}} \cdot  \tilde{C}_{m,2},$$}
and $\tilde{C}_{m,1}$, and $\tilde{C}_{m,2}$ are constants defined in Lemma \ref{lemma_upper_m}.

\end{lemma}
The proof of Lemma \ref{lemma_II_2} is shown in  Appendix \ref{appendix:lemma_II_2}. Combining  \text{(II-1)},      \text{(II-2)}, and \text{(II-3)} together, we have the following lower bound for (II):
{\small
\begin{eqnarray}
    \text{(II)} &=&      \text{(II-1)} +     \text{(II-2)} +     \text{(II-3)} \nonumber \\
    &\overset{\text{\eqref{eq_II-1}}}{ \geq}& -  \frac{\widetilde{C}_3}{k} +  \text{(II-2)} +     \text{(II-3)}  \nonumber\\
          &\overset{\text{Lemma \ref{lemma_II_2}}}{ \geq}&  \underbrace{\frac{\eta_{k}}{1 - \beta_1^n} \sum_{l=1}^d\sum_{i=0}^{n-1}  \partial_l f(x_{k,0})\frac{m_{l,k,i} - \beta_1^n m_{l,k-1,i}}{\sqrt{v_{l,k,\yushunrevise{-1}}}} \yushunrevise{\mathbb{I} \left( \vert \partial_lf(x_{k,0})\vert \geq Q_k \right)}  }_{\text{(III)}}
  \nonumber \\
  &&-  \frac{1}{1 - \beta_1^n} \left( \eta_{k-1}  - \eta_k\right)
 \beta_1^n \sum_{l=1}^d\sum_{i=0}^{n-1} \left| \frac{\partial_l f(x_{k,0})}{\sqrt{v_{l,k,\yushunrevise{-1}}}} \right| \left| m_{l,k-1,i}\right| \nonumber\\
     && - \left(\eta_k \tilde{\delta}_3(\beta_2) + \eta_{k-1} \tilde{\delta}_4(\beta_2) \right)\left(\vert \partial_\alpha f(x_{k,0})\vert  + \sqrt{\frac{D_0}{D_1d } }\right) -\frac{\widetilde{C}_3 + \widetilde{C}_{4} +\widetilde{C}_5}{k} \nonumber\\ 
      &\overset{\text{(i)}}{\geq} & \underbrace{\frac{\eta_{k}}{1 - \beta_1^n} \sum_{l=1}^d\sum_{i=0}^{n-1}  \partial_l f(x_{k,0})\frac{m_{l,k,i} - \beta_1^n m_{l,k-1,i}}{\sqrt{v_{l,k,\yushunrevise{-1}}}}\yushunrevise{\mathbb{I} \left( \vert \partial_lf(x_{k,0})\vert \geq  Q_k \right)}  }_{\text{(III)}}
  \nonumber\\
  &&-  \frac{d\beta_1^{\frac{n}{2}} 4n\sqrt{n}\eta_0 }{(1 - \beta_1^n) k} 
 \left(  C_{m,1}\left(\vert \partial_\alpha f(x_{k,0})\vert  + \sqrt{\frac{D_0}{D_1d } }\right) + \frac{C_{m,2}}{\sqrt{k}} \right) \nonumber\\
     && - \left(\eta_k \tilde{\delta}_3(\beta_2) + \eta_{k-1} \tilde{\delta}_4(\beta_2) \right)\left(\vert \partial_\alpha f(x_{k,0})\vert  + \sqrt{\frac{D_0}{D_1d } }\right) -\frac{\widetilde{C}_3 + \widetilde{C}_{4} +\widetilde{C}_5}{k} \nonumber \\
      &\overset{}{\geq} & \underbrace{\frac{\eta_{k}}{1 - \beta_1^n} \sum_{l=1}^d\sum_{i=0}^{n-1}  \partial_l f(x_{k,0})\frac{m_{l,k,i} - \beta_1^n m_{l,k-1,i}}{\sqrt{v_{l,k,\yushunrevise{-1}}}} \yushunrevise{\mathbb{I} \left( \vert \partial_lf(x_{k,0})\vert \geq  Q_k \right)}   }_{\text{(III)}} \nonumber \\
 &  &  - \eta_k \tilde{\delta}_5(\beta_2)
  \left(\vert \partial_\alpha f(x_{k,0})\vert  + \sqrt{\frac{D_0}{D_1d } }\right) -  \frac{\widetilde{C}_6}{k} 
 \left(\vert \partial_\alpha f(x_{k,0})\vert  + \sqrt{\frac{D_0}{D_1d } }\right)  - \frac{\widetilde{C}_7}{k}, \label{eq_lower_bound_II}
\end{eqnarray}
}
where 
{\small $$\tilde{\delta}_5(\beta_2) = \tilde{\delta}_3(\beta_2) + \sqrt{2} \tilde{\delta}_4(\beta_2),\quad  \widetilde{C}_6 := \frac{d \beta_1^{\frac{n}{2}} 4n\sqrt{n}\eta_0  C_{m,1}}{(1 - \beta_1^n) }, \quad \widetilde{C}_7 = \widetilde{C}_3 + \widetilde{C}_4 +\widetilde{C}_5 + \frac{d \beta_1^{\frac{n}{2}} 4n\sqrt{n}\eta_0 C_{m,2}}{(1 - \beta_1^n) }.$$}
(i) is due to Lemma \ref{lemma_f_over_v} and Lemma  \ref{lemma_upper_m}  and $\frac{1}{\sqrt{k-1}} - \frac{1}{\sqrt{k}} \leq \frac{2}{k}$ for $k\geq 2$.
\RED{We now take conditional expectations $\mathbb E_k[\cdot]$ and then $\mathbb E_{k-1}[\cdot]$ and bound (III) in the following lemma.}

\begin{lemma}
    \label{lemma_III}
          \begingroup\emergencystretch=2em Consider Algorithm \ref{algorithm_rr}. Assume $k \geq \lceil \frac{\log (1/2)}{n \log\beta_2} \rceil +1 $. Then for any  $f \in \functionclass$, $0\leq \beta_1<\sqrt{\beta_2}<1$,  we have\par\endgroup
  {\small \begin{eqnarray}
  \label{eq_lower_bound_III}
     \Ex_{k-1} \Exk  \text{(III)} 
      &\geq&  \eta_k\mathbb{E}_{k-1} \left[\frac{(\partial_\alpha f(x_{k,0}))^2}{\sqrt{v_{\alpha,k,\yushunrevise{-1}}}}
\right] -  \eta_k \tilde{\delta}_6(\beta_2)\left(  \Ex_{k-1}\left|\partial_{\alpha} f\left(x_{k, 0}\right)\right|+\sqrt{\frac{D_{0}}{D_{1} d}}\right) -   \frac{\widetilde{C}_9}{k},
\end{eqnarray}}
where
{\small \begin{eqnarray*}
    \tilde{\delta}_6(\beta_2)&=& \frac{1}{1-\beta_1^n} \tilde{\delta}_1(\beta_2)n^2 \sqrt{D_1} \sqrt{n} d^2,\\
   \widetilde{C}_8&=&\frac{\eta_0d\Delta_1}{1-\beta_1^n}\Bigl(
   2\sqrt{\frac{n}{\beta_2^n}}\cdot n^3+n^3\tilde{\delta}_1(\beta_2)\nonumber\\
   &&\quad+\sqrt{2}\Bigl(\frac{1}{\sqrt{\beta_2^n}}(1-\beta_1)
   \cdot2\sqrt{\frac{n}{\beta_2^n}}\cdot n^2\nonumber\\
   &&\qquad+4\sqrt{\frac{n}{\beta_2^n}}\cdot n^3\Bigr)\Bigr)\nonumber\\
   &&\quad+\frac{\eta_0d}{1-\beta_1^n}\cdot2\sqrt{\frac{n}{\beta_2^n}}
   \cdot n^2\cdot3Q_1.
\end{eqnarray*}}
\end{lemma}

The proof of Lemma \ref{lemma_III} is shown in Appendix \ref{appendix:lemma_III}. 

\brrevise{
Let $V_{1,-1}:=\|v_{1,-1}\|_\infty$.
The exact second-moment recursion at the beginning of epoch $k$ is
{\small \[
v_{l,k,-1}
=
\beta_2^{n(k-1)}v_{l,1,-1}
+
(1-\beta_2)
\sum_{s=1}^{k-1}\sum_{i=0}^{n-1}
\beta_2^{\,n(k-1-s)+n-1-i}
\bigl(\partial_l f_{\tau_{s,i}}(x_{s,i})\bigr)^2.
\]}
Thus, in the interior case, we additionally take $k$ sufficiently large
so that $\beta_2^{n(k-1)}V_{1,-1}\leq D_0nd$.
}

Combining all the results from (I) to (III), we have
{\scriptsize
\begin{eqnarray}
    && \Ex_{k-1} \Ex_{k} \langle \nabla f(z_k),  z_{k} - z_{k+1} \rangle  \nonumber\\&\overset{\eqref{eq_lower_bound_I},\eqref{eq_lower_bound_II},\eqref{eq_lower_bound_III}}{\geq}& \underbrace{ - \frac{\widetilde{C}_2}{k}}_{\text{from (I)}}  \underbrace{ - \eta_k \tilde{\delta}_5(\beta_2)
  \left(\mathbb{E}_{k-1} \vert \partial_\alpha f(x_{k,0})\vert  + \sqrt{\frac{D_0}{D_1d } }\right)  -  \frac{\widetilde{C}_6}{k} 
  \left(\mathbb{E}_{k-1}\vert \partial_\alpha f(x_{k,0})\vert  + \sqrt{\frac{D_0}{D_1d } }\right) - \frac{\widetilde{C}_7}{k}}_{\text{from (II)}}  \nonumber\\
 && + \underbrace{\eta_k\mathbb{E}_{k-1} \left[\frac{(\partial_\alpha f(x_{k,0}))^2}{\sqrt{v_{\alpha,k,\yushunrevise{-1}}}}
\right] -  \eta_k \tilde{\delta}_6(\beta_2)\left(  \Ex_{k-1}\left|\partial_{\alpha} f\left(x_{k, 0}\right)\right|+\sqrt{\frac{D_{0}}{D_{1} d}}\right) -   \frac{\widetilde{C}_8}{k}}_{\text{from (III)}} \nonumber \\
& = & \eta_k\mathbb{E}_{k-1} \left[\frac{(\partial_\alpha f(x_{k,0}))^2}{\sqrt{v_{\alpha,k,\yushunrevise{-1}}}}
\right]  - \eta_k \left(\tilde{\delta}_5(\beta_2) + \tilde{\delta}_6(\beta_2)\right)\left(  \Ex_{k-1}\left|\partial_{\alpha} f\left(x_{k, 0}\right)\right|+\sqrt{\frac{D_{0}}{D_{1} d}}\right) \nonumber \\
&&  -  \frac{\widetilde{C}_6}{k} 
 \ \left(\mathbb{E}_{k-1}\vert \partial_\alpha f(x_{k,0})\vert  + \sqrt{\frac{D_0}{D_1d } }\right)  - \frac{\widetilde{C}_2 +\widetilde{C}_7 + \widetilde{C}_8}{k}. \label{eq.thm3.1_E_k-1E_k}
\end{eqnarray}
}

By taking expectations of both sides of \eqref{eq.thm3.1_E_k-1E_k} and repeating the steps used to derive \eqref{eqn:min_gradient_norm}, we arrive at the following inequality:
When
$\tilde{\delta}_5(\beta_2)+\tilde{\delta}_6(\beta_2) \leq \frac{1}{4d\sqrt{5D_1n}}$, which can be achieved by setting $1-\beta_2 = \mathcal{O}((1-\beta_1^n)/n^{5.5})$, and $k$ is large enough such that $\frac{\tilde{C}_6}{\sqrt{k}} \leq \frac{\eta_0}{4d\sqrt{5D_1n}}$,
{\footnotesize
\begin{eqnarray}\label{eqn:shuffle_min_gradient_norm}
    \Ex\left[\frac{(\partial_{\alpha}f(x_{k,0}))^2}{n \sqrt{ v_{\alpha,k,\yushunrevise{-1}}}}-(\tilde{\delta}_5(\beta_2)+\tilde{\delta}_6(\beta_2)+\frac{\tilde{C}_6}{\eta_0\sqrt{k}})\vert\partial_\alpha f(x_{k,0})\vert\right]
    &\geq& \Ex  \left[\min\left\{   \frac{\|\nabla f(x_{k,0})\|_2^2}{d\sqrt{5D_0nd}},    \frac{\|\nabla f\left(x_{k,0}\right)\|_2}{2d^2\sqrt{5D_1n}}  \right\} \right] \nonumber\\ 
    &&-(\tilde{\delta}_5(\beta_2)+\tilde{\delta}_6(\beta_2)) \sqrt{\frac{D_0}{D_1 d}}  - \frac{\tilde{C}_{9}}{\sqrt{k}}, \nonumber
\end{eqnarray}
}
where $\tilde{C}_{9}=\frac{\tilde{C}_6}{\eta_0}\sqrt{\frac{D_0}{D_1 d}}+\frac{32\Delta_1^2}{(1-\beta_2)^2D_1 d^2\sqrt{5D_0nd}}+\frac{2\sqrt{2}\Delta_1}{(1-\beta_2)D_1 n d^2\sqrt{5}}+(\tilde{\delta}_5(\beta_2)+\tilde{\delta}_6(\beta_2)+\frac{\tilde{C}_6}{\eta_0})\frac{4\sqrt{2n}\Delta_1}{(1-\beta_2)d\sqrt{D_1}}$. 

{\small \begin{eqnarray}
\Ex \langle \nabla f(z_k),  z_{k} - z_{k+1} \rangle \geq \eta_k \left\{\Ex \left[\min\left\{   \frac{\|\nabla f(x_{k,0})\|_2^2}{d\sqrt{5D_0nd}},    \frac{\|\nabla f\left(x_{k,0}\right)\|_2}{2d^2\sqrt{5D_1n}}  \right\} \right] -\tilde{\delta}(\beta_2) \sqrt{D_0} -\frac{\tilde{C}}{\sqrt{k}}  
\right\},  \nonumber
\end{eqnarray}}
where $\tilde{\delta}(\beta_2)=(\tilde{\delta}_5(\beta_2)+\tilde{\delta}_6(\beta_2)) \sqrt{\frac{1}{D_1 d}}$,
$\tilde{C}=\tilde{C}_{9}+\frac{\tilde{C}_6\sqrt{\frac{D_0}{D_1d}}+\tilde{C}_2+\tilde{C}_7+\tilde{C}_8}{\eta_0}.$

Following the same procedure as in the proof of Theorem \ref{thm_wr}, we have

{\small \begin{eqnarray*}
        \min_{k \in [1, T]} \Ex\left[\min  \left\{ \frac{\|\nabla f(x_{k, 0})\|_2^2 }{\sqrt{D_{0}} }  , \frac{\|\nabla f(x_{k, 0})\|_2}{2\sqrt{d D_1 }}\right\}  \right] &\leq & \mathcal{O}\left(\frac{\log T }{\sqrt{T}} \right)  +\mathcal{O}( \tilde{\delta}(\beta_2)\sqrt{D_0}).
   \nonumber
\end{eqnarray*}}

\brrevise{Finally, the argument above proves the interior case $D_0>0$ and
$D_1>0$. As discussed at the beginning of this section, the remaining
boundary cases follow the same procedure as the corresponding boundary
cases for Theorem \ref{thm_wr} in Appendix
\ref{appendix:D0_D1_boundary}.} 
\brrevise{
Combining the interior estimate above with the boundary cases, and applying the same
descent-lemma telescoping argument as in the proof of Theorem
\ref{thm_wr}, gives
{\small \[
\min_{k\in[1,T]}
\Ex\left[
\mathcal{M}_{D_0,D_1}
\bigl(\|\nabla f(x_{k,0})\|_2\bigr)
\right]
=
\mathcal{O}\left(\frac{\log T}{\sqrt{T}}\right)
+
\mathcal{O}\left(
\widetilde\delta_{D_0,D_1}(\beta_2)\sqrt{D_0}
\right).
\]}
This proves Theorem \ref{thm_rr}.}
\qed

\section{\brrevise{Boundary Cases for 
Theorem \ref{thm_wr} and Theorem \ref{thm_rr} }}
\label{appendix:D0_D1_boundary}

Theorem \ref{thm_wr} and Theorem \ref{thm_rr}  cover all
$D_0,D_1\geq0$ by cases. Their main derivation \yushunrevise{in Section \ref{appendix:lemma_descent_WR} and Appendix \ref{appendix:thm_rr}} prove the interior branch
$D_0>0$, $D_1>0$. We now prove the remaining boundary branches without
interpreting any quotient with a zero denominator.

\brrevise{
The same case split applies to Algorithm \ref{algorithm_rr}. The only
initialization-dependent replacement is
{\small \[
\beta_2^k\|v_0\|_\infty
\quad\longrightarrow\quad
\beta_2^{n(k-1)}\|v_{1,-1}\|_\infty.
\]}
For $D_1=0<D_0$, it suffices to take $k$ large enough that
$\beta_2^{n(k-1)}\|v_{1,-1}\|_\infty\leq D_0$. For $D_0=0<D_1$, this
initialization term decays exponentially and is absorbed into the
$k^{-1/2}$ error term. Hence the following boundary arguments give the
corresponding conclusions for both sampling schemes.}

\paragraph{The case $D_0=D_1=0$.}
Assumption \ref{assum2} gives
{\small \[
    \sum_{i=0}^{n-1}\|\nabla f_i(x)\|_2^2\leq 0
    \qquad\text{for every }x\in\mathbb{R}^d.
\]}
Every summand is nonnegative, so $\nabla f_i(x)=0$ for every $i$ and every
$x$. Hence every component function, and therefore $f$, is constant on
$\mathbb{R}^d$. In particular, $\|\nabla f(x_k)\|_2=0$ for every iterate,
and the vanishing-gradient conclusion is immediate.

\paragraph{The case $D_0=0<D_1$.}
All terms proportional to $\sqrt{D_0}$ or
$\sqrt{D_0/(D_1d)}$ in the proof of Lemma \ref{lemma_descent_WR} vanish.
In the proof of \eqref{eqn:min_gradient_norm}, the first sub-case under
Case~(a), defined by
{\small \[
    |\partial_\alpha f(x_k)|^2\leq \frac{D_0}{D_1d},
\]}
contains only $|\partial_\alpha f(x_k)|=0$, for which the desired lower
bound is zero. Whenever $|\partial_\alpha f(x_k)|>0$, the second sub-case
applies and gives the linear-gradient lower bound. Combining it with
Case~(b) yields
{\small \begin{equation}
\label{eq:D0_zero_descent_component}
\Ex\left[
 \frac{(\partial_\alpha f(x_k))^2}
      {n\sqrt{\Exk(v_{\alpha,k})}}
 -\left(\delta_{12}(\beta_2)+\frac{C_7}{\sqrt{k}}\right)
  |\partial_\alpha f(x_k)|
\right]                                                   
\geq
\Ex\left[
 \frac{\|\nabla f(x_k)\|_2}
      {2nd^2\sqrt{5D_1n}}
\right]
-\frac{C_{10}}{\sqrt{k}},
\end{equation}}
where $C_{10}$ is the finite constant in \eqref{def_C_10}. Substituting
\eqref{eq:D0_zero_descent_component} into \eqref{eq:b}, and then combining
with \eqref{eq:a}, gives a finite constant
{\small \begin{equation}
    \label{def:D0_zero_constant}
    C^{(0)}=C_1+C_9+C_{10}
\end{equation}} such that
{\small \begin{equation}
\label{eq:D0_zero_descent}
\Ex\langle\nabla f(z_k),z_k-z_{k+1}\rangle
\geq
\eta_k\left\{
 \Ex\left[
  \frac{\|\nabla f(x_k)\|_2}{2nd^2\sqrt{5D_1n}}
 \right]
 -\frac{C^{(0)}}{\sqrt{k}}
\right\}.
\end{equation}}
This is \eqref{eq:descent_all_parameter_cases} with
$\delta_{0,D_1}(\beta_2)=0$ and $C_{0,D_1}=C^{(0)}$.
Applying the same descent-lemma telescope as in
\eqref{eq:descent_unsimplified}--\eqref{eq:thm_wr_telescoping_bound} gives
{\small \[
    \min_{k\in[1,T]}\Ex\|\nabla f(x_k)\|_2
    =
    \mathcal{O}\left(
      \frac{\log T}{(1-\beta_2)\sqrt{T}}
    \right),
\]}
which is Corollary \ref{thm1_sgc_wr}.

\paragraph{The case $D_1=0<D_0$.}
Assumption \ref{assum2} becomes
{\small \begin{equation}
\label{eq:D1_zero_bounded_gradients}
\sum_{i=0}^{n-1}\|\nabla f_i(x)\|_2^2\leq D_0
\qquad\text{for every }x\in\mathbb{R}^d.
\end{equation}}
Consequently,
{\small \begin{equation}
\label{eq:D1_zero_gradient_bounds}
\|\nabla f_i(x)\|_2\leq\sqrt{D_0},
\qquad
\|\nabla f(x)\|_2^2
=
\left\|
\frac1n\sum_{i=0}^{n-1}\nabla f_i(x)
\right\|_2^2
\leq\frac{D_0}{n}.
\end{equation}}

The component-gradient transport estimate used in the proof of Lemma
\ref{lemma_descent_WR} now reduces to
{\small \begin{equation}
\label{eq:D1_zero_transport}
|\partial_l f_{\tau_{k-j}}(x_{k-j})|
\leq
\sum_{t=1}^{j}\Delta_{k-t}+\sqrt{ndD_0}.
\end{equation}}
In particular, this estimate contains neither
$|\partial_\alpha f(x_k)|$ nor a negative power of $D_1$.
The exact unrolling of the second-moment recursion and
\eqref{eq:D1_zero_bounded_gradients} give, for every coordinate $l$,
{\small \begin{align}
\Exk(v_{l,k})
&=
\beta_2^k v_{l,0}
+
(1-\beta_2)
\sum_{j=0}^{k-1}
\beta_2^j
\Exk\left[
\bigl(\partial_l f_{\tau_{k-j}}(x_{k-j})\bigr)^2
\right] \nonumber\\
&\leq
\beta_2^k\|v_0\|_\infty+D_0.
\label{eq:D1_zero_v_bound}
\end{align}}
Hence, for all sufficiently large $k$ such that
$\beta_2^k\|v_0\|_\infty\leq D_0$,
{\small \begin{equation}
\label{eq:D1_zero_positive_term}
\frac{(\partial_\alpha f(x_k))^2}
     {n\sqrt{\Exk(v_{\alpha,k})}}
\geq
\frac{(\partial_\alpha f(x_k))^2}
     {n\sqrt{2D_0}}
\geq
\frac{\|\nabla f(x_k)\|_2^2}
     {nd\sqrt{2D_0}}.
\end{equation}}

Repeating the estimates leading to \eqref{eq:b}, but using
\eqref{eq:D1_zero_transport}, removes all terms generated by the
unbounded-gradient part of the argument. In particular,
{\small \[
C_7=0,
\qquad
\delta_{10}(\beta_2)=0,
\qquad
\delta_8(\beta_2)=\delta_6(\beta_2).
\]}
Define
{\small \begin{equation}
\label{eq:D1_zero_constants}
\delta_{12}^{\mathrm{bd}}(\beta_2)
:=
\delta_4(\beta_2)+\delta_6(\beta_2)+\delta_1(\beta_2),
\end{equation}}
and let $C_9^{\mathrm{bd}}$ denote the finite constant obtained from
$C_9$ after applying \eqref{eq:D1_zero_transport}. The analogue of
\eqref{eq:b} is therefore
{\small \begin{equation}
\label{eq:D1_zero_b}
\begin{split}
\Ex(b)\geq\eta_k\Bigg\{&
\Ex\left[
\frac{(\partial_\alpha f(x_k))^2}
     {n\sqrt{\Exk(v_{\alpha,k})}}
-
\delta_{12}^{\mathrm{bd}}(\beta_2)
|\partial_\alpha f(x_k)|
\right] \\
&-
\bigl(\delta_9(\beta_2)+\delta_{11}(\beta_2)\bigr)\sqrt{D_0}
-\frac{C_9^{\mathrm{bd}}}{\sqrt{k}}
\Bigg\}.
\end{split}
\end{equation}}

By \eqref{eq:D1_zero_gradient_bounds},
{\small \[
|\partial_\alpha f(x_k)|
\leq
\|\nabla f(x_k)\|_2
\leq
\sqrt{\frac{D_0}{n}}.
\]}
Combining this bound with \eqref{eq:D1_zero_positive_term},
\eqref{eq:D1_zero_b}, and the estimate
$\Ex(a)\geq-\eta_k C_1/\sqrt{k}$ gives
{\small \begin{equation}
\label{eq:D1_zero_descent}
\begin{split}
\Ex\langle\nabla f(z_k),z_k-z_{k+1}\rangle
\geq\eta_k\Bigg\{&
\Ex\left[
\frac{\|\nabla f(x_k)\|_2^2}
     {nd\sqrt{2D_0}}
\right] \\
&-\delta_{\mathrm{bd}}(\beta_2)\sqrt{D_0}
-\frac{C_{\mathrm{bd}}}{\sqrt{k}}
\Bigg\},
\end{split}
\end{equation}}
where
{\small \begin{eqnarray}
    \label{def:D1_zero_constants}
\delta_{\mathrm{bd}}(\beta_2)
:=
\frac{\delta_{12}^{\mathrm{bd}}(\beta_2)}{\sqrt n}
+\delta_9(\beta_2)+\delta_{11}(\beta_2),
\qquad
C_{\mathrm{bd}}
:=
C_1+C_9^{\mathrm{bd}}.
\end{eqnarray}}
All these constants are finite and contain no negative power of $D_1$.
Moreover,
{\small \[
\delta_{\mathrm{bd}}(\beta_2)\longrightarrow0
\qquad\text{as }\beta_2\longrightarrow1.
\]}

Since $n,d\geq1$,
{\small \[
\frac{1}{nd\sqrt{2}}
\geq
\frac{1}{nd\sqrt{5nd}}.
\]}
Thus \eqref{eq:D1_zero_descent} implies
{\small \[
\Ex\langle\nabla f(z_k),z_k-z_{k+1}\rangle
\geq
\eta_k\left\{
\frac{1}{nd\sqrt{5nd}}
\Ex\left[
\mathcal{M}_{D_0,0}
\bigl(\|\nabla f(x_k)\|_2\bigr)
\right]
-\delta_{\mathrm{bd}}(\beta_2)\sqrt{D_0}
-\frac{C_{\mathrm{bd}}}{\sqrt{k}}
\right\}.
\]}
This proves the $D_1=0<D_0$ branch of Lemma
\ref{lemma_descent_WR}. Applying the same telescoping argument as in the
interior proof gives
{\small \[
\min_{k\in[1,T]}
\Ex\left[
\frac{\|\nabla f(x_k)\|_2^2}{\sqrt{D_0}}
\right]
=
\mathcal{O}\left(
\frac{\log T}{(1-\beta_2)\sqrt{T}}
\right)
+
\mathcal{O}\left(
(1-\beta_2)\sqrt{D_0}
\right).
\]}

\section{Discussion on Bias Correction Terms and Non-Zero $\epsilon$}
 \label{appendix:bias_correction}

In the above analysis, we focus on Adam without bias correction terms and consider $\epsilon=0$ ($\epsilon$ is the hyperparameter for numerical stability in Algorithm \ref{algorithm_wr} and \ref{algorithm_rr}).  
For completeness, we now briefly discuss how to incorporate the bias correction terms and non-zero $\epsilon$   into our analysis above. Based on the current convergence proof, we only require several additional simple changes. We will first discuss non-zero $\epsilon$ and then discuss bias correction terms. We will use the notations for Algorithm \ref{algorithm_wr} but all the arguments also hold for Algorithm \ref{algorithm_rr}.

\paragraph{Adam with non-zero $\epsilon$:} In our current analysis, we consider $\epsilon = 0$. In practice, $\epsilon$ is often set to be a small positive number such as $10^{-8}$.
Proving convergence with $\epsilon > 0$
is strictly simpler. 
It only requires a few simple changes based on the current proof. We explain as below. 

When $\epsilon \neq 0$, the new 2nd-order momentum becomes $\sqrt{\hat{v}_{l,k}}:=\sqrt{v_{l,k}}+\epsilon$. With this change, the new proof can be conducted with the following minor changes over the current proof in in Section \ref{appendix:lemma_descent_WR}

First of all, we make the following slight changes in the current proof:
\begin{itemize}
    \item First, for Lemma \ref{lemma_k-j_k}, we change the condition from $|\partial_l f_i (x_k)|  \geq  \frac{8\sqrt{2} j \Delta_1}{\sqrt{k}}$ to $|\partial_l f_i (x_k)|  \geq  \max\left(1, \frac{8\sqrt{2} j \Delta_1}{\sqrt{k}}\right)$. The conclusion of Lemma \ref{lemma_k-j_k} is unaffected and remains unchanged.
    \item  Second, based on the new version of Lemma \ref{lemma_k-j_k}, we change the condition of Lemma \ref{lemma_concentrate_v} from $\max_i |\partial_l f_i (x_k)| \geq R_k$ to $\max_i |\partial_l f_i (x_k)| \geq \max \left(1, R_k \right)$. The conclusion of Lemma \ref{lemma_concentrate_v} is unaffected and remains unchanged. 
    \item We define the constant $\hat{Q}_k = \max \left(1, Q_k \right)$.
\end{itemize}

Now, we re-state the two cases from the proof in Section \ref{appendix:lemma_descent_WR}:

\begin{itemize}
    \item {\bf Case 1 (bounded gradient):} when  $\max_i |\partial_l f_i (x_k)| \leq \hat{Q}_k$, we consider two sub-cases: if $\hat{Q}_k = Q_k$, then changing  from $\sqrt{v_{l,k}}$ to $\sqrt{\hat{v}_{l,k}}$ does not affect the proof since it does not change the result of \eqref{eq:m_over_v}. In particular, the derivation in the paragraph ``{\bf Lower bound of $\Ex(c)$}'' still holds. If $\hat{Q}_k  = 1$, then the gradient is bounded by the constant 1. This is equivalent to analyzing Adam under bounded gradient condition with constant 1. In particular, we have:
        {\small \begin{equation}
    \nonumber
        \sqrt{v_{l,k}} \leq \sqrt{\hat{v}_{l,k}} = \sqrt{v_{l,k}}+\epsilon \leq C,
    \end{equation}}
    where $C >0$ is some constant that is independent of $\epsilon$. 
     Then, whenever we need an upper bound for $\frac{1}{ \sqrt{\hat{v}_{l,k}}}$, we use $\frac{1}{ \sqrt{\hat{v}_{l,k}}} \leq \frac{1}{\sqrt{v_{l,k}}}$, and  then we follow the same steps in the current proof.
    Whenever we need a lower bound for $\frac{1}{ \sqrt{\hat{v}_{l,k}}}$, we use $\frac{1}{ \sqrt{\hat{v}_{l,k}}} \geq \frac{1}{C}$. This makes the proof strictly easier than the current proof since it reduces Adam to SGD.  We omit the proof for brevity.
    
    \item {\bf Case 2 (unbounded gradient):} the complement of Case 1, the \textit{unbounded}  partial gradient event $B_{l,k}^c$.  In this case, we will have $\sqrt{v_{l,k}}$ is lower bounded by a large constant in the same order as $\hat{Q}_k$, with high probability. This can be seen following the proof of Lemma \ref{lemma_concentrate_v} (from \eqref{eq_Exk_vlk_expansion}  to \eqref{eq_lemma_concentrate_v_lower_bd}).   Since $\epsilon$ is usually chosen to be significantly smaller than 1, we have $\sqrt{v_{l,k}} \approx \hat{Q}_k  > \epsilon$.  Therefore, we have the following relation with high probability:
    {\small \begin{equation}
    \nonumber
        \sqrt{v_{l,k}} \leq \sqrt{\hat{v}_{l,k}} = \sqrt{v_{l,k}}+\epsilon \leq 2\sqrt{v_{l,k}}.
    \end{equation}}
    Then, whenever we need an upper bound for $\frac{1}{ \sqrt{\hat{v}_{l,k}}}$, we use $\frac{1}{ \sqrt{\hat{v}_{l,k}}} \leq \frac{1}{\sqrt{v_{l,k}}}$, and  then we follow the same steps in the current proof.
    Whenever we need a lower bound for $\frac{1}{ \sqrt{\hat{v}_{l,k}}}$, we use $\frac{1}{ \sqrt{\hat{v}_{l,k}}} \geq \frac{1}{2 \sqrt{v_{l,k}}}$, and  then we follow the same steps in the current proof with minor changes on the constant. Finally, the tail probability can be controlled using the same procedure as in the current paragraph ``{\bf Lower bound of $\Ex(e)$}'', and the tail bound vanishes exponentially as $\beta_2 \rightarrow 1$.
\end{itemize}

 The final convergence result will be independent of $\epsilon$. 
  The above arguments also hold for Algorithm \ref{algorithm_rr} by changing the notation to $v_{k,0}$,  $\hat{v}_{k,0}$, $\nabla f_{\tau_{k,0}}(x)$.

\paragraph{Adam with bias correction terms:} These bias correction terms are introduced  by \citep{kingma2014adam}. It has the following form:
{\small $$x_{k+1} = x_k - \eta_k \frac{m_k /(1-\beta_1^k)}{\sqrt{v_k / (1-\beta_2^k)} + \epsilon}  =  x_k - \eta_k \frac{\sqrt{1-\beta_2^k}}{1-\beta_1^k}\frac{m_k}{\sqrt{v_k } + \epsilon \sqrt{(1-\beta_2^k)} }  :=  x_k -  \hat{\eta}_k \frac{m_k}{\sqrt{v_k } + \hat{\epsilon}}  $$}

As shown above, bias correction terms can be implemented by changing the stepsize $\eta_k$ into $\hat{\eta}_k = \frac{\sqrt{1-\beta_2^k}}{1-\beta_1^k}\eta_k = \frac{\sqrt{1-\beta_2^k}}{1-\beta_1^k} \frac{\eta_0}{\sqrt{k}}$ and changing $\epsilon$ into $ \hat{\epsilon} = \epsilon \sqrt{(1-\beta_2^k)} $.  We now explain how to include this change into our analysis.

We observe that the new stepsize  $\hat{\eta}_k$ is well bounded around the old stepsize $\eta_k$, i.e.,  $\hat{\eta}_k \in [\sqrt{1-\beta_2} \eta_k, \frac{1}{1-\beta_1}\eta_k]$. Therefore, to prove the convergence of Adam with $\hat{\eta}_k$, we add the following steps to the current proof. 

\begin{itemize}
 \item  Since $ \hat{\epsilon} \leq \epsilon$, the previous analysis for $\epsilon$ can be directly applied to $\hat{\epsilon}$. 
  \item Whenever we need an upper bound on  $\hat{\eta}_k$, we use $\hat{\eta}_k\leq \frac{1}{1-\beta_1}\eta_k$. Then we follow the original analysis with an extra constant $\frac{1}{1-\beta_1}$. It turns out we only need to  change the constant $\Delta_{k}$ in \eqref{eq_constants_wr} into $\frac{1}{1-\beta_1}\frac{\eta_{0} }{\sqrt{k}}\frac{L\sqrt{d}}{\sqrt{1-\beta_{2}}} \frac{1-\beta_{1} }{1-\frac{\beta_{1}}{\sqrt{\beta_{2}}}} $. The rest of the analysis remains the same.  
  \item Whenever we need a lower bound on  $\hat{\eta}_k$, we use $\hat{\eta}_k\geq \sqrt{1-\beta_2}\eta_k$. Then we follow the original analysis with an extra constant $\sqrt{1-\beta_2}$. As such, we only need to change the constant terms in the final result. The rest of the analysis remains the same. 
\end{itemize}

\section{Experimental Settings}
\label{appendix:exp_setting}

Here, we introduce our experimental settings.
\begin{itemize}
    \item {\bf Experiments on the non-realizable function used in Figure \ref{fig:cifar_nlp_mnist} (b).} Here, we state non-realizable function we used in Figure \ref{fig:cifar_nlp_mnist} (b). This example is restated from \citep[Appendix A.4]{shi2020rmsprop}.

{\small \begin{equation}\label{eq_nonrealizable}
f_j(x)=\left\{\begin{array}{l}
(x-a)^2 \text { if } j=0 \\
-0.1\left(x-\frac{10}{9} a\right)^2 \text { if } 1 \leq j \leq 9
\end{array}\right.
\end{equation}}

We can see that $f(x)= \frac{1}{10}\sum_{j=0}^9 f_j(x)=\frac{1}{10} \left( \frac{1}{10} x^2-\frac{1}{9} a^2 \right)$ is a convex function and $f(x)$ is lower bounded by $-\frac{1}{90}a^2$. We used $a = 10$ in Figure \ref{fig:cifar_nlp_mnist} (b). Note that we have $D_0 >0$  for  function \eqref{eq_nonrealizable}.

    \item {\bf  Experiments on the counter-example \eqref{counterexample1}.}
    We minimize function \eqref{counterexample1} using Algorithm \ref{algorithm_rr} with cyclic order $f_0$, $f_1$, $f_2$ and so on. We report the optimality gap $x-x^*$ after  50k iterations, or equivalently  $50000/n$ epochs. We use $\epsilon =10^{-8}$ for numerical stability. We use diminishing stepsize $\eta_k=0.1/\sqrt{k}$, where $k$ is the index of epoch.  Unless otherwise stated, this setting applies to all the other experiments on function \eqref{counterexample1}.

    \item {\bf MNIST \citep{deng2012mnist}.} We use one-hidden-layer neural network with width =16. We set batch size =1, weight decay =0, stepsize =0.0001 and train for 20 epochs.  We use $\epsilon =10^{-8}$ for numerical stability. 
    \item {\bf CIFAR-10 \citep{krizhevsky2009learning}.} We use ResNet-18 \citep{he2016deep} as the architecture. We choose batch size =16, weight decay =5e-4 , initial stepsize=1e-3.   We use a stage-wise constant learning rate scheduling with a multiplicative factor of 0.1 on epoch 30, 60 and 90.   We use $\epsilon =10^{-8}$ for numerical stability.  
    
    For  MNIST and CIFAR-10, larger batch size will bring similar pattern as that in Figure \ref{fig:intro_paper}, but
the phase transition will occur at some smaller $\beta_2$.

\end{itemize}

\clearpage
\bibliography{reference.bib}

\end{document}